\documentclass[10pt,letterpaper]{report}
\usepackage[utf8]{inputenc}
\usepackage{mathtools,amsthm,amssymb,latexsym,amsmath,graphics,graphicx,tabularx,caption,subcaption,mathtools}
\usepackage{bm}
\usepackage[usenames,dvipsnames]{color}
\usepackage[linesnumbered,ruled,vlined]{algorithm2e}
\SetKwComment{Comment}{(}{)}
\usepackage[shortlabels]{enumitem}
\usepackage[normalem]{ulem}
\usepackage[
backend=biber
]{biblatex}
\usepackage{hyperref}
\hypersetup{
  colorlinks   = true,    % Colours links instead of ugly boxes
  urlcolor     = blue,    % Colour for external hyperlinks
  linkcolor    = blue,    % Colour of internal links
  citecolor    = red      % Colour of citations
}
\usepackage{setspace}
\usepackage{nomencl}
\makenomenclature
\usepackage{geometry}
\theoremstyle{plain}
\newtheorem{theorem}{Theorem}[section]

\newtheorem{corollary}{Corollary}[section]
\newtheorem{lemma}{Lemma}[section]
\newtheorem{proposition}{Proposition}[section]
\theoremstyle{definition}
\newtheorem{definition}{Definition}[section]
\newtheorem{example}{Example}[section]
\newtheorem{rem}{Remark}[section]
\newtheorem{remark}{Remark}[section]

\newcommand\norm[1]{\left|\left|#1\right|\right|}
\newcommand\abs[1]{\left|#1\right|}
\newcommand\alg[1]{\left\langle#1\right\rangle}
\newcommand\floor[1]{\left\lfloor#1\right\rfloor}

\renewcommand{\vec}[1]{\mathbf{#1}}

\newcommand{\be}{\begin{equation}}
\newcommand{\ee}{\end{equation}}
\newcommand{\ba}{\begin{aligned}}
\newcommand{\ea}{\end{aligned}}
\newcommand{\ben}{\begin{enumerate}}
\newcommand{\een}{\end{enumerate}}
\newcommand{\bit}{\begin{itemize}}
\newcommand{\eit}{\end{itemize}}
\newcommand\nc[2]{\nomenclature{$#1$}{#2}}

\def\XX{{\mathbb X}}
\def\BB{{\mathbb B}}
\def\HH{{\mathbb H}}

\def\TT{{\mathbb T}}
\def\RR{{\mathbb R}}

\def\ZZ{{\mathbb Z}}

\def\MM{\mathbb{M}}
\def\PPI{{{\rm I}\kern-1pt\Pi}}
\def\SS{{\mathbb S}}

\newcommand\dist{\operatorname{dist}}
\DeclareMathOperator*{\argmin}{\arg\!\min}
\DeclareMathOperator*{\argmax}{\arg\!\max}

\newcommand{\x}{\mathbf{x}}
\newcommand{\y}{\mathbf{y}}

\def\ls{\lesssim}
\def\D{{\mathcal D}}
\def\gs{\gtrsim}
\def\disp{\displaystyle}

\renewcommand{\vec}[1]{\mathbf{#1}}

\newcommand{\yadi}{\nomenclature}

\def\tn{|\!|\!|}

\numberwithin{equation}{section}

\begin{document}

\pagenumbering{gobble}
%TITLE PAGE
\vspace*{.3\textheight}
\begin{center}
\huge 
Learning Without Training\\
\large
\vspace{.5in}
By\\
Ryan O'Dowd\\
\vfill
Claremont Graduate University\\
2026
\normalsize
\end{center}

\newpage
%COPYRIGHT PAGE
\qquad
\vfill
\begin{center}
$\copyright$ Copyright Ryan O'Dowd, 2026.\\
All rights reserved.
\end{center}

\newpage
%COMMITTEE APPROVAL PAGE
\begin{center}
\textbf{Committee Approval}
\end{center}

This dissertation has been duly read, reviewed, and critiqued by the committee listed
below, which hereby approves the manuscript of Ryan O'Dowd as fulfilling the scope
and quality requirements for meriting the degree of Doctor of Philosophy in Mathematics.
\vspace{10pt}

\begin{center}
Hrushikesh Mhaskar, Chair\\
Claremont Graduate University\\
Distinguished Research Professor of Mathematics

\vspace{10pt}
Asuman Aksoy\\
Claremont McKenna College\\
Crown Professor of Mathematics and George R. Roberts Fellow

\vspace{10pt}
Alexander Cloninger\\
 University of California, San Diego\\
Professor of Mathematics

\vspace{10pt}
Allon Percus\\
Claremont Graduate University\\
Joseph H. Pengilly Professor of Mathematics

\end{center}

\doublespacing
\newpage
%ABSTRACT

\begin{abstract}

We live in an era of big data. Whether it be algorithms designed to help corporations efficiently allocate the use of their resources, systems to block or intercept transmissions in times of war, or the helpful pocket companion known as ChatGPT, machine learning is at the heart of managing the real-world problems associated with massive data. With the success of neural networks on such large-scale problems, more research in machine learning is being conducted now than ever before. This dissertation focuses on three different projects rooted in mathematical theory for machine learning applications. Common themes throughout involve the synthesis of mathematical ideas with problems in machine learning, yielding new theory, algorithms, and directions of study.

The first project deals with supervised learning and manifold learning. In theory, one of the main problems in supervised learning is that of function approximation: that is, given some data set $\mathcal{D}=\{(x_j,f(x_j))\}_{j=1}^M$, can one build a model $F\approx f$? At the surface level, classical approximation theory seems readily applicable to such a problem, but under the surface there are technical difficulties including unknown data domains, extremely high dimensional feature spaces, and noise. We introduce a method which aims to tackle these difficulties and remedies several of the theoretical shortcomings of the current paradigm.

The second project deals with transfer learning, which is the study of how an approximation process or model learned on one domain can be leveraged to improve the approximation on another domain.
This can be viewed as the lifting of a function from one manifold to another.
This viewpoint enables us to connect some inverse problems in applied mathematics (such as the inverse Radon transform) with transfer learning.
We study such liftings of functions when the data is assumed to be known only on a part of the whole domain.
We are interested in determining subsets of the target data space on which the lifting can be defined, and how the local smoothness of the function and its lifting are related.

The third project is concerned with the classification task in machine learning, particularly in the active learning paradigm. Classification has often been treated as an approximation problem as well, but we propose an alternative approach leveraging techniques originally introduced for signal separation problems.
The analogue to point sources are the supports of distributions from which data belonging to each class is sampled from.
We introduce theory to unify signal separation with classification and a new algorithm which yields competitive accuracy to other recent active learning algorithms while providing results much faster.

\end{abstract}

%\newpage
%DEDICATION

%\begin{center}
%\textbf{Dedication}
%\end{center}

\newpage
%ACKNOWLEDGEMENTS

\begin{center}
\textbf{Acknowledgements}
\end{center}

My advisor Hrushikesh Mhaskar has my deepest gratitude for introducing me to the world of mathematics research and pushing me to improve and accomplish far beyond what I imagined possible for myself. His support, advice, kindness, and feedback have been instrumental to the construction of this dissertation and to my development as a mathematician. I will always remember fondly our discussions of mathematics and philosophy, walks along campus, meals shared, and travel abroad together. 

I am grateful for the advice and kindness that Allon Percus---who has been my academic advisor during my years at CGU---has provided to me over the years. His guidance led me along the path that has culminated in this work. I also would like to thank Asuman Aksoy and Alexander Cloninger for graciously agreeing to sit on my dissertation committee and provide valuable feedback. 

I would like to acknowledge my fellow students who I have shared in the challenges of homeworks, the cheer of Math Club, or even the occasional conversation with. 
The friendly faces which are always present at CGU make the campus such a joy to be around.
I would also like to acknowledge the professors at CGU and the Claremont Colleges who all provide exceptionally well-led courses and have immense care for the success of their students.
%This work would also not have been possible without the support of grants from the Office of Naval Research and the National Science Foundation.

Beyond CGU, I am thankful for the advice I have gotten from mathematicians and scientists alike. In particular, Richard Spencer, Raghu Raj, Jerry Kim, and Frank Filbir stand out for the professional guidance and knowledge they have passed to me.

I am immensely thankful for my parents, Sean and Allison. The pursuit of my dreams and the culmination of my work in this dissertation has been upheld by their ongoing support and unconditional love throughout my life. I am thankful for my fianc\'ee Diana, and all of the moments she encouraged me forward when doubts held me back. Her love is the reservoir through which I find my strength. I am thankful for my siblings Tyler, Trevor, and Kelsey, and their uncanny ability to raise my spirits. I am also thankful for Monica, whose warmth and kindness is always inspiring. I love you all.

This dissertation is the product of an accumulation of support more plentiful than I can mention or fully appreciate. To all those who I have known as friends in my life, thank you.

\pagenumbering{roman}

\newpage
\tableofcontents

\newpage

\pagenumbering{arabic}

\chapter{Introduction}
\label{ch:intro}

Suppose we are given data about the weather and want to build a model to predict future weather occurrences. Our first thought is to visualize the data to understand trends, but the issue is that each data point has so many different parameters: geographic location, humidity, wind speed, wind direction, temperature, climate, human pollution, atmospheric composition, time of day, etc. It is impossible to visualize. The goal is to use all of this data to estimate a weather \textit{function} for future time data points. Although the fundamental problem is just that of function approximation, the function in this case is described by so much data that constructing a mathematical model also seems infeasible. So where to even begin?

Problems of large and changing data like this one are not unique to meteorology. They face scientists and researchers in every field, from finances, geology, astronomy, biology, etc. Some challenges are unique to each field and even each specific problem, but common ones include large scale data sets, noisy data, high dimensionality, changing data, and more. Machine learning algorithms have become a popular choice to tackle these large-scale problems where data visualization and traditional mathematical modeling approaches are simply not feasible. 
This chapter serves to review background information on machine learning, and investigate the shortcomings which have motivated the work in this dissertation. 

We discuss machine learning along with its many paradigms and problems of interest in Section~\ref{sec:machineback}. We delve into the paradigm of supervised learning in Section~\ref{sec:mlparadigm}. We point out the traditional role approximation theory plays in supervised learning in Section~\ref{sec:approxthy}. 
To conclude the discussion of the supervised learning paradigm, we describe some shortcomings of this paradigm in Section~\ref{sec:obstacles}.
In Section~\ref{sec:constapprox}, we explore an alternative way in which approximation theory could play a more important role in machine learning. 
A relatively recent paradigm which motivated part of our research is the area of manifold learning, which we introduce in Section~\ref{sec:manifoldlearning}. 
We then point out in Section~\ref{sec:approxshort} the reasons why we think classical approximation theory cannot play a direct role in machine learning.
We then conclude this chapter by providing an introduction to our thesis in Section~\ref{sec:thesisintro}.

Since it is absolutely essential for clarity in the presentation of our results, we would also like to state now how constants will be treated in this dissertation. The symbols $c, c_1, c_2, \cdots$ will denote generic positive constants depending upon the fixed quantities in the discussion. 
Their values may be different at different occurrences, even within a single formula.
The notation $A\lesssim B$ means $A\le cB$, $A\gtrsim B$ means $B\lesssim A$, and $A\sim B$ means $A\lesssim B\lesssim A$.
In some cases where we believe it may be otherwise unclear, we will clarify which values a constant may depend on, which may appear in the subscript of the above-mentioned symbols. For example, $A\lesssim_d B$ means there exists $c(d)>0$ such that $A\leq c(d)B$.

In this chapter only, we denote the $\ell_p$ norms on vectors by $\norm{\circ}_p$, and the $L^p$ norms for functions defined on a domain $X$ by $\norm{\circ}_{X,p}$. In later chapters we will adjust our notation for norms to simplify our discussion. In this context, we will always assume that $1\le p\le \infty$.

\section{Machine Learning Background}
\label{sec:machineback}

Machine learning is a field at the intersection of mathematics, computer science, and data science concerned with the development of theory and algorithms capable of learning complex tasks. Typically, machine learning models are given some data from which they \textit{learn} the task at hand. The tasks of interest for machine learning approaches are often those where traditional mathematical modeling leaves something to be desired. As mentioned, this may be due to a large size of data, changing data, high dimensionality, noise, complicated factors influencing the data, or even the lack of a known mathematical model behind the data to begin with. We cite \cite{goodfellow} as an introductory resource for machine learning, and much of the content in this section can be learned in further depth there.

We start our discussion on the background of machine learning by first introducing several settings for problems that are studied in the field.
\bit
\item \textbf{Supervised Learning:} The main goal of \textit{supervised learning} is to generate a model to \textit{approximate} a function $f$ on unseen data points. Let $\mathcal{D}=\{(x_j,y_j)\}_{j=1}^M$ be a given data set, where $(x_j,y_j)$ are sampled from some unknown joint distribution $\tau=(X,Y)$. The goal is to approximate the target function $f(x)=\mathbb{E}_{\tau}(Y|X=x)$.

\item \textbf{Unsupervised Learning:} Let $\mathcal{D}=\{x_j\}_{j=1}^M$ be a given data set, where $x_j$ are sampled from some unknown distribution $\nu$ on a high-dimensional, Euclidean space $\mathbb{X}$. The goal of \textit{unsupervised learning} is to use $\mathcal{D}$ to \textit{construct} a function $f$ that satisfies some desired property.

\item \textbf{Semi-Supervised Learning:} The problems considered in \textit{semi-supervised learning} are broad and encompass the case where the $y_j$ from the supervised setting are given for only some subset of the total data. In a common setup, one may seek to construct the missing $y_j$ data points and then proceed in a supervised manner to approximate new data points.

\item \textbf{Active Learning:} Active learning incorporates ideas from both unsupervised and supervised learning. Let $\mathcal{D}=\{x_j\}_{j=1}^M$ be a given data set like that of the unsupervised setting, but in active learning we are also given $f$ called an \textit{oracle} in this context. We assume that there is some cost to query $f$, so we start in the unsupervised setting to determine some set of points we would like to query to maximize the information we can gain. Then once we have queried this set of points, we are in the semi-supervised setting where the goal is to approximate $f$ for the non-queried points.
\eit 

We note briefly that there are numerous other settings one may consider in the realm of machine learning, including \textit{distributed learning}, \textit{online learning}, and \textit{self-supervised learning}. Since these topics are not directly relevant to this dissertation we have omitted discussion on them.
There are two major problems or \textit{tasks} in machine learning which we will discuss.
\bit
\item \textbf{Classification:} In \textit{classification} problems, the function $f$ is discrete, taking on only some finite set of values called \textit{class labels}.

\item \textbf{Regression:} In \textit{regression} problems the function $f$ may take on any value on a continuum (and is often assumed to be a continuous function). Furthermore, in many cases it is assumed that $y_j=f(x_j)+\epsilon_j$ where $\epsilon_j$ is sampled from some unknown (often assumed mean-zero) distribution $\epsilon$.

\eit

Unsupervised learning is used mostly for classification problems. 
In unsupervised classification, one seeks to \textit{cluster} the data by somehow constructing $f$ to give the same class label to points that are ``similar" in some sense.
Supervised, semi-supervised, and active learning each are known to be used for both classification and regression problems.
We now turn our focus to the general paradigm in machine learning for solving problems in the supervised setting.

%\subsection{Transfer Learning}
%\label{subsec:transferlearning}
%
%Another area of machine learning which is relevant to this dissertation is known as transfer learning.
%The problem of transfer learning is to learn the parameters of an approximation process based on one data set and leverage this information to aid in the determination of an approximation process on another data set \cite{valeriyasmartphone, maskey2023transferability, maurer2013sparse}. An overview of transfer learning can be found in \cite{hosna-transfer}. Since training large (by today's standards) neural networks requires computing power simply unavailable to most researchers, the idea of transfer learning may be appealing to many wishing to use neural network models on new problems. While training a neural network from scratch on a new problem is not feasible in many cases, tuning a pre-trained network for a similar problem as the one it was trained on is doable and often yields better results anyways. Of course, this leads to some major questions such as
%\bit
%\item How does one identify which parameters for one problem are important for another?
%
%\item What if the new problem has a feature (or multiple) unlike any from the pre-trained model?
%
%\item How do we interpret the outputs of the model after it has been tuned for the new problem?
%\eit
%Since these questions are very broad and remain open in many settings, transfer learning is an exciting and active area of research.

\section{Machine Learning Paradigm for Supervised Learning}
\label{sec:mlparadigm}

In this section, we discuss the current machine learning paradigm for supervised learning.

As mentioned above, the fundamental problem of supervised learning is to build a model to approximate a function based on finitely many (noisy) data points from said function.
Therefore, we begin with a function approximation framework.
We assume that $f$ belongs to some class of functions called the \textit{universe of discourse} $\mathcal{X}$ (typically a Banach space), and then decide on some \textit{hypothesis spaces} $V_n$ of functions to model $f$ by (\textit{approximants} in the language of approximation theory).
As in approximation theory, a typical approach is to choose a nested family of hypothesis spaces
\be
V_1\subseteq V_2 \subseteq \dots \subseteq V_n\subseteq \dots,
\ee
where the value of the subscript often represents some notion of complexity. The idea is that these hypothesis spaces should satisfy a \textit{density property}. That is, as $n\to \infty$ the space $\bigcup_{n=1}^\infty V_n$ should be dense in $\mathcal{X}$.

\subsection{Examples of Hypothesis Spaces}

As an example, suppose the universe of discourse is chosen to be continuous, real functions on $[-1,1]^d$. Perhaps the simplest choice for a family of hypothesis spaces $\{V_n\}$ would be the sets of polynomials of degree at most $n$ on $[-1,1]^d$. Fitting data to this model choice is known as \textit{polynomial regression}. From the Stone-Weierstrass theorem, we know that if $f\in \mathcal{X}$, then for any $\epsilon>0$ there exists $n$ and $P\in V_n$ such that $\norm{f-P}_{[-1,1]^d,\infty}<\epsilon$. That is, the model spaces allow us to approximate $f$ arbitrarily closely as $n$ increases, demonstrating the density property. In the case $d=1$, polynomials form a simple model. However as $d$ grows they may become unwieldy, with the number of parameters growing exponentially with $d$.

A second choice for hypothesis spaces that has increased in popularity immensely in recent years is that of neural networks. 
A \textit{shallow neural network} is a function approximation model typically taking the form
\be\label{eq:shallownet}
F_n(\vec{x})=\sum_{k=1}^n a_k\phi(\vec{w}_k^T\vec{x}+b_k),
\ee
where $\displaystyle\bigcup_{k=1}^n\{a_k,w_{k,1},\dots,w_{k,d},b_k\}$ are \textit{learning parameters} associated with the model, $\vec{x}$ is the input value, $n$ is known as the \textit{width} (also \textit{complexity}) of the network, and $\phi$ is a judiciously chosen \textit{activation} function.

One advantage of a neural network model is that it can be implemented using parallel computation. As to whether or not a class of neural networks satisfies the density property is another matter. For instance, it is known that if $\phi(t)=t$, then there is no network which can reproduce even the exclusive-OR function. In Section~\ref{subsec:universal} we will return to this question.

Other common hypothesis spaces include radial basis functions and kernel-based approximation based on a variety of kernel functions.
Once a space $V_n$ is chosen to model the function by, the paradigm for supervised learning diverges from approximation theory instead opting to use ideas from statistics and optimization theory to solve the problem.

\subsection{Empirical Risk, Generalization, and Optimization}

The primary approach to select a model from $V_n$ is to introduce the notion of a \textit{generalization error}, which is given as a loss functional taking in the model as an input and returning some error value. Many choices can be used for such a loss functional, but perhaps the most common example is \textit{mean-squared-error} (MSE), which is given as
\be\label{eq:generalizationerror}
\mathcal{L}(P)=\int(y-P(x))^2d\tau(x,y).
\ee
Notice that the generalization error does not depend on the given data; instead it is an expectation over the entire distribution $\tau$, which again is not known. This means that one cannot expect to attain a minimizer of the generalization error, $\tilde{P}\coloneqq \argmin_{P\in V_n}\mathcal{L}(P)$, from only finite data. 
Instead, one typically seeks to find a minimizer of the \textit{empirical risk}, which is a discretized version of the generalization error based on the data. For example, the MSE empirical risk is given as
\be\label{eq:MSE-e}
L(\mathcal{D};P)=\frac{1}{M}\sum_{j=1}^M (y_j-P(x_j))^2=\frac{1}{M}\norm{\vec{y}-P(\vec{x})}^2_2.
\ee
The minimizer of the empirical risk, $P^\#\coloneqq \argmin_{P\in V_n}L(\mathcal{D};P)$, is the value that is ultimately sought after in machine learning applications.
However, there are two questions with this approach that remain active areas of research:
\bit
\item How can we attain the minimizer $P^\#$?

\item If we can attain it, how closely does it estimate $\tilde{P}$?
\eit
Both questions are essential to the performance of machine learning algorithms trained by empirical risk minimization, and in many cases neither question has a simple solution.

 When given large amounts of high-dimensional data, and especially in the presence of finite noise realizations, analytically deducing $P^\#$ may be impossible. So instead, a process known as \textit{optimization} is used.
The field of optimization theory is vast so we will limit our discussion to a commonly used method known as \textit{gradient (or steepest) descent}. The idea of optimization is to parameterize the hypothesis space and then leverage the parameterization with respect to a chosen loss functional to make successive (typically improving) guesses at an optimal value. For example, if $V_n$ is the space of all neural networks of the form \eqref{eq:shallownet} with $n$ neurons, then setting $\bm{\theta}$ to be a vectorization of $\displaystyle\bigcup_{k=1}^n\{a_k,w_{k,1},\dots,w_{k,d},b_k\}$ may serve as the vector of parameters.
Gradient descent starts with some initial guess of the parameters for the model to approximate the target function, known as the \textit{initialization}. Then, the guess is moved by some step size in the opposite direction of the gradient of the chosen loss function with respect to those parameters at that point. 
That is, the guess at iteration $j$, which we will denote by $\bm{\theta}_j$, is updated by the following rule:
\be
\bm{\theta}_{j+1}=\bm{\theta}_j-\eta \nabla L(\bm{\theta}_j),
\ee
where $\eta$ is called the \textit{learning rate}, or \textit{step size}.
For theoretical analysis, it is convenient to express this update as a dynamical process involving $\bm{\theta}$ as a function of of a variable $t$.
Then the process is described by
\be\label{eq:gdprocess}
\dot{\bm{\theta}}=-\eta \nabla L(\bm{\theta}),
\ee
where $\dot{}$ denotes the derivative with respect to $t$. Then, 
\be\label{eq:lossfndecay}
\dot{L(\bm{\theta})}=\nabla L(\bm{\theta})\cdot \dot{\bm{\theta}}=-\eta\norm{\nabla L(\bm{\theta})}_2^2.
\ee
Thus, with the update process \eqref{eq:gdprocess}, the loss functional is non-increasing, and hence reaches a limit as $t\to\infty$.

We examine the implementation of gradient descent for a shallow neural network with inputs from $\mathbb{R}$ using the MSE empirical risk as in \eqref{eq:MSE-e}. For simplicity, we will break up the parameters $\bm{\theta}$ into $\vec{a}=[a_1,\dots,a_n]^T$, $\vec{w}=[w_1,\dots,w_n]^T$, and $\vec{b}=[b_1,\dots,b_n]^T$. By defining
\be
\ba
\bm{\phi}=&\begin{bmatrix}
\phi(w_1x_1+b_1) & \phi(w_2x_1+b_2) & \dots & \phi(w_nx_1+b_n)\\
\phi(w_1x_2+b_1) & \phi(w_2x_2+b_2) & \dots & \phi(w_nx_2+b_n)\\
\vdots & \vdots & \ddots & \vdots\\
\phi(w_1x_M+b_1) & \phi(w_2x_M+b_2) & \dots & \phi(w_nx_M+b_n)
\end{bmatrix},\\
\bm{\phi}'=&\begin{bmatrix}
\phi'(w_1x_1+b_1) & \phi'(w_2x_1+b_2) & \dots & \phi'(w_nx_1+b_n)\\
\phi'(w_1x_2+b_1) & \phi'(w_2x_2+b_2) & \dots & \phi'(w_nx_2+b_n)\\
\vdots & \vdots & \ddots & \vdots\\
\phi'(w_1x_M+b_1) & \phi'(w_2x_M+b_2) & \dots & \phi'(w_nx_M+b_n)
\end{bmatrix},
\ea
\ee
we can set the vector
\be
\vec{z}=\bm{\phi}\vec{a}-f(\vec{x}),
\ee
in order to simplify our notation of the MSE loss functional:
\be
L(\mathcal{D};F_n)=\frac{1}{M}\norm{\vec{z}}^2_2.
\ee
Then we observe that
\be
\frac{d L}{d \vec{z}}=\frac{2}{M}\vec{z},
\ee
allowing us to compute each of the following by the chain rule:
\be
\ba
\frac{\partial L}{\partial \vec{a}}=&\frac{d L}{d\vec{z}}\frac{\partial \vec{z}}{\partial \vec{a}}=\frac{2}{M}\bm{\phi}^T\vec{z},\\
\frac{\partial L}{\partial \vec{w}}=&\frac{d L}{d\vec{z}}\frac{\partial \vec{z}}{\partial \vec{w}}=\frac{2}{M}(\bm{\phi}')^T\operatorname{diag}(\vec{x})\vec{z},\\
\frac{\partial L}{\partial \vec{b}}=&\frac{d L}{d\vec{z}}\frac{\partial \vec{z}}{\partial \vec{b}}=\frac{2}{M}(\bm{\phi}')^T\vec{z},
\ea
\ee
where $\operatorname{diag}(\vec{x})$ represents the matrix with each diagonal entry $i,i$ given by $x_i$ and each non-diagonal entry set to $0$. Thus, the gradient descent updates can be done in three parts:
\be
\ba
\vec{a}_{t+1}=&\vec{a}_t-\eta_1 \frac{2}{M}\bm{\phi}^T(\bm{\phi}\vec{a}-f(\vec{x})),\\
\vec{w}_{t+1}=&\vec{w}_t-\eta_2 \frac{2}{M}(\bm{\phi}')^T\operatorname{diag}(\vec{x})(\bm{\phi}\vec{a}-f(\vec{x})),\\
\vec{b}_{t+1}=&\vec{b}_t-\eta_3 \frac{2}{M}(\bm{\phi}')^T(\bm{\phi}\vec{a}-f(\vec{x})).
\ea
\ee

Since the calculation involves the computation of two matrices depending upon $\phi,\phi'$, one significant way to simplify the gradient descent calculation is to use an activation function which satisfies
\be
\phi'(x)=g(\phi(x)),
\ee
for a relatively simple function $g$. For instance, the function $\phi(x)=\tanh(x)$ is a sigmoidal activation function satisfying $\phi'(x)=1-\phi(x)^2$. Another example is the popular choice of activation function known as the \textit{rectified linear unit} (ReLU), defined by
\be
\phi(x)=\begin{cases} x & x>0\\ 0 & x\leq 0\end{cases}.
\ee
Even though this function is not differentiable at $0$, it is piecewise differentiable and for every non-zero point we can represent $\phi'(x)=x\phi(x)$. These simplifications allow the $\bm{\phi}'$ matrix above to be calculated directly from the $\bm{\phi}$ matrix, saving computation time.

\section{Approximation Theory}\label{sec:approxthy}

One major question in supervised learning is how to choose the hypothesis space.
The classical intersection of approximation theory with machine learning seeks to provide some theoretical justification for this choice.
There are two main questions which are investigated extensively. The first question is whether the sequence of hypothesis spaces is appropriate. That is, does the loss functional for $V_n$ converge to $0$ as $n\to \infty$? We discuss a couple of examples of early research in this direction in Section~\ref{subsec:universal}.
The second question is: how well can the target function be approximated by elements of $V_n$? In Section~\ref{subsec:degapprox}, we give some examples of early research in this direction.

\subsection{Universal Approximation}\label{subsec:universal}

We say that a sequence of hypothesis spaces, $\{V_n\}$, satisfies a \textit{universal approximation property} if for every $d\geq 1$, every compact set $K\subseteq \mathbb{R}^d$, every $f\in C(K)$, and every $\epsilon>0$, there exists some $N\in\mathbb{N}$ and $P\in V_N$ such that
\be
\norm{f-P}_{K,\infty}\leq \epsilon.
\ee
Any activation function which yields a family of universal approximator neural networks is called a \textit{Kolmogorov} function.
An activation function is called sigmoidal of order $k$ if $\lim_{x\to \infty}x^{-k}\phi(x)=1$ and $\lim_{x\to -\infty}x^{-k}\phi(x)=0$.
In \cite[Theorem 1 and Lemma 1]{cybenko-net} and \cite[Theorem 2.1]{chuili1992}, it was shown that shallow neural networks with continuous sigmoidal order $0$ activation functions satisfy the universal approximation property for continuous functions defined on compact sets.
In \cite[Corollary 3.4]{mhaskar-activation} it was shown that a continuous function $\phi$ satisfying the condition
\be
\sup_{x\in\mathbb{R}} (1+x^2)^{-n}\abs{\phi(x)}<\infty
\ee
for some integer $n$ is a Kolmogorov function if and only if it is not a polynomial. The work of \cite{leshno-activation} also showed that the activation function being non-polynomial was sufficient under different assumptions on $\phi$. 

\subsection{Degree of Approximation}
\label{subsec:degapprox}

While the above discussion on neural networks gives some indication that they may be a suitable hypothesis space for machine learning, a crucial question remains unanswered: how well can a fixed-width network approximate a function? Even if neural networks can approximate any function as $n\to \infty$, this gives us no idea about the rate at which we would converge to $f$.

To answer this question, researchers have often turned to degree of approximation results. The \textit{degree of approximation} is defined to be the least possible distance from $V_n$ to $f$. That is,
\be
E_n(f)\coloneqq \inf_{P\in V_n}\norm{f-P},
\ee 
where $\|\cdot\|$ is a suitably defined norm, typically the $L^2$ norm with respect to the marginal distribution of $\tau$ on $x$. 
%Notice that in this case, the degree of approximation would be exactly the bias term from \eqref{eq:biasvariance}.
The \textit{best approximation}, $P^*=\argmin_{P\in V_n}E_n(f)$, is the model from $V_n$ that minimizes the degree of approximation. This value indicates a global best that one can hope their hypothesis space (and thereby algorithm) to achieve. In the context of neural networks, significant research has gone into degree of approximation theorems in different settings.

We say that $f: \mathbb{R}^d\to \mathbb{R}$ belongs to the \textit{Barron Space} with parameter $s>0$, denoted by $B_s$, if it satisfies the following norm condition
\be
\norm{f}_{B_s}\coloneqq \int_{\mathbb{R}^d} (1+\abs{\vec{x}}^2)^{s/2}\abs{\hat{f}(\vec{x})}d\vec{x}<\infty.
\ee
In \cite[Theorem 1]{barron-universal}, it was shown that if $f\in B_1$, then there exists a network $G_n$ with a sigmoidal activation function and width $n$ such that
\be
\norm{f-G_n}_{\mathbb{B}^d(0,1),2}=O(1/\sqrt{n}).
\ee
Another degree of approximation result of note is for Sobolev spaces. The Sobolev space with parameters $r,p$ on a set $\mathbb{X}\subseteq \mathbb{R}^d$ is defined as $W^d_{r,p}(\mathbb{X})=\{f\in L^p(\mathbb{X}): D^{\vec{k}}f\in L^p(\mathbb{X}), \abs{\vec{k}}\leq r\}$. These spaces are associated with the following (respective) semi-norm and norm:
\be
\abs{f}_{W^d_{r,p}}\coloneqq \left(\sum_{\abs{\vec{k}}=r} \norm{f^{(\vec{k})}}_p^p\right)^{1/p},\qquad \norm{f}_{W^d_{r,p}}\coloneqq \left(\sum_{\abs{\vec{k}}\leq r} \norm{f^{(\vec{k})}}_p^p\right)^{1/p}.
\ee
If $\mathbb{X}=[-1,1]^d$, it was shown in \cite[Theorem 2.1]{mhaskar-net} that if $f\in W^d_{r,p}$ then one can \textbf{construct} a network $G_n$ with width $n$ such that
\be\label{eq:mhaskartheo}
\norm{f-G_N}_{\mathbb{X},p}\lesssim n^{-r/d}\norm{f}_{W^d_{r,p}}.
\ee
The constructive process starts with a polynomial approximation which satisfies the bound in \eqref{eq:mhaskartheo}. For example, a shifted average of the partial sums of the Chebyshev expansion of $f$ can be used in the uniform approximation case ($p=\infty$). Then each of the monomials in the polynomial can be approximated by divided differences of the function defined by $\bm{w}\mapsto \phi(\bm{w}\cdot \x +b)$.

In \cite{mhaskar-multilayer}, degree of approximation results were shown for multilayer neural networks in terms of a variational modulus of smoothness. A survey of some of the main results of the time can be found in \cite{pinkus-survey}. Over the years, many others have studied optimal approximation properties of neural networks in different contexts including \cite{kutyniok-net,petersen-net,shen-net}.

Degree of approximation results like these are oft-cited to motivate the usage of neural networks in practical applications. In the following section, we discuss why this setup may not be as desirable as it first seems.
To summarize this section, we show the current paradigm graphically in Figure~\ref{fig:mlparadigm}.

\begin{figure}[!ht]
\centering
    \includegraphics[scale=0.20]{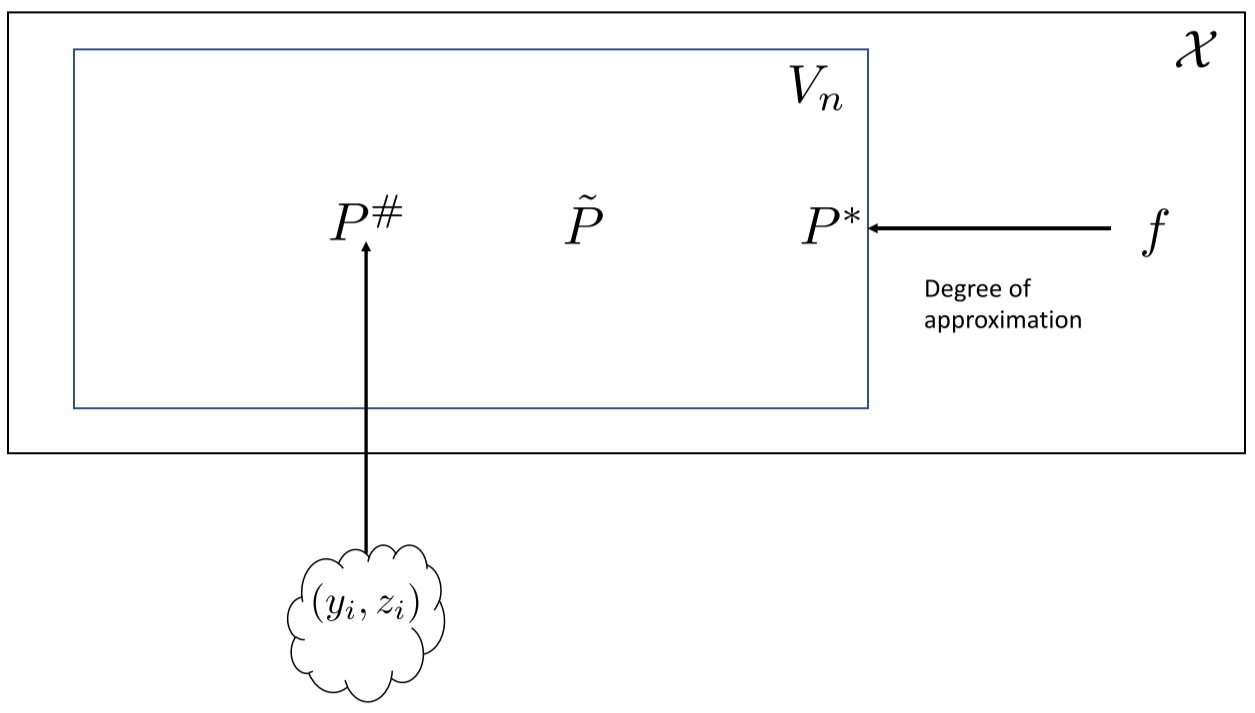}\\
    \caption{A depiction of the standard supervised learning paradigm. The universe of discourse $\mathcal{X}$ is assumed to contain a target function $f$ and hypothesis spaces $V_n$ are judiciously chosen based on the algorithm of choice. $P^\#$ denotes the empirical risk minimizer, $\tilde{P}$ denotes the minimizer of the generalization error, and $P^*$ denotes the best approximation.}
    \label{fig:mlparadigm}
\end{figure}

%\textcolor{red}{Perhaps include a section for paradigm of unsupervised learning including clustering by minimization of some entropy functional. This may be a good way to introduce the idea of linear vs non-linear clustering}.

\section{Shortcomings of the Supervised Learning Paradigm}
\label{sec:obstacles}

In this section we note some of the shortcomings of the current paradigm for supervised learning from a theoretical point of view, in spite of its tremendous success in practice.

\begin{enumerate}
\item Any optimization problem, gradient descent in particular, suffers from such shortcomings as getting stuck in a local minimum, slow convergence or even a lack of convergence, instability, and sensitivity to the initial choice of parameters. For example, in \cite{lu2019dying} a neural network tasked with approximating $|t|=t_++(-t)_+$ yields a constant output when given a poor choice of initialization. The phenomenon has become known as \textit{dead on arrival}. 
Another issue is how to decide when to stop iterating. One may halt the iterations once $\abs{\bm{\theta}_j-\bm{\theta}_{j-1}}$ are sufficiently small, but this runs into the issue of \textit{false stabilization}, or reaching a point where the iterations seem to converge to a point but in actuality will continue changing given enough iterations.

\item The use of degree of approximation for deciding upon the choice of the model space might be misleading. We elaborate upon this point in further detail in this section. 

\item The use of a global loss functional such as the one defined in \eqref{eq:generalizationerror} is insensitive to local artifacts in the target function. Finding minimizers $\tilde{P}$ (of the loss functional) or even $P^*$ (best approximation) may not preserve the important local effects to the problem at hand. We will elaborate on this aspect more in Section~\ref{sec:constapprox}.

\item In many cases, one can give a direct solution to the fundamental problem of machine learning, so that the approximation at each point is affected only by the smoothness of the target function in the vicinity of that point. Furthermore, this can be done while preserving the order of magnitude of the degree of approximation according to the local smoothness of the function on that neighborhood. We will elaborate on this aspect in Section~\ref{sec:constapprox}.
\end{enumerate}

It has become common to motivate the usage of neural networks for machine learning applications by citing degree of approximation results as in Section~\ref{subsec:degapprox} or similar. But it is important to recognize that these types of results are often existence-based and ignore the fact that in practice the approximation needs to be constructed from the data. This distinction should not be understated.
To give an example, we look at a ReLU activation function (defined by $t\to |t|=(t)_++(-t)_+$) network to approximate functions on the $d$-dimensional sphere $\mathbb{S}^d$. The functions in question are assumed to have a representation of the form
\be\label{eq:relunative}
f(\vec{x})=\int_{\mathbb{S}^d} |\vec{x}\cdot \vec{y}|\mathcal{
D}_{|\cdot|}(f)(\vec{y})d\mu_d^*(\vec{y}),
\ee
where $\mu_d^*$ is the volume measure of $\SS^d$ normalized to be a probability measure and $D_{|\cdot|}(f)\in C(\SS^d)$.
We define, in this discussion, $\|f\|_{W_{|\cdot|}}=\|f\|_\infty+\|\mathcal{D}_{|\cdot|}(f)\|_{\infty}$, and $W_{|\cdot|}$ to be the class of functions with $\|f\|_{W_{|\cdot|}}<\infty$.
The results in \cite[Corollary 4.1 and Remark 4.1]{mhaskar2020dimension} show that for $f\in W_{|\cdot|}$, there exist $a_k$, $\vec{y}_k$ depending upon $f$ in an unspecified manner such that
\be\label{eq:dimindbd}
\norm{f-\sum_{k=1}^N a_k\abs{\circ\cdot \vec{y}_k}}_{\mathbb{S}^d}\lesssim (N^{-(d+3)/(2d)})\|f\|_{W_{|\cdot|}}.
\ee
In contrast, it is proved in \cite[Corollary 4.1 and Remark 4.5]{sphrelu} that for any choice of $\vec{y}_k$ which admit a quadrature formula exact for integrating spherical polynomials of a certain degree, and scattered data of the form $(\bm{\xi}_j, f(\bm{\xi}_j))$, $\bm{\xi}_j\in\SS^d$, there exist linear functionals $a_k$ (given explicitly in \cite{sphrelu}) depending only on this data such that
\be\label{eq:dimdepbd}
\norm{f-\sum_{k=1}^N a_k\abs{\circ\cdot \vec{y}_k}}_{\mathbb{S}^d}\lesssim (N^{-2/d})\|f\|_{W_{|\cdot|}}.
\ee
 We note that  unlike the bounds for Barron spaces and Sobolev spaces discussed before, both the bounds above are for the same space of functions. 
 The dimension-independent bounds are obtained purely as an existence theorem, while the dimension-dependent bounds are based on explicit constructions, essentially solving the problem of supervised learning directly. 
In particular, utilizing degree of approximation results to motivate the use of a particular model in practice can be misleading.

%As another note, much effort must go into training neural networks in practice as well. This is because the training procedure, and even the initialization of this training procedure are crucial to the success of neural networks as a model. 
%For example, in \cite{lu2019dying} a neural network tasked with approximating $|t|=t_++(-t)_+$ yields a constant output when given a poor choice of initialization. The phenomenon has become known as \textit{dead on arrival}.

\section{Constructive Approximation}\label{sec:constapprox}
The purpose of this section is to introduce a direct and simple way of solving the problem of machine learning in some special cases. In particular, since trigonometric approximation serves as a role model for almost all approximation processes, we elaborate our ideas in this context.

We introduce the approximation of multivariate $2\pi$-periodic functions by trigonometric polynomials. We define $\mathbb{T}=\mathbb{R}/(2\pi\mathbb{Z})$ and denote in this section only $\mu^*_d$ to be the probability Lebesgue measure on $\mathbb{T}^d$.
In this section, we choose the universe of discourse to be the space $C(\TT^d)$ comprising continuous real functions on $\TT^d$, equipped with the uniform norm.
We define the hypothesis spaces $V_n$ by
$$
V_n=\mathsf{span}\{\vec{x}\mapsto \exp(i\vec{k}\cdot \vec{x}) : |\vec{k}| <n\}.
$$
%For $f\in C(\TT^d)$, we define the degree of approximation by
%$$
%E_n(f)=\min_{P\in V_n}\|f-P\|.
%$$
The Fourier coefficients of a function $f\in L^1(\mathbb{T}^d)$ are defined by
\be
\hat{f}(\vec{k})=\int_{\mathbb{T}^d} f(\vec{x})e^{-i\vec{k}\cdot \vec{x}}d\mu^*_d(\vec{x}).
\ee
The best approximation in the sense of the global $L^2$ norm is given by the Fourier projection, defined by
\be
s_n(f)(\vec{x})=\sum_{\abs{\vec{k}}< n} \hat{f}(\vec{k})e^{i\vec{k}\cdot \vec{x}}.
\ee
It is well known that $s_n(f)(0)$ might not converge to $f(0)$ for every $f\in C(\mathbb{T}^d)$. Furthermore, the best approximation in the sense of the uniform norm is difficult to compute in general, where the process may be aided by methods such as the Remez algorithm.
In order to achieve convergence in the uniform norm as well as the right order of magnitude for the degree of approximation on different arcs according to the smoothness of the target function on those arcs, we define a kernel by
\be\label{eq:kerntrig}
\Phi_n(\vec{x})=\sum_{\abs{\vec{k}}<n}h\left(\frac{\abs{\vec{k}}}{n}\right)e^{i\vec{k}\cdot \vec{x}},
\ee
where $h$ is a $C^\infty$ even function with $h(t)=1$ for $|t|\in [0,1/2]$ and $h(t)=0$ for $|t|\geq 1$. Then our reconstruction operator is given by
\be\label{eq:sigmatrig}
\sigma_n(f)(\vec{x})=\int_{\mathbb{T}^d}f(\vec{z})\Phi_n(\vec{x}-\vec{z})d\mu^*_d(\vec{z})=\sum_{\abs{\vec{k}}<n}h\left(\frac{\abs{\vec{k}}}{n}\right)\hat{f}(\vec{k})e^{i\vec{k}\cdot \vec{x}}.
\ee
One can show that for any $f\in C(\TT^d)$ and $n\ge 1$,
\be\label{eq:goodapprox}
E_n(f)\le \|f-\sigma_n(f)\|\lesssim E_{n/2}(f).
\ee
Due to this inequality, we say that $\sigma_n(f)$ is a \textit{good approximation}. 
Given that the reconstruction operator $\sigma_n$ has the form in \eqref{eq:sigmatrig}, it is in general much easier to compute than the best approximation.
Moreover, we can write down the solution of the supervsied learning problem in this context directly, without any optimization.
We reproduce the following theorem from \cite{mhaskar-survey}.

 In order to present the theorem, we need to define a discretization of \eqref{eq:sigmatrig}, given data $\{(\vec{x}_j,y_j)\}_{j=1}^M$:
\be\label{eq:sigmadisctrig}
\tilde{\sigma}_n(x)=\frac{1}{M}\sum_{j=1}^M y_j\Phi_n(x-x_j).
\ee

\begin{theorem}[{\cite[Theorem~3.5]{mhaskar-survey}}]
Let the marginal distribution of the points $\{x_j\}$ be $\mu_d^*$. Let $\gamma>0$ and $f$ belong to a smoothness class with parameter $\gamma$; i.e.,
$$
\|f\|_{W_\gamma}=\|f\| +\sup_{n\ge 0}2^{n\gamma}E_{2^n}(f) <\infty.
$$
 If $n\gtrsim 1$ and $M\gtrsim n^{d+2\gamma}\log(n)$, then with probability going to $1$ as $M\to \infty$, we have
\be\label{eq:machinesolve}
\norm{\tilde{\sigma}_n-f}_{\mathbb{T}^d,\infty}\lesssim n^{-\gamma}.
\ee
\end{theorem}

While this construction may seem disparate from neural networks, there is actually a trick to represent the approximation in the form \eqref{eq:sigmadisctrig} as a neural network structure. From \cite{mhaskarmicchelli}, we have
\be
e^{i\vec{k}\cdot \vec{x}}=\frac{1}{2\pi \hat{\phi}(1)}\int_{\mathbb{T}} \phi(\vec{k}\cdot \vec{x}-t)e^{it}dt.
\ee
Then, this integral can be discretized and substituted into \eqref{eq:kerntrig} so that \eqref{eq:sigmadisctrig} forms a neural network with activation function $\phi$.

Next, we discuss the notion of local approximation using our reconstruction operator in the case when $d=1$. We look at an example of approximating the function $f(\theta)=\abs{\cos\theta}^{1/4}$. We note that the function has singularities at $-\pi/2,\pi/2$. In Figure~\ref{fig:trigcomparison}, we see the recovery error of $f$ by $\sigma_{64}$, $\sigma_{128}$, and $\sigma_{256}$ compared with the recovery errors by $s_{64}$, $s_{128}$, and $s_{256}$ respectively. We see that the good approximation actually reduces error much faster away from the singularities than the best approximation.

\begin{figure}
\begin{center}
\includegraphics[width=.9\textwidth]{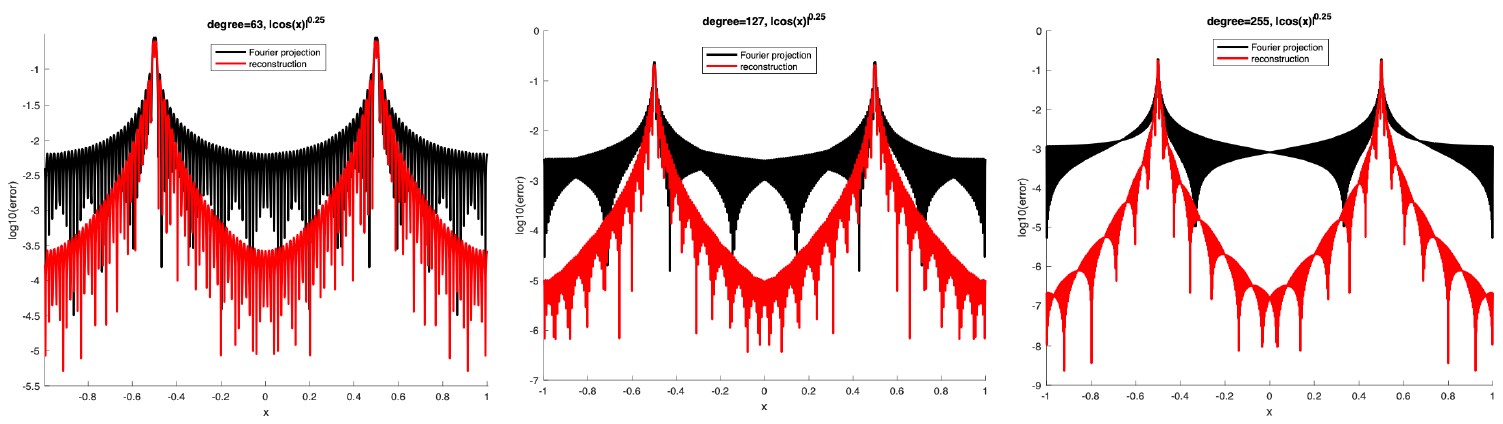}
\caption{Comparison of recoveries for $f(\theta)=\abs{\cos\theta}^{1/4}$ by the best approximation (black) and a good approximation (red) for degrees $n=63,127,255$. Figure credit: Hrushikesh Mhaskar.}
\label{fig:trigcomparison}
\end{center}
\end{figure}

\section{Manifold Learning}\label{sec:manifoldlearning}

The following section contains excerpts from our papers \cite{mhaskarodowd,mhaskarodowdtransfer}.

The purpose of this section is to introduce a relatively new paradigm of manifold learning. 
First, we examine a phenomenon known as the \textit{curse of dimensionality}. It is a critical issue in machine learning which informs our later choice to work with functions defined on manifolds.
We can formalize this notion through the lense of \textit{nonlinear widths}, shown graphically in Figure~\ref{fig:nonlinwidth}. We start with a subset of a Banach space of functions $\mathbb{K}\subseteq \mathbb{X}$ defined on a domain $X\subseteq \mathbb{R}^Q$, a (continuous) parameterization scheme $\mathcal{P}: \mathbb{K}\to \mathbb{R}^M$, and a (continuous) approximation scheme $\mathcal{A}:\mathbb{R}^M\to \mathbb{X}$. The nonlinear $L^p$ width is defined by
\be
d_M(\mathbb{K},\mathbb{X})_p\coloneqq\inf_{\mathcal{P},\mathcal{A}} \sup_{f\in \mathbb{K}}\norm{f-\mathcal{A}(\mathcal{R}(f))}_{X,p}.
\ee
This value represents the worst possible error resulting from the approximation of any $f\in \mathbb{K}$ using the best possible (continuous) parameterization and approximation process.

\begin{figure}[!ht]
\begin{center}
\includegraphics[width=.5\textwidth]{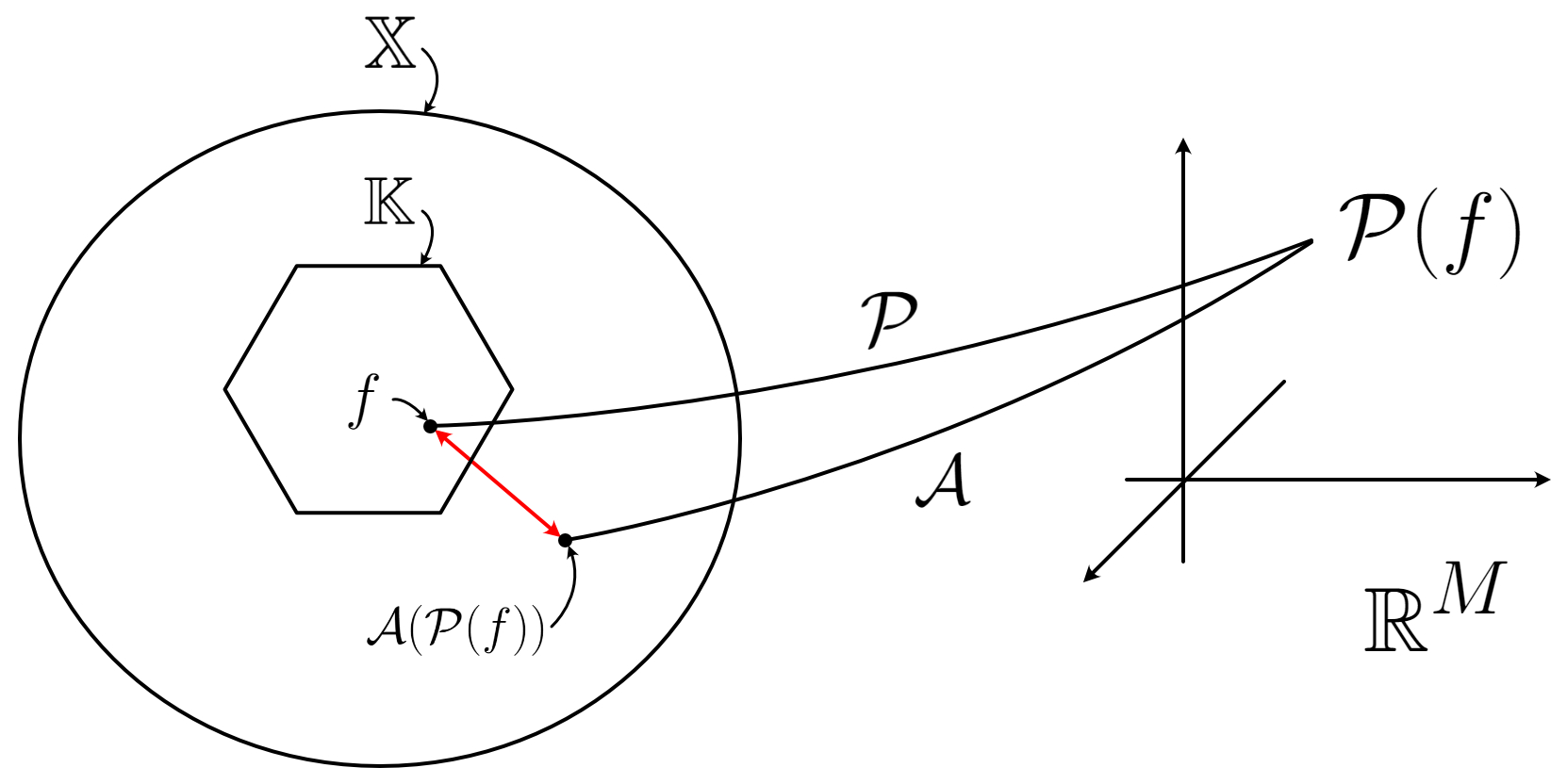}
\caption{Depiction of nonlinear width. Here $\mathbb{X}$ is a metric space of functions with $\mathbb{K}$ some subset. The goal is to understand the error associated with approximating any function $f\in \mathbb{K}$, where $\mathcal{P}$ is a continuous parameterization of $\mathbb{K}$ and $\mathcal{A}$ is an approximation scheme mapping into $\mathbb{X}$.}
\label{fig:nonlinwidth}
\end{center}
\end{figure}

In particular, we have the following theorem from \cite{devorenonlinear} demonstrating this curse for the Sobolev spaces as defined in Section~\ref{subsec:degapprox}. 

\begin{theorem}[{\cite[Theorem 4.2]{devorenonlinear}}]
Let $K^Q_{r,p}=\{f\in L^P([0,1]^Q): \abs{f}_{W^Q_{r,p}}\leq 1\}$, and $1\leq p\leq q$. Then,
\be
d_M(K^Q_{r,p},W^Q_{r,p})_{q}\gtrsim_r M^{-r/Q}.
\ee
\end{theorem}
This theorem informs us that if we want an approximation process that estimates every function in $K_{r,p}$ by some corresponding function in $W_{r,p}$ with an $L^p$ error less than $\epsilon$, we need $M \gtrsim_r \epsilon^{-Q/r}$. The curse of dimensionality comes into play when $Q$ is large and $\epsilon$ is small, making this lower bound on $M$ so large as to be impractical. The curse of dimensionality is a symptom of the choice of the hypothesis spaces the approximation process belongs to.

A relatively recent idea to alleviate the curse of dimensionality is to assume the values $x_j$ are sampled from an unknown, low-dimensional submanifold of the high-dimensional ambient space and build the approximation scheme based on functions from this lower-dimensional structure.
This has become known as the \textit{manifold assumption}.
In theory, this assumption implies that the overall structure of the data can be preserved on a much lower dimensional space.

Many methods for such dimensionality reduction have been studied, including Isomaps \cite{tenenbaum2000global}, locally linear embedding
 (LLE) \cite{roweis2000nonlinear}, Hessian locally
 linear embedding (HLLE) \cite{david2003hessian},  diffusion maps (Dmaps) \cite{coifmanlafondiffusion}, maximum variance unfolding (MVU) which is also known as
 semidefinite programming (SDP) \cite{weinberger2005nonlinear}, local tangent space alignment (LTSA) \cite{zhang2004principal}, Laplacian eigenmaps (Leigs) \cite{belkinlaplacian}, and randomized anisotropic transform \cite{chuiwang2010}. 
Chui and Wang have given a survey on these methods in \cite{chuidimred2015}. The special issue \cite{achaspissue} of Applied and Computational Harmonic Analysis (2006) provides a great introduction on diffusion geometry. 
Applications of these methods include, but are by no means limited to, semi-supervised learning \cite{niyogi1, niyogi2}, document analysis \cite{coifmanmauro2006}, 
face recognition \cite{chuiwang2010, niyogiface, ageface2011}, hyperspectral imaging \cite{chuihyper}, image processing \cite{arjuna1, donoho2005image},  cataloguing of galaxies \cite{donoho2002multiscale}, and social networking \cite{bertozzicommunity}.

The term \emph{manifold learning} has generally become associated with a two-step procedure, where there is first finding information about the manifold itself, and second utilizing this information to do function approximation. With the two-step procedure, the estimates obtained in function approximation need to be tempered by the errors accrued in the manifold learning step.
In turn, the errors in the manifold learning step may be sensitive to the choice of different parameters used in the process as well as noise in the data.

One approach is to estimate an atlas of the manifold, which thereby allows function approximation to be conducted via local coordinate charts. 
One such effort is to utilize the underlying parametric structure of the functions to determine the dimension of the manifold and the parameters involved \cite{manoni2020effective}.
Approximations utilizing estimated coordinate charts have been implemented, for example, via deep learning  \cite{cloninger-net,coifman_deep_learn_2015bigeometric,schmidt2019deep}, moving least-squares  \cite{soberleastsquares}, local linear regression  \cite{Wuregression},  or using Euclidean distances among the data points \cite{chui_deep}.  It is also shown in \cite{jones2010universal, jones2008parameter} that an atlas on the unknown manifold can be defined in terms of the heat kernel corresponding to the Laplace-Beltrami operator on the manifold.

This leads to another approach, which is to look at an eigendecomposition of the Laplace-Beltrami operator.
It has been shown that the so-called graph Laplacian (and the corresponding eigendecomposition) constructed from data points converges to the manifold Laplacian and its eigendecomposition  \cite{belkinlaplacian,belkinfound,singerlaplacian}.
An introduction to the subject is given in \cite{achaspissue}. 
In \cite{coifmanmauro2006, heatkernframe}, a multi-resolution analysis is constructed using the heat kernel. 
Another important tool is the theory of localized kernels based on the eigen-decomposition of the heat kernel. 
These were introduced in \cite{mhaskarmaggioniframe} based on certain assumptions on the spectral function and the property of finite speed of wave propagation. 
In the context of manifolds, this later property was proved in \cite{sikora2004riesz, frankbern} to be equivalent to the so called Gaussian upper bounds on the heat kernels. 
These bounds have been studied in many contexts by many authors, e.g., \cite{grigoryan1995upper, grigor1997gaussian, davies1990heat, kordyukov1991p}, and recently for a general smooth manifold in \cite{mhaskardata}.

Function approximation on manifolds based on \emph{scattered data} (i.e., data points $x_j$ whose locations are not prescribed analytically) has been studied in detail in many papers, starting with \cite{mhaskarmaggioniframe}, e.g., \cite{frankbern, modlpmz, eignet, compbio, heatkernframe, mhaskardata}. 
A theory was applied successfully in \cite{mhas_sergei_maryke_diabetes2017} to construct  deep networks for predicting blood sugar levels based on continuous glucose monitoring devices.
In \cite[Theorem~4.3]{tauberian}, localized kernels were constructed based on the Gaussian upper bound on the heat kernel.	

\section{Shortcomings of Classical Approximation Theory}\label{sec:approxshort}
As we discussed, the fundamental problem of supervised learning is that of function approximation, yet approximation theory has played only a marginal role.  We point out some  reasons.
\begin{enumerate}
\item We do not typically know whether the assumptions on the target function involved   in the approximation theory bounds are satisfied in practice, or whether the number of parameters is the right criterion to look at in the first place.
For example, when one considers approximation by radial basis function (RBF) networks, it is  observed in many papers (e.g., \cite{eignet}) that the minimal separation among the centers is the right criterion rather than the number of parameters. 
It is shown that if one measures the degree of approximation in terms of the minimal separation, then one can determine the smoothness of the underlying target function by examining the rate at which the degrees of approximation converge to $0$.
\item Most of the approximation theory literature focuses on the question of estimating the degree of approximation in various norms and conditions on $f$, where the support of the marginal distribution $\nu$ is assumed to be a known domain, such as a torus, a cube, the whole Euclidean space, a hypersphere, etc.; equivalently, one assumes that the data points $y_j$ are ``dense'' on such a domain.
This creates a gap between theory, where the domain of $\nu$ is known, and practice, where it is not.
\item One consequence of approximating, say on a cube, is the curse of dimensionality as mentioned. Manifold learning seeks to alleviate this problem, but the current methodology in that area is a two-fold procedure, introducing extra errors in the problem of function approximation itself, apart from the choice of hyper-parameters, need to find the eigen-decomposition of a large matrix, out-of-sample extension, etc. 
\end{enumerate}

\section{Organization of the Thesis}\label{sec:thesisintro}
This dissertation introduces three new approaches for problems of machine learning. A unifying theme across the methods is that they are developed from a harmonic analysis and approximation theory viewpoint.

The first project, discussed in Chapter~\ref{ch:manifoldapprox}, involves the approximation of noisy functions defined on an unknown submanifold of a hypersphere directly, without training and without estimating information about the submanifold. We note that our theory may also work with data from a Euclidean space via a projection to the hypersphere of the same dimension. This project aims to remedy ailments in the current supervised learning paradigm by introducing a method working in a new paradigm, where the results attained from real data are tied directly to those in the theory.

The second project, discussed in Chapter~\ref{ch:transferlearning}, introduces a new approach for localized transfer learning. The idea of transfer learning is to take information learned in one setting to aid with learning in another. Our theory works with data spaces which are a generalization of manifolds.

The third project, discussed in Chapter~\ref{ch:classification}, tackles problem of classification in machine learning. The problem of classification in the supervised setting has often been approached by function approximation. This yields several theoretical challenges which were not present when dealing with regression. Much work has been done to overcome these challenges. There is however an alternate perspective, which draws inspiration from signal separation problems. We introduce the problem of signal separation and show how the problem of machine learning classification can be viewed as a generalization of this problem. Our results are built from this alternative perspective which solves the classification problem completely when the classes are well-separated. We also introduce theory and an algorithm to tackle the case where the data classes may not be well-separated.

We note here that the chapters discussing these projects will be excerpts from our papers, and therefore have self-contained notation that may not apply across chapters.

\chapter{Approximation on Manifolds}
\label{ch:manifoldapprox}

%\title{Learning on manifolds without manifold learning}
%\author{H. N. Mhaskar\thanks{
%Institute of Mathematical Sciences, Claremont Graduate University, Claremont, CA 91711. 
%\textsf{email:} hrushikesh.mhaskar@cgu.edu.
%The research is  supported in part by NSF grant DMS 2012355, and ONR grants N00014-23-1-2394, N00014-23-1-2790.} \and Ryan O'Dowd\thanks{Institute of Mathematical Sciences, Claremont Graduate University, Claremont, CA 91711. 
%\textsf{email:} ryan.o'dowd@cgu.edu.}}
%\date{\today}

The content in this chapter is sourced from our paper published in \textit{Neural Networks} titled ``Learning on manifolds without manifold learning" \cite{mhaskarodowd}.

\section{Introduction}
\label{sec:intromanifold}

Let $\mathcal{D}=\{(y_j, z_j)\}_{j=1}^M$ be our data, drawn randomly from an unknown probability distribution $\tau$,  and set  $f(y)=\mathbb{E}_\tau(z|y)$. We recall the fundamental problem of supervised learning is to approximate $f$ given the data $\mathcal{D}$.
The purpose of this chapter is to introduce a direct method of approximation on \emph{unknown}, compact, $q$-dimensional submanifolds of $\mathbb{R}^Q$ without trying to find out anything about the manifold other than its dimension. The motivation of this approach is outlined in Section~\ref{sec:manifoldlearning}.

As discussed in Chapter~\ref{ch:intro}, there are many shortcomings to the current paradigm for supervised learning. The results in this chapter belong to an alternate machine learning paradigm proposed by H.N. Mhaskar, shown in Figure~\ref{fig:newparadigm}. In this paradigm, one is still concerned with assuming that $f$ belongs to some universe of discourse and approximating it by some function in a hypothesis space $V_n$. The major difference lies in the approximation. Instead of using optimization to minimize empirical risk, the idea is to directly \textit{construct} a \textit{good} approximation $\sigma_n\in V_n$ in the approximation theory sense.

\begin{figure}[!ht]
\centering
    \includegraphics[scale=0.20]{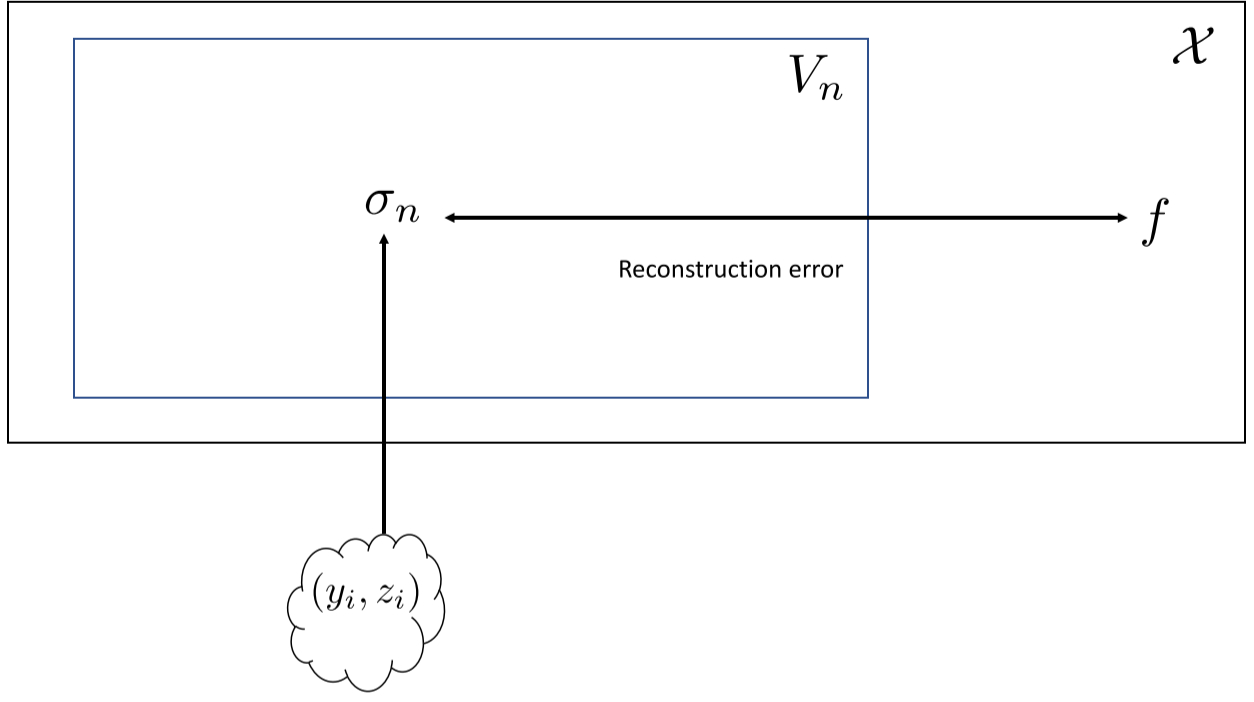}\\
    \caption{A depiction of a new machine learning paradigm, where one constructs an approximation $\sigma_n$ in the space $V_n$ directly from the data. This is done in such a way that one can also measure a direct reconstruction error from the approximation to the target function.}
    \label{fig:newparadigm}
\end{figure}

More specifically, in \cite{mhaskar2020deep} it is shown how a specific construction of kernels results in a direct function approximation method when the data space is an unknown, smooth, compact, connected manifold. The approximation procedure is additionally done directly from the data, without the two step procedure described in Section~\ref{sec:manifoldlearning}.
The key to that work was the construction of localized kernels on the Euclidean space in terms of Hermite functions. We define
\be
\mathcal{P}_{n,q}\coloneqq \begin{cases}
(-1)^n\frac{\sqrt{(2n)!}}{2^nn!\pi^{1/4}}h_{2n}(x)&\text{if }q=1\\
\frac{1}{\pi^{(2q-1)/4}}\Gamma((q-1)/2)\sum_{\ell=0}^n(-1)^\ell \frac{\Gamma((q-1)/2+n-\ell)\sqrt{(2\ell)!}}{(n-\ell)!2^\ell \ell!}h_{2\ell}(x)&\text{else,}
\end{cases}
\ee
where $h_n$ is the orthonormalized Hermite function as defined in \eqref{eq:hermitefun}.
Then the localized kernel is defined by the sum
\be
\tilde{\Phi}_{n,q}(x)=\sum_{k=0}^{\floor{n^2/2}}H\left(\frac{\sqrt{2k}}{n}\right)\mathcal{P}_{k,q}(x),
\ee
where $H:[0,\infty)\to [0,1]$ is a $C^\infty$ function such that $H(t)=1$ for $t\in [0,1/2]$ and $H(t)=0$ for $t\geq 1$. Letting $\mathcal{D}=\{(\vec{x}_j,y_j)\}_{j=1}^M$, the approximation process for $\vec{x}\in \mathbb{R}^q$ is then defined by
\be\label{eq:hermiteapprox}
\tilde{F}_{n,\alpha}(\mathcal{D};\vec{x})\coloneqq \frac{n^{q(1-\alpha)}}{M} \sum_{j=1}^M y_j\tilde{\Phi}_{n,q}(n^{1-\alpha}\norm{\vec{x}-\vec{x}_j}_2).
\ee
The main theorem of that work is the following.
\begin{theorem}[{\cite[Theorem~3.1]{mhaskar2020deep}}]
Let $\gamma>0$, and $\mathcal{D}$ be a data set sampled from a distribution $\tau$ supported on $\mathbb{X}\times \Omega$ where the marginal distribution of $\tau$ restricted to the manifold $\mathbb{X}$ (where the $\vec{x}_j$'s are sampled) is absolutely continuous with respect to $\mu^*$ and has density $f_0\in W_\gamma(\mathbb{X})$, where $W_\gamma$ is a smoothness class analogous to \ref{def:manifold_smoothness}. Let $f\in W_\gamma(\mathbb{X})$, $0<\alpha<\frac{4}{2+\gamma}$, and $\alpha\leq 1$. Then for every $n\geq 1$ and $0<\delta<1$, if $M\geq n^{q(2-\alpha)+2\alpha\gamma}\log(n/\delta)$, with probability at least $1-\delta$ we have
\be
\norm{\tilde{F}_{n,\alpha}(\mathcal{D};\circ)-ff_0}_{\mathbb{X},\infty}\lesssim \frac{\sqrt{\norm{f_0}_{\mathbb{X},\infty}}\norm{Y}_{\mathbb{X}\times \Omega,\infty}+\norm{ff_0}_{W_\gamma(\mathbb{X})}}{n^{\alpha\gamma}}.
\ee
\end{theorem}
This theorem demonstrates an approximation process for functions defined on an unknown manifold. There are some limitations however. The first is that the complexity of the approximation grows like $n^2$ even though the approximation rate is dependent on $n^{\alpha\gamma}$. The second is that it is based on Hermite functions, which have the issue of becoming unstable when implemented too far from the origin, due to the quickly decreasing $e^{-\vec{x}^2/2}$ term. The third is that the class of approximants used to generate the approximation in \eqref{eq:hermiteapprox} depends explicitly on the choice $\vec{x}$ at which one is approximating. In this chapter, we look to alleviate these drawbacks.

The upshot is that this method does not require optimization since it is built from mathematical theory with guaranteed good approximation properties, and the localization of the kernels allows for the approximation to be successful despite the usual pitfalls of approximation theory in machine learning applications. Furthermore, it succeeds in doing the above without utilizing the classical approach of approximating the manifold via an eigen-decomposition or atlas estimate as discussed in Section~\ref{sec:manifoldlearning}. Instead, only the dimension of the manifold is assumed and in practice can be considered a hyperparameter.
 
In the present work we project the $q$-dimensional manifold $\XX$ in question from the ambient space $\RR^Q$ to a sphere $\SS^Q$ of the same dimension.
We can then use a specially designed, localized, univariate kernel $\Phi_{n,q}$  (cf. \eqref{eq:kernel}) which is a spherical polynomial of degree $< n$ on $\SS^Q$, with $n$ and $q$ being tunable hyperparameters. 
Our construction is very simple; we define
\begin{equation}\label{eq:proto_const}
    F_{n}(\mathcal{D};x)\coloneqq \frac{1}{M}\sum_{j=1}^M z_j\Phi_{n,q}(x\cdot y_j).
\end{equation}
We note that $F_n(\mathcal{D};\circ)$ is a function defined on the ambient sphere $\SS^Q$.
The localization of the kernel allows us to adapt the approximation to the unknown manifold.

Our main theorem (cf. Theorem~\ref{theo:manifoldapproxmainthm}) has the following form:

\begin{theorem}\label{theo:proto_mainthm} (\textbf{Informal statement})
Let $\mathcal{D}=\{(y_j,z_j)\}_{j=1}^M$ be a set of random samples chosen from a distribution $\tau$. Suppose $f$ belongs to a smoothness class $W_\gamma$ (detailed in Definition~\ref{def:manifold_smoothness}) with associated norm $\norm{\circ}_{W_\gamma}$. Then under some additional conditions and with a judicious choice of $n$, we have with high probability:
\begin{equation}\label{eq:proto_est}
    \norm{F_n(\mathcal{D};\circ)-f}_\mathbb{X}\leq c\left(\norm{z}+\norm{f}_{W_\gamma}\right)\left(\frac{\log M}{M}\right)^{\gamma/(q+2\gamma)},
\end{equation}
where $c$ is a positive constant independent of $f$.
\end{theorem}

We note some mathematical features of our construction and theorem which we find interesting.
\begin{enumerate}
\item The usual approach in machine learning is to construct the approximation using an optimization procedure, usually involving a regularization term.
The setting up of this optimization problem, especially the regularization term, requires one to assume that the function belongs to some special function class, such as a reproducing kernel Hilbert/Banach space.
Thus, the constructions are not explicit nor universal.
In contrast, our construction \eqref{eq:proto_const} does not require  a prior on the function in order to use our model.
Of course, the theorem and its high-probability convergence rates do require various assumptions on $\tau$, the marginal distribution, the dimension of the manifold, the smoothness of the target function, etc. 
The point is that the construction itself does not require any assumptions.
\item 
A major problem in manifold learning is one of out of sample extension; i.e., extending the approximation to outside the manifold. 
A usual procedure for this in the context of approximation using the eigenstructure of the Laplace-Beltrami operator on the manifold is the Nystr\"om extension \cite{coifman2006geometric}.
However, this extension is no longer in terms of any orthogonal system on the ambient space, and hence there is no guarantee of the quality of approximation even if the function is known outside the manifold.
In contrast, the point $x$ in \eqref{eq:proto_const} is not restricted to the manifold, but rather freely chosen from $\mathbb{S}^Q$. 
That is, our construction defines an out of sample extension in terms of spherical polynomials on the ambient sphere, whose approximation capabilities are well studied.
\item In terms of $M$, the estimate in \eqref{eq:proto_est} depends asymptotically on the dimension $q$ of the manifold rather than the dimension $Q$ of the ambient space.
\item We do not need to know \textbf{anything} about the manifold (e.g., eigendecomposition or atlas estimate) itself apart from its dimension in order to prove our theorem.
There are several papers in the literature for estimating the dimension from the data, for example \cite{liao2016learning, liao2016adaptive, manoni2020effective}. 
However, the simplicity of our construction allows us to treat the dimension $q$ as a tunable parameter to be determined by the usual division of the data into training, validation, and test data.
\end{enumerate}

There are several other approaches superficially similar to our constructions. 
We will comment on some of these in Section~\ref{sec:relatedmanifold}.
We describe the main idea behind our proofs in Section~\ref{sec:overview}.
The results require an understanding of the approximation properties of spherical polynomials. Accordingly, we describe some background on the spherical polynomials, our localized kernels, and their use in approximation theory on subspheres of the ambient sphere
in Section~\ref{sec:approxthy}. 
The main theorems for approximation on the unknown manifold are given in Section~\ref{sec:manifoldapprox}. 
The theorems are illustrated with three numerical examples in Section~\ref{sec:numericalmanifold}. 
One of these examples is closely related to an important problem in magnetic resonance relaxometry, in which one seeks to find the proportion of water molecules in the myelin covering in the brain based on a model that involves inversion of the Laplace transform.
The proofs of the main theorems are given in Section~\ref{sec:manifoldapproxproofs}.

We would like to thank Dr. Richard Spencer at the National Institute of Aging (NIH) for his helpful comments, especially on Section~\ref{subsec:spencerdata}, verifying that our simulation is consistent with what is used in the discipline of magnetic resonance relaxometry.

\section{Related Ideas}\label{sec:relatedmanifold}

Since our method is based on a highly localized kernel, it is expected to be comparable to the simple nearest neighbor algorithm.
However, rather than specifying the number of neighbors to consider in advance, our method allows the selection of neighbors adaptively  for each test point, controlled by the parameter $n$.
Also, rather than taking a simple averaging, our method is more sophisticated, designed to give an optimal order of magnitude of the approximation error.

One of the oldest ideas for data based function approximation is the so-called Nadaraya-Watson estimator (NWE), given by
$$
NW_h(x)=\frac{\sum_{j=1}^M z_jK(|x-y_j|/h)}{\sum_{j=1}^M K(|x-y_j|/h)},
$$
where $K$ is a kernel with an effectively small support---the Gaussian kernel $K(t)=\exp(-t^2)$, as a common example---and $h$ is a scaling parameter.
Another possible choice is a $B$-spline (including Bernstein polynomials) which has a compact support.
This construction is designed to work on a Euclidean space by effectively shrinking the support of $K$ using the scaling parameter $h\to 0$, analogously to spline approximation.
The degree of approximation of such methods is measured in terms of $h$.
It is well known (e.g., \cite{de2006approximation}) that  the use of a positive kernel $K$  suffers from the so-called saturation phenomenon: the degree of approximation cannot be smaller than $\mathcal{O}(h^2)$ unless the function is a trivial one in some sense.

Radial basis function (RBF) networks and neural networks are used widely for function approximation, using either interpolation or least square fit. 
Standard RBF networks, such as Gaussian networks or thin plate spline networks, use a fixed, scaled kernel.
Typically, the matrices involved in either interpolation or least square approximation are very ill-conditioned, and the approximation is not highly localized. 

Restricted to the sphere, both of the notions are represented by a zonal function (ZF) network. 
A \emph{zonal function} on a sphere is a function of the form $x\mapsto g(x\cdot x_0)$. 
A ZF network is a linear combination of finitely many zonal functions.
One may notice that
$$
g(x\cdot x_0)=g\left(1-\frac{|x-x_0|^2}{2}\right),
$$
so we can see that a ZF network is also a neural/RBF network. 
Conversely, a neural/RBF network restricted to the sphere is a ZF network. 
The same observations about RBF networks hold for ZF networks as well.
We note that all the papers we are aware of which deal with approximation by ZF networks actually end up approximating a spherical polynomial by the networks in question.

Rather than working with a fixed, scaled kernel, in this chapter we deal with a sequence of highly localized polynomial kernels.
We do not need to solve any system of equations or do any optimization to arrive at our construction.
RBF networks and NWE were developed for approximation on Euclidean domains instead of unknown manifolds. Both have a single hyperparameter $h$ and work analogously to the spline approximation.
In contrast, our method is designed for approximation on unknown manifolds without having to learn anything about the manifold besides the dimension. It has two integer hyperparameters ($n$ and $q$) and yields a polynomial approximation.

If one chooses $h$ small enough relative to a fixed $n$ then NWE may be able to outperform our method as measured in terms of a global error bound, such as the root mean square (RMS) error. If one instead chooses $n$ large enough relative to a fixed $h$ then our method may be able to outperform NWE. So in order to give a fair comparison in Example~\ref{ex:nswvsphin}, we force the RMS error of both methods to be approximately equivalent and investigate the qualitative differences of the errors produced by each method. We additionally show that both methods in the example outperform an interpolatory RBF network.

\begin{example}\label{ex:nswvsphin}
This example serves to illustrate two points.
The first point is to compare the performance our method with NWE and an RBF interpolant. In doing so, we show that the error associated with our method is localized to singularities of the target function, whereas the other methods do not exhibit this behavior.
The second point is that using a global error estimate such as RMS can be misleading. 
Even if the RMS error with a given method might be large, the percentage of test data points at which it is smaller than a threshold could be substantially higher due to the local effects in the target function.

To ensure fair comparison, we use each of the three methods for approximation on $\SS^1=\{(\cos\theta,\sin\theta): \theta\in (-\pi,\pi]\}$, where the Gaussian kernel can be expressed in the form of a zonal function as explained above.

We consider the function
\be\label{eq:toyexample}
f(\theta)=1+\abs{\cos\theta}^{7/2}\sin(\cos\theta+\sin\theta)/2, \qquad \theta\in (-\pi,\pi].
\ee
We note that the function is analytic except at $\theta=\pm\pi/2$, where it has a discontinuity in the 4th order derivative. 
Our training data consists of $2^{13}$ equidistantly spaced $y_j$'s along the circle. We set $z_j=f(y_j)$, and examine the resulting error on a test data consisting of $2^{11}$ points chosen randomly according to the uniform distribution on $\SS^1$.  

We consider three approximation processes : (1)  Nadaraya-Watson estimator NW$_h$ with\\
 $K_h(t)=\exp(-t^2/h^2)$, (2) interpolatory approximation by the RBF network of the form $\sum a_k\exp(-|\circ-y_j|^2/h^2)$, (3) our method with the kernel $\Phi_{50,1}$.

We experimentally determined the optimal $h$ value in NWE to be $\approx 7.45e\text{-4}$ (effectively simulating the minimization of the actual generalization error on the test data), resulting in an RMS error of 1.8462e-7. 
The same value of $h$ was used for interpolation with the Gaussian RBF network, yielding a RMS error of 2.2290e-4. 
We then chose $n$ so as to yield a (comparable to NWE) RMS error of 1.8594e-7 (though we note that in this case our method continues to provide a better approximation if $n$ is further increased).

The detailed results are summarized in Figure~\ref{fig:nwecomp} below.
\begin{figure}[!ht]
\centering
\begin{tabular}{cc}
    \includegraphics[width=.45\textwidth]{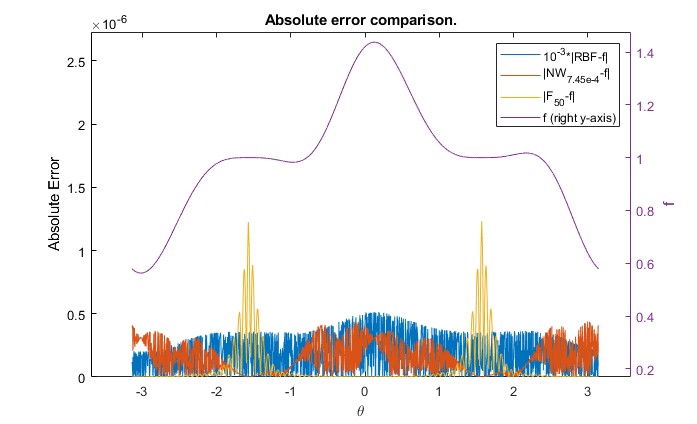}
    &
    \includegraphics[width=.45\textwidth]{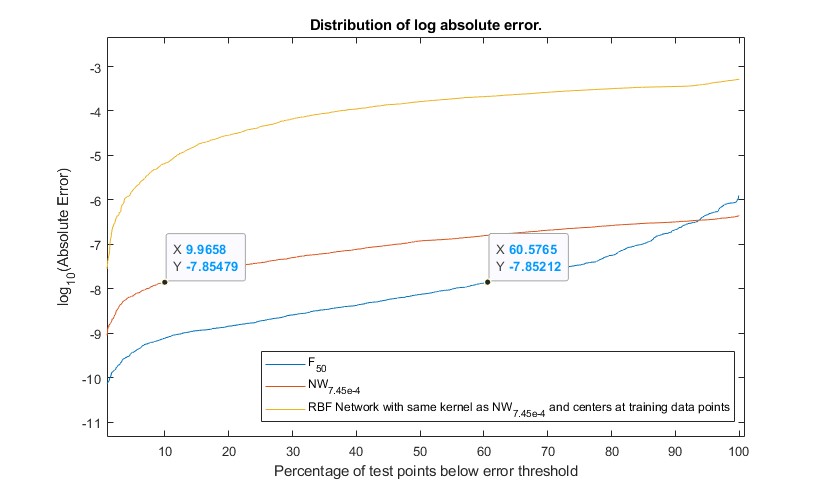}
\end{tabular}
    \caption{Error comparison between our method, the Nadaraya-Watson estimator, and an interpolatory RBF network. (Left) Comparison of absolute errors between the methods with the target function plotted on the right $y$-axis for benefit of the viewer. We note that the error from the RBF method is scaled by $10^{-3}$ so as to not dominate the figure. (Right) Percent point plot of the log absolute error for all three methods.}
        \label{fig:nwecomp}

\end{figure}

In the left plot in Figure~\ref{fig:nwecomp},  we can see a clear difference between the errors of the three methods. The (scaled by $10^{-3}$) error from the RBF network jumps throughout the whole domain, signaling the ill-conditioned nature of the matrix.
The error from the Nadaraya-Watson estimator exhibits some oscillation across the whole domain as well. The error with our method is localized to the two singularity points of the function. 
In other words, our method exhibits 1) sensitivity to the singularities of a function and 2) error adapting to the local smoothness of the function.
In comparison, RBF networks and NWE do not always exhibit such behavior. 
On the right plot of Figure~\ref{fig:nwecomp}, we give a percent point plot of the log absolute error for all three methods. There are three curves corresponding to the three methods being compared. Each point $(x,y)$ along a given curve indicates that the corresponding method approximated $x\%$ of test points with absolute error below $10^{y}$. This plot can also be thought of as the inverse CDF for the random variable of the resulting log absolute error for a test point sampled uniformly at random.
For example, whereas the Nadaraya-Watson estimator yields an error below $\approx 10^{-7.85}$ for only about $10\%$ of the tested points, our method exhibits the same error or below for about $60\%$ of the test points. Our method has the higher uniform error, but lower error for over $90\%$ of the test points.
Although the overall RMS error is roughly the same, our method exhibits a quicker decay from the uniform error.
This illustrates, in particular, that measuring the performance using a global measure for the error, such as the uniform or RMS error can be misleading.
The interpolatory RBF network performs the worst of the three methods as the right plot of Figure~\ref{fig:nwecomp} shows clearly. \qed
\end{example}

There are some efforts \cite{fuselier2012scattered,lehmann2019ambient} to do function approximation on manifolds using tensor product splines or RBF networks defined on an ambient space by first extending the target function to the ambient space.
A locally adaptive polynomial approach is used in \cite{soberleastsquares} for accomplishing function approximation on manifolds using the data. 
All these works require that the manifold be known.

In \cite{mhaskar2020deep}, we have suggested a direct approach to function approximation on an unknown submanifold of a Euclidean space using a localized kernel based on Hermite polynomials.
This construction was used successfully in predicting diabetic sugar episodes \cite{gaussian_diabetes} and recognition of hand gestures \cite{mason2021manifold}.
In  particular, in \cite{gaussian_diabetes}, we constructed our approximation based on one clinical data set and used it to predict the episodes based on another clinical data set.
In order to extend the applicability of such results to wearable devices, it is important that the approximation should be encoded by a hopefully small number of real numbers, which can then be hardwired or used for a simpler approximation process \cite{valeriyasmartphone}.
However, the construction in \cite{mhaskar2020deep} is a linear combination of kernels of the form $\Psi(|\circ-y_j|)$, where $\Psi(t)=P(t)\exp(-t^2/2)$ is a univariate kernel utilizing a judiciously chosen polynomial $P$. 
This means that we get a good approximation, but the space from which the approximation takes place changes with the point at which the approximation is desired.
This does not allow us to encode the approximation using finitely many real numbers.
In contrast, the method proposed in this chapter allows us to encode the approximation using  coefficients of the target function in the spherical harmonic expansion (defined in a distributional sense), computed empirically.
Moreover, the degree of the polynomials involved in  \cite{mhaskar2020deep} to obtain same the rate of convergence in terms of the number of samples is $\mathcal{O}(n^2)$, while the degree of the polynomials involved in this chapter is $\mathcal{O}(n)$. 
We note that the construction in both involve only univariate polynomials, so that the dimension of the input space enters only linearly in the complexity of the construction.

\section{An Overview of the Proof}\label{sec:overview}

\begin{figure}[!ht]
\centering
    \includegraphics[width=0.4\textwidth]{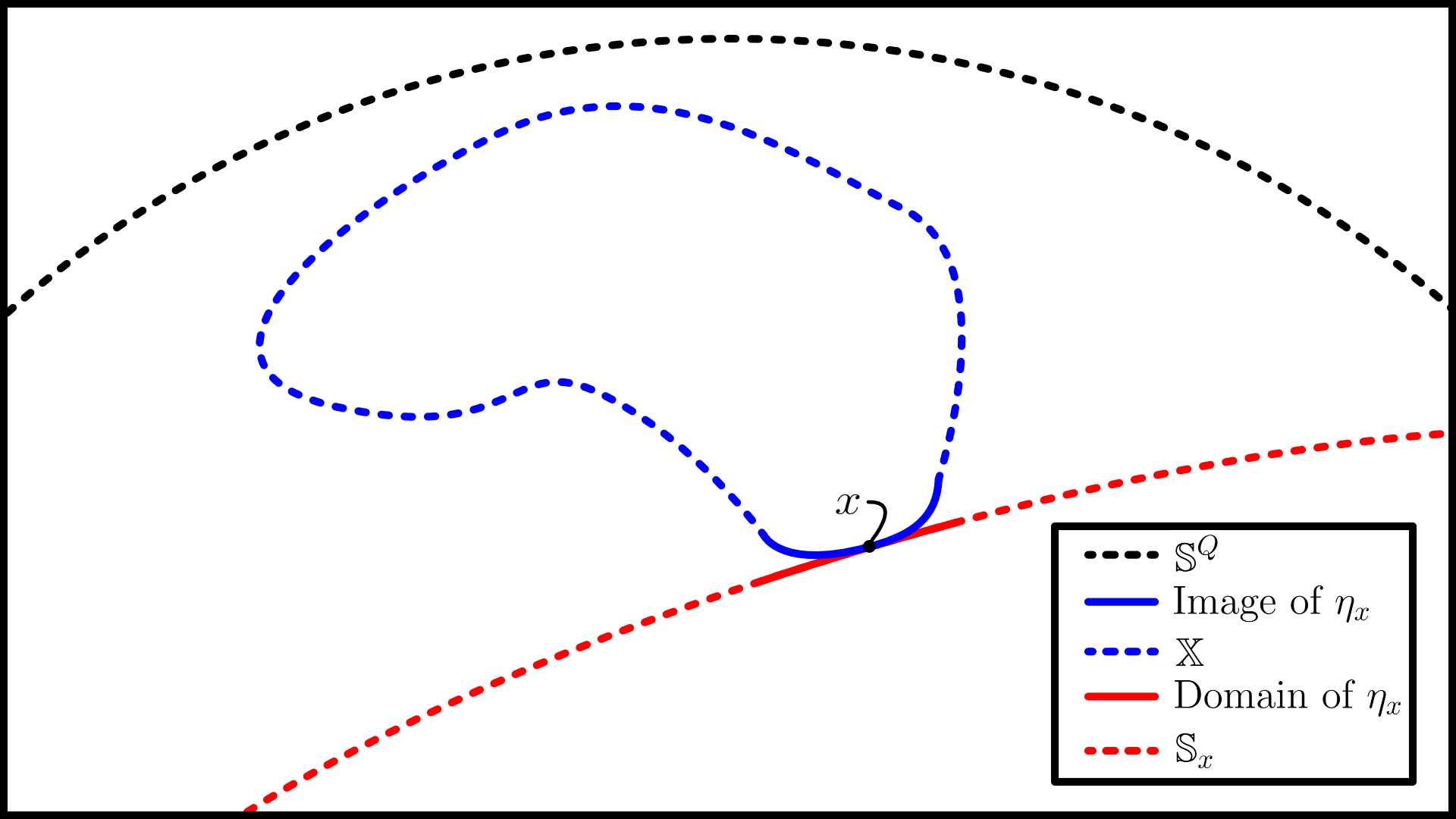}
    \caption{Visualization of our approximation approach. Here, $\mathbb{X}$ is a submanifold of the sphere $\mathbb{S}^Q$. The map $\eta_x$, analogous to the exponential map, allows us to relate the part of the integral in \eqref{eq:manifold_summabilityop} near $x$ with an integral on the tangent sphere at $x$ via a change of variables (solid curves). The localization of the kernels in our method allow for the approximation to be extended over $\XX$ and the tangent sphere $\mathbb{S}_x$ (dotted curves).}
        \label{fig:method}

\end{figure}

We can think of $F_n(\mathcal{D};x)$ defined in \eqref{eq:approximation} as an empirical approximation for an expected value with respect to the data distribution $\tau$:
\be\label{eq:approx_vs_expect}
\mathbb{E}_\tau(F_n(\mathcal{D};x))=\int z\Phi_{n,q}(x\cdot y)d\tau(y). 
\ee
Assuming that the marginal distribution of $\tau$ on $\XX$ is absolutely continuous with respect to the Riemannian volume measure $\mu^*$ on  $\XX$; i.e., given by $f_0d\mu^*$ for some smooth function $f_0$, we have
\be\label{eq:expect_vs_volumemeasure}
\mathbb{E}_\tau(F_n(\mathcal{D};x))=\int_\XX f(y)f_0(y)\Phi_{n,q}(x\cdot y)d\mu^*(y).
\ee
Accordingly, we define an \emph{integral reconstruction operator} by
\be\label{eq:manifold_summabilityop}
\sigma_n(\mathbb{X},f)(x)\coloneqq \int_{\mathbb{X}}\Phi_{n,q}(x\cdot y)f(y)d\mu^*(y), \qquad f\in C(\mathbb{X}), \ x\in\XX,
\ee
study the approximation properties of this operator, and use it with $ff_0$ in place of $f$.
The approximation properties of the operator $\sigma_n$ in the case of when $\XX$ is the $q$-dimensional sphere $\SS^q$ are well known (Proposition~\ref{prop:sphapprox}), and can be easily transferred to a $q$-dimensional equator of the ambient sphere $\SS^Q$ (Section~\ref{subsec:eqapprox}, Theorem~\ref{theo:subthm}).
We introduce a local exponential map $\eta_x$ at $x\in\XX$ between $\XX$ and the tangent equatorial sphere $\SS_x$ (i.e., a rotated version of $\SS^q$ that shares the tangent space with $\XX$ at $x$). We give an illustration of this setup in Figure~\ref{fig:method}.
Locally, a change of variable formula and the properties of this map allow us to compare the integral over a small manifold ball with that of its image on $\SS_x$ (cf. \eqref{eq:lemma_7_1_main_est}).
We keep track of the errors using the Bernstein inequality for spherical polynomials (cf. \eqref{eq:sph_bernstein}) and standard approximations between geodesic distances and volume elements on the manifold by those on $\mathbb{S}_x$.
This constitutes the main part of the proof of the critical Lemma~\ref{lemma:manifold_keylemma}.
We use the high localization property of our kernel $\Phi_{n,q}$ to lift the rest of the integral in \eqref{eq:manifold_summabilityop} on $\mathbb{X}$ at any point $x\in\XX$ to the rest of $\SS_x$ with small error (cf. \eqref{eq:lemma_7_1_awaymanifold},~\eqref{eq:lemma_7_1_awaysphere}).   
After this, we can use known results from the theory of approximation on the sphere by spherical polynomials (cf. Proposition~\ref{prop:sphapprox} and Theorem~\ref{theo:subthm}).
A partition of unity argument is used often in the proof.

Having obtained the approximation result for the integral reconstruction operator, we then discretize the integral and keep track of the errors using concentration inequalities.

\section{Background}
\label{sec:manifoldbackground}

In this section, we outline some important details about spherical harmonics (Section~\ref{subsec:sphharm}) which leads to the construction of the kernels of interest in this chapter (Section~\ref{subsec:LocalizedKernels}). We then review some classical approximation results using these kernels on spheres (Section~\ref{subsec:sphapprox}) and equators of spheres (Section~\ref{subsec:eqapprox}).

\subsection{Spherical Harmonics}
\label{subsec:sphharm}
The material in this section is based on \cite{mullerbk, steinweissbk}.
Let $0\le q\le Q$ be integers. We define the \textit{$q$-dimensional sphere} embedded in $Q+1$-dimensional space as follows
\begin{equation}\label{eq:sphere}
    \mathbb{S}^q\coloneqq\{(x_1,\dots,x_{q+1},\underbrace{0,\dots,0}_{Q-q}): x_1^2+\cdots+x_{q+1}^2=1\}.
\end{equation}
Observe that $\mathbb{S}^q$ is a $q$-dimensional compact manifold with geodesic defined by $\rho(x,y)=\arccos(x\cdot y)$. 

 Let $\mu_q^*$ denote the normalized volume measure on $\mathbb{S}^q$. 
 By representing a point $x\in \mathbb{S}^q$ as $(x' \sin\theta,\cos\theta)$ for some $x'\in \mathbb{S}^{q-1}$, one has the recursive formula for measures
\be\label{eq:recursivemeasure}
\frac{\omega_q}{\omega_{q-1}} d\mu^*_q(x)=\sin^{q-1}(\theta)d\theta d\mu^*_{q-1}(x'),
\ee
where $\omega_q$ denotes the surface volume of $\mathbb{S}^q$. One can write $\omega_q$ recursively by
\be\label{eq:volumerecurs}
\omega_{q}= \frac{2\pi^{(q+1)/2}}{\Gamma((q+1)/2)}=
\begin{cases}
2\pi, &\mbox{ if $q=1$},\\[1ex]
\displaystyle\sqrt{\pi}\frac{\Gamma(q/2)}{\Gamma(q/2+1/2)}\omega_{q-1}, &\mbox{ if $q\ge 2$,}
\end{cases}
\ee
where $\Gamma$ denotes the Gamma function. 

The restriction of a homogenous harmonic polynomial in $q+1$ variables to the $q$-dimensional unit sphere $\mathbb{S}^q$
is called a spherical harmonic. 
The space of all spherical harmonics of degree $\ell$ in $q+1$ variables will be denoted  by $\mathbb{H}_\ell^q$. 
The  space of the restriction of  all $q+1$ variable polynomials of degree $< n$ to $\mathbb{S}^q$ will be denoted by $\Pi_n^q$. 
We extend this notation for an arbitrary real value $x>0$ by writing  $\Pi_x^q\coloneqq \Pi_{\floor{x}}^q$.
It is known that $\mathbb{H}_\ell^q$ is orthogonal to $\mathbb{H}_j^q$ in $L^2(\mu_q^*)$ whenever $j\neq \ell$, and $\disp\Pi_n^q=\bigoplus_{\ell=0}^{n-1   }\mathbb{H}_\ell^q$.
In particular, $\disp L^2(\mu^*_q)=\bigoplus_{\ell=0}^{\infty}\mathbb{H}_\ell^q$.

 If we let $\{Y_{\ell,k}\}_{k=1}^{\dim (\mathbb{H}_\ell^q)}$ be an orthonormal basis for $\mathbb{H}_\ell^q$ with respect to $\mu^*_q$, we can define
\begin{equation}\label{eq:reproduce}
    K_{q,\ell}(x,y)\coloneqq \sum_{k=1}^{\dim (\mathbb{H}^q_\ell)} Y_{\ell,k}(x)Y_{\ell,k}(y).
\end{equation}
In \cite[Theorem 2]{mullerbk} and \cite[Theorem 2.14]{steinweissbk}, it is shown that 
\be\label{eq:summation}
K_{q,\ell}(x,y)=\frac{\omega_{q}}{\omega_{q-1}}p_{q,\ell}(1)p_{q,\ell}(x\cdot y),
\ee
where $p_{q,\ell}$ denotes the orthonormalized ultraspherical polynomial of dimension $q$ and degree $\ell$. These ultraspherical polynomials satisfy the following orthogonality condition.
\begin{equation}\label{j-orth}
    \int_{-1}^1 (1-x^2)^{(q/2-1)} p_{q,m}(x)p_{q,n}(x)dx=\delta_{m,n}.
\end{equation}
Computationally, it is customary to use the following recurrence relation:
\begin{equation}\label{eq:recurrence}
\begin{aligned}
\sqrt{\frac{(n+1)(n+q-1)}{(2n+q-1)(2n+q+1)}}&p_{q,n+1}(x)=xp_{q,n}(x)-\sqrt{\frac{n(n+q-2)}{(2n+q-1)(2n+q-3)}}p_{q,n-1}(x), \qquad n\ge 1,\\
& p_{q,0}(x)=p_{q,0}=2^{1/2-q/2}\frac{\Gamma(q-1)}{\Gamma(q/2)},\qquad p_{q,1}(x)=2^{1/2-q/2}\frac{\sqrt{\Gamma(q)\Gamma(q+1)}}{\Gamma(q/2)}x.
\end{aligned}
\end{equation}
We note further that
\begin{equation}\label{eq:pnat1}
    p_{q,n}(1)=\frac{2^{1/2-q/2}}{\Gamma(q/2)}\sqrt{\frac{\Gamma(n+q-1)(2n+q-1)}{\Gamma(n+1)}}.
\end{equation}

\begin{rem}{\rm
Many notations have been used for ultraspherical polynomials in the past. 
For example, \cite{szegopoly} uses the notation of $P_n^{(\lambda)}$ for the Gegenbauer polynomials, also commonly denoted by $C_n^{(\lambda)}$. 
It is also usual to use a normalization, which we will denote by $R_n^q$ in this remark, given by $R_n^q=p_{q,n}/p_{q,n}(1)$.
Ultraspherical polynomials are also simply a special case of the Jacobi polynomials $P_n^{(\alpha,\beta)}$ where $\alpha=\beta$. 
Setting
\be\label{eq:normalization_coeff}
h_{q,n}\coloneqq 2^{q-1}\frac{\Gamma(n+q/2)^2}{n!\Gamma(n+q-1)(2n+q-1)},
\ee
 we have the following connection between these notations:
\be\label{eq:notation_relation}
\ba
p_{q,n}(x)&=h_{q,n}^{-1/2}P_{n}^{(q/2-1,q/2-1)}(x)=\frac{\Gamma(q-1)}{\Gamma(q/2)}\sqrt{\frac{n!(2n+q-1)}{2^{q-1}\Gamma(n+q-1)}}C_n^{(q/2-1/2)}(x)\\
&=\frac{2^{1/2-q/2}}{\Gamma(q/2)}\sqrt{\frac{\Gamma(n+q-1)\Gamma(2n+q-1)}{\Gamma(n+1)}}R_n^q.
\ea
\ee
\qed}
\end{rem}

Furthermore, the ultraspherical polynomials for the sphere of dimension $d_1$ can be represented by those for the sphere of dimension $d_2$ in the following manner
\be\label{eq:dimchange}
p_{d_1,n}=\sum_{\ell=0}^n C_{d_2,d_1}(\ell,n)p_{d_2,\ell}.
\ee
The coefficients $C$ have been studied, and explicit formulas are given in \cite[Equation 7.34]{askeypoly} and \cite[Equation 4.10.27]{szegopoly}.

\subsection{Localized Kernels}
\label{subsec:LocalizedKernels}

Let $h$ be an infinitely differentiable function supported on $[0,1]$ where $h(x)=1$ on $[0,1/2]$. 
This function will be fixed in the rest of this chapter, and its mention will be omitted from the notation.
Then we define the following univariate kernel for $t\in [-1,1]$:
\begin{equation}\label{eq:kernel}
\Phi_{n,q}(t)\coloneqq \Phi_{n,q}(h;t)=\sum_{\ell=0}^n h\left(\frac{\ell}{n}\right)K_{\ell,q}(t)=\frac{\omega_q}{\omega_{q-1}}\sum_{\ell=0}^{n}h\left(\frac{\ell}{n}\right)p_{q,\ell}(1)p_{q,\ell}(t).
\end{equation}
The following proposition lists some technical properties of these kernels which we will often use, sometimes without an explicit mention.
\begin{proposition}\label{prop:kernelprop}
Let $x,y\in \mathbb{S}^Q$. For any $S>0$, the kernel $\Phi_{n,q}(x,y)$ satisfies the \textbf{localization bound}
\begin{equation}
\label{eq:sphkernloc}
    |\Phi_{n,q}(x\cdot y)|\lesssim \frac{n^q}{\max(1,n\arccos(x\cdot y))^S},
\end{equation}
where the constant involved may depend upon $S$.
Further, we have the Lipschitz condition:
\begin{equation}
\label{eq:sph_bernstein}
|\Phi_{n,q}(x\cdot y)-\Phi_{n,q}(x\cdot y')|\lesssim n^{q+1}|\arccos(x\cdot y)-\arccos(x\cdot y')|,\qquad y'\in \mathbb{S}^Q.
\end{equation}
%and
%
\end{proposition}

\begin{proof}
The estimate \eqref{eq:sphkernloc} is proved in \cite[Lemma 4.9]{mhaskarpoly}.
Since $\theta\mapsto \Phi_{n,q}(\cos\theta)$ is a trigonometric polynomial of degree $< n$, the Bernstein inequality for the derivatives of trigonometric polynomials implies that
$$
|\Phi_{n,q}(\cos\theta)-\Phi_{n,q}(\cos\phi)|\le n\|\Phi_{n,q}\|_\infty |\theta-\phi|\ls n^{q+1}|\theta-\phi|.
$$
This leads easily to \eqref{eq:sph_bernstein}.
%
% For Equation~\eqref{K3}, 
%let $P\in \Pi_{n/2}^q$, write $P(x)=\sum_{\ell=0}^{n/2}\sum_{k}^{\dim (\mathbb{H}^q_{\ell})} a_{\ell,k} Y_{\ell,k}(x)$, and see that $\sigma_n(P)$ takes the same form.
%\begin{align}
%\sigma_n(f)=&\sum_{\ell=0}^{n/2}h\left(\frac{\ell}{n}\right)\sum_{k}^{\dim (\mathbb{H}_\ell^q)}\sum_{i=0}^n \sum_{j=0}^{\dim (\mathbb{H}_{\ell}^q)} Y_{i,j}(x)a_{\ell,k}\left[\int_{\mathbb{S}^q} Y_{i,j}(y)Y_{\ell,k}(y)d\mu_q^*(y)\right]\\
%    =&\sum_{\ell=0}^{n/2}\sum_{k}^{\dim (\mathbb{H}_\ell^q)}\sum_{i=0}^n \sum_{j=0}^{\dim (\mathbb{H}_{\ell}^q)} Y_{i,j}(x)a_{\ell,k}\delta_{i,\ell}\delta_{j,k}\\
%    =&\sum_{\ell=0}^{n/2}\sum_{k}^{\dim (\mathbb{H}_\ell^q)} a_{\ell,k}Y_{\ell,k}(x)\\
%    =&P(x),
%\end{align}
%which shows (\ref{K3}).
\end{proof}

\subsection{Approximation on Spheres}
\label{subsec:sphapprox}

Methods of approximating functions on $\mathbb{S}^q$ have been studied in, for example, \cite{mhaskarsphere,rustamovsphere} and some details are summarized in Proposition~\ref{prop:sphapprox}. 

For a compact set $A$, let $C(A)$ denote the space of continuous functions on $A$, equipped with the supremum norm $\norm{f}_{A}=\max_{x\in A}|f(x)|$. 
We define the degree of approximation for a function $f\in C(\mathbb{S}^q)$ to be
\begin{equation}\label{eq:degapproxmanifold}
E_n(f)\coloneqq\inf_{P\in \Pi_n^q} \norm{f-P}_{\mathbb{S}^q}.
\end{equation}
Let $W_\gamma(\mathbb{S}^q)$ be the class of all $f\in C(\mathbb{S}^q)$ such that
\begin{equation}\label{eq:sphsoboldef}
||f||_{W_\gamma(\mathbb{S}^q)}\coloneqq ||f||_{\mathbb{S}^q}+\sup_{n\geq 0}2^{n\gamma}E_{2^n}(f)<\infty.
\end{equation}
We note that an alternative smoothness characterized in terms of constructive properties of $f$ is explored by many authors; some examples are given in  \cite{daiapprox}. 
We define the approximation operator for $\mathbb{S}^q$ by
\be\label{eq:sphapproxop}
\sigma_n(f)(x)\coloneqq\sigma_n(\mathbb{S}^q,f)(x)= \int_{\mathbb{S}^q}\Phi_{n,q}(x\cdot u)f(u)d\mu^*_q(u).
\ee
With this setup, we now review some bounds on how well $\sigma_n(f)$ approximates $f$.
\begin{proposition}[{\cite[Proposition 4.1]{mhaskarsphere}}]\label{prop:sphapprox} Let $n\geq 1$.\newline
{\rm (a)} For  all $P\in \Pi_{n/2}^q$, we have $\sigma_n(P)=P$.\\
{\rm (b)}
For any $f\in C(\mathbb{S}^q)$, we have
\begin{equation}\label{eq:sphgoodapprox}
    E_n(f)\leq \norm{f-\sigma_n(f)}_{\mathbb{S}^q}\lesssim E_{n/2}(f).
\end{equation}
In particular, if $\gamma>0$ then $f\in W_\gamma(\mathbb{S}^q)$ if and only if
\be
\norm{f-\sigma_n(f)}_{\mathbb{S}^q}\ls \norm{f}_{W_\gamma(\mathbb{S}^q)}n^{-\gamma}.
\ee
\end{proposition}

\begin{rem}\label{rem:sphapprox}
Part (a) is known as a \textit{reproduction} property, which shows that polynomials up to degree $<n/2$ are unchanged when passed through the  operator $\sigma_n$. Part (b) demonstrates that $\sigma_n$ yields what we term a \textit{good approximation}, where its approximation error is no more than some constant multiple of the degree of approximation. Part (c)  not only gives the approximation bounds in terms of the smoothness parameter $\gamma$, but shows also that the rate of decrease of the approximation error obtained by $\sigma_n(f)$ \textbf{determines} the smoothness $\gamma$.\qed
\end{rem}

\subsection{Approximation on Equators}
\label{subsec:eqapprox}

Let $SO(Q+1)$ denote group of all unitary $(Q+1)\times (Q+1)$ matrices with determinant equal to $1$. 
A \emph{$q$-dimensional equator} of $\mathbb{S}^Q$ is a set of the form $\mathbb{Y}=\{\mathcal{R}u:u\in\mathbb{S}^q\}$ for some $\mathcal{R}\in SO(Q+1)$.
 The goal in the remainder of this section is to give approximation results for equators. 
 
 Since there exist infinite options for $\mathcal{R}\in \operatorname{SO}(Q+1)$ to generate the set $\mathbb{Y}$, we first give a definition of degree of approximation in terms of spherical polynomials that is invariant to the choice of $\mathcal{R}$.

Fix $\mathbb{Y}$ to be a given $q$-dimensional equator of $\mathbb{S}^Q$ and let $\mathcal{R},\mathcal{S}\in \operatorname{SO}(Q+1)$ mapping $\mathbb{S}^q$ to $\mathbb{Y}$. Observe that if $P\in \Pi_n^q$, then $P(\mathcal{R}^T\mathcal{S}\circ)\in \Pi_n^q$ and vice versa. As a result, the functions $F_\mathcal{R}=f(\mathcal{R}\circ)$ and $F_\mathcal{S}=f(\mathcal{S}\circ)$ defined on $\mathbb{S}^q$ satisfy
\be
E_n(F_\mathcal{R})=E_n(F_\mathcal{S}).
\ee
Since the degree of approximation in this context is invariant to the choice of $\mathcal{R}\in \operatorname{SO}(Q+1)$, we may simply choose any such matrix that maps $\mathbb{S}^q$ to $\mathbb{Y}$, drop the subscript $\mathcal{R}$ from $F_\mathcal{R}$, and define
\be
E_n(\mathbb{Y},f)\coloneqq E_n(F).
\ee
This allows us to define the space $W_\gamma(\mathbb{Y})$ as the class of all $f\in C(\mathbb{Y})$ such that
\begin{equation}\label{eq:spheresmooth}
||f||_{W_\gamma(\mathbb{Y})}\coloneqq ||f||_{\mathbb{Y}}+\sup_{n\geq 0} 2^{n\gamma}E_{2^n}(\mathbb{Y},f)<\infty.
\end{equation}
We can also define the approximation operator on the set $\mathbb{Y}$ as
\begin{equation}\label{eq:sphapproxop2}
\sigma_n(\mathbb{Y},f)(x)\coloneqq \int_\mathbb{Y} \Phi_{n,q}(x\cdot y)f(y)d\mu^*_\mathbb{Y}(y),
\end{equation}
where $\mu^*_\mathbb{Y}(y)$ is the probability volume measure on $\mathbb{Y}$. Let $F_\mathcal{R}\in C(\mathbb{S}^q)$ satisfy $F_\mathcal{R}=f\circ \mathcal{R}$. We observe that
\begin{equation}
\begin{aligned}
\sigma_n(\mathbb{Y},f)(x)=&\int_{\mathbb{S}^q}\Phi_{n,q}(x\cdot \mathcal{R}u)f(\mathcal{R}u)d\mu_q^*(u)\\
=&\int_{\mathbb{S}^q}\Phi_{n,q}(\mathcal{R}^Tx\cdot u)F_\mathcal{R}(u)d\mu^*_q(u)\\
=&\sigma_n(\mathbb{S}^q,F_\mathcal{R})(\mathcal{R}^Tx).
\end{aligned}
\end{equation}

We now give an analogue of Proposition~\ref{prop:sphapprox} for approximation on equators.

\begin{theorem}\label{theo:subthm}
Let $f\in C(\mathbb{Y})$. \\
{\rm (a)} We have
\begin{equation}\label{eq:subthm1}
    E_n(\mathbb{Y},f)\leq\norm{\sigma_n(\mathbb{Y},f)-f}_\mathbb{Y}\lesssim E_{n/2}(\mathbb{Y},f).
\end{equation}
{\rm (b)} If $\gamma>0$, then $f\in W_\gamma(\mathbb{Y})$ if and only if
\begin{equation}\label{eq:subthm2}
\norm{\sigma_n(\mathbb{Y},f)-f}_\mathbb{Y}\lesssim n^{-\gamma}\norm{f}_{W_\gamma(\mathbb{Y})}.
\end{equation}

\end{theorem}

\begin{proof} Let $F(\circ)=f(\mathcal{R}\circ)$ for some $\mathcal{R}\in \operatorname{SO}(Q+1)$ with $\mathbb{Y}=\{\mathcal{R}u:u\in\mathbb{S}^q\}$. To see \eqref{eq:subthm1}, we check using Proposition~\ref{prop:sphapprox} that
\begin{equation}
\norm{\sigma_n(\mathbb{Y},f)-f}_\mathbb{Y}=\norm{\sigma_n(\mathbb{S}^q,F)(\mathcal{R}^T\circ)-F(\mathcal{R}^T\circ)}_\mathbb{Y}=\norm{\sigma_n(\mathbb{S}^q,F)-F}_{\mathbb{S}^q}\lesssim E_{n/2}(F)=E_{n/2}(\mathbb{Y},f).
\end{equation}
Additionally, $E_n(\mathbb{Y},f)\leq \norm{\sigma_n(\mathbb{Y},f)-f}_\mathbb{Y}$ since $\sigma_n(\mathbb{Y},f)=\sigma_n(\mathbb{S}^q,F)(\mathcal{R}^Tx)\in \Pi_n^q$.
Part (b) can be seen from part (a) and the definitions.
\end{proof}

\section{Function Approximation on Manifolds}
\label{sec:manifoldapprox}

In this section, we develop the notion of \textit{smoothness} for the target function defined on a manifold, and state our main theorem: Theorem~\ref{theo:manifoldapproxmainthm}. For a brief introduction to manifolds and some results we will be using in this chapter, see Appendix~\ref{sec:manifoldintro}.

Let $Q\geq q\geq 1$ be integers and $\mathbb{X}$ be a $q$-dimensional, compact, connected, submanifold of $\mathbb{S}^Q$ without boundary. 
Let $\rho$ denote the geodesic distance and $\mu^*$ be the normalized volume measure (that is, $\mu^*(\mathbb{X})=1$). 
For any $x\in \mathbb{X}$, observe that the tangent space $\mathbb{T}_x(\mathbb{X})$ is a $q$-dimensional vector space tangent to $\mathbb{S}^Q$. We define $\mathbb{S}_x=\mathbb{S}_x(\mathbb{X})$ to be the $q$-dimensional equator of $\mathbb{S}^Q$ passing through $x$ whose own tangent space at $x$ is also $\mathbb{T}_x(\mathbb{X})$. 
 As an important note, $\mathbb{S}_x$ is also a $q$-dimensional compact manifold.

In this chapter we will consider many spaces, and need to define balls on each of these spaces, which we list in Table~\ref{tab:balldef} below.

\begin{table}[ht]

\begin{center}
\begin{tabular}{|c|c|c|}
\hline
Space & Description & Definition\\
\hline
Ambient space & Euclidean ball & $B_{Q+1}(x,r)=\{y\in \mathbb{R}^{Q+1}:\norm{x-y}_2\leq r\}$\\
\hline Ambient sphere & Spherical cap & $S^Q(x,r)=\{y\in \mathbb{S}^Q:\arccos(x\cdot y) \le r\}$\\
\hline
Tangent space & Tangent ball & $B_{\mathbb{T}}(x,r)=\{y\in \mathbb{T}_x(\mathbb{X}): ||x-y||_{2}\le r\}$\\
\hline
Tangent sphere & Tangent cap & $\mathbb{S}_x(r)=\{y\in \mathbb{S}_x:\arccos(x\cdot y)\le r\}$ \\
\hline
Manifold & Geodesic ball & $\mathbb{B}(x,r)=\{y\in \mathbb{X}:\rho(x,y)\le  r\}$\\
\hline
\end{tabular}\end{center}
\caption{Definition and description of balls in different spaces.}
\label{tab:balldef}
\end{table}

We also need to define the smoothness classes we will be considering for functions on $\mathbb{X}$. 
Let $C(\mathbb{X})$ denote the space of all continuous functions on $\mathbb{X}$, and $C^\infty(\mathbb{X})\subset C(\mathbb{X})$ denote the space of all infinitely differentiable functions on $\mathbb{X}$. 
Let $\overline{\varepsilon}_x$ be the exponential map at $x$ for $\mathbb{S}_x$ and $\varepsilon_x$ be the exponential map at $x$ for $\mathbb{X}$.
Since both $\mathbb{X}$ and $\mathbb{S}_x$ are compact, we have some $\iota_1,\iota_2$ such that $\varepsilon_x,\overline{\varepsilon}_x$ are defined on $B_{\mathbb{T}}(x,\iota_1),B_{\mathbb{T}}(x,\iota_2)$ respectively for any $x$. 
We write $\iota^*=\min\{1, \iota_1,\iota_2\}$ and define $\eta_x :\mathbb{S}_x(\iota^*)\to \XX$ by 
 $\eta_x: \varepsilon_x \circ\overline{\varepsilon}_x^{-1}$.
 Thus,
\be\label{eq:rho_sphdist}
\rho(x,\eta_x(y))=\arccos(x\cdot y), \qquad x\in \XX, \ y\in \mathbb{S}_x(\iota^*).
\ee

\begin{definition}\label{def:manifold_smoothness}
We say that $f\in C(\mathbb{X})$ is \textbf{$\gamma$-smooth} for some $\gamma>0$, or also that $f\in W_\gamma(\mathbb{X})$, if for every $x\in \mathbb{X}$ and $\phi\in C^\infty(\mathbb{X})$ supported on $\mathbb{B}(x,\iota^*)$, the function $F_{x,\phi}:\mathbb{S}_x\to \mathbb{R}$ defined by $F_{x,\phi}\coloneqq f(\eta_x(u))\phi(\eta_x(u))$ belongs to $W_\gamma(\mathbb{S}_x)$ as outlined in Section~\ref{subsec:sphapprox} (in particular, Equation~\eqref{eq:spheresmooth}). 
We also define
\be\label{eq:manifold_sobnorm}
\norm{f}_{W_\gamma(\mathbb{X})}\coloneqq \sup_{x\in\mathbb{X},\norm{\phi}_{W_\gamma(\mathbb{S}_x)}\leq 1} \norm{F_{x,\phi}}_{W_\gamma(\mathbb{S}_x)}.
\ee
\end{definition}

Our main theorem, describing the approximation of $ff_0$ (the target function weighted by the density of data points) by the operator defined in \eqref{eq:proto_const}, is the following.
We note that approximation of $ff_0$ includes local approximation on $\XX$ in the sense that when the training data is sampled only from a subset of $\XX$, this fact can be encoded by $f_0$ being supported on this subset.

\begin{theorem}\label{theo:manifoldapproxmainthm}
We assume that
\begin{equation}\label{eq:ballmeasure}
    \sup_{x\in \mathbb{X},r>0}\frac{\mu^*(\mathbb{B}(x,r))}{r^q}\ls 1.
\end{equation}
Let $\mathcal{D}=\{(y_j,z_j)\}_{j=1}^M$ be a random sample  from a joint distribution $\tau$. 
We assume that the  marginal distribution of $\tau$ restricted to $\mathbb{X}$ is absolutely continuous with respect to $\mu^*$ with density $f_0$, and that the random variable $z$ has a bounded range. We say $z\in [-\norm{z},\norm{z}]$.
Let
\begin{equation}\label{eq:fdef}
    f(y)\coloneqq \mathbb{E}_\tau(z|y),
\end{equation}
and
\begin{equation}\label{eq:approximation}
    F_{n}(\mathcal{D};x)\coloneqq \frac{1}{M}\sum_{j=1}^M z_j\Phi_{n,q}(x\cdot y_j),\qquad x\in \mathbb{S}^Q,
\end{equation}
where $\Phi_{n,q}$ is defined in \eqref{eq:kernel}.

Let $0<\gamma<2$ and $ff_0\in W_\gamma(\mathbb{X})$.
 Then for every $n\geq 1$, $0<\delta<1/2$ and 
 \be\label{eq:Mncond}
  M\gtrsim n^{q+2\gamma}\log(n/\delta),
  \ee
   we have with $\tau$-probability (i.e., probability over the distribution $\tau$) $\geq 1-\delta$:
\begin{equation}\label{eq:approxest}
    \norm{F_n(\mathcal{D};\circ)-ff_0}_\mathbb{X}\lesssim \frac{\sqrt{\norm{f_0}_{\mathbb{X}}}\norm{z}+\norm{ff_0}_{W_\gamma(\mathbb{X})}}{n^\gamma}.
\end{equation}
Equivalently, for integer $M\ge 2$ and $n$ satisfying \eqref{eq:Mncond}, we have with $\tau$-probability $\geq 1-\delta$:
\be\label{eq:sample_approxest}
 \norm{F_n(\mathcal{D};\circ)-ff_0}_\mathbb{X}\lesssim \left\{\sqrt{\norm{f_0}_{\mathbb{X}}}\norm{z}+\norm{ff_0}_{W_\gamma(\mathbb{X})}\right\}\left(\frac{\log(M/\delta^{q+2\gamma})}{M}\right)^{\gamma/(q+2\gamma)}.
\ee
\end{theorem}

We discuss two corollaries of this theorem, which demonstrate how the theorem can be used for both estimation of the probability density $f_0$ and the approximation of the function $f$ in the case when the training data is sampled from the volume measure on $\XX$.

The first corollary is a result on function approximation in the case when the marginal distribution of $y$ is $\mu^*$; i.e., $f_0\equiv 1$.

\begin{corollary}\label{cor:maincor}
Assume the setup of Theorem~\ref{theo:manifoldapproxmainthm}. Suppose also that the marginal distribution of $\tau$ restricted to $\mathbb{X}$ is uniform. Then we have with $\tau$-probability $\geq 1-\delta$:
\begin{equation}\label{eq:fnapprox}
    \norm{F_n(\mathcal{D};\circ)-f}_\mathbb{X}\lesssim \frac{\norm{z}+\norm{f}_{W_\gamma(\mathbb{X})}}{n^\gamma}.
\end{equation}
\end{corollary}

The second corollary is obtained by setting $f\equiv 1$, to point out that our theorem gives a method for density estimation. 
In practice, one may not have knowledge of $f_0$ (or even the manifold $\mathbb{X}$). So, the following corollary can be applied to estimate this critical quantity. We use this fact in our numerical examples in Section~\ref{sec:numericalmanifold}.

Typically, a positive kernel is used for the problem of density estimation in order to ensure that the approximation is also a positive measure.
It is well known in approximation theory that this results in a saturation for the rate of convergence.
Our method does not use positive kernels, and does not suffer from such saturation.

\begin{corollary}\label{cor:densityest}
Assume the setup of Theorem~\ref{theo:manifoldapproxmainthm}. Then
we have with $\tau$-probability $\geq 1-\delta$:
\be\label{eq:densityest}
\norm{\ \ \left|\frac{1}{M}\sum_{j=1}^M\Phi_{n,q}(\circ\cdot y_j)\right|-f_0}_\mathbb{X}\lesssim \frac{\norm{f_0}_{W_\gamma(\mathbb{X})}}{n^\gamma}.
\ee
\end{corollary}

\section{Numerical Examples}
\label{sec:numericalmanifold}

In this section, we illustrate our theory with some numerical experiments.
In Section~\ref{subsec:expiecewise}, we consider the approximation of a piecewise differentiable function, and demonstrate how the localization of the kernel leads to a determination of the locations of the singularities. 
The example in Section~\ref{subsec:spencerdata} is motivated by magnetic resonance relaxomety. 
Since it is relevant to our method for practical applications, we have included some discussion and results about how $q$ effects the approximation in this example. 
The example in Section~\ref{subsec:darcyflow} illustrates how our method can be used for inverse problems in the realm of differential equations.
In all the examples, we will examine how the accuracy of the approximation depends on the maximal  degree $n$ of the polynomial, the number $M$ of samples, and the level of noise.

\subsection{Piecewise Differentiable Function}
\label{subsec:expiecewise}

In this example only we define the function to be approximated as
\be
f(\theta)\coloneqq 1+\abs{\cos\theta}^{1/2}\sin(\cos\theta+\sin\theta)/2,
\ee 
defined on the ellipse
\be
E=\{(3\cos\theta,6\sin\theta): \theta\in (-\pi,\pi]\}.
\ee
We project $E$ to the sphere $\mathbb{S}^2$ using an inverse stereographic projection defined by
\be
\mathbf{P}(\vec{x})=\frac{(\vec{x},1)}{\norm{(\vec{x},1)}_2},
\ee 
and call $\mathbb{X}=\mathbf{P}(E)$. Each $\vec{x}\in\mathbb{X}$ is associated with the value $\theta_\vec{x}$ satisfying $\vec{x}=\mathbf{P}\big((3\cos\theta_\vec{x},6\sin\theta_\vec{x})\big)$, so that $f(\vec{x})\coloneqq f(\theta_\vec{x})$ is a continuous function on $\mathbb{X}$.

We generate our data points by taking $\vec{y}_j=\mathbf{P}\big((3\cos\theta_j,6\sin\theta_j)\big)$, where $\theta_j$ are each sampled uniformly at random from $(-\pi,\pi]$. We then define $z_j=f(\vec{y}_j)+\epsilon_j$ where $\epsilon_j$ are sampled from some mean-zero normal noise. 
Our data set is thus $\mathcal{D}\coloneqq\{(\vec{y}_j,f(\vec{y}_j)+\epsilon_j)\}_{j=1}^M$.
We will measure the magnitude of noise using the signal-to-noise ratio (SNR), defined by 
\be\label{eq:snr}
20\log_{10}\Big(\norm{\big(f(\vec{y}_1),\dots,f(\vec{y}_M)\big)}_2\Big/\norm{\big(\epsilon_1,\dots,\epsilon_M\big)}_2\Big).
\ee  
Since $f_0\neq 1$ in this case, we could calculate $f_0$ from the projection, or we may estimate it using Corollary~\ref{cor:densityest}. That is,
\be
f_0(\vec{x})\approx\frac{1}{M}\sum_{j=1}^M\Phi_{n,1}(\vec{x}\cdot \vec{y}_j).
\ee
This option may be desirable in cases where $f_0$ is not feasible to compute (i.e. if the underlying domain of the data is unknown or irregularly shaped). Our approximation is thus:
\be
F_n(\mathcal{D};\vec{x})= \sum_{j=1}^M (f(\vec{y}_j)+\epsilon_j)\Phi_{n,1}(\vec{x}\cdot \vec{y}_j)\left/\left(\sum_{j=1}^M \Phi_{n,1}(\vec{x}\cdot \vec{y}_j)\right).\right.
\ee

\begin{figure}[!ht]
\centering
    \includegraphics[width=.4\textwidth]{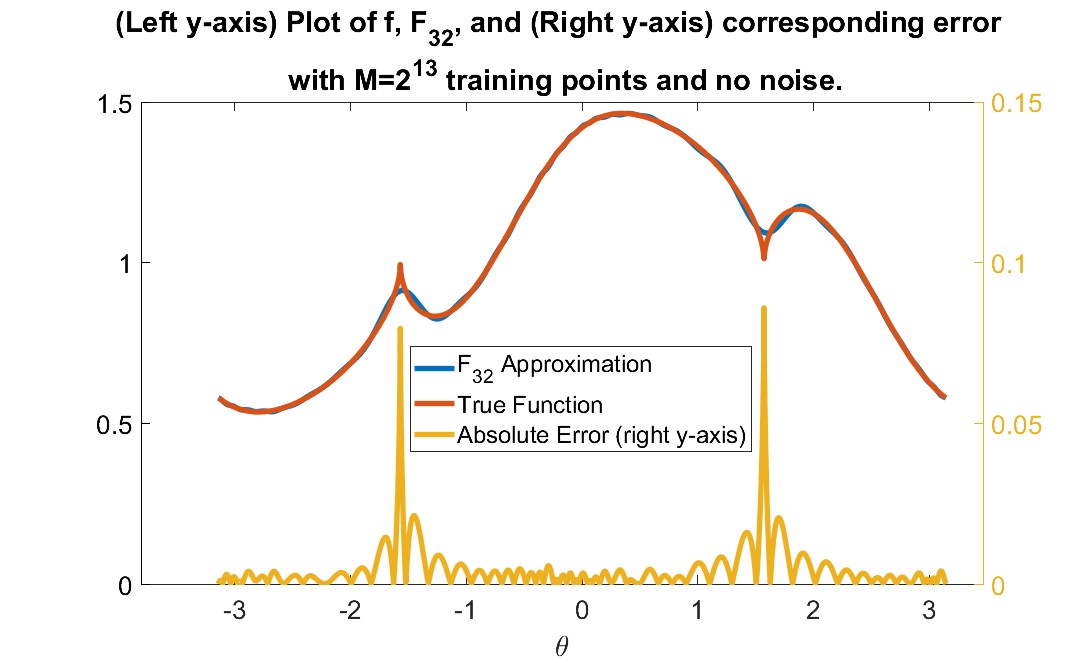}\\
    \caption{Left y-axis: Plot of the true function $f$ compared with $F_{32}$ constructed by $2^{13}$ noiseless training points. Right y-axis: Plot of $\abs{f-F_{32}}$.}
    \label{fig:errorplot}
\end{figure}

\begin{figure*}[!ht]
\centering
\begin{tabular}{ccc}
    \includegraphics[width=.3\textwidth]{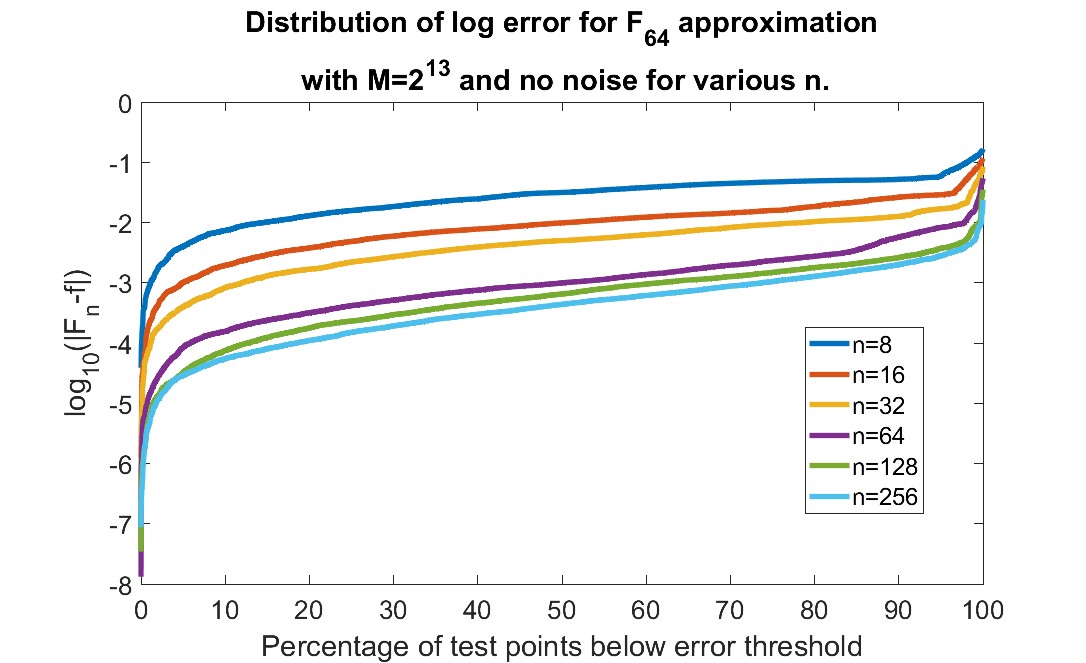}
&
    \includegraphics[width=.3\textwidth]{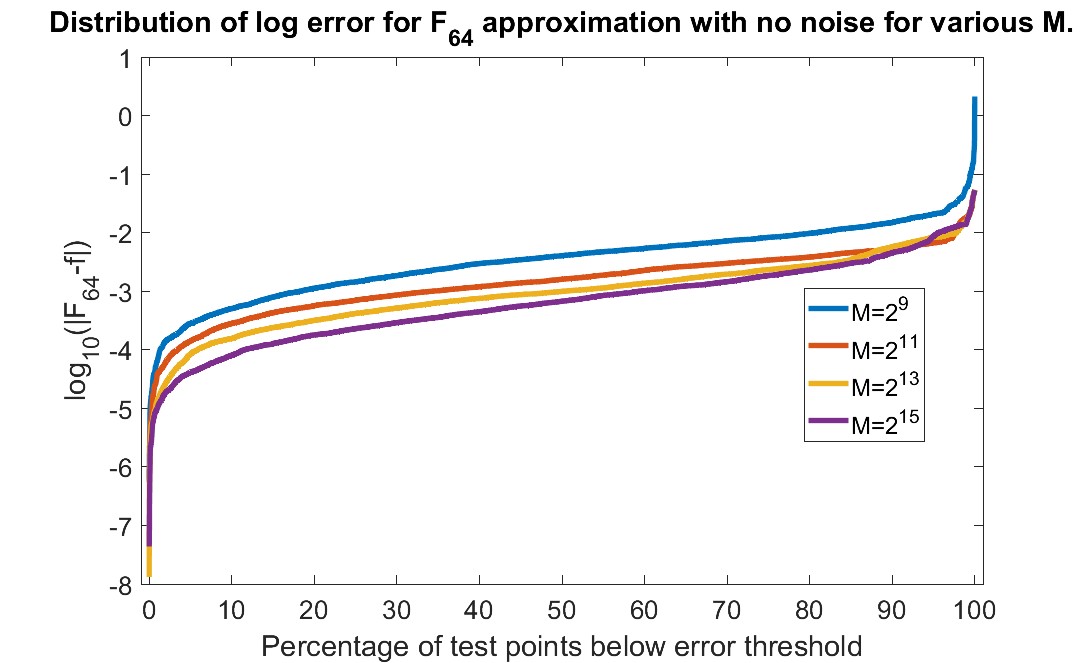}
&
    \includegraphics[width=.3\textwidth]{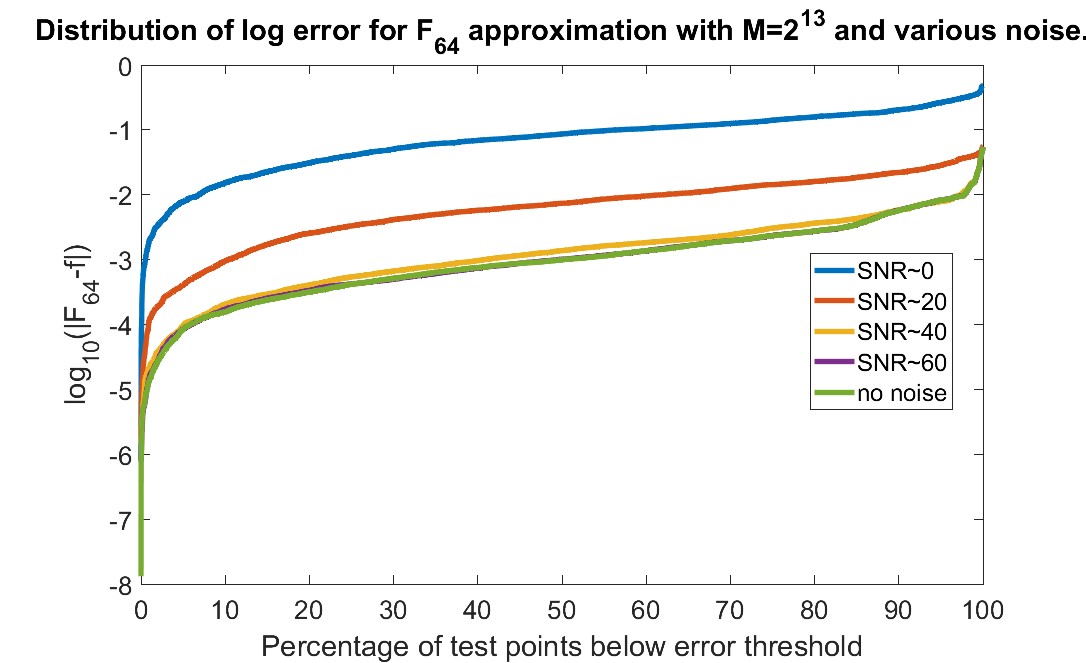}
\end{tabular}
    \caption{(Left) Percent point plot of log absolute error for various $n$ with $M=2^{13}$ training points and no noise. (Center) Percent point plot of log absolute error for various choices of $M$ with no noise. (Right) Percent point plot of log absolute error for various noise levels with $M=2^{13}$ training points.}
        \label{fig:3plot}
\end{figure*}

Figure~\ref{fig:errorplot} shows a plot of the true function and our approximation on the left y-axis and the absolute error on the right y-axis. The plot demonstrates that the approximation achieves much lower error than the uniform error bound at points where the function is relatively smooth, and only spikes locally at the singularities of the function ($\theta=\pm \pi/2$). Figure~\ref{fig:3plot} displays three percent point plots illustrating how the distribution of $\log_{10}|F_n-f|$ behaves for various choices of $n,M,\epsilon$. Each point $(x,y)$ on a curve indicates that $x\%$ of test points were approximated by our method with absolute error below $10^{y}$ for the $n$, $M$, and $\epsilon$ value associated with the curve.
The first graph shows the trend for various $n$ values.  As we increase $n$, we see consistent drop in log error.
The second graph shows various values of $M$. We again see a decrease in the overall log error as $M$ is increased.
The third graph shows how the log error decreases as the noise decreases. We can see that the approximation is much worse for low SNR values, but nearly indistinguishable from the noiseless case when the SNR is above 60.

\subsection{Parameter Estimation in Bi-exponential Sums}
\label{subsec:spencerdata}

This example is motivated by magnetic resonance relaxometry, in which the proton nuclei of water are first
excited with radio frequency pulses and then exhibit an exponentially decaying electromagnetic signal.
When one may assume the presence of two water compartments undergoing slow exchange, with signal corrupted by additive Gaussian noise, the
 signal is modeled typically as a bi-exponential decay function  \eqref{eq:ex2eq} (cf. \cite{spencer-biexponential}):
$$
 F(t)=c_1\exp(-t/T_{2,1})+c_2\exp(-t/T_{2,2})+E(t),
 $$
 where $E$ is the noise, $T_{2,1}, T_{2,2}>0$, and the time $t$ is typically sampled at equal intervals.
The problem is to determine $c_1, c_2, T_{2,1}, T_{2,2}$. 
The problem appears also in many other medical applications, such as intravoxel incoherent motion studies in magnetic resonance.
 An accessible survey of these applications is given in \cite{Istratov1999}.
 
 Writing $t=j\delta$, $\lambda_1=\delta/T_{2,1}$, $\lambda_2=\delta/T_{2,2}$, we may reformulate the data as
\be\label{eq:ex2eq}
f(j)\coloneqq f(\lambda_1,\lambda_2,j)=c_1 e^{-\lambda_1j}+c_2 e^{-\lambda_2j}+\epsilon(j),
\ee 
where $\epsilon(j)$ are samples of mean-zero normal noise.

In this example, suggested by Dr. Spencer at the National Institute of Aging (NIH), we  consider the case where $c_1=.7,c_2=.3$ and use our method to  
 determine the values $\lambda_1,\lambda_2$, given data of the form
\be\label{eq:ex2eq2}
\tilde{\mathbf{y}}(\lambda_1,\lambda_2)\coloneqq(f(1),f(2),\dots,f(100)).
\ee
 We ``train" our approximation process with $M$ samples of $(\lambda_1,\lambda_2)\in [.1,.7]\times [1.1,1.7]$ chosen uniformly at random and then plugging those values into~\eqref{eq:ex2eq} to generate vectors of the form shown in~\eqref{eq:ex2eq2}. 
 The dimension of the input data is $Q=100$, however (in the noiseless case) the data lies on a $q=2$ dimensional manifold, so we will use $\Phi_{n,2}$ to generate our approximations. 
 
 We note that our method is agnostic to the particular model \eqref{eq:ex2eq2} used to generate the data. 
 We treat $\lambda_1, \lambda_2$ as functions of $\tilde{\mathbf{y}}$ without a prior knowledge of this function.
 In the noisy case, this problem does not perfectly fit the theory studied in this chapter since the noise is applied to the input values $f(t)$ meaning we cannot assume they lie directly on an unknown manifold anymore. 
 Nevertheless, we can see some success with our method. 
We define the operators
\be
\mathbf{T}(\tilde{\mathbf{y}})=1000\tilde{\mathbf{y}}-(380,189,116,0,\dots,0),\qquad \mathbf{P}(\circ)=\frac{(\circ, 100)}{\norm{(\circ, 100)}_2}
\ee
and denote $\mathbf{y}=\mathbf{P}(\mathbf{T}(\tilde{\mathbf{y}}))$.
 In practice, the values used to define $\mathbf{T}$ and $\mathbf{P}$   need to be treated as hyperparameters of the model. In this example, we did not conduct a rigorous grid search.
 We use the same density estimation as in Section~\ref{subsec:expiecewise}:
\be
\operatorname{DE}\big(\x(\lambda_1,\lambda_2)\big)=\sum_{j=1}^M\Phi_{n,2}\big(\vec{x}(\lambda_1,\lambda_2)\cdot \mathbf{y}(\lambda_{1,j}, \lambda_{2,j})\big).
\ee
As a result, our approximation process looks like:
\be
\begin{bmatrix}\lambda_1\\
\lambda_2\end{bmatrix}\approx \vec{F}_n\big(\vec{x}(\lambda_1,\lambda_2)\big)=\sum_{j=1}^M\begin{bmatrix}\lambda_{1,j}\\\lambda_{2,j}\end{bmatrix}\Phi_{n,2}\big(\x(\lambda_1,\lambda_2)\cdot \mathbf{y}(\lambda_{1,j}, \lambda_{2,j})\big)\Big/\operatorname{DE}\big(\x(\lambda_1,\lambda_2)\big).
\ee 

Similar to Example~\ref{subsec:expiecewise}, we will include figures showing how the results are effected as $n,M,\epsilon$ are adjusted. We measure noise using the signal-to-noise ratio (SNR) defined by
\be
20\log_{10}\left(\norm{\tilde{\vec{y}}}_2\Big/\norm{\big(\epsilon(1),\dots,\epsilon(100)\big)}_2\right).
\ee
Unlike Example~\ref{subsec:expiecewise}, we will now be considering percent approximation error instead of uniform error as it is more relevant in this problem. We define the \textit{combined error} to be
\be
\sum_{j=1}^2 \frac{\abs{\lambda_{j,\text{true}}-\lambda_{j,\text{approx}}}}{\lambda_{j,\text{true}}}.
\ee

\begin{figure*}[!ht]
\centering
\begin{tabular}{ccc}
    \includegraphics[width=.3\textwidth]{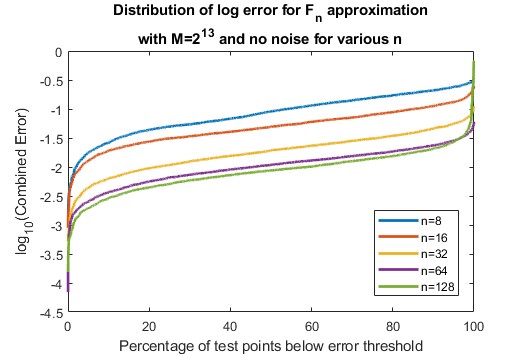}
&
    \includegraphics[width=.3\textwidth]{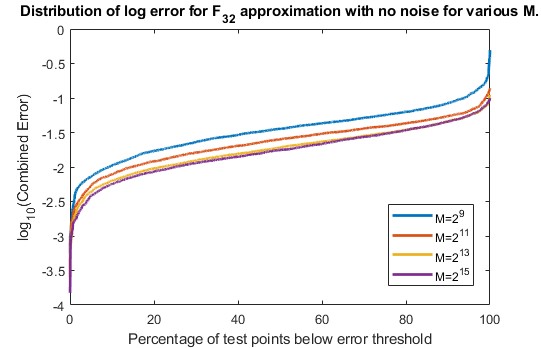}
&
    \includegraphics[width=.3\textwidth]{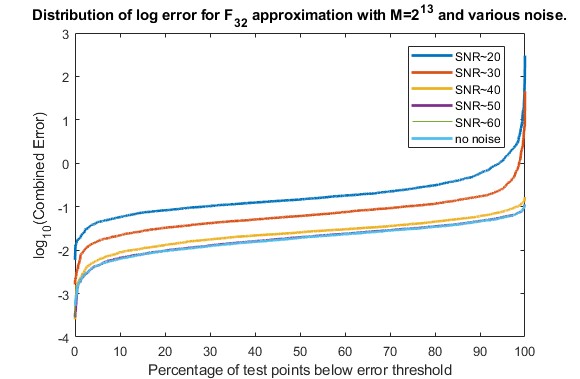}
\end{tabular}
    \caption{(Left) Percent point plot of log combined error for various $n$ with $M=2^{13}$ training points, and no noise. (Center) Percent point plot of log combined error for fixed $n=32$, various choices of $M$, and no noise. (Right) Percent point plot of log combined error for fixed $n=32$, fixed $M=2^{13}$ training points, and various noise levels.}
        \label{fig:ex23plot}
\end{figure*}

\begin{figure*}[!ht]
\centering
    \includegraphics[width=.45\textwidth]{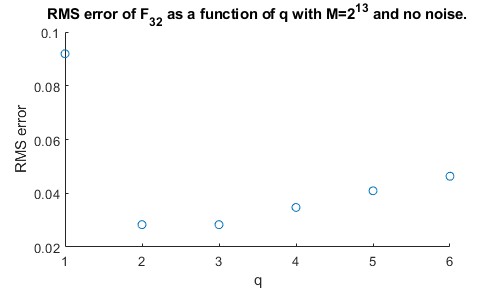}
    \caption{Plot of RMS error for approximation by $F_{32}$ for various $q$ values with $M=2^{13}$ and no noise.}
        \label{fig:ex2q}
\end{figure*}

Figure~\ref{fig:ex23plot} depicts three percent point plots showing the distribution of sorted $\log_{10}$(Combined Error) points for various $n,M,\epsilon$. Each point $(x,y)$ on a curve indicates that $x\%$ of test points were approximated by our method with combined error below $10^{y}$ for the $n$, $M$, and $\epsilon$ associated with the curve.
In the first graph, we see the distribution of for various choices of $n$. 
As $n$ increases, the overall log error decreases. 
An interesting phenomenon occurring in this figure is with the $n=128$ case where the uniform error is actually higher than the $n=64$ case. 
This is likely due to the fact  that overfitting can occur if $n$ gets too large relative to a fixed $M$. 
The second graph illustrates how the approximation improves as $M$ is increased. 
As expected, we see log error decay as we include more and more training points. In the third graph, we see that the approximation improves up to a limit as the noise decreases. 
There is very little noticeable difference between the noiseless case and any case where SNR$>50$.

Another question that may arise when utilizing our method on various data is what value of $q$ to use. While the theory predicts that $q$ should be associated with the intrinsic dimension of the manifold underlying the data, in practice this can only be estimated and so $q$ should be treated as a hyperparameter. In Figure~\ref{fig:ex2q}, we explore how changing $q$ effects the approximation in this example. In this case, the intrinsic dimension is $2$, and when $q=2,3$ the approximation does well. If $q$ is chosen too high or two low, the approximation yields a greater error.

\subsection{Darcy Flow Problem}
\label{subsec:darcyflow}

In this section we will look at a numerical example from the realm of PDE inverse problems. Steady-state Darcy flow is given by the following PDE (see for example,  \cite[Eq. (4.7)]{raonic2024convolutional}):
\be
-\nabla\cdot(a\nabla y)=f,
\ee
defined on a domain $D$ with the property that $y|_{\partial D}=0$. The problem is to predict the \textit{diffusion coefficient} $a$ and \textit{forcing term} $f$ given some noisy samples of $y$ on $D$. In this chapter we consider a 1-dimensional version and suppose that $a=e^{-st}$ and $f=pe^{-st}$ for some $p,s$. We take noisy samples of $y(p,s;\circ)=y$ satisfying the following boundary value problem:
\be\label{eq:ex3difeq}
-\frac{d}{dt}(e^{-st}y'(t))=pe^{-st},\qquad y(1)=0, y(0)=1.
\ee
In this sample, we take a similar approach to that of Example~\ref{subsec:spencerdata} by ``training'' our model with a data set of the form $\{\vec{y}_j,(p_j,s_j)\}_{j=1}^M$, where $(p_j,s_j)\in [.1,.25]\times [1.5,2.5]$ are sampled uniformly at random for each $j$. Letting $y_j$ denote the $y$ satisfying \eqref{eq:ex3difeq} with $p=p_j,s=s_j$, then $\vec{y}_j=\mathbf{P}(y_j(t_1),y_j(t_2),\dots,y_j(t_{100}))$, where $t_1,t_2,\dots,t_{100}$ are sampled uniformly from $[0,1]$ and $\mathbf{P}$ is the projection to the sphere. In this example, the projection first consists of finding the center $C$ and maximum spread over a single feature $r$ of the data. That is,
\be
\ba
C=&\left(\max_{j}y_j(t_1)+\min_j y_j(t_1),\dots,\max_j y_j(t_{100})+\min_j y_j(t_{100})\right)\Big/2,\\ r=&\max_{t_i}\left(\max_{j}y_j(t_1)-\min_j y_j(t_1),\dots,\max_j y_j(t_{100})-\min_j y_j(t_{100})\right).
\ea
\ee
Then, we define
\be
\mathbf{P}(\circ)=\frac{(\circ-C,r)}{\norm{(\circ-C,r)}_2}.
\ee
Our approximation process then looks like:
\be
\begin{bmatrix}p\\
s\end{bmatrix}\approx \vec{F}_n\big(\vec{y}\big)=\sum_{j=1}^M\begin{bmatrix}p_j\\s_j\end{bmatrix}\Phi_{n,2}(\vec{y}\cdot \mathbf{y}_j)\Big/\operatorname{DE}(\vec{y}),
\ee 
where
\be
\operatorname{DE}(\vec{y})=\sum_{j=1}^M \Phi_{n,2}(\vec{y}\cdot \vec{y}_j).
\ee
Also similar to Example~\ref{subsec:spencerdata}, we use the same notion of SNR and evaluate the success of our model using a \textit{combined error}, now defined to be
\be
\left(\abs{\frac{p_{\text{true}}-p_{\text{approx}}}{p_{\text{true}}}}+\abs{\frac{s_{\text{true}}-s_{\text{approx}}}{s_{\text{true}}}}\right).
\ee

\begin{figure*}[!ht]
\centering
\begin{tabular}{ccc}
    \includegraphics[width=.3\textwidth]{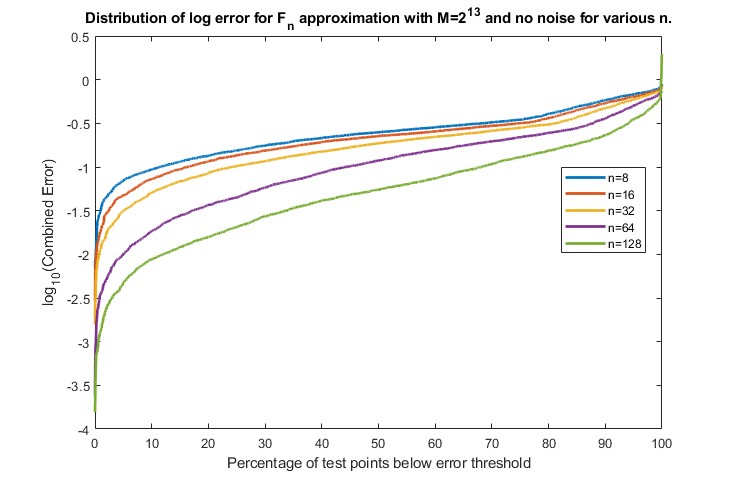}
&
    \includegraphics[width=.3\textwidth]{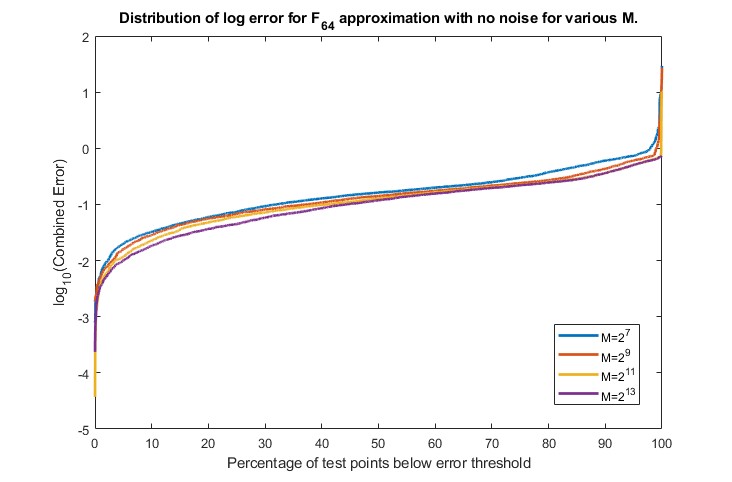}
&
    \includegraphics[width=.3\textwidth]{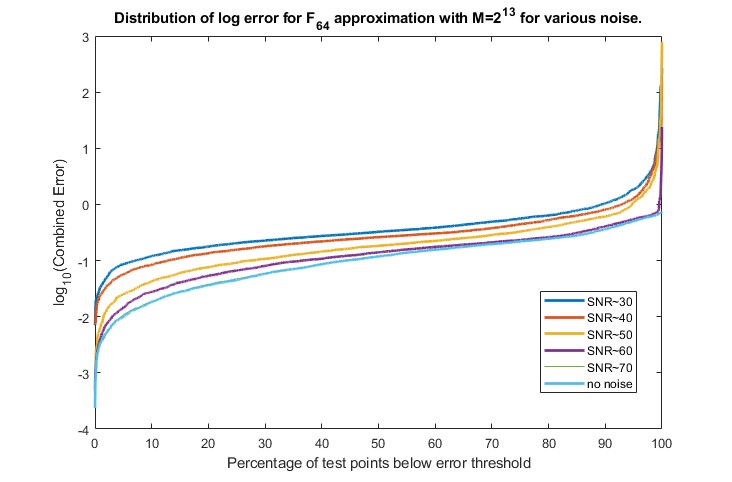}
\end{tabular}
\caption{(Left) Percent point plot of log combined error for various $n$, fixed $M=2^{13}$, and no noise. (Center) Percent point plot of log combined error for fixed $n=64$, various $M$, and no noise. (Right) Percent point plot of log combined error for fixed $n=64$, fixed $M=2^{13}$, and various noise levels.}
        \label{fig:ex3plot}
\end{figure*}

In Figure~\ref{fig:ex3plot}, we provide some percent point plots from using our method on this data. Each point $(x,y)$ on a curve indicates that $x\%$ of test points were approximated by our method with combined error below $10^{y}$ for the $n$, $M$, and $\epsilon$ associated with the curve. We see in the left-most plot that as we increase $n$, the error tends to decrease. In contrast to previous examples, the middle plot does not show much improvement by increasing M. This may be an indication of the fact that we have chosen a tight parameter space in this example (as compared to \ref{subsec:spencerdata}) and not many samples are needed to sufficiently cover the space. On the right-most plot, we see a decrease in error with the decrease of noise as expected, with convergence appearing to occur around the SNR=70 mark, as indicated by the green and light-blue lines being so close together.

\section{Proofs}\label{sec:manifoldapproxproofs}

The purpose of this section is to prove Theorem~\ref{theo:manifoldapproxmainthm}. 

In Section~\ref{bhag:summabilityop}, we study the approximation properties of the integral reconstruction operator  defined in \eqref{eq:manifold_summabilityop} (Theorem~\ref{theo:manifold_summop}).
In Section~\ref{bhag:resultsproof}, we use this theorem with $ff_0$ in place of $f$, and  use the Bernstein concentration inequality (Proposition~\ref{prop:berncon}) to discretize the integral expression in \eqref{eq:manifold_summabilityop} and complete the proof of Theorem~\ref{theo:manifoldapproxmainthm}.

\subsection{Integral Reconstruction Operator}\label{bhag:summabilityop}
In this section, we prove the following theorem which is an integral analogue of Theorem~\ref{theo:manifoldapproxmainthm}.

\begin{theorem}\label{theo:manifold_summop}
Let $0<\gamma<2$, $f\in W_\gamma(\mathbb{X})$, $\sigma_n$ be as defined in \eqref{eq:manifold_summabilityop}.
 Then for $n\geq 1$, we have
\begin{equation}\label{eq:theobound}
    \norm{f-\sigma_n(\mathbb{X},f)}_\mathbb{\mathbb{X}}\lesssim n^{-\gamma}\norm{f}_{W_\gamma(\mathbb{X})}.
\end{equation}
\end{theorem}

In order to prove this theorem, we will use a covering of $\XX$ using balls of radius $\iota^*$, and a corresponding partition of unity. 
A key lemma to facilitate the details here is the following.

\begin{lemma}\label{lemma:manifold_keylemma}
Let  $x\in \mathbb{X}$. 
Let $g\in C(\mathbb{X})$ be supported on $\mathbb{B}(x,\iota^*)$. If $G(u)=g(\eta_x(u))$, $0<\gamma<2$. 
Then
\begin{equation}\label{eq:manifold_localest}
    \abs{\int_\mathbb{X} \Phi_{n,q}(x\cdot y)g(y)d\mu^*(y)-\int_{\mathbb{S}_x} \Phi_{n,q}(x\cdot u) G(u)d\mu_{\mathbb{S}_x}^*(u)}\lesssim n^{-\gamma}||g||_\mathbb{X}.
\end{equation}
\end{lemma}

If $\phi\in C^\infty(\XX)$ is supported on $\mathbb{B}(x,\iota^*)$, then we may apply this theorem with $g=f\phi$, thereby providing locally a lifting of the integral on $\XX$ to the tangent equator $\SS_x$ with the function corresponding to $g$ on this tangent equator.

Naturally, the first step in this proof is to show that the Lebesgue constant for the kernel $\Phi_{n,q}$ is bounded independently of $n$ (cf. \eqref{eq:sigmaall}). 
Moreover, one can even leverage the localization of the kernel to improve on this bound when the integral is taken away from the point $x$ (cf. \eqref{eq:sigmaaway}). These are both done in the following lemma.

\begin{lemma}\label{lemma:manifold_lebesgue_number}
Let $r>0$ and $n\geq 1/r$. If  $\Phi_{n,q}$ is given as in \eqref{eq:kernel} with $S>q$, then
\be\label{eq:sigmaaway}
\sup_{x\in \mathbb{X}}\int_{\mathbb{X}\setminus \mathbb{B}(x,r)}\abs{\Phi_{n,q}(x\cdot y)}d\mu^*(y)\lesssim \max(1,nr)^{q-S}.
\ee
Additionally,
\be\label{eq:sigmaall}
\sup_{x\in\mathbb{X}}\int_{\mathbb{X}}|\Phi_{n,q}(x\cdot y)|d\mu^*(y)\lesssim 1.
\ee
\end{lemma}

\begin{proof}\ % Proof of Lemma~\ref{lemma:manifold_lebesgue_number}
Recall from Proposition~\ref{prop:taylorprop} that $\rho(x,y)\sim \arccos(x\cdot y)$, so \eqref{eq:sphkernloc} implies
\be\label{eq:manifoldloc}
\abs{\Phi_{n,q}(x\cdot y)}\lesssim \frac{n^q}{\max(1,n\arccos(x\cdot y))^S}\lesssim \frac{n^q}{\max(1,n\rho(x, y))^S}.
\ee
In this proof only, we fix $x\in \mathbb{X}$. 
Let $A_0=\mathbb{B}(x,r)$ and $A_k=\mathbb{B}(x,2^kr)\setminus \mathbb{B}(x,2^{k-1}r)$, $k\ge 1$.
Then $\mu^*(A_k)\ls 2^kqr^q$, and
  for any $y\in A_k$, $2^{k-1}r\leq \rho(x,y)\leq 2^{k}r$.
  
  First, let  $nr\geq 1$. 
 In view of \eqref{eq:ballmeasure}~and~\eqref{eq:manifoldloc}, it follows that
\be\ba\label{eq:lem01}
\int_{\mathbb{X}\setminus \mathbb{B}(x,r)}\abs{\Phi_{n,q}(x\cdot y)}d\mu^*(y)=&\sum_{k=1}^\infty \int_{A_k} \abs{\Phi_{n,q}(x\cdot y)}d\mu^*(y)
\lesssim \sum_{k=1}^\infty \frac{\mu^*(A_k)n^q}{(n2^{k-1}r)^S}\\
\lesssim&(nr)^{q-S}\sum_{k=0}^\infty 2^{k(q-S)}
\leq (nr)^{q-S}.
\ea\ee
Using this estimate with $r=1/n$ and the condition \eqref{eq:ballmeasure} on the measures of balls we see that
$$
\int_{\mathbb{X}}\abs{\Phi_{n,q}(x\cdot y)}d\mu^*(y)=\int_{A_0}\abs{\Phi_{n,q}(x\cdot y)}d\mu^*(y)+\int_{\mathbb{X}\setminus\mathbb{B}(x,r)}\abs{\Phi_{n,q}(x\cdot y)}d\mu^*(y)\lesssim 1+(nr)^{q-S}\sim 1.
$$
Since the choice of $x$ was arbitrary, we have proven \eqref{eq:sigmaall}. Then \eqref{eq:sigmaall}~and~\eqref{eq:lem01} combined give the bounds for \eqref{eq:sigmaaway}.
\end{proof}

Next, we prove Lemma~\ref{lemma:manifold_keylemma}.

\noindent\textsc{Proof of  Lemma~\ref{lemma:manifold_keylemma}.}

Since $\gamma<2$, we may choose (for sufficiently large $n$)
\be\label{eq:deltaSchoice}
 \delta=n^{-(\gamma+q+1)/(q+3)}, \quad n\delta=n^{(2-\gamma)/(q+3)} >1, \quad S> \frac{2q+3\gamma}{2-\gamma}.
 \ee
We may assume further that $\delta<\iota^*$.
Then, by using \eqref{eq:sph_bernstein} and Proposition~\ref{prop:taylorprop}, we see that
\be\label{eq:phidiff}
\ba
\biggl|\Phi_{n,q}(x\cdot \eta_x(u))&-\Phi_{n,q}(x\cdot u)\biggr|\lesssim n^{q+1}\abs{\arccos(x\cdot \eta_x(u))-\arccos(x\cdot u)}\\
=&n^{q+1}\abs{\arccos(x\cdot \eta_x(u))-\rho(x, \eta_x(u))}
\lesssim n^{q+1}\rho(x,\eta_x(u))^3
\lesssim n^{q+1}\delta^3,
\ea
\ee
for any $u\in\mathbb{S}_x(\delta)$.
Let $\mathbf{g}_1,\mathbf{g}_2$ be the metric tensors associated with the exponential maps $\varepsilon_x : \mathbb{T}_x(\mathbb{X})\to \mathbb{X}$ and $\overline{\varepsilon}_x: \mathbb{T}_x(\mathbb{X})\to \mathbb{S}_x$, respectively. Then we have the following change of variables formulas (cf. Table~\ref{tab:balldef}):
\be\label{eq:changevars}
\int_{\mathbb{B}(x,\delta)}d\mu^*(\varepsilon_x(v))=\int_{B_{\mathbb{T}}(x,\delta)}\sqrt{|\mathbf{g}_1|}dv,\qquad \int_{\mathbb{S}_x(\delta)}d\mu^*_q(u)=\int_{B_{\mathbb{T}}(x,\delta)}\sqrt{|\mathbf{g}_2|}d\overline{\varepsilon}^{-1}_x(u).
\ee
We set $v=\overline{\varepsilon}^{-1}_x(u)$ and use the fact (cf. \eqref{eq:volume_element})  that on $B_{\mathbb{T}}(x,\delta)$, $|\sqrt{|\mathbf{g}_1|}-1|\lesssim\delta^2$ and $|\sqrt{|\mathbf{g}_2|}-1|\lesssim\delta^2$. Then by applying Equations~\eqref{eq:phidiff},~\eqref{eq:changevars},~\eqref{eq:sphkernloc},~and~\eqref{eq:sph_bernstein}, we can deduce
\be\label{eq:lemma_7_1_main_est}
\begin{aligned}
\Biggl|\int_{\mathbb{B}(x,\delta)}\Phi_{n,q}(x\cdot y)& g(y)d\mu^*(y)-\int_{\mathbb{S}_x(\delta)}\Phi_{n,q}(x\cdot u)G(u)d\mu^*_{\mathbb{S}_x}(u)\Biggr| \\
\leq&\abs{\int_{B_{\mathbb{T}}(x,\delta)}\Phi_{n,q}(x\cdot \varepsilon_x(v))g(\varepsilon_x(v))(\sqrt{|\mathbf{g}_1|}-\sqrt{|\mathbf{g}_2|})dv}\\
&\qquad+\abs{\int_{\mathbb{S}_x(\delta)}(\Phi_{n,q}(x\cdot \eta_x(u))-\Phi_{n,q}(x\cdot u))G(u)d\mu^*_{\mathbb{S}_x}(u)}\\
\lesssim&\norm{g}_\mathbb{X}(\delta^{q+2}n^q+\delta^{q+3}n^{q+1}) \le \delta^{q+3}n^{q+1}\norm{g}_\mathbb{X}(1/(n\delta)+1)
\lesssim\norm{g}_\mathbb{X}n^{-\gamma}.
\end{aligned}
\ee
Now it only remains to examine the terms away from $\mathbb{S}_x(\delta),\mathbb{B}(x,\delta)$. Utilizing Lemma~\ref{lemma:manifold_lebesgue_number}, and the fact that $S\geq \frac{2q+3\gamma}{2-\gamma}$, we have
\be\label{eq:lemma_7_1_awaymanifold}
  \abs{\int_{\mathbb{X}\setminus\mathbb{B}(x,\delta)}\Phi_{n,q}(x\cdot y)g(y)d\mu^*(y)}\lesssim\norm{g}_\mathbb{X}(n\delta)^{q-S}=\norm{g}_\mathbb{X} n^{(q-S)(2-\gamma)/(q+3)}\ls \norm{g}_{\mathbb{X}}n^{-\gamma}.
\ee
Similarly, again using Lemma~\ref{lemma:manifold_lebesgue_number} (with $\mathbb{S}_x$ as the manifold) and observing $\norm{g}_\mathbb{X}=\norm{G}_{\mathbb{S}_x}$, we can conclude
\be\label{eq:lemma_7_1_awaysphere}
\abs{\int_{\mathbb{S}_x\setminus \mathbb{S}_x(\delta)} \Phi_{n,q}(x\cdot u)G(u)d\mu_{\mathbb{S}_x}^*(u)}\lesssim \norm{G}_{\mathbb{S}_x}(n\delta)^{q-S}
\lesssim\norm{g}_\mathbb{X}n^{-\gamma},
\ee
completing the proof.
\qed

We are now in a position to complete the proof of Theorem~\ref{theo:manifold_summop}.

\noindent\textsc{Proof  of Theorem~\ref{theo:manifold_summop}.}

Let $x\in \mathbb{X}$. Choose $\phi\in C^\infty$ such that $0\leq \phi(y)\leq 1$ for all $y\in \mathbb{X}$, $\phi(y)=1$ on $\mathbb{B}(x,\iota^*/2)$, and $\phi(y)=0$ on $\mathbb{X}\backslash \mathbb{B}(x,\iota^*)$. Then $f\phi$ is supported on $\mathbb{B}(x,\iota^*)$ and $F(u)\coloneqq \phi(\eta_x(u))f(\eta_x(u))$ belongs to $W_\gamma (\mathbb{S}_x)$. We observe that $\norm{f}_\mathbb{X}\lesssim \norm{f}_{W_\gamma(\mathbb{X})}$. By Lemma~\ref{lemma:manifold_lebesgue_number},
\be
\abs{\int_\mathbb{X}\Phi_{n,q}(x\cdot y)f(y)(1-\phi(y))d\mu^*(y)}\leq\norm{f}_\mathbb{X}\int_{\mathbb{X}\setminus\mathbb{B}(x,\iota^*/2)}\abs{\Phi_{n,q}(x\cdot y)}d\mu^*(y)\lesssim n^{-\gamma}\norm{f}_{W_\gamma(\mathbb{X})}.
\ee
Note that the constant above is chosen to account for the case where $n<2/\iota^*$. By Lemma~\ref{lemma:manifold_keylemma},
\begin{equation}
    \abs{\int_\mathbb{X} \Phi_{n,q}(x\cdot y)f(y)\phi(y)d\mu^*(y)-\sigma_n(\mathbb{S}_x,F)(x)}\lesssim n^{-\gamma}\norm{f\phi}_{\mathbb{X}}\lesssim n^{-\gamma}\norm{f}_{W_\gamma(\mathbb{X})}.
\end{equation}
Observe that since $f(x)=F(x)$ and $\norm{F}_{W_\gamma(\mathbb{S}_x)}\leq\norm{f}_{W_\gamma(\mathbb{X})}$,
\be
\ba
&|f(x)-\sigma_n(\mathbb{X},f)(x)|\\
\leq& |f(x)-F(x)|+|F(x)-\sigma_n(\mathbb{S}_x,F)(x)|+|\sigma_n(\mathbb{S}_x,F)(x)-\sigma_n(\mathbb{X},f)(x)|\\
\lesssim& 0+n^{-\gamma}\norm{F}_{W_\gamma(\mathbb{S}_x)}+\abs{\sigma_n(\mathbb{S}_x,F)(x)-\int_{\mathbb{X}}\Phi_{n,q}(x\cdot y)f(y)\phi(y)d\mu^*(y)}+\abs{\int_\mathbb{X}\Phi_{n,q}(x\cdot y)f(y)(1-\phi(y))d\mu^*(y)}\\
\leq& n^{-\gamma}\norm{f}_{W_\gamma(\mathbb{X})}.
\ea
\ee
Since this bound is independent of $x$, the proof is completed.
\qed

\subsection{Discretization}
\label{bhag:resultsproof}

In order to complete the proof of Theorem~\ref{theo:manifoldapproxmainthm}, we need to discretize the integral operator in Theorem~\ref{theo:manifold_summop} while keeping track of the error.
If the manifold were known and we could use the eigendecomposition of the Laplace-Beltrami operator, we could do this discretization  without losing the accuracy using quadrature formulas (cf., e.g., \cite{mhaskardata}).
In our current set up, it is more natural to use
 concentration inequalities. 
 We will  use the inequality summarized in Proposition~\ref{prop:berncon} below (c.f.~\cite[Section 2.8]{concentration}).

\begin{proposition}[Bernstein concentration inequality]\label{prop:berncon}
Let $Z_1,\cdots, Z_M$ be independent real valued random variables such that for each $j=1,\dots,M$, $|Z_j|\leq R$, and $\mathbb{E}(Z_j^2)\leq V$. Then for any $t>0$,
\begin{equation}
    \operatorname{Prob}\left(\abs{\frac{1}{M}\sum_{j=1}^M (Z_j-\mathbb{E}(Z_j))}\geq t\right)\leq 2\exp\left(-\frac{Mt^2}{2(V+Rt/3)}\right).
\end{equation}
\end{proposition}
In the following, we will set $Z_j(x)=z_j\Phi_{n,q}(x\cdot y_j)$, where $(y_j,z_j)$ are sampled from  $\tau$.
The following lemma estimates the variance of $Z_j$.

%LEMMA 3 STATEMENT
\begin{lemma}\label{lemma:variance_est}
With the setup from Theorem~\ref{theo:manifoldapproxmainthm}, we have
\begin{equation}\label{eq:variance}
    \sup_{x\in \mathbb{X}}\int\abs{z\Phi_{n,q}(x\cdot y)}^2d\tau(y,z)\lesssim n^q \norm{z}^2\norm{f_0}_\mathbb{X}, \qquad x\in\SS^Q.
\end{equation}
\end{lemma}

%LEMMA 3 PROOF
\begin{proof}\ % of Lemma~\ref{lemma:variance_est}

We observe that \eqref{eq:sphkernloc} and Lemma~\ref{lemma:manifold_lebesgue_number} imply that
\be
\sup_{x\in\mathbb{X}}\int_{\mathbb{X}}\Phi_{n,q}(x\cdot y)^2d\mu^*(y)\lesssim n^q\sup_{x\in\mathbb{X}}\int_\mathbb{X}\abs{\Phi_{n,q}(x\cdot y)}d\mu^*(y)\lesssim n^q.
\ee
Hence,
\begin{equation}\label{eq:lem2p1}
    \sup_{x\in\mathbb{X}}\int\abs{z(y,\epsilon)\Phi_{n,q}(x\cdot y)}^2d\tau(y,z)\leq \norm{z}^2\norm{f_0}_\mathbb{X}\sup_{x\in\mathbb{X}}\int_\mathbb{X}\Phi_{n,q}(x\cdot y)^2d\mu^*(y)
    \lesssim n^q \norm{z}^2\norm{f_0}_\mathbb{X}.
\end{equation}
\end{proof}

A limitation of the Bernstein concentration inequality is that it only considers a single $x$ value. Since we are interested in supremum-norm bounds, we must first relate the supremum norm of $Z_j$ over all $x\in \mathbb{S}^Q$ to a finite set of points. We set up the connection in the following lemma.

%LEMMA 4 STATEMENT
\begin{lemma}\label{lemma:supnorm_concentration}
Let $\nu$ be any (bounded variation) measure on $\mathbb{X}$. Then there exists a finite set $\mathcal{C}$ of size $|\mathcal{C}|\sim n^{Q}$ such that
\be\label{eq:polybound}
\norm{\int_{\mathbb{X}}\Phi_{n,q}(\circ\cdot y)d\nu(y)}_{\mathbb{S}^Q}\leq 2\max_{x\in \mathcal{C}}\abs{\int_{\mathbb{X}}\Phi_{n,q}(x\cdot y)d\nu(y)}.
\ee
\end{lemma}

%LEMMA 4 PROOF
\begin{proof}\ % of Lemma~\ref{lemma:supnorm_concentration}.

In view of the Bernstein  inequality for the derivatives of spherical polynomials, we see that
\be\label{eq:polybern}
\abs{P(x)-P(y)}\leq cn\norm{x-y}_{\infty}\norm{P}_\infty, \qquad P\in\Pi_n^Q.
\ee
We can see by construction that $\int_{\mathbb{X}}\Phi_{n,q}(t\cdot y)d\nu(y)$ is a polynomial of degree $<n$ in the variable $t$. Since $\mathbb{S}^Q$ is a compact space and polynomials of degree $<n$ are continuous functions, there exists some $x^*\in\mathbb{S}^Q$ such that
\be
\norm{\int_{\mathbb{X}}\Phi_{n,q}(\circ\cdot y)d\nu(y)}_{\mathbb{S}^Q}=\abs{\int_{\mathbb{X}}\Phi_{n,q}(x^*\cdot y)d\nu(y)}.
\ee
Let $c$ be the same as in \eqref{eq:polybern} and $\mathcal{C}$ be a finite set satisfying
\be\label{eq:meshnorm}
\max_{x\in \mathbb{S}^Q}\min_{y\in \mathcal{C}}\norm{x-y}_\infty\leq \frac{1}{2cn}.
\ee
Since $\mathbb{S}^Q$ is a compact $Q$-dimensional space, the set $\mathcal{C}$  needs no more than $\sim n^{Q}$ points.

Then there exists some $z^*\in\mathcal{C}$ such that
\be
\abs{\int_\mathbb{X}(\Phi_{n,q}(x^*\cdot y)-\Phi_{n,q}(z^*\cdot y))d\nu(y)}\lesssim n\norm{x^*-z^*}_\infty\abs{\int_{\mathbb{X}}\Phi_{n,q}(x^*\cdot y)d\nu(y)},
\ee
which implies \eqref{eq:polybound}.

\end{proof}

With this preparation, we now state the following theorem which gives a bound on the difference between our discrete approximation $F_n$ and continuous approximation $\sigma_n$ with high probability.

%THEOREM 2 STATEMENT
\begin{theorem}\label{theo:theorem2} Assume the setup of Theorem~\ref{theo:manifoldapproxmainthm}. Then for every $n\geq 1$ and $M\gtrsim n^{q+2\gamma}\log(n/\delta)$ we have
\be\label{eq:thm2}
\operatorname{Prob}_\tau\left(\norm{F_n(\mathcal{D};\circ)-\sigma_n(\mathbb{X},ff_0)}_{\mathbb{S}^Q}\geq c \norm{z}n^{-\gamma}\sqrt{\norm{f_0}_\mathbb{X}}\right)\leq \delta.
\ee
\end{theorem}

%THEOREM 2 PROOF
\begin{proof}
In this proof only, constants $c,c_1,c_2,\dots$ will maintain their value once used. Let $Z_j(x)=z_j\Phi_{n,q}(x\cdot y_j)$. Since $z$ is integrable with respect to $\tau$, one has the following for any $x\in \mathbb{S}^Q$:
\be
\mathbb{E}_\tau(Z_j(x))=\int_{\mathbb{X}} \mathbb{E}_\tau(z|y)\Phi_{n,q}(x\cdot y)d\nu^*(y)=\int_{\mathbb{X}}f(y)\Phi_{n,q}(x\cdot y)f_0(y)d\mu^*(y)=\sigma_n(\mathbb{X},ff_0)(x).
\ee
We have from \eqref{eq:sphkernloc} that $\abs{Z_j}\lesssim n^q\norm{z}$. Lemma~\ref{lemma:variance_est} informs us that $\mathbb{E}_\tau(Z_j^2)\lesssim n^q\norm{z}^2\norm{f_0}_\mathbb{X}$. Assume $0<r\leq 1$ and set $t=r\norm{z}\norm{f_0}_\mathbb{X}$. From Proposition~\ref{prop:berncon}, we see
\be\label{eq:thm20}
\begin{aligned}
\operatorname{Prob}_\tau\left(\abs{\frac{1}{M}\sum_{j=1}^M Z_j(x)-\sigma_n(\mathbb{X},ff_0)(x)}\geq t\right)\leq& 2\exp\left(-c_1\frac{Mt^2}{(n^q\norm{z}^2\norm{f_0}_\mathbb{X}+n^q\norm{z}t/3)}\right)\\
\leq&2\exp\left(-c_2\frac{M\norm{f_0}_\mathbb{X}r^2}{n^q}\right).
\end{aligned}
\ee
Let $\delta\in (0,1/2)$, $\mathcal{C}$ be a finite set satisfying~\eqref{eq:meshnorm} with $\abs{\mathcal{C}}\leq c_3n^{Q}$ (without loss of generality we assume $c_3\geq 1$),
\be\label{eq:creq}
c_4\geq \frac{\max\big(\log_2(c_3)+1,Q\big)}{c_2},
\ee
and
\be\label{eq:Mreq}
M\geq c_4n^{q+2\gamma}\log(n/\delta).
\ee
We now fix
\be\label{eq:rdef}
r\equiv\sqrt{c_4\frac{n^q}{M\norm{f_0}_\mathbb{X}}\log(n/\delta)}.
\ee
Notice that since $\norm{f_0}_\mathbb{X}\geq 1$, our assumption of $M$ in \eqref{eq:Mreq} implies
\be
r\leq n^{-\gamma}\Big/\sqrt{\norm{f_0}_\mathbb{X}}\leq 1,
\ee
so our choice of $r$ may be substituted into \eqref{eq:thm20}. Further,
\be\label{eq:r}
r\norm{z}\norm{f_0}_\mathbb{X}\leq c\norm{z} n^{-\gamma}\sqrt{\norm{f_0}_\mathbb{X}}.
\ee
With this preparation, we can conclude
\be\label{eq:thm21}
\begin{aligned}
\operatorname{Prob}_\tau\bigg(\big\|F_n(\mathcal{D};\circ)&-\sigma_n(\mathbb{X},ff_0)\big\|_{\mathbb{S}^Q}\geq c \norm{z}n^{-\gamma}\sqrt{\norm{f_0}_\mathbb{X}}\bigg)&\\
\leq&\operatorname{Prob}_\tau\left(\norm{\frac{1}{M}\sum_{j=1}^M Z_j-\sigma_n(\mathbb{X},ff_0)}_{\mathbb{S}^Q}\geq r\norm{z}\norm{f_0}_\mathbb{X}\right)&\text{(from \eqref{eq:r})}\\
\leq&\operatorname{Prob}_\tau\left(\max_{x_k\in \mathcal{C}}\left(\abs{\frac{1}{M}\sum_{j=1}^M Z_j(x_k)-\sigma_n(\mathbb{X},ff_0)(x_k)}\right)\geq t \right)&\text{(by Lemma~\ref{lemma:supnorm_concentration})}\\
\leq&\sum_{k=1}^{\abs{\mathcal{C}}}\operatorname{Prob}_\tau\left(\abs{\frac{1}{M}\sum_{j=1}^M Z_j(x_k)-\sigma_n(\mathbb{X},ff_0)(x_k)}\geq t \right)&\\
\leq& |\mathcal{C}|\exp\left(-c_2\frac{M\norm{f_0}_\mathbb{X}r^2}{n^q}\right)&\text{(from \eqref{eq:thm20})}\\
 \leq& c_3n^{Q-c_2c_4}\delta^{c_2c_4}&\text{(from \eqref{eq:rdef})}\\
 \leq&c_3n^{Q-Q}\left(\frac{1}{2}\right)^{\log_{1/2}(1/c_3)}\delta&\text{(from \eqref{eq:creq} and $\delta<1/2$)}\\
 \leq&\delta.
\end{aligned}
\ee
\end{proof}

We are now ready for the proof of Theorem~\ref{theo:manifoldapproxmainthm}.

\begin{proof}[Proof of Theorem~\ref{theo:manifoldapproxmainthm} (and Corollary~\ref{cor:maincor}~and~\ref{cor:densityest}).]
Since $f,f_0\in W_\gamma(\mathbb{X})$, we can determine that $ff_0\in W_\gamma(\mathbb{X})$ as well. Utilizing Theorem~\ref{theo:manifold_summop} with $ff_0$ and Theorem~\ref{theo:theorem2}, we obtain with probability at least $1-\delta$ that
\be
\ba
\norm{F_n(\mathcal{D};\circ)-ff_0}_{\mathbb{X}}\leq&\norm{F_n(\mathcal{D};\circ)-\sigma_n(\mathbb{X};ff_0)}_\mathbb{X}+\norm{\sigma_n(\mathbb{X};ff_0)-ff_0}_\mathbb{X}\\
\lesssim&\frac{\sqrt{\norm{f_0}_{\mathbb{X}}}\norm{z}+\norm{ff_0}_{W_\gamma(\mathbb{X})}}{n^\gamma}.
\ea
\ee
Corollary~\ref{cor:maincor} is seen immediately by setting $f_0=1$. Corollary~\ref{cor:densityest} follows from setting $z=1$ and then observing that $f=1$ and $\sqrt{\norm{f_0}_\mathbb{X}}\lesssim \norm{f_0}_{W_\gamma(\mathbb{X})}$.
\end{proof}

\nc{\mathcal{D},\tau}{Set of data $\mathcal{D}=\{y_j,z_j\}_{j=1}^M$ sampled from distribution $\tau$. It is assumed the $y_j$'s lie on a $q$-dimensional submanifold of $\mathbb{S}^Q$.}
\nc{\mathbb{S}^q,\mu^*_q}{Sphere of $q$-dimensions with probability measure $\mu^*_q$ as defined in \eqref{eq:sphere}.}
\nc{\omega_q}{Volume of the $q$-dimensional sphere.}
\nc{\mathbb{Y}}{Equator of $\mathbb{S}^Q$, as defined in Section~\ref{subsec:eqapprox}.}
\nc{\mathbb{X},\rho,\mu^*}{Submanifold of $\mathbb{S}^Q$ with geodesic $\rho$ and normalized volume measure $\mu^*$.}
\nc{Q,q}{Ambient dimension of the data and dimension of the underlying manifold, respectively.}
\nc{p_{q,\ell}}{orthonormalized ultraspherical polynomial of degree $\ell$ and dimension $q$ as defined in Section~\ref{subsec:sphharm}.}
\nc{\Phi_{n,q}}{Localized kernels as defined in \eqref{eq:kernel}.}
\nc{\sigma_n}{Continuous approximation (a.k.a. integral reconstruction) operator as defined on $\mathbb{S}^q$ in \eqref{eq:sphapproxop}, $\mathbb{Y}$ in \eqref{eq:sphapproxop2}, and $\mathbb{X}$ in \eqref{eq:manifold_summabilityop}.}
\nc{F_n}{Our proposed constructive approximation for finite data as defined in \eqref{eq:approximation}.}
\nc{Y_{\ell,k}}{Basis elements for the space of homogenous, harmonic polynomials of degree $\ell$.}
\nc{\mathbb{H}_\ell^q}{Space of homogenous, harmonic polynomials of degree $\ell$ in $q$ dimensions.}
\nc{\Pi_n^q}{Space of spherical polynomials of degree $<n$. See Section~\ref{subsec:sphharm}.}
\nc{E_n}{Degree of approximation as defined in Sections~\ref{subsec:sphapprox}~and~\ref{subsec:eqapprox}.}
\nc{W_\gamma}{Smoothness class of functions as defined on $\mathbb{S}^q$ in Section~\ref{subsec:sphapprox}, on $\mathbb{Y}$ in Section~\ref{subsec:eqapprox}, and on $\mathbb{X}$ in Definition~\ref{def:manifold_smoothness}.}
\nc{B_{Q+1}(x,r),S^Q(x,r),B_\mathbb{T},\mathbb{S}_x,\mathbb{B}}{Balls on various spaces. See Table~\ref{tab:balldef} for reference.}
\nc{\mathbb{S}_x}{The unique $q$-dimensional equator of $\mathbb{S}^Q$ that shares a tangent space $\mathbb{T}_x$ with the point $x\in\mathbb{X}$.}
\nc{\varepsilon_x,\overline{\varepsilon}_x}{Exponential map for $\mathbb{X},\mathbb{S}_x$ respectively. See \ref{sec:manifoldintro} for details.}
\nc{\eta_x}{Composite map $\varepsilon_x\circ \overline{\varepsilon}^{-1}_x$ defined in Section~\ref{sec:manifoldapprox}.}
\nc{\iota_1,\iota_2,\iota^*}{Injectivity radii of $\varepsilon_x,\overline{\varepsilon}_x,\eta_x$, respectively.}
\nc{\mathbf{g}_1,\mathbf{g}_2}{Metric tensors associated with $\varepsilon_x,\overline{\varepsilon}_x$, respectively. See \ref{sec:manifoldintro} for details.}

\chapter{Local Transfer Learning}
\label{ch:transferlearning}

The content in this part is sourced from our paper published in \textit{Inverse Problems, Regularization Methods and Related Topics: A Volume in Honour of Thamban Nair} titled ``Local transfer learning from one data space to another" \cite{mhaskarodowdtransfer}.

\section{Introduction}\label{bhag:introduction}

The problem of transfer learning is to learn the parameters of an approximation process based on one data set and leverage this information to aid in the determination of an approximation process on another data set \cite{valeriyasmartphone, maskey2023transferability, maurer2013sparse}. An overview of transfer learning can be found in \cite{hosna-transfer}. Since training large (by today's standards) neural networks requires computing power simply unavailable to most researchers, the idea of transfer learning may be appealing to many wishing to use neural network models on new problems. While training a neural network from scratch on a new problem is not feasible in many cases, tuning a pre-trained network for a similar problem as the one it was trained on is doable and often yields better results anyways. Of course, this leads to some major questions such as
\bit
\item How does one identify which parameters for one problem are important for another?

\item What if the new problem has a feature (or multiple) unlike any from the pre-trained model?

\item How do we interpret the outputs of the model after it has been tuned for the new problem?
\eit
Since these questions are very broad and remain open in many settings, transfer learning is an exciting and active area of research.

In the context of manifold learning, a data set (point cloud) determines a manifold, so that different data sets would correspond to different manifolds. 
In the context of data spaces, we can therefore interpret transfer learning as ``lifting'' a function from one data space (the \emph{base data space}) to another (the \emph{target data space}).
This viewpoint allows us to unify the topic of transfer learning with the study of some inverse problems in image/signal processing.
For example, the problem of synthetic aperture radar (SAR) imaging can be described in terms of an inverse Radon transform \cite{nolan2002synthetic,  cheney2009fundamentals, munson1983tomographic}. 
The domain and range of the Radon transform are different, and hence, the problem amounts to approximating the actual image on one domain based on observations of its Radon transform, which are located on a different domain.
Another application is in analyzing hyperspectral images changing with time \cite{coifmanhirn}. 
A similar problem arises in analyzing the progress of Alzheimer's disease from MRI images of the brain taken over time, where one is interested in the development of the cortical thickness as a function on the surface of the brain, a manifold which is changing over time \cite{kim2014multi}.

Motivated by these applications and the paper \cite{coifmanhirn} of Coifman and Hirn, the question of lifting a function from one data space to another, when certain landmarks from one data space were identified with those on the other data space, was studied in \cite{tauberian}. 
For example, it is known \cite{lerch2005focal} that in spite of the changing brain, one can think of each brain to be parametrized by an inner sphere, and the cortical thickness at certain standard points based on this parametrization are important in the prognosis of the disease.
In \cite{tauberian} we investigated certain conditions on the two data spaces which allow the lifting of a function from one to the other, and analyzed the effect on the smoothness of the function as it is lifted.

In many applications, the data about the function is available only on a part of the base data space. 
The novel part of this chapter is to investigate the following questions of interest: 
(1) determine on what subsets of the target data space the lifting is defined, and (2) how the local smoothness on the base data space translates into the local smoothness of the lifted function. In limited angle tomography, one observes the Radon transform on a limited part of a cylinder and needs to reconstruct the image as a function on a ball from this data. 
A rudimentary introduction to the subject is given in the book of Natterer \cite{natterer2001mathematics}.
We do not aim to solve  the limited angle tomography problem itself, but we will study in detail an example motivated by the singular value decomposition of the Radon transform, which involves two different systems of orthogonal polynomials on the interval $[-1,1]$.
We are given the coefficients in the expansion of a function $f$ on $[-1,1]$ in terms of Jacobi polynomials with certain parameters (the base space expansion in our language), and use them as the coefficients in an expansion in terms of Jacobi polynomials with respect to a different set of parameters (the target space in our language). 
Under what conditions on $f$ and the parameters of the two Jacobi polynomial systems will the expansion in the target space converge and in which $L^p$ spaces?
We will illustrate our general theory by obtaining a localized transplantation theorem for uniform approximation.

In Section~\ref{bhag:singlespace}, we review certain important results in the context of a single \emph{data space} (our abstraction of a manifold). In particular, we present a characterization of local approximation of functions on such spaces.
In Section~\ref{bhag:jointspaces}, we review the notion of joint spaces (introduced under a different name in \cite{tauberian}).
The main new result of our chapter is to study the lifting of a function from a subset (typically, a ball) on one data space to another. 
These results are discussed in Section~\ref{bhag:locapprox}. 
The proofs are given in Section~\ref{bhag:proofs}.
An essential ingredient in our constructions is the notion of localized kernels which, in turn, depend upon a Tauberian theorem. 
For the convenience of the reader, this theorem is presented in Appendix~\ref{bhag:tauberian}. 
Appendix~\ref{sec:orthopolys} lists some important properties of Jacobi polynomials which are required in our examples.

\section{Data Spaces}
\label{bhag:singlespace}

As mentioned in the introduction, a good deal of research on manifold learning is devoted to the question of learning the geometry of the manifold. 
For the purpose of harmonic analysis and approximation theory on the manifold, we do not need the full strength of the differentiability structure on the manifold. 
Our own understanding of the correct hypotheses required to study these questions has evolved, resulting in a plethora of terminology such as data defined manifolds, admissible systems, data defined spaces, etc., culminating in our current understanding with the definition of a data space given in \cite{mhaskardata}. 
For the sake of simplicity, we will restrict our attention in this chapter to the case of compact spaces. 
We do not expect any serious problems in extending the theory to the general case, except for a great deal of technical details.

Thus, the set up is the following.

We consider a compact metric measure space $\XX$ with metric $d$ and a probability measure $\mu^*$. We take $\{\lambda_k\}_{k=0}^\infty$ to be a non-decreasing sequence of real numbers with $\lambda_0=0$ and $\lambda_k\to\infty$ as $k\to\infty$, and $\{\phi_k\}_{k=0}^\infty$ to be an orthonormal set in $L^2(\mu^*)$.
We assume that each $\phi_k$ is continuous.
The elements of the space
\be\label{eq:diffpolyspace}
\Pi_n=\mathsf{span}\{\phi_k: \lambda_k <n\}
\ee
are called \emph{diffusion polynomials} (of order $<n$). 
We write $\disp\Pi_\infty=\bigcup_{n>0}\Pi_n$.
We introduce the following notation.
\be\label{eq:balldeftransfer}
\BB(x,r)=\{y\in \XX : d(x,y)\le r\},  \qquad x\in\XX, \ r>0.
\ee
If $A\subseteq \XX$ we define
\be\label{eq:setnbds}
\BB(A,r)=\bigcup_{x\in A}\BB(x,r).
\ee

With this set up, the definition of a compact data space is the following.

\begin{definition}\label{def:ddrdef}
The  tuple $\Xi=(\XX,d,\mu^*, \{\lambda_k\}_{k=0}^\infty, \{\phi_k\}_{k=0}^\infty)$ is called a \textbf{(compact) data space} if 
each of the following conditions is satisfied.
\begin{enumerate}
\item For each $x\in\XX$, $r>0$, $\mathbb{B}(x,r)$ is compact.
\item (\textbf{Ball measure condition}) There exist $q\ge 1$ and $\kappa>0$ with the following property: For each $x\in\XX$, $r>0$,
\be\label{eq:ballmeasurecond}
\mu^*(\mathbb{B}(x,r))=\mu^*\left(\{y\in\XX: d(x,y)<r\}\right)\le \kappa r^q.
\ee
(In particular, $\mu^*\left(\{y\in\XX: d(x,y)=r\}\right)=0$.)
\item (\textbf{Gaussian upper bound}) There exist $\kappa_1, \kappa_2>0$ such that for all $x, y\in\XX$, $0<t\le 1$,
\be\label{eq:gaussianbd}
\left|\sum_{k=0}^\infty \exp(-\lambda_k^2t)\phi_k(x)\phi_k(y)\right| \le \kappa_1t^{-q/2}\exp\left(-\kappa_2\frac{d(x,y)^2}{t}\right).
\ee
\end{enumerate}
We refer to $q$ as the \textbf{exponent} for $\Xi$.
\end{definition}

The primary example of a data space is, of course, a Riemannian manifold.

\begin{example}\label{uda:manifold}
{\rm  Let $\XX$ be a smooth, compact, connected Riemannian manifold (without boundary), $d$ be the geodesic distance on $\XX$, $\mu^*$ be the Riemannian volume measure normalized to be a probability measure, $\{\lambda_k\}$ be the sequence of eigenvalues of the (negative) Laplace-Beltrami operator on $\XX$, and $\phi_k$ be the eigenfunction corresponding to the eigenvalue $\lambda_k$; in particular, $\phi_0\equiv 1$. 
It has been proven in  \cite[Appendix~A]{mhaskardata} that the Gaussian upper bound is satisfied. 
Therefore, if the condition in Equation~\eqref{eq:ballmeasurecond} is satisfied, then $(\XX,d,\mu^*, 
\{\lambda_k\}_{k=0}^\infty, \{\phi_k\}_{k=0}^\infty)$ is a data space with exponent equal to the dimension of the manifold. 
\qed}
\end{example}

\begin{remark}\label{rem:graph}
{\rm
In \cite{friedman2004wave}, Friedman and Tillich give a construction for an orthonormal system on a graph which leads to a finite speed of wave propagation. 
It is shown in \cite{frankbern} that this, in turn, implies the Gaussian upper bound. 
Therefore, it is an interesting question whether appropriate definitions of measures and distances can be defined on a graph to satisfy the assumptions of a data space.
\qed}
\end{remark}

%JACOBI EXAMPLE
\begin{example}\label{uda:jacobispace}
{\rm In this example, we let $\mathbb{X}=[0,\pi]$ and for $\theta_1,\theta_2
\in\mathbb{X}$ we simply define the distance as
\begin{equation}
d(\theta_1,\theta_2)=\abs{\theta_1-\theta_2}.
\end{equation}
We will consider the so-called \textit{trigonometric functions}:
\begin{equation}
\phi_n^{(\alpha,\beta)}(\theta)=(1-\cos\theta)^{\alpha/2+1/4}(1+\cos\theta)^{\beta/2+1/4}p_n^{(\alpha,\beta)}(\cos\theta),
\end{equation}
where $p_n^{(\alpha,\beta)}$ are orthonormalized Jacobi polynomials defined as in Appendix~\ref{sec:orthopolys} and $\alpha,\beta\geq -1/2$. We define
\begin{equation}
d\mu^*(\theta)=\frac{1}{\pi}d\theta.
\end{equation}
We see that a change of variables $x=\cos\theta$ in Equation~\eqref{eq:jacobiortho} results in the following orthogonality condition
\begin{equation}
\int_0^\pi \phi_n^{(\alpha,\beta)}(\theta)\phi_m^{(\alpha,\beta)}(\theta)d\theta=\delta_{n,m}.
\end{equation}
So our orthonormal set of functions with respect to $\mu^*$ will be $\{\sqrt{\pi}\phi_n^{(\alpha,\beta)}\}$. It was proven in \cite[Theorem A in view of Equation (3)]{nowak2011sharp} that with
\begin{equation}
\lambda_n=n+\frac{\alpha+\beta+1}{2},
\end{equation}
we have
\begin{equation}
\pi\sum_{n=0}^\infty \exp\left(-\lambda_{n}^2t\right)\phi_n^{(\alpha,\beta)}(\theta_1)\phi_n^{(\alpha,\beta)}(\theta_2)\lesssim t^{-1/2}\exp\left(-c\frac{d(\theta_1,\theta_2)^2}{t}\right), \quad \theta_1,\theta_2\in\mathbb{X}.
\end{equation}
In conclusion,
\begin{equation}
	\Xi=(\mathbb{X},d,\mu^*,\{\lambda_n\},\{\sqrt{\pi}\phi_n^{(\alpha,\beta)}\})
\end{equation}
is a data space with exponent $1$.
\qed}
\end{example}

The following example illustrates how a manifold with boundary can be transformed into a closed manifold as in Example~\ref{uda:manifold}.
We will use the notation and facts from Appendix~\ref{sec:orthopolys} without always referring to them explicitly.
We adopt the notation 
\be\label{eq:spheredef}
\SS^q=\{\x\in\RR^{q+1}: |\x|=1\}, \qquad \SS^q_+=\{\x\in\SS^q : x_{q+1}\ge 0\}.
\ee

%BALL EXAMPLE
\begin{example}\label{uda:pnjk}
{\rm
Let $\mu^*_q$ denote the volume measure of $\mathbb{S}^q$, normalized to be a probability measure. Let $\HH_{n}^{q}$ be the space of the restrictions to $\SS^{q}$ of homogeneous harmonic polynomials of degree $n$ on $q+1$ variables, and $\{Y_{n,k}\}_k$ be an orthonormal (with respect to $\mu^*_q$) basis for $\HH_{n}^{q}$.
The polynomials $Y_{n,k}$ are eigenfunctions of the Laplace-Beltrami operator on the manifold $\SS^q$ with eigenvalues $n(n+q-1)$. The geodesic distance between $\xi,\eta\in\mathbb{S}^q$ is $\arccos(\xi\cdot \eta)$, so the Gaussian upper bound for manifolds takes the form
\be\label{eq:sphgauss}
\sum_{n,k}\exp(-n(n+q-1)t)Y_{n,k}(\x)\overline{Y_{n,k}(\y)}\ls t^{-q/2}\exp\left(-c\frac{(\arccos(\x\cdot\y))^2}{t}\right).
\ee
As a result, $(\mathbb{S}^q,\arccos(\circ\cdot \circ),\mu^*_q,\{\lambda_{n}\}_n,\{Y_{n,k}\}_{n,k})$ is a data space with dimension $q$.

Now we consider 
$$
\XX=\BB^q=\{\x\in\RR^q: |\x|\le 1\}.
$$
We can identify $\BB^q$ with $\SS^q_+$ as follows. Any point $\x\in\BB^q$ has the form $\x=\omega\sin\theta$ for some   $\omega\in\SS^{q-1}$, $\theta\in [0,\pi/2]$. We write $\hat{\x}=(\omega\sin\theta, \cos\theta)\in\SS^q_+$. 
With this identification, $\SS^q_+$ is parameterized by $\BB^q$ and we define
\begin{equation}\label{eq:ballmeasuretransfer}
\begin{aligned}
d\mu^*(\x)=d\mu^*_q(\hat{\x})=&\frac{\operatorname{Vol}(\mathbb{B}^q)}{\operatorname{Vol}(\mathbb{S}^q_+)}(1-|\x|^2)^{-1/2}dm^*(\x)\\
=&\frac{\Gamma((q+1)/2)}{\sqrt{\pi}\Gamma(q/2+1)}(1-|\x|^2)^{-1/2}dm^*(\x),
\end{aligned}
\end{equation}
where $\mu^*_q$ is the probability volume measure on $\mathbb{S}^q_+$, and $m^*$ is the probability volume measure on $\mathbb{B}^q$.
It is also convenient to define the distance on $\BB^q$ by 
\begin{equation}
d(\x_1,\x_2)=\arccos(\hat{\x}_1\cdot\hat{\x}_2)=\arccos(\x_1\cdot\x_2 +\sqrt{1-|\x_1|^2}\sqrt{1-|\x_2|^2}).
\end{equation}
All spherical harmonics of degree $2n$ are even functions on $\mathbb{S}^q$. So with the identification of measures as above, one can represent the even spherical harmonics as an orthonormal system of functions on $\mathbb{B}^q$. That is, by defining
\be
P_{2n,k}(\x)=\sqrt{2}Y_{2n,k}(\hat{\x}),
\ee
we have
\be\label{eq:ballortho}
\begin{aligned}
\int_{\BB^q} P_{2n,k}(\x)\overline{P_{2n',k'}(\x)}d\mu^*(\x) =&2\int_{\mathbb{S}_+^q}Y_{2n,k}(\hat\x)\overline{Y_{2n',k'}(\hat\x)}d\mu^*_q(\hat\x)\\
=&\int_{\mathbb{S}^q}Y_{2n,k}(\xi)\overline{Y_{2n',k'}(\xi)}d\mu^*_q(\xi)\\
=&\delta_{(n,k), (n',k')}.
\end{aligned}
\ee
To show the Gaussian upper bound for $Y_{2n,k}$ on $\mathbb{B}^q$, we first see that in view of the summation formula \eqref{eq:summation} (from Chapter~\ref{ch:manifoldapprox}) and \eqref{eq:evenjacobi} we have
\begin{equation}
\begin{aligned}
&\sum_{k=1}^{\operatorname{dim}(\mathbb{H}_{2n}^q)}P_{2n,k}(\x)\overline{P_{2n,k}(\y)}\\
=&\sum_{k=1}^{\operatorname{dim}(\mathbb{H}_{2n}^q)}Y_{2n,k}(\hat\x)\overline{Y_{2n,k}(\hat\y)}\\
=&\frac{\omega_q}{\omega_{q-1}} p_{2n}^{(q/2-1,q/2-1)}(1)p_{2n}^{(q/2-1,q/2-1)}(\hat\x\cdot\hat\y)\\
=&\frac{\omega_q}{\omega_{q-1}}2^{(q-1)/2}p_n^{(q/2-1,-1/2)}(1)p_n^{(q/2-1,-1/2)}(\cos(2\arccos(\hat\x\cdot \hat\y))).
\end{aligned}
\end{equation}
In light of Equation~\eqref{eq:jacobidifeq} we define
\begin{equation}
\lambda_{n}=\sqrt{n(n+q/2-1/2)},
\end{equation}
which is conveniently not dependent upon $k$. Using \eqref{eq:specialjacobigauss},  we see that for $t>0$
\begin{equation}\label{eq:pgauss}
\begin{aligned}
&\sum_{n=0}^\infty\sum_{k=1}^{\operatorname{dim}(\mathbb{H}_{2n}^q)}\exp\left(-\lambda_{n}^2t\right)P_{2n,k}(\mathbf{x})\overline{P_{2n,k}(\y)}\\
\sim& \sum_{n=0}^\infty\exp(-n(n+q/2-1/2)t)p_{n}^{(q/2-1,-1/2)}(1)p_{n}^{(q/2-1,-1/2)}(\cos(2\arccos(\hat\x\cdot \hat\y)))\\
\lesssim& t^{-q/2}\left(-4c\frac{\arccos(\hat{\x}_1\cdot \hat{\x}_2)^2}{t}\right).
\end{aligned}
\end{equation}
Therefore, $(\mathbb{B}^q,d,\mu^*,\{\lambda_n\}_n,\{P_{2n,k}\}_{n,k})$ is a data space with exponent $q$.
\qed}
\end{example}

In this section, we will assume $\Xi$ to be a fixed data space and omit its mention from the notations. 
We will mention it later in other parts of the chapter in order to avoid confusion.
Next, we define smoothness classes of functions on $\XX$. 
In the absence of any differentiability structure, we do this in a manner that is customary in approximation theory. We define first the \emph{degree of approximation} of a function $f\in L^p(\mu^*)$ by
\be\label{eq:degapprox}
E_n(p,f)=\min_{P\in \Pi_n}\|f-P\|_{p,\mu^*}, \qquad n>0,  1\le p \le \infty,\ f\in L^p(\mu^*).
\ee
We find it convenient to denote by $X^p$ the space $\{f\in L^p(\mu^*) : \disp\lim_{n\to\infty}E_n(p,f)=0\}$; e.g., in the manifold case, $X^p=L^p(\mu^*)$ if $1\le p<\infty$ and $X^\infty=C(\XX)$.
In the case of Example~\ref{uda:pnjk}, we need to restrict ourselves to even functions.

\begin{definition}\label{def:sobolev}
Let $1\le p\le \infty$, $\gamma>0$. \\
{\rm (a)} For $f\in X^p$, we define
\be\label{sobnorm}
\|f\|_{W_{\gamma,p}}=\|f\|_{p,\mu^*}+\sup_{n>0}n^\gamma E_n(p,f),
\ee
and note that
\be\label{sobnormuseful}
\|f\|_{W_{\gamma,p}}\sim \|f\|_{p,\mu^*}+\sup_{n\in\ZZ_+}2^{n\gamma}E_{2^n}(p,f).
\ee
The space $W_{\gamma,p}$ comprises all $f$ for which $\|f\|_{W_{\gamma,p}} <\infty$.\\
{\rm (b)}
We write $C^\infty=\displaystyle\bigcap_{\gamma>0}W_{\gamma,\infty}$.
If $B$ is a ball in $\XX$, $C^\infty(B)$ comprises functions $f\in C^\infty$ which are supported on $B$.\\
{\rm (c)} If $x_0\in\XX$, the space $W_{\gamma,p}(x_0)$ comprises functions $f$ such that there exists $r>0$ with the property that for every $\phi\in C^\infty(\mathbb{B}(x_0,r))$, $\phi f\in W_{\gamma,p}$. 
If $A\subset \XX$, the space $W_{\gamma,p}(A)=\displaystyle\bigcap_{x_0\in A}W_{\gamma,p}(x_0)$; i.e., $W_{\gamma,p}(A)$ comprises functions which are in $W_{\gamma,p}(x_0)$ for each $x_0\in A$.
\end{definition}

A central theme in approximation theory is to characterize the smoothness spaces $W_{\gamma,p}$ in terms of the degree of approximation from some spaces; in our case we consider $\Pi_n$'s. 

For this purpose, we define some localized kernels and operators.

The kernels are defined by
\be\label{eq:kerndef}
\Phi_{n}(H;x,y)= \sum_{m=0}^\infty H\left(\frac{\lambda_m}{n}\right)\phi_m(x)\phi_m(y),
\ee
where $H :\RR\to\RR$ is a compactly supported function.

The operators corresponding to the kernels $\Phi_n$ are defined by
\be\label{eq:opdef}
\sigma_{n}(H;f,x)= \int_{\mathbb{X}}\Phi_{n}(H;x,y)f(y)d\mu^*(y) =\sum_{k: \lambda_k<n}H\left(\frac{\lambda_k}{n}\right)\hat{f}(k)\phi_k(x),
\ee
where
\be\label{eq:fourcoeff}
\hat{f}(k)=\int_\XX f(y)\phi_k(y)d\mu^*(y).
\ee

The following proposition recalls an important property of these kernels. 
Proposition~\ref{prop:kernloc}  is proved in \cite[Theorem 4.1]{mhaskarmaggioniframe}, and more recently  in much greater generality in \cite[Theorem~4.3]{tauberian}.
\begin{proposition}[{\cite[Theorem 4.3]{tauberian}}]\label{prop:kernloc}
Let  $S>q+1$ be an integer, $H:\mathbb{R}\to \mathbb{R}$ be an even, $S$ times continuously differentiable, compactly supported function. 
 Then for every $x,y\in \mathbb{X}$, $N>0$,
\begin{equation}\label{eq:kernlocest}
| \Phi_N(H;x,y)|\ls \frac{N^{q}}{\max(1, (Nd(x,y))^S)},
\end{equation}
where the constant may depend upon $H$ and $S$, but not on $N$, $x$, or $y$.
\end{proposition}

In the remainder of this chapter, we fix a filter $h$; i.e., an infinitely differentiable function $h: [0,\infty)\to [0,1]$, such that $h(t)=1$ for $0\le t\le 1/2$, $h(t)=0$ for $t\ge 1$. 
The domain of the filter $h$ can be extended to $\RR$ by setting $h(-t)=h(t)$. 
Since $h$ is fixed, its mention will be omitted from the notation unless we feel that this would cause a confusion.
The following theorem gives a crucial property of the operators, proved in several papers in different contexts, see \cite[Theorem 5.1]{mhaskardata} for a recent proof.
\begin{theorem}[{\cite[Theorem 5.1]{mhaskardata}}]\label{theo:goodapprox}
Let $n>0$. If $P\in\Pi_{n/2}$, then $\sigma_n(P)=P$. Also,
for any $p$ with $1\le p\le\infty$, 
\be\label{opbd}
\|\sigma_n(f)\|_p \ls \|f\|_p, \qquad f\in L^p.
\ee 
If $1\le p\le \infty$, and $f\in L^p(\XX)$, then
\be\label{goodapprox}
E_n(p,f)\le \|f-\sigma_n(f)\|_{p,\mu^*}\ls E_{n/2}(p,f).
\ee
\end{theorem}

While Theorem~\ref{theo:goodapprox} gives, in particular, a characterization of the global smoothness spaces $W_{\gamma,p}$, the characterization of local smoothness requires two more assumptions: the partition of unity and product assumption. 

\begin{definition}[\textbf{Partition of unity}] We say that a set $X$ has a partition of unity if for every $r>0$, there exists a countable family $\mathcal{F}_r=\{\psi_{k,r}\}_{k=0}^\infty$ of $C^\infty$ functions with the following properties:
\begin{enumerate}
    \item Each $\psi_{k,r}\in \mathcal{F}_r$ is supported on $\mathbb{B}(x_k,r)$ for some $x_k\in X$.
    \item For every $\psi_{k,r}\in\mathcal{F}_r$ and $x\in X$, $0\leq \psi_{k,r}(x)\leq 1$.
    \item For every $x\in X$ there exists a finite subset $\mathcal{F}_r(x)\subseteq \mathcal{F}_r$ (with cardinality bounded independently of $x$) such that for all $y\in\mathbb{B}(x,r)$
    \begin{equation}
        \sum_{\psi_{k,r}\in\mathcal{F}_r(x)}\psi_{k,r}(y)=1.
    \end{equation}
\end{enumerate}
\end{definition}

\begin{definition}[\textbf{Product assumption}]\label{def:prod}
    We say that a data space $\Xi$ satisfies the product assumption if there exists $A^*\geq 1$ and a family $\{R_{j,k,n}\in \Pi_{A^*n}\}$ such that for every $S>0$,
    \begin{equation}
        \lim_{n\to \infty}n^S\left(\max_{\lambda_k,\lambda_j<n}\norm{\phi_k\phi_j-R_{j,k,n}}_\mathbb{X}\right)=0.
    \end{equation}
If instead for every $n>0$ and $P,Q\in \Pi_n$ we have $PQ\in \Pi_{A^*n}$, then we say that $\Xi$ satisfies the \textbf{strong product assumption}.
\end{definition}

In the most important manifold case, the partition of unity assumption is always satisfied \cite[Chapter~0, Theorem~5.6]{docarmo}. 
It is shown in \cite{geller2011band, modlpmz} that the strong product assumption is satisfied if $\phi_k$'s are eigenfunctions of certain differential equations on a Riemannian manifold and the $\lambda_k$'s are the corresponding eigenvalues.
We do not know of any example where this property does not hold, yet cannot prove that it holds in general. 
Hence, we have listed it as an assumption.

Our characterization of local smoothness (\cite{compbio, heatkernframe, mhaskardata}) is the following.
\begin{theorem}\label{theo:paleywiener}
Let $1\le p\le \infty$, $\gamma>0$, $f\in X^p$, $x_0\in\XX$. We assume the partition of unity and the  product assumption. 
Then the following are equivalent.\\
{\rm (a)} $f\in W_{\gamma,p}(x_0)$.\\
{\rm (b)}  There exists a ball $\BB$ centered at $x_0$ such that
\be\label{eq:op_implies_sobol}
\sup_{n\ge 0}2^{n\gamma}\|f-\sigma_{2^n}(f)\|_{p,\mu^*,\BB} <\infty.
\ee
\end{theorem}

A direct corollary is the following.
\begin{corollary}\label{cor:set_loc_smooth}
Let $1\le p\le \infty$, $\gamma>0$, $f\in X^p$, $A$ be a compact subset of $\XX$. We assume the partition of unity and the product assumption. 
Then the following are equivalent.\\
{\rm (a)} $f\in W_{\gamma,p}(A)$.\\
{\rm (b)} There exists $r>0$ such that
\be\label{eq:op_implies_sobol_bis}
\sup_{n\ge 0}2^{n\gamma}\|f-\sigma_{2^n}(f)\|_{p,\mu^*,\BB(A,r)} <\infty.
\ee
\end{corollary}

%JOINT DATA SPACES SECTION
\section{Joint Data Spaces}\label{bhag:jointspaces}

In order to motivate our definitions in this section, we first consider  a couple of examples.

\begin{example}\label{uda:coifmanhirn}
{\rm
Let $\Xi_j=(\XX_j, d_j, \mu_j^*, \{\lambda_{j,k}\}_{k=0}^\infty, \{\phi_{j,k}\}_{k=0}^\infty)$, $j=1,2$ be two data spaces with exponent $q$. 
We denote the heat kernel in each case by
$$
K_{j,t}(x,y)=\sum_{k=0}^\infty \exp(-\lambda_{j,k}^2t)\phi_{j,k}(x)\phi_{j,k}(y), \qquad j=1,2, \ x,y\in \XX, \ t>0,
$$
In the paper \cite{coifmanhirn}, Coifman and Hirn assumed that $\XX_1=\XX_2=\XX$, $\mu_1^*=\mu_2^*=\mu^*$, and proposed the diffusion distance between points $x_1, x_2$ to be the square root of
$$
K_{1,2t}(x_1,x_2)+K_{2,2t}(x_2,x_2)-2\int_\XX K_{1,t}(x_1,y)K_{2,t}(y,x_2)d\mu^*(y).
$$
Writing, in this example only, 
\be\label{eq:coif_hirn_conn}
A_{j,k}=\int_\XX \phi_{1,j}(y)\phi_{2,k}(y)d\mu^*(y),
\ee
we get 
\be\label{eq:coif_hirn_heat}
\int_\XX K_{1,t}(x_1,y)K_{2,t}(y,x_2)d\mu^*(y)=\sum_{j,k}\exp\left(-(\lambda_{1,j}^2+\lambda_{2,k}^2)t\right)A_{j,k}\phi_{1,j}(x_1)\phi_{2,k}(x_2).
\ee
Furthermore,  the Gaussian upper bound conditions imply that
\be\label{eq:coif_hirn_joint}
\begin{aligned}
\int_\XX K_{1,t}(x_1,y)K_{2,t}(y,x_2)d\mu^*(y)&\ls t^{-q}\int_\XX \exp\left(-c\frac{d_1(x_1,y)^2+d_2(y,x_2)^2}{t}\right)d\mu^*(y)\\
&\ls t^{-q}\exp\left(-c\frac{\left(\min_{y\in\XX}\left(d_1(x_1,y)+d_2(y,x_2)\right)\right)^2}{t}\right).
\end{aligned}
\ee
Writing, in this example only, 
$$
d_{1,2}(x_1,x_2)= \min_{y\in\XX}\left(d_1(x_1,y)+d_2(y,x_2)\right)=d_{2,1}(x_2,x_1),
$$
we observe that for any $x_1,x_1', x_2,x_2'\in\XX$,
$$
\begin{aligned}
d_{1,2}(x_1,x_2)&\le d_{1,2}(x_1',x_2)+d_1(x_1,x_1'),\\
d_{1,2}(x_1,x_2)&\le d_{1,2}(x_1,x_2')+d_1(x_2,x_2').
\end{aligned}
$$
\qed}
\end{example}

\begin{example}\label{uda:jacobispace2}
{\rm
In this example we let $\alpha_i, \beta_i \ge -1/2$ for $i=1,2$ and assume that $a=\abs{\alpha_1-\alpha_2}/2,b=\abs{\beta_1-\beta_2}/2\in \mathbb{N}$. Then we select the following two data spaces as defined in Example~\ref{uda:jacobispace}
\begin{equation}
\Xi_i=([0,\pi],d_i,\frac{1}{\pi}d\theta,\{\lambda_{i,n}\},\{\sqrt{\pi}\phi^{(\alpha_i,\beta_i)}_n\}).
\end{equation}
Since both spaces already have the same distance, we will define a joint distance for the systems accordingly:
\begin{equation}
d_{1,2}(\theta_1,\theta_2)=d_1(\theta_1,\theta_2)=d_2(\theta_1,\theta_2)=|\theta_1-\theta_2|.
\end{equation}
Similar to Example~\ref{uda:coifmanhirn} above, we are considering two data spaces with the same underlying space and measure. However, we now proceed in a different manner. Let us denote
\begin{equation}
\varOmega(\theta)=(1-\cos\theta)^{a}(1+\cos\theta)^{b}.
\end{equation}
Let $\overline{\alpha}=\max(\alpha_1,\alpha_2)$ and $\overline{\beta}=\max(\beta_1,\beta_2)$. Then we define
\begin{equation}
\begin{aligned}
A_{m,n}=&\int_{0}^\pi \phi_m^{(\alpha_1,\beta_1)}(\theta)\phi_n^{(\alpha_2,\beta_2)}(\theta)\varOmega(\theta)d\theta\\
=&\int_0^\pi p_m^{(\alpha_1,\beta_1)}(\cos\theta)p_n^{(\alpha_2,\beta_2)}(\cos\theta)(1-\cos\theta)^{\overline{\alpha}+1/2}(1+\cos\theta)^{\overline{\beta}+1/2}d\theta.
\end{aligned}
\end{equation}
The orthogonality of the Jacobi polynomials tells us that $A_{m,n}=0$ at least when $m>n+2a+2b$ or $n>m+2a+2b$. Furthermore, we have the following two sums
\begin{equation}\label{eq:orthosums}
\sum_{n}A_{m,n}\phi_{n}^{(\alpha_2,\beta_2)}(\theta)=\varOmega(\theta)\phi_{m}^{(\alpha_1,\beta_1)}(\theta),\hspace{5pt}\sum_{m}A_{m,n}\phi_{m}^{(\alpha_1,\beta_1)}(\theta)=\varOmega(\theta)\phi_{n}^{(\alpha_2,\beta_2)}(\theta).
\end{equation}
We define $\ell_{m.n}=\sqrt{\lambda_{1,m}^2+\lambda_{2,n}^2}$, utilize the Gaussian upper bound property for $\Xi_i$ and Equation~\eqref{eq:orthosums} to deduce as in Example~\ref{uda:coifmanhirn} that
\begin{equation}
\begin{aligned}
&\abs{\pi\sum_{m,n}\exp\left(-\ell_{m,n}^2t\right)A_{m,n}\phi_m^{(\alpha_1,\beta_1)}(\theta_1)\phi_n^{(\alpha_2,\beta_2)}(\theta_2)}\\
=&\abs{\int_0^\pi K_{1,t}(\theta_1,\phi)K_{2,t}(\phi,\theta_2)\varOmega(\phi)d\phi}\\
\lesssim& t^{-1}\exp\left(-c\frac{d_{1,2}(\theta_1,\theta_2)^2}{t}\right).
\end{aligned}
\end{equation}
We note (cf.  \cite[Lemma 5.2]{mhaskardata}) that
\begin{equation}\begin{aligned}
&\pi\sum_{m,n:\ell_{m,n}<N}\abs{A_{m,n}\phi_m^{(\alpha_1,\beta_1)}(\theta_1)\phi_n^{(\alpha_2,\beta_2)}(\theta_2)}\\
\lesssim& \norm{\varOmega}_{[0,\pi]}\sum_{m:\lambda_{1,m}<N}\abs{\phi_m^{(\alpha_1,\beta_1)}(\theta_1)\phi_m^{(\alpha_1,\beta_1)}(\theta_2)}\\
\lesssim& N.
\end{aligned}
\end{equation}
\qed}
\end{example}

Motivated by these examples, we now give a series of definitions, culminating in Definition~\ref{def:jointspace}.
First, we define the notion of a joint distance.

\begin{definition}\label{def:jointdist}
Let $\XX_1$, $\XX_2$ be  metric  spaces, with each $\XX_j$ having a metric $d_j$.
 A function $d_{1,2}:\mathbb{X}_1\times \mathbb{X}_2\to [0,\infty)$ will be called a \textbf{joint distance} if the following \textbf{generalized triangle inequalities} are satisfied for $x_1,x_1'\in\mathbb{X}_1$ and $x_2,x_2'\in\mathbb{X}_2$:
\be\label{eq:joint_triangle}
\begin{aligned}
d_{1,2}(x_1,x_2) &\le d_1(x_1,x_1')+d_{1,2}(x_1',x_2),\\
d_{1,2}(x_1,x_2) &\le  d_{1,2}(x_1,x_2')+d_{2}(x_2',x_2).\\
\end{aligned}
\ee
\end{definition}

For convenience of notation we denote $d_{2,1}(x_2,x_1)= d_{1,2}(x_1,x_2)$. 
Then for $r>0$, $x_1\in\mathbb{X}_1$, $x_2\in \mathbb{X}_2$, $A_1\subset \XX_1$, $A_2\subset \XX_2$, we  define
\begin{align*}
    \mathbb{B}_1(x_1,r)&=\{z\in \mathbb{X}_1:d_1(x_1,z)\leq r\}, &\mathbb{B}_2(x_2,r)&=\{z\in \mathbb{X}_2:d_2(x_2,z)\leq r\},\\
    \mathbb{B}_{1,2}(x_1,r)&=\{z\in \mathbb{X}_2:d_{1,2}(x_1,z)\leq r\}, &\mathbb{B}_{2,1}(x_2,r)&=\{z\in \mathbb{X}_1:d_{2,1}(x_2,z)\leq r\},
 \end{align*}
\be\label{eq:jointballs}
\begin{aligned}
    d_{1,2}(A_1,x_2)=&\inf_{x\in A_1\subseteq \mathbb{X}_1}d_{1,2}(x,x_2),\\
    d_{1,2}(x_1,A_2)=d_{2,1}(A_2,x_1)=&\inf_{y\in A_2\subseteq \mathbb{X}_2}d_{2,1}(y,x_1).
\end{aligned}
\ee
We recall here that an infimum over an empty set is defined to be $\infty$. 

\begin{definition}\label{def:jointheatkern}
Let $\mathbf{A}=(A_{j,k})_{j,k=0}^\infty$ (\textbf{connection coefficients}) and $\mathbf{L}=(\ell_{j,k})_{j,k=0}^\infty$ (\textbf{joint eigenvalues}) be bi-infinite matrices. For $x_1\in\XX_1$, $x_2\in\XX_2$, $t>0$, the \textbf{joint heat kernel} is defined formally by
\be\label{eq:jointheatkerndef}
\begin{aligned}
K_t(\Xi_1,\Xi_2;x_1,x_2)=&K_t(\Xi_1,\Xi_2;\mathbf{A}, \mathbf{L}; x_1,x_2)\\
=&\sum_{j,k=0}^\infty \exp(-\ell_{j,k}^2 t)A_{j,k}\phi_{1,j}(x_1)\phi_{2,k}(x_2)\\
=&\lim_{n\to\infty}\sum_{j,k:\ell_{j,k}<n} \exp(-\ell_{j,k}^2 t)A_{j,k}\phi_{1,j}(x_1)\phi_{2,k}(x_2).
\end{aligned}
\ee
\end{definition}
\begin{definition}\label{def:jointspace}
For $m=1,2$, let $\Xi_m=\left(\XX_m, d_m,\mu_m^*, \{\lambda_{m,k}\}_{k=0}^\infty, \{\phi_{m,k}\}_{k=0}^\infty\right) $ be compact data spaces.
With the notation above, assume each $\ell_{j,k}\ge 0$ and that for any $u>0$, the set $\{(j,k): \ell_{j,k}<u\}$ is finite.
A \textbf{joint (compact) data space} $\Xi$ is  a tuple 
$$
(\Xi_1,  \Xi_2, d_{1,2}, \mathbf{A}, \mathbf{L}),
$$
where each of the following conditions is satisfied for some $Q>0$:
\begin{enumerate}
\item (\textbf{Joint regularity})
There exist $q_1, q_2>0$ such that
\be\label{eq:ballmeasurecond1}
\mu_1^*(\BB_{2,1}(x_2,r))\le cr^{q_1}, \quad \mu_2^*(\BB_{1,2}(x_1,r))\le cr^{q_2}, \qquad x_1\in\XX_1,\ x_2\in\XX_2,\  r>0.
\ee
\item (\textbf{Variation bound)}
For each $n>0$,
\begin{equation}\label{eq:variation}
    \sum_{j,k:\ell_{j,k}<n}\abs{A_{j,k}\phi_{1,j}(x_1)\phi_{2,k}(x_2)}\lesssim n^Q, \qquad x_1\in\XX_1, \ x_2\in\XX_2.
\end{equation}
\item (\textbf{Joint Gaussian upper bound})
The limit in \eqref{eq:jointheatkerndef} exists for all $x_1\in\XX_1$, $x_2\in\XX_2$, and
\be\label{eq:jointoffdiagbd}
|K_t(\Xi_1,\Xi_2;x_1,x_2)| \le c_1t^{-c}\exp\left(-c_2\frac{d_{1,2}(x_1,x_2)^2}{t}\right), 
\qquad x_1\in\XX_1,\ x_2\in\XX_2.
\ee
\end{enumerate}
We refer to $(Q, q_1, q_2)$ as the \textbf{(joint) exponents} of the joint data space.
\end{definition}

The kernel corresponding to the one defined in Equation~\eqref{eq:kerndef} is the following, where $H: [0,\infty)\to[0,\infty)$ is a compactly supported function.
\begin{equation}\label{eq:jointkernel}
\Phi_n(H,\Xi_1,\Xi_2; x_1,x_2)=\sum_{j,k=0}^\infty H\left(\frac{\ell_{j,k}}{n}\right)A_{j,k}\phi_{1,j}(x_1)\phi_{2,k}(x_2).
\end{equation}
For $f\in L^1(\mu^*_2)+L^\infty(\mu^*_2)$ and $x_1\in \mathbb{X}_1$, we also define
\begin{equation}\begin{aligned}\label{eq:jointopdef}
\sigma_{n}(H, \Xi_1,\Xi_2;f)(x_1)=&\int_{\mathbb{X}_2}f(x_2)\Phi_n(H,\Xi_1,\Xi_2;x_1,x_2)d\mu^*_2(x_2)\\
=&\sum_{j,k=0}^\infty H\left(\frac{\ell_{j,k}}{n}\right)A_{j,k}\hat{f}(\Xi_2;k)\phi_{1,j}(x_1).
\end{aligned}
\end{equation}

The localization property of the kernels is given in the following proposition.
\begin{proposition}[{\cite[Equation (4.5)]{tauberian}}]\label{prop:jointkernloc}
Let  $S>Q+1$ be an integer, $H:\mathbb{R}\to \mathbb{R}$ be an even, $S$ times continuously differentiable, compactly supported function. 
 Then for every $x_1\in \mathbb{X}_1$, $x_2\in \mathbb{X}_2$, $N>0$,
\begin{equation}\label{eq:jointkernlocest}
| \Phi_N(H,\Xi_1,\Xi_2; x_1,x_2)|\ls \frac{N^Q}{\max(1, (Nd_{1,2}(x_1,x_2))^S)},
\end{equation}
where the constant involved may depend upon $H$,  and $S$, but not on $N$, $x_1$, $x_2$. 
\end{proposition}
In the sequel, we will fix $H$ to be the filter $h$ introduced in Section~\ref{bhag:singlespace}, and will omit its mention from all notations. 
Also, we take $S>\max(Q, q_1, q_2)+1$ to be fixed, although we may put additional conditions on $S$ as needed.
As before, all constants may depend upon $h$ and $S$.

In the remainder of this chapter, we will take $p=\infty$, work only with continuous functions on $\XX_1$ or $\XX_2$, and use $\|f\|_K$ to denote the supremum norm of $f$ on a set $K$. Accordingly, we will omit the index $p$ from the notation for the smoothness classes; e.g., we will write $W_\gamma(\Xi_1;B)$ instead of $W_{\gamma,\infty}(\Xi_1;B)$. The results in the sequel are similar in the case where $p<\infty$ due to the Riesz-Thorin interpolation theorem, but more notationally exhausting without adding any apparent new insights.

We end the section with a condition on the operator defined in Equation~\eqref{eq:jointopdef} that is useful for our purposes.

\begin{definition}[\textbf{Polynomial preservation condition}]\label{def:polypres} Let  $(\Xi_1,\Xi_2,d_{1,2},\mathbf{A},\mathbf{L})$ be a joint data space. We say the \textbf{polynomial preservation condition} is satisfied if there exists some $c^*>0$ with the property that if $P_{n}\in \Pi_{n}(\Xi_2)$, then $\sigma_m(\Xi_1,\Xi_2;P_n)=\sigma_{c^*n}(\Xi_1,\Xi_2;P_n)$ for all $m\geq c^*n$.
\end{definition}

\begin{remark}\label{rem:polypres}
The polynomial preservation condition is satisfied if, for any $n>0$, we have the following inclusion:
\begin{equation}\label{eq:polyprescon}
\{(i,j):A_{i,j}\neq 0,\lambda_{2,j}<n\}\subseteq \{(i,j):\ell_{i,j}\leq c^*n,\lambda_{1,i}<c^*n\}.
\end{equation}
\end{remark}

\begin{example}\label{uda:jacobispace3}
{\rm
We utilize the same notation as in Examples~\ref{uda:jacobispace}~and~\ref{uda:jacobispace2}. We now see, in light of Definition~\ref{def:jointspace}, that $(\Xi_1,\Xi_2,d_{1,2},\mathbf{A},\mathbf{L})$ is a joint data space with exponents $(1,1,1)$. It is clear that both the partition of unity and strong product assumption hold in these spaces. One may also recall that $A_{m,n}=0$ at least whenever $m>n+2a+2b$, so there exists $c^*$ such that Equation~\eqref{eq:polyprescon} is satisfied. As a result, we conclude the polynomial preservation condition holds.
\qed}
\end{example}

%THEORY SECTION
\section{Local Approximation in Joint Data Spaces}\label{bhag:locapprox}
 In this section, we assume a fixed joint data space as in Section~\ref{bhag:jointspaces}.  
 We are interested in the following questions. 
 Suppose $f\in C(\XX_2)$, and we have information about $f$ only in the neighborhood of a compact set $A\subseteq\XX_2$.
 Under what conditions on $f$ and a subset $B\subseteq \XX_1$ can $f$ be lifted to a function $\mathcal{E}(f)$ on $B$? 
 Moreover, how does the local smoothness of $\mathcal{E}(f)$ on $B$ depends upon the local smoothness of $f$ on $A$? We now give definitions for $\mathcal{E}(f),A,B$ for which we have considered these questions.

 \begin{definition}\label{def:liftedfunction} Given $f\in C(\mathbb{X}_2)$, we define the \textbf{lifted function} $\mathcal{E}(f)$ to be the limit
 \begin{equation}\label{eq:ef}
\mathcal{E}(f)=\lim_{n\to\infty}\sigma_n(\Xi_1,\Xi_2;f),
 \end{equation}
 if the limit exists.
 \end{definition}

\begin{definition}\label{def:imageset}
 Let $r,s>0$ and $A\subseteq \mathbb{X}_2$ be a compact subset with the property that there exists a compact subset $B^{-}\subset \mathbb{X}_1$ such that
\begin{equation}\label{eq:rcon}
B^-\subseteq \{x_1: d_{1,2}(x_1,\mathbb{X}_2\setminus A)\geq s+r\}
\end{equation}
for some $r>0$. 
We then define the \textbf{image set} of $A$ by 
\be\label{eq:imageset}
\mathcal{I}(r,s;A)=\BB_1(B^-,s)=\{x_1:d_1(x_1,B^-)\leq s\}.
\ee
If the set $B^{-}$ does not exist, then we define $\mathcal{I}(r,s;A)=\emptyset$.
\end{definition}
\begin{remark}\label{rem:imageset}
{\rm 
In the sequel we fix $r, s>0$ and  a compact subset $A\subseteq \mathbb{X}_2$ such that $B^-$ defined in Equation~\eqref{eq:rcon} is nonempty. We write $B= \mathcal{I}(r,s;A)$. We note that, due to the generalized triangle inequality~\eqref{eq:joint_triangle}, we have the important property
\begin{equation}
B^-\varsubsetneq B\subseteq \{x_1: d_{1,2}(x_1,\mathbb{X}_2\setminus A)\geq r\}.
\end{equation}
\qed}
\end{remark}
We now state our main theorem. Although there is no explicit mention of $B^-$ in the statement of the theorem, Remark~\ref{rem:thmclarify} and Example~\ref{uda:jacobispace4} clarify the benefit of such a construction.

%MAIN THEOREM

\begin{theorem}\label{theo:mainthm}
Let  $(\Xi_1,\Xi_2,d_{1,2},\mathbf{A},\mathbf{L})$ be a joint data space with exponents $(Q, q_1, q_2)$.
We assume that the polynomial preservation condition holds with parameter $c^*$. Suppose $\mathbb{X}_2$ has a partition of unity. \\
{\rm (a)} Let $f\in C(\mathbb{X}_2)$, satisfying 
\begin{equation}\label{eq:thmcon1}
\sum_{m=0}^\infty 2^{m(Q-q_2)}\norm{\sigma_{2^{m+1}}(\Xi_2;f)-\sigma_{2^m}(\Xi_2;f)}_A<\infty.
\end{equation}
Then $\mathcal{E}(f)$ as defined in Definition~\ref{def:liftedfunction} exists on $B$ and for $c^*r2^n\geq 1$ we have
\begin{equation}\label{eq:ef-sigma}
\begin{aligned}
\norm{\mathcal{E}(f)-\sigma_{c^*2^n}(\Xi_1,\Xi_2;f)}_{B}\lesssim& 2^{n(Q-q_2)}\norm{f-\sigma_{2^n}(\Xi_2;f)}_A+\norm{f}_{\mathbb{X}_2}2^{n(Q-S)}r^{q_2-S}\\
&+\sum_{m=n}^\infty 2^{m(Q-q_2)}\norm{\sigma_{2^{m+1}}(\Xi_2;f)-\sigma_{2^m}(\Xi_2;f)}_A.
\end{aligned}
\end{equation}
In particular, if $\Xi_1$ satisfies the strong product assumption, $\mathbb{X}_1$ has a partition of unity, and $\alpha>0$ is given such that $\alpha\ell_{j,k}\geq \lambda_{1,j}$ for all $j,k\in\mathbb{N}$, then $\sigma_{n}(\Xi_1,\Xi_2;f)\in \Pi_{\alpha n}(\Xi_1)$. \\[1ex]
{\rm (b)}
If additionally, $f\in W_{\gamma}(\Xi_2; A)$ with $Q-q_2<\gamma<S-q_2$, then $\mathcal{E}(f)$ is continuous on $B$ and for $\phi\in C^\infty(B)$, we have $\phi \mathcal{E}(f)\in W_{\gamma-Q+q_2}(\Xi_1)$.
\end{theorem}

\begin{remark}\label{rem:thmclarify}
{\rm
Given the assumptions of Theorem~\ref{theo:mainthm}, $\mathcal{E}(f)$ is not guaranteed to be continuous on the entirety of $\mathbb{X}_1$ (or even defined outside of $B$). As a result, in the setting of \ref{theo:mainthm}(b) we cannot say $\mathcal{E}(f)$ belongs to any of the smoothness classes defined in this chapter. However we can still say, for instance, that 
\begin{equation}\label{eq:b-inf}
\inf_{P\in \Pi_{2^n}(\Xi_1)}\norm{\mathcal{E}(f)-P}_{B^-}\lesssim 2^{-n(\gamma-Q+q_2)}
\end{equation}
(this can be seen directly by taking $\phi\in C^\infty(\Xi_1)$ such that $\phi(x)=1$ when $x\in B^-$ and $\phi(x)=0$ when $x\in \mathbb{X}_1\setminus B$). Consequently, if it happens that $\mathcal{E}(f)\in C(\Xi_1)$, then $\mathcal{E}(f)\in W_{\gamma-Q+q_2}(\Xi_1;B^-)$.
\qed}
\end{remark}

\begin{example}\label{uda:jacobispace4}
{\rm We now conclude the running examples from \ref{uda:jacobispace}, \ref{uda:jacobispace2}, and \ref{uda:jacobispace3} by demonstrating how one may utilize Theorem~\ref{theo:mainthm}. We assume the notation given in each of the prior examples listed.
First, we find the image set for
 $A=\mathbb{B}_2(\theta_0,r_0)$ given some $\theta_0\in [0,\pi]$ and $r_0>0$. We let $r=s=r_0/8$ in correspondence to Definition~\ref{def:imageset} and define
\begin{equation}\begin{aligned}
  B^-=&\mathbb{B}_1\left(\theta_0,\frac{3r_0}{4}\right)\\
	 =&\left\{\theta_1\in[0,\pi]: d_1(\theta_1,[0,\pi]\setminus \mathbb{B}_1(\theta_0,r_0))\geq \frac{r_0}{4}\right\}\\
=&\left\{\theta_1\in[0,\pi]: d_{1,2}(\theta_1,[0,\pi]\setminus A)\geq r+s\right\}.
\end{aligned}\end{equation}
Then we can let $B=\mathbb{B}_1\left(\theta_0,\frac{7r_0}{8}\right)=\mathbb{B}_1(B^-,r)$. By Theorem~\ref{theo:mainthm}(a), $f\in C([0,\pi])$ can be lifted to $\mathbb{B}_1\left(\theta_0,\frac{7r_0}{8}\right)$ (where we note that Equation~\eqref{eq:thmcon1} is automatically satisfied due to $Q=q_2=1$). Since $\ell_{m,n}= \lambda_{1,n}$, we have $\sigma_{n}(\Xi_1,\Xi_2;f)\in \Pi_{n}(\Xi_1)$. If we suppose $f\in W_\gamma(\Xi_2;A)$for some $\gamma>0$ (with $h$ chosen so $S$ is sufficiently large), then Theorem~\ref{theo:mainthm}(b) informs us that $\phi\mathcal{E}(f)\in W_{\gamma}(\Xi_1)$ for $\phi\in C^\infty(B)$. Lastly, as a result of Equation~\eqref{eq:b-inf}, we can conclude
\begin{equation}
\inf_{P\in \Pi_{2^n}(\Xi_1)}\norm{\mathcal{E}(f)-P}_{\mathbb{B}_1\left(\theta_0,\frac{3r_0}{r}\right)}\lesssim 2^{-n\gamma}.
\end{equation}
\qed}
\end{example}

%PROOFS
\section{Proofs}\label{bhag:proofs}
 In this section, we give a proof of Theorem~\ref{theo:mainthm} after proving some preperatory results. We assume that $(\Xi_1,\Xi_2,d_{1,2}, \mathbf{A}, \mathbf{L})$ is a joint data space with exponents $Q, q_1, q_2$.
 
\begin{lemma}\label{lemma:opbdlemma}
Let $x_1\in\XX_1$, $r>0$.
We have
\begin{equation}\label{eq:lockeraway}
    \int_{\mathbb{X}_2\setminus \mathbb{B}_{1,2}(x_1,r)}\abs{\Phi_n(\Xi_1,\Xi_2;x_1,x_2)}d\mu_2^*(x_2)\lesssim n^{Q-q_2}(\max(1, nr))^{q_2-S}.
\end{equation}
In particular, 
\begin{equation}\label{eq:globker}
    \int_{\mathbb{X}_2}\abs{\Phi_n(\Xi_1,\Xi_2;x_1,x_2)}d\mu_2^*(x_2)\lesssim n^{Q-q_2}.
\end{equation}
\end{lemma}

\begin{proof}\ % of Lemma~\ref{lemma:opbdlemma}
In this proof only, define
\begin{equation}
    A_0=\mathbb{B}_{1,2}(x_1,r),\qquad A_m=\mathbb{B}_{1,2}(x_1,r2^m)\setminus \mathbb{B}_{1,2}(x_1,r2^{m-1})\text{ for all $m\in \mathbb{N}$.}
\end{equation}
Then the joint regularity condition \eqref{eq:ballmeasurecond1} implies
$
    \mu^*_2(A_m)\lesssim (r2^m)^{q_2},
$
for each $m$. 
We can also see by definition that when $x\in A_m$, then $d_{1,2}(x_1,x)>r2^{m-1}$. 
Since $S>q_2$,  we deduce that for $r n\ge 1$, 
\be\begin{aligned}\label{eq:pf1eqn1}
    \int_{\mathbb{X}_2\setminus A_0}\abs{\Phi_n(\Xi_1,\Xi_2; x_1,x_2)}d\mu_2^*(x_2)\lesssim& \sum_{m=1}^\infty \frac{n^Q\mu_2^*(A_m)}{(rn2^{m-1})^S}\\
    \lesssim& r^{q_2-S}n^{Q-S}\sum_{m=1}^\infty 2^{m(q_2-S)}\\
    \lesssim& n^{Q-q_2}(nr)^{q_2-S}.
\end{aligned}\ee
This completes the proof of \eqref{eq:lockeraway} when $nr\ge 1$.
The joint regularity condition and Proposition~\ref{prop:jointkernloc} show further that 
\be\label{eq:pf1eqn2}
\int_{A_0} \abs{\Phi_n(\Xi_1,\Xi_2;x_1,x_2)}d\mu_2^*(x_2)\ls n^Q\mu_2^*(A_0)\ls n^Qr^{q_2}=n^{Q-q_2}(nr)^{q_2}.
\ee
We use $r=1/n$ in the estimates \eqref{eq:pf1eqn1} and \eqref{eq:pf1eqn2} and add the estimates to arrive at both \eqref{eq:globker} and the case $r\le 1/n$ of \eqref{eq:lockeraway}.
\end{proof}

The next lemma gives a local bound on the kernels $\sigma_n$ defined in \eqref{eq:jointopdef}.

\begin{lemma}\label{lemma:jointop_loc_lemma}
Let $A$ and $B$ be as defined in Remark~\ref{rem:imageset}.
For a continuous $f:A\to\RR$, we have
\be\label{eq:joint_loc_opbd}
\|\sigma_n(\Xi_1,\Xi_2;f)\|_B \ls n^{Q-q_2}\left\{\|f\|_A +\|f\|_{\XX_2}(\max(1, nr))^{q_2-S}\right\}.
\ee
\end{lemma}

\begin{proof}\ %of Lemma~\ref{lemma:jointop_loc_lemma}
Let $x_1\in B$.
In view of the joint triangle inequality \eqref{eq:joint_triangle}, we have $d_{1,2}(x_1,x_2)\ge r$ for all $x_2\in \XX_2\setminus A$.
Therefore, Lemma~\ref{lemma:opbdlemma} shows that
\begin{equation}
\begin{aligned}
|\sigma_n(\Xi_1,\Xi_2;f)(x_1)|\le& \int_{\XX_2} |f(x_2)\Phi_n(\Xi_1,\Xi_2;x_1,x_2)|d\mu_2^*(x_2)\\
=& \int_A |f(x_2)\Phi_n(\Xi_1,\Xi_2;x_1,x_2)|d\mu_2^*(x_2)\\
&+\int_{\XX_2\setminus A}|f(x_2)\Phi_n(\Xi_1,\Xi_2;x_1,x_2)|d\mu_2^*(x_2) \\
\ls& n^{Q-q_2}\|f\|_A +\|f\|_{\XX_2}\int_{\XX_2\setminus \BB_{1,2}(x_1,r)}|\Phi_n(\Xi_1,\Xi_2;x_1,x_2)|d\mu_2^*(x_2)\\
\ls& n^{Q-q_2}\left\{\|f\|_A +\|f\|_{\XX_2}(\max(1, nr))^{q_2-S}\right\}.
\end{aligned}
\end{equation}
\end{proof}

\begin{lemma}\label{lemma:E}
We assume  the polynomial preservation condition with parameter $c^*$. Let $f\in C(\XX_2)$ satisfy \eqref{eq:thmcon1}.
Then
\begin{equation}
    \mathcal{E}(f)= \lim_{n\to \infty}\sigma_{2^n}(\Xi_1,\Xi_2;f)
\end{equation}
exists on $B$. 
Furthermore, when $c^*2^n>1/r$, we have
\begin{equation}
\begin{aligned}
    &\norm{\mathcal{E}(f)-\mathcal{E}(\sigma_{2^n}(\Xi_2;f))}_B\\
    \lesssim& \sum_{m=n}^\infty 2^{m(Q-q_2)}\norm{f-\sigma_{2^m}(\Xi_2;f)}_A+\norm{f}_{\mathbb{X}_2}n^{Q-S}r^{q_2-S}.
\end{aligned}
\end{equation}
\end{lemma}

\begin{proof}\ %of Lemma~\ref{lemma:E}
In this proof only we denote $P_{n}=\sigma_{n}(\Xi_2;f)$. 
Since $P_n\in \Pi_{n}(\Xi_2)$, the condition \eqref{eq:polyprescon} implies that
\begin{equation}
\mathcal{E}(P_n)=\sigma_{c^*n}(\Xi_1,\Xi_2;P_n)=\lim_{k\to \infty}\sigma_{k}(\Xi_1,\Xi_2;P_k)
\end{equation}
is defined on $\mathbb{X}_1$.
Theorem~\ref{theo:goodapprox} and Lemma~\ref{lemma:jointop_loc_lemma} then imply that
\begin{equation}
\begin{aligned}
&\norm{\mathcal{E}(P_{2^{m+1}})-\mathcal{E}(P_{2^{m}})}_B\\
=&
\norm{\sigma_{c^*2^{m+1}}(\Xi_1,\Xi_2;P_{2^{m+1}})-\sigma_{c^*2^{m+1}}(\Xi_1,\Xi_2; P_{2^m})}_B\\
\lesssim& 2^{m(Q-q_2)}(\norm{P_{2^{m+1}}-P_{2^m}}_A+\norm{f}_{\mathbb{X}_2}(\max(1, 2^mr))^{q_2-S}).
\end{aligned}
\end{equation}
We conclude that
\begin{equation}
\begin{aligned}
&\norm{\mathcal{E}(P_1)}_B+\sum_{m=0}^\infty \norm{\mathcal{E}(P_{2^{m+1}})-\mathcal{E}(P_{2^m})}_B\\
\lesssim&\norm{P_1}+\sum_{m=0}^\infty 2^{m(Q-q_2)}\norm{P_{2^{m+1}}-P_{2^m}}_A\\
&+\norm{f}_{\mathbb{X}_2}\left(\sum_{c^*2^m\leq 1/r} 2^{m(Q-q_2)}+r^{q_2-S}\sum_{c^*2^m>1/r}2^{m(Q-S)}\right)<\infty.
 \end{aligned}
\end{equation}
Thus, 
\begin{equation}
\mathcal{E}(f)=\mathcal{E}(P_1)+\sum_{m=0}^\infty \left(\mathcal{E}(P_{2^{m+1}})-\mathcal{E}(P_{2^m})\right)
\end{equation}
is defined on $B$. In particular, when $c^*2^n\geq 1/r$ it follows
\begin{equation}
    \norm{\mathcal{E}(f)-\mathcal{E}(P_{2^n})}_B\leq \sum_{m=n}^\infty 2^{m(Q-q_2)}\norm{P_{2^{m+1}}-P_{2^m}}_A+\norm{f}_{\mathbb{X}_2}2^{n(Q-S)}r^{q_2-S}.
\end{equation}
\end{proof}

%MAIN THEOREM PROOF

Now we give the proof of Theorem~\ref{theo:mainthm}.

\begin{proof}\ %of Theorem~\ref{theo:mainthm}
In this proof only denote $P_n=\sigma_{n}(\Xi_2;f)\in\Pi_n(\Xi_2)$. 
We can deduce from Theorem~\ref{theo:goodapprox} and Lemma~\ref{lemma:jointop_loc_lemma} that for $c^*r2^n\geq 1$,
\begin{equation}
\begin{aligned}\label{eq:pf2eqn1}
&\norm{\sigma_{c^*2^n}(\Xi_1,\Xi_2;f)-\sigma_{c^*2^{n}}(\Xi_1,\Xi_2;P_{2^n})}_{B}\\
\lesssim& 2^{n(Q-q_2)}(\norm{f-P_{2^n}}_A+\norm{f-P_{2^n}}_{\mathbb{X}_2}2^{n(q_2-S)}r^{q_2-S})\\
\lesssim& 2^{n(Q-q_2)}(\norm{f-P_{2^n}}_A+\norm{f}_{\mathbb{X}_2}2^{n(q_2-S)}r^{q_2-S}).
\end{aligned}
\end{equation}
The polynomial preservation condition (Definition~\ref{def:polypres}) gives us that
\begin{equation}
\norm{\sigma_{c^*2^n}(\Xi_1,\Xi_2;P_{2^n})-\mathcal{E}(P_{2^n})}_{B}=0.
\end{equation}
Then utilizing Equation~\eqref{eq:pf2eqn1} and Lemma~\ref{lemma:E}, we see
\begin{equation}\begin{aligned}
&\norm{\sigma_{c^*2^n}(\Xi_1,\Xi_2;f)-\mathcal{E}(f)}_{B}\\
\leq&\norm{\sigma_{c^*2^n}(\Xi_1,\Xi_2;f)-\sigma_{c^*2^n}(\Xi_1,\Xi_2;P_{2^n})}_{B}\\
&+\norm{\sigma_{c^*2^n}(\Xi_1,\Xi_2;P_{2^n})-\mathcal{E}(P_{2^n})}_{B}+\norm{\mathcal{E}(P_{2^n})-\mathcal{E}(f)}_{B}\\
\lesssim& 2^{n(Q-q_2)}\norm{f-P_{2^n}}_A+\sum_{m=n}^\infty 2^{m(Q-q_2)}\norm{P_{2^{m+1}}-P_{2^m}}_A+\norm{f}_{\mathbb{X}_2}2^{n(Q-S)}r^{q_2-S}.
\end{aligned}
\end{equation}
This proves Equation~\eqref{eq:ef-sigma}.

In particular, when $\alpha\ell_{j,k}\geq \lambda_{1,j}$ and $\alpha>0$, the only $\phi_{1,j}(x_1)$ with non-zero coefficients in Equation~\eqref{eq:jointopdef} are those where $\ell_{j,k}<n$, which implies $\lambda_{1,j}<\alpha n$ and further that $\sigma_{n}(\Xi_1,\Xi_2;f)\in \Pi_{\alpha n}(\Xi_1)$.
This completes the proof of part (a).

In the proof of part (b),  we may assume without loss of generality that $\|f\|_{W_{\gamma}(\Xi_2;A)}+\norm{f}_{\mathbb{X}_2}=1$. We can see from Corollary~\ref{cor:set_loc_smooth} that for each $m$
\begin{equation}
    \norm{P_{2^{m+1}}-P_{2^m}}_A\leq \norm{P_{2^{m+1}}-f}_A+\norm{P_{2^{m}}-f}_A\lesssim 2^{-m\gamma},
\end{equation}
which implies that whenever $Q-q_2<\gamma$ we have
\begin{equation}
    \sum_{m=n}^\infty 2^{m(Q-q_2)}\norm{P_{2^{m+1}}-P_{2^m}}_A\lesssim 2^{n(Q-q_2-\gamma)}.
\end{equation}
Further, the assumption that $\gamma<S-q_2$ gives us
\begin{equation}
    2^{n(Q-S)}\lesssim 2^{n(Q-q_2-\gamma)}.
\end{equation}
Since $f\in W_\gamma(\Xi_2;A)$, we have from Corollary~\ref{cor:set_loc_smooth} that
\begin{equation}
\norm{f-P_{2^n}}_A\lesssim 2^{-n\gamma}.
\end{equation}
Using Equation~\eqref{eq:ef-sigma} from part (a), we see
\begin{equation}\label{eq:ef-sigmagamma}
\norm{\mathcal{E}(f)-\sigma_{c^*2^n}(\Xi_1,\Xi_2;f)}_{B}\lesssim (1+r^{q_2-S}\norm{f}_{\mathbb{X}_2})2^{n(Q-q_2-\gamma)}.
\end{equation}
Thus, $\{\sigma_{c^*2^n}(\Xi_1,\Xi_2;f)\}$ is a sequence of continuous functions converging uniformly to $\mathcal{E}(f)$ on $B$, so $\mathcal{E}(f)$ itself is continuous on $B$.
Let us define $R_{c^*\alpha 2^n}\in \Pi_{c^*\alpha 2^n}$ for each n such that $\norm{R_{c^*\alpha2^n}-\phi}_{\mathbb{X}_1}\lesssim 2^{-n\gamma}$. Theorem~\ref{theo:goodapprox} and the strong product assumption (Definition~\ref{def:prod}) allow us to write
\begin{equation}\label{eq:goodapproxapp}
    \sigma_{c^*A^*\alpha 2^{n+1}}(\Xi_1;R_{c^*\alpha2^n}\sigma_{c^*2^n}(\Xi_1,\Xi_2;f))=R_{c^*\alpha2^n}\sigma_{c^*2^n}(\Xi_1,\Xi_2;f).
\end{equation}
Using Equations~\eqref{eq:globker}~and~\eqref{eq:goodapproxapp}, Theorem~\ref{theo:goodapprox}, and the fact $\phi$ is supported on $B$, we can deduce
\begin{equation}\label{eq:phiRnpolybound}
\begin{aligned}
&\norm{\sigma_{c^*A^*\alpha 2^{n+1}}\big(\Xi_1;R_{c^*\alpha 2^n}\sigma_{c^*2^n}(\Xi_1,\Xi_2;f)-\phi\mathcal{E}(f)\big)}_{\mathbb{X}_1}\\
\lesssim& \norm{R_{c^*\alpha 2^n}\sigma_{c^*2^n}(\Xi_1,\Xi_2;f)-\phi\mathcal{E}(f)}_{\mathbb{X}_1}\\
\lesssim& \norm{\phi\mathcal{E}(f)-\phi\sigma_{c^*2^n}(\Xi_1,\Xi_2;f)}_{\mathbb{X}_1}+\norm{R_{c^*\alpha 2^n}-\phi}_{\mathbb{X}_1}\norm{\sigma_{c^*2^n}(\Xi_1,\Xi_2;f)}_{\mathbb{X}_1}\\
\lesssim&\norm{\mathcal{E}(f)-\sigma_{c^*2^n}(\Xi_1,\Xi_2;f)}_B+2^{n(Q-q_2-\gamma)}\norm{f}_{\mathbb{X}_2}.
\end{aligned}
\end{equation}
In view of Equations~\eqref{eq:ef-sigmagamma}~and~\eqref{eq:phiRnpolybound}, we can conclude that
\begin{equation}\begin{aligned}
&E_{c^*A^*\alpha 2^{n+1}}(\Xi_1,\phi\mathcal{E}(f))\\
    \lesssim&\norm{\phi \mathcal{E}(f)-\sigma_{c^*A^*\alpha 2^{n+1}}(\Xi_1,\phi \mathcal{E}(f))}_{\mathbb{X}_1}\\
    \leq&\norm{\phi \mathcal{E}(f)-R_{c^*\alpha 2^n}\sigma_{c^*2^n}(\Xi_1,\Xi_2;f)}_{\mathbb{X}_1}\\
    &+\norm{\sigma_{c^*A^*\alpha 2^{n+1}}\big(\Xi_1;R_{c^*\alpha 2^n}\sigma_{c^*2^n}(\Xi_1,\Xi_2;f)-\phi\mathcal{E}(f)\big)}_{\mathbb{X}_1}\\
    \lesssim& \norm{\mathcal{E}(f)-\sigma_{c^*2^n}(\Xi_1,\Xi_2;f)}_{B}+\norm{f}_{\mathbb{X}_2}2^{n(Q-q_2-\gamma)}\\
\lesssim&(1+\norm{f}_{\mathbb{X}_2}(1+r^{q_2-S})) 2^{n(Q-q_2-\gamma)}.
\end{aligned}\end{equation}
Thus, $\phi\mathcal{E}(f)\in W_{\gamma-Q+q_2}(\Xi_1)$, completing the proof of part (b).
\end{proof}

\chapter{Classification}
\label{ch:classification}

%\title{A signal separation view of classification}
%\author{H. N. Mhaskar\thanks{
%Institute of Mathematical Sciences, Claremont Graduate University, Claremont, CA 91711. 
%\textsf{email:} hrushikesh.mhaskar@cgu.edu.
%The research is  supported in part by  ONR grants N00014-23-1-2394, N00014-23-1-2790.} \and Ryan O'Dowd\thanks{Institute of Mathematical Sciences, Claremont Graduate University, Claremont, CA 91711. 
%\textsf{email:} ryan.o'dowd@cgu.edu.}}
%\date{\today}

The content in this chapter is sourced from our paper pending publication titled ``A signal separation view of classification" \cite{mhaskarodowdclassification}.

\section{Introduction}
\label{sec:intro}
A fundamental problem in machine learning is the following. Let $\{(x_j,y_j)\}_{j=1}^M$ be random samples from an \textbf{unknown} probability distribution $\tau$. 
The problem is to approximate the conditional expectation $f(x)=\mathbb{E}_\tau(y|x)$ as a function of $x$. 
Naturally, there is a huge amount of literature studying function approximation by commonly used tools in machine learning such as neural and kernel based networks. 
For example, the universal approximation theorem gives conditions under which a neural network can approximate an arbitrary \textbf{continuous} function on an arbitrary compact subset of the ambient Euclidean space.
The estimation of the complexity of the approximation process typically assumes some smoothness conditions on $f$, examples of which include, the number of derivatives, membership in various classes such as Besov spaces, Barron spaces, variation spaces, etc.

A very important problem is one of classification.
Here the values of $y_j$ can take only finitely many (say $K$) values, known as the class labels. 
In this case, it is fruitful to approximate the classification function, defined by $f(x)=\argmax_k \mathsf{Prob}(k|x)$ \cite{murphy2022probabilistic}. 
Obviously, this function is only piecewise continuous, so that the universal approximation theorem does not apply directly. 
In the case when the classes are supported on well-separated sets, one may refer to extension theorems such as Stein extension theorems \cite[Chapter 6]{stein} in order to justify the use of the various approximation theorems to this problem.

While these arguments are sufficient for pure existence theorems, they also create difficulties in an actual implementation, in particular, because these extensions are not easy to construct. 
In fact, this would be impossible if the classes are not well-separated, and might even overlap.
Even if the classes are well-separated, and  each class represents a Euclidean domain, any lack of smoothness in the boundaries of these domains is a problem.
Some recent efforts, for example,  by Petersen and Voigtl\"ander  \cite{petersen-net} deal with the question of accuracy in approximation when the class boundaries are not smooth. 
However, a popular assumption in the last twenty years or so is that the data is distributed according a probability measure supported on a low dimensional manifold of a high dimensional ambient Euclidean space.
In this case, the classes have boundary of measure $0$ with respect to the Lebesgue measure on the ambient space.
Finally, approximation algorithms, especially with deep networks, utilize a great deal of labeled data.

In this chapter, we propose a different approach as advocated in \cite{cloningercluster}.
Thus, we do not assume that $\mathsf{Prob}(k|x)$ is a function, but assume instead that the points $x_j$ in class $k$ comprise the support of a probability measure $\mu_k$. 
The marginal distribution $\mu$ 
\yadi{${\mu}$}{Data distribution of the input} 
of $\tau$ along $x$ is then a convex combination of the measures $\mu_k$\yadi{$\mu_k$}{Probability distribution for class $k$}.
The fundamental idea is to determine the \textbf{supports} of the measures $\mu_k$ rather than approximating $\mu_k$'s themselves\footnote{If $\nu$ is a positive measure on a metric space $\MM$, we define the support of  any positive measure $\nu$ by $\mathsf{supp}(\nu)=\{x\in\MM, \nu(\BB(x,r))>0 \mbox{ for all } r>0\}$, where $\BB(x,r)$ is the ball of radius $r$ centered at $x$.}.  \yadi{$\mathsf{supp}$}{Support of a measure}
This is done in an unsupervised manner, based only on the $x_j$'s with no label information.
Having done so, we may then query an oracle for the label of one point in the support of each measure, which is then necessarily the label for every other point in the support. 
Thus, we aim to achieve in theory a  perfect classification using a minimal amount of judiciously chosen labeled data.

In order to address the problem of overlapping classes, we take a hierarchical multiscale approach motivated by a paper \cite{dasgupta2010} of Chaudhury and Dasgupta.
Thus, for each value $\eta$ of the minimal separation among classes, we assume that the support of $\mu$ is a disjoint union of $K_\eta$ subsets, each representing one of $K_\eta$ classes, leaving an extra set, representing the overlapping region. 
When we decrease $\eta$, we may eventually capture all the classes, leaving only a negligible overlapping region (ideally with $\mu$-probability $0$).

In \cite{cloningercluster}, it is seen that the problem is analogous to the problem of point source signal separation. 
If each $\mu_k$ were a Dirac delta measure supported at say $\omega_k$, the point source signal separation problem is to find these point sources from finitely many observations of the Fourier transform of $\mu$. 
In the classification problem we do not have point sources and the information comprises samples from $\mu$ rather than its Fourier transform.
Nevertheless, we observed in \cite{cloningercluster} that the techniques developed for the point source signal separation problem can be adapted to the classification problem viewed as the the problem of separation of the supports of $\mu_k$.
In that paper, it was assumed only that the data is supported on a compact subset of a Eulidean space, and used a specially designed localized kernel based on Hermite polynomials \cite{chuigaussian} for this purpose. 
Since Hermite polynomials are intrinsically defined on the whole Euclidean space, this creates both numerical and theoretical difficulties.
In this chapter, we allow the data to come from an arbitrary compact metric space, and use localized trigonometric polynomial kernels instead.
We feel that this leads to a more satisfactory theory, although one of the accomplishments of this chapter is to resolve the technical difficulties required to achieve this generalization.

To summarize, the main accomplishments of this chapter are:
\begin{itemize}
\item We provide a unified approach to signal separation problems and classification problems.
\item We deal with the classification of data coming from an arbitrary metric space with no further structure, such as the manifold structure.
\item Our results suggest a multiscale approach which does not assume any constraints on class boundaries, including that the classes not overlap.
\item In theory, the number of classes at each scale is an output of the theorem rather than a prior assumption.
\item We develop an algorithm to illustrate the theory, especially in the context of active learning on hyperspectral imaging data.
\end{itemize}

Our work belongs in the general theory of active learning. 
In Section~\ref{sec:related}, we review some literature in this area which is somewhat related to the present work.
In Section~\ref{sec:signals}, we give a brief discussion of the point source signal separation problem and the use of localized trigonometric polynomial kernels to solve it.
In Section~\ref{sec:background}, we describe the background needed to formulate our theorems, which are given in Section~\ref{sec:mainresults}.
The algorithm MASC to implement these results in practice is given in Section~\ref{sec:alg}, and illustrated in the context of a simulated circle and ellipse data set, a document dataset, and two hyperspectral datasets.
The proofs of the results in Section~\ref{sec:mainresults} are given in Section~\ref{sec:proofs}.
%Our algorithm requires some function approximation on submanifolds of a Euclidean sphere.
%This theory is summarized in the Appendix.

\section{Related works}\label{sec:related}

Perhaps the most relevant work to this chapter is that of \cite{cloningercluster}. 
That paper also outlines the theory and an algorithm for a classification procedure using a thresholding set based on a localized kernel. There are three major improvements we have made relative to that work in this chapter:
\ben
\item We have constructed the kernel in this chapter in terms of trigonometric functions, whereas in \cite{cloningercluster} the kernel was constructed from Hermite polynomials. The trigonometric kernel is much faster in implementations for two reasons: 1) each individual polynomial is extremely quick to compute and 2) the trigonometric kernel deals only with trigonometric polynomials up to degree $n$, whereas the Hermite polynomial based kernel needs polynomials up to degree $n^2$ to achieve the same support estimation bounds.

\item This chapter deals with arbitrary compact metric spaces (allowing for a rescaling of the data so that the maximum distance between values is $\leq \pi$), whereas \cite{cloningercluster} dealt with compact subsets of the Euclidean space and had a requirement on the degree of the kernel dependent upon the diameter of the data in terms of Euclidean distance.

\item In \cite{cloningercluster}, an algorithm known as Cautious Active Clustering (CAC) was developed. In this chapter, we present a new algorithm with several implementation advantages over CAC. 
We discuss this topic in more depth in Section~\ref{subsec:CACcomparison}.
\een

Another pair of related works is that of the Learning by Active Nonlinear Diffusion (LAND) and Learning by Evolving Nonlinear Diffusion (LEND) algorithms \cite{murphy-land,murphy-lend}. 
Like the present work, these algorithms use a kernel-based density estimation when deciding points to query. However, LAND and LEND both use a Gaussian kernel applied on $k$ neighbors for the density estimation and weight it by a diffusion value. 
Then, the queried points are simply those with the highest of the combined weights. 
The diffusion value corresponds to a minimal diffusion distance among points with a higher density estimation. 
For the point with the maximal density estimation, a maximal diffusion distance among other data points is taken as the weight. 
This extra weighting procedure is absent from our theory and algorithm, which uses an  estimation approach based purely on a localized kernel to decide on points to query. In our algorithm, we take a multiscale approach and decide on query points at each level instead of a global listing of the data points.

In \cite{nowak-network}, an active learning approach using neural networks is developed. This work focuses on binary classification and developing models using a neural network framework such that a sufficient number of queries will achieve a desired accuracy. 

A study of two types of uncertainty in active learning problems is discussed in \cite{sharma-active}. The two critical types of uncertainty are 1) a data point is likely to belong to multiple labels, 2) a data point is not likely to belong to any label. Our work also seeks to distinguish between data points which are uncertain in the second sense, using a graph construction approach and potentially also a thresholding set for high-density points. When our algorithm encounters points which are uncertain in the first sense, it elects not to assign a label right away, instead coming back to it once the ``confident" points have been classified.

Our method and algorithm are meant to work on general data sampled from compact metric spaces. One difficulty that algorithms may face is the presence of highly imbalanced data (i.e. where some class labels dominate over others in a data set). The problem of tackling this difficulty is studied in \cite{tharwat-active}, where an approach to querying imbalanced data using a balance of two principles is employed: exploration and exploitation. During the exploration phase, the algorithm seeks out points to query in low-sought regions. During the exploitation phase, the algorithm seeks to query points in the most critical explored regions. Our algorithm works in a different fashion, by querying points which we believe to be in high-density portions of a label's support and extending the label to nearby points until it ``bumps" against points which may belong to another label.

We list the survey by Tharwat and Schenck \cite{tharwat-survey} as a resourceful survey of recent developments in active learning.

\printnomenclature

\section{Point source signal separation}\label{sec:signals}

The problem of signal separation goes back to early work of de Prony \cite{prony}, and can be stated as: estimate the coefficients $a_k$ and locations $\omega_k$ constituting $\mu=\sum_{k=1}^K a_k\delta_{\omega_k}$, from observations of the form
\be\label{eq:signal}
\hat{\mu}(x)=\hat{\mu}(x)=\sum_{k} a_k e^{-i\omega_k x}, \qquad x\in\RR.
\ee
There is much literature on methods to approach this problem, and we cite \cite{plonkafourier} as a text one can use to familiarize themselves with the topic. If we assume $\omega_k=k\Delta$ for some $\Delta\in\mathbb{R}^+$ and allow measurements for any $x\in [-\Omega,\Omega]$ for some $\Omega\in \mathbb{R}^+$, then recovery is possible so long as we are above the Rayleigh threshold, i.e. $\Omega\geq \pi/\Delta$ \cite{donoho-superres}. Since  The case where this threshold is not satisfied is known as super-resolution. Much further research has gone on to investigate the super-resolution problem, such as \cite{batenkov-superres,candes-superres,li-superres}.

We now introduce a particular method of interest for signal separation from \cite{loctrigwave} and further developed in \cite{mhaskar-kitimoon-raj}. The method takes the following approach to estimate the coefficients and locations of $\mu$, without the assumption that the $\omega_k$'s should be at grid points, and the additional restriction that only \textbf{finitely many} integer values of $x$ are allowed. 
We start with the \textbf{trigonometric moments} of $\mu$:
$$
\hat{\mu}(\ell)=\sum_{k} a_k e^{-i\omega_k \ell}, \qquad |\ell|<n,
$$
where $n\ge 1$ is an integer.
Clearly, the quantities $\hat{\mu}(\ell)$ remain the same if any $\omega_k$ is replaced by $\omega_k$ plus an integer multiple of $2\pi$. Therefore, this problem is properly treated as the recuperation of a periodic measure $\mu$ from its Fourier coefficients rather than the recuperation of a measure defined on $\RR$ from its Fourier transform.
Accordingly, we define the quotient space $\mathbb{T}=\mathbb{R}/(2\pi\ZZ)$, and denote in this context, $|x-y|=|(x-y)\mbox{ mod } 2\pi|$.\yadi{$\TT$}{Quotient space $\RR/(2\pi\ZZ)$, sometimes referred to as the unit circle}
Here and in the rest of this chapter, we consider a smooth band pass filter $h$; i.e., an even function $h\in C^\infty(\mathbb{R})$ such that $h(u)=1$ for $|u|\leq 1/2$ and $h(u)=0$ for $|u|\geq 1$. \yadi{$h$}{Smooth, bandpass filter}
We then define
\yadi{$\sigma_n$}{Reconstruction operator in different settings, \eqref{eq:sigmaintro} and \eqref{eq:sigmametricdef}}
\be\label{eq:sigmaintro}
\sigma_n(\mu)(x)\coloneqq \sum_{|\ell|<n} h\left(\frac{\ell}{n}\right)\hat{\mu}(\ell)e^{i\ell x}, \qquad x\in \mathbb{T}.
\ee
With the kernel defined by \yadi{$\Phi_n$}{Localized trigonometric polynomial kernel, \eqref{eq:trigkerndef}}
\be\label{eq:trigkerndef}
\Phi_n(t)\coloneqq \sum_{|k|<n} h\left(\frac{k}{n}\right)e^{ikt}, \qquad t\in \mathbb{T},
\ee
it is easy to deduce that
\be\label{eq:trigconv}
\sigma_n(\mu)(x)=\frac{1}{2\pi}\int_\TT \Phi_n(x-t)d\mu(t)=\sum_{k} a_k \Phi_n(x-\omega_k).
\ee
A key property of $\Phi_n$ is the \textbf{localization property} (cf. \cite{mhaskar-prestin-filbir, kitimoon2025localized}, where the notation is different): For any integer $S\ge 3$,
\be\label{eq:kernloc}
|\Phi_n(t)|\le 7\sqrt{\frac{\pi}{2}}\left\{\int_{-1}^1 |h^{(S)}(t)|dt \right\}\frac{n}{\max(1, (n|t|)^S)}.
\ee
Together with the fact that $h=1$ on $[-1/2,1/2]$, this implies that $\Phi_n$ is approximately a Dirac delta supported at $0$; in particular,
$$
\sigma_n(\mu)(x)\approx \sum_k a_k\delta_{\omega_k}(x).
$$
The theoretical details of this sentiment are described more rigorously in \cite{mhaskar-prestin-filbir, kitimoon2025localized}. 
Here, we only give two examples to illustrate.

\begin{example}\label{uda:pointsource}
{\rm
We consider the measure 
\be\label{eq:udapointsource}
\mu=5\delta_{-1}+30\delta_{2}+20\delta_{2.05},
\ee
so that the data is
\be\label{eq:udamoments}
\hat{\mu}(\ell)=5\exp(i\ell)+30\exp(-2i\ell)+20\exp(-2.05i\ell), \qquad |\ell|<n.
\ee
In Figure~\ref{fig:pointsource}, we show the graphs of the ``power spectrum'' $|\sigma_n(\mu)(x)|$ for $n=64, 256$.

\begin{figure}[h]
\begin{center}
\begin{minipage}{0.3\textwidth}
\includegraphics[width=\textwidth]{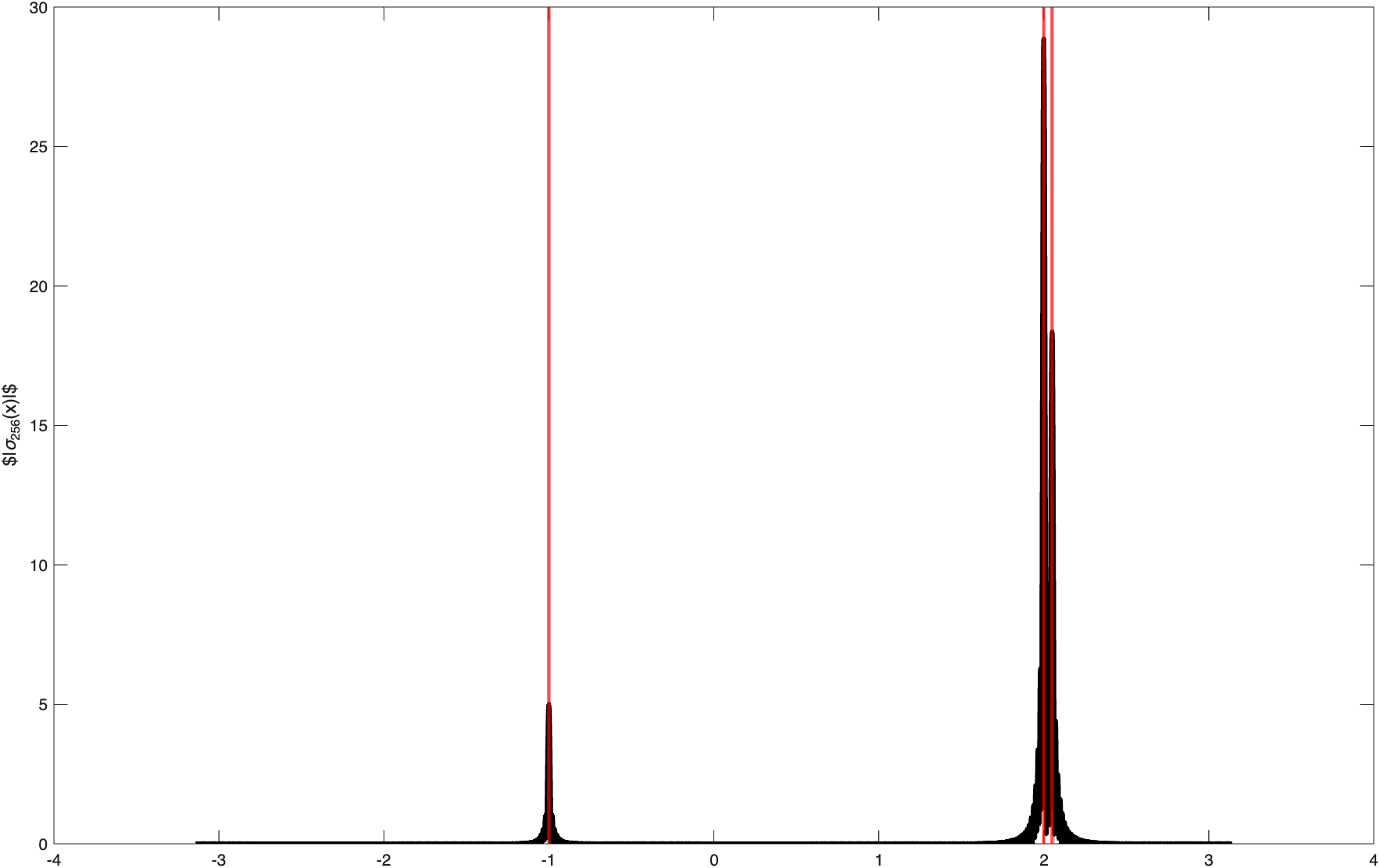} 
\end{minipage}
\begin{minipage}{0.3\textwidth}
\includegraphics[width=\textwidth]{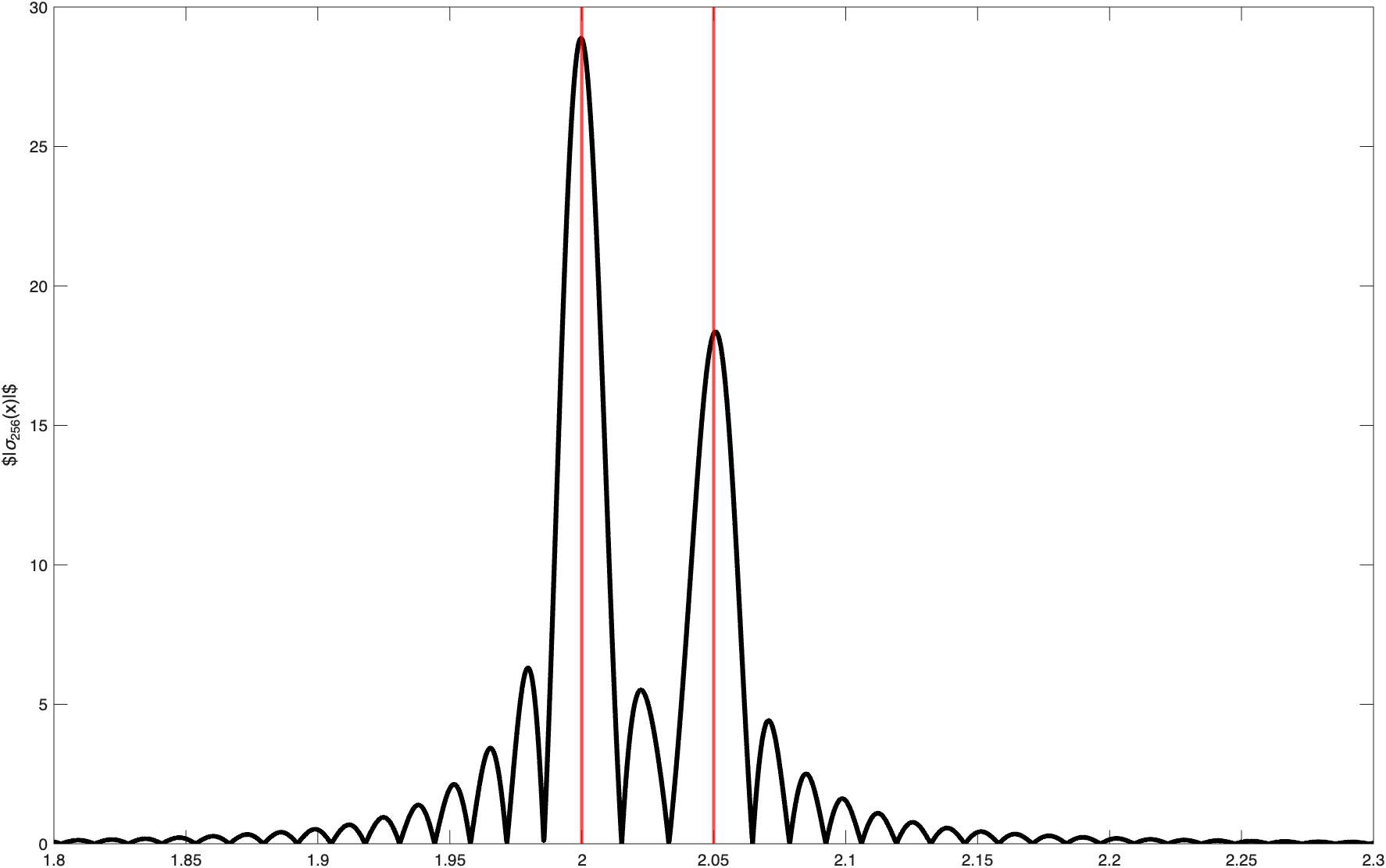} 
\end{minipage}
\begin{minipage}{0.3\textwidth}
\includegraphics[width=\textwidth]{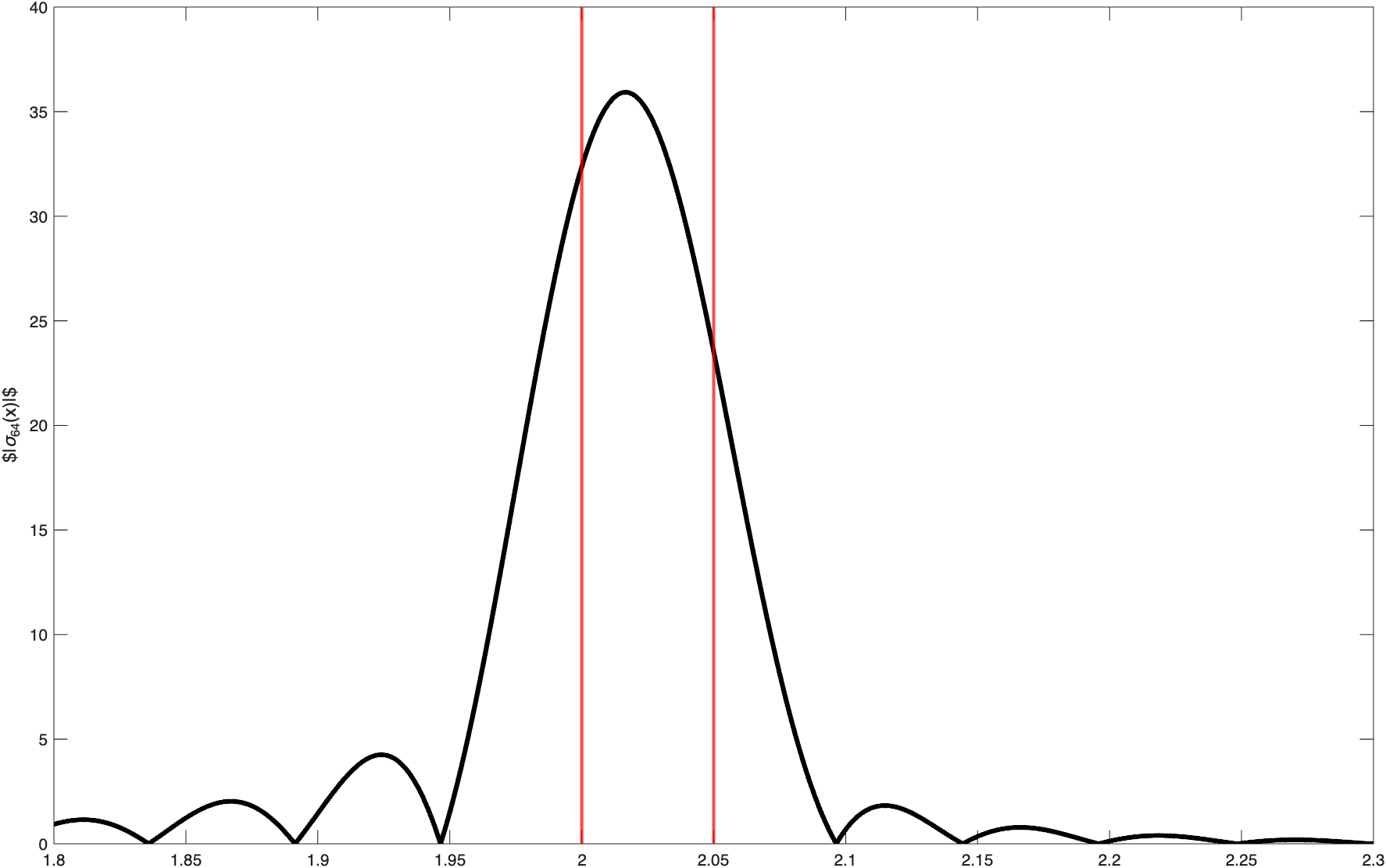} 
\end{minipage}
\end{center}
\caption{Left: $|\sigma_{256}(\mu)(x)|$ has peaks at the points $-1, 2, 2.05$, and is small everywhere else. Vertical red lines indicate the positions of these points. Middle: A close-up view of $|\sigma_{256}(\mu)(x)|$ near $x=2$ to show an accurate detection of the close-by points $2, 2.05$. Right: A close-up view of $|\sigma_{64}(\mu)(x)|$ near $x=2$ to show the non-detection of the close-by points $2, 2.05$. }
\label{fig:pointsource}
\end{figure}

We see from the leftmost figure that $\abs{\sigma_{256}}$ has peaks at approximately $-1, 2, 2.05$, and is very small everywhere else on $[-\pi,\pi]$. The middle figure is a close-up view to highlight an accurate detection of the close-by point sources $2, 2.05$. The rightmost figure shows that with $n=64$, such a resolution is not possible. 
When we wish to automate this, we need to figure out a threshold so that we should look only at peaks above the threshold. As the middle figure shows, there are sidelobes around each peak (and in fact, small sidelobes at many other places on $[-\pi,\pi]$).
In theory, this threshold is $\min_k |a_k|/2$, which we do not know in practice. If we set it too low, then we might  ``detect'' non-existent point sources near $2, 2.05$.
On the other hand, if we set it too high, then we would lose the low amplitude point source at $-1$. 
Some ideas on how to set an appropriate threshold, especially in the presence of noise are discussed in \cite{kitimoon2025localized}.
Another important quantity is the minimal separation $\eta=\min_{k\neq j}|(\omega_k-\omega_j) \mbox{ mod } 2\pi|$ among the point sources.
As the middle and right figures show, the detection of point sources which are very close-by requires the knowledge of a larger number of moments.
It is shown in \cite{mhaskar-superres} that one must have $n\gtrsim \eta^{-1}$ in order to have sufficient resolution to recover the point sources in a stable manner.
\qed}
\end{example}

Our next example is a precursor of the main results of this chapter.

\begin{example}\label{uda:measureseparation} {\rm
We define a distribution $\mu$ on $\TT$ as a convex combination of:
\bit
\item A sum of two uniform distributions each supported on $ [-0.6,-0.4]$, with a weight of $1200/3900$.

\item A normal distribution with mean $0.05$ and variance $0.04$, with a weight of $2400/3900$.

%\item A uniform distribution supported on $ [.46,.66]$, with a weight of $1200/3150$.

\item Three point-mass distributions at $-2,0.4,1.5$, with weights of $60/3900, 120/3900, 120/3900$ respectively (anomaly).
\eit

We take 3900 samples from this distribution (the number of points from each part of the distribution corresponding to the numerator of the weight). 
The samples from the distribution are visualized in Figure~\ref{fig:histogram} (a) as a normalized histogram. 
Then, we apply $\sigma_{128}$ to get an estimation of the support, as seen in Figure~\ref{fig:histogram} (b). 
Not only do we get an idea of the support of the distribution by looking at $\sigma_{128}$, but also the amplitudes of the non-atomic components of the distribution. 
 Since we are dealing with finitely many samples, we in fact are only estimating the integral in the definition of $\sigma_n$ as a Monte-Carlo type summation. That is, with data $\{u_j\}_{j=1}^M$ sampled randomly from $\mu$, we estimate
$$
\sigma_n(t)\approx \frac{1}{M}\sum_{j=1}^M\Phi_n(t- u_j).
$$

\begin{figure}[!ht]
\begin{center}
\begin{subfigure}[t]{.45\textwidth}
\begin{center}
\includegraphics[width=\textwidth]{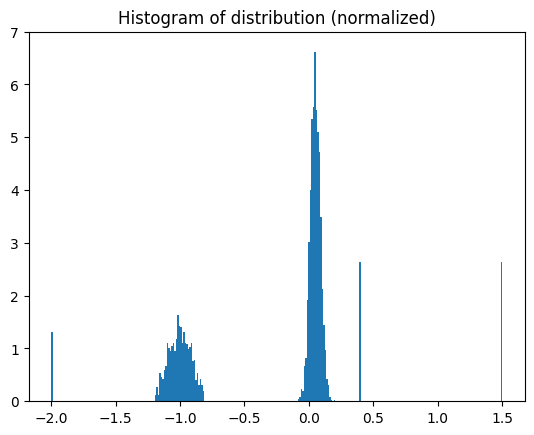}
\caption{Histogram.}
\end{center}
\end{subfigure}
\begin{subfigure}[t]{.45\textwidth}
\begin{center}
\includegraphics[width=\textwidth]{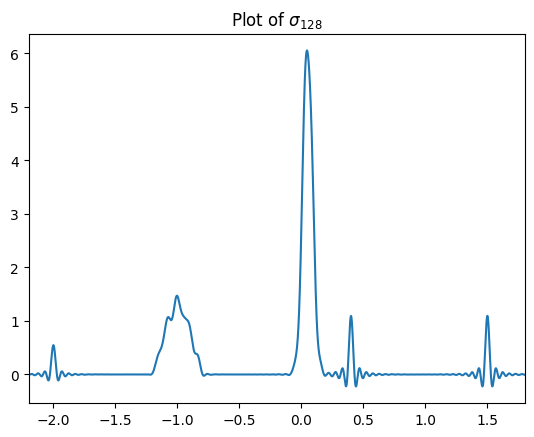}
\caption{Plot of $\sigma_{128}$.}
\end{center}
\end{subfigure}
\caption{Normalized histogram of the density of interest (left), paired with our density estimation by $\sigma_{128}$ based on $3900$ samples (right).} 
\label{fig:histogram}
\end{center}
\end{figure}
\qed}
\end{example}

In this chapter, we will use similar ideas to separate the support of probability measures $\mu_k$ (rather than $\delta_{\omega_k}$), based on random samples taken from the convex combination $\mu=\sum_k a_k\mu_k$ (rather than Fourier coefficients of a general linear combination), where all the measures are supported on a compact metric measure space. 
Table~\ref{tab:supervsclass} summarizes the similarities and differences between the signal separation problem and classification problem as studied in this chapter.

\begin{table}
\begin{center}
\begin{tabular}{|c|c|c|}
\hline
 & Signal separation & Classification\\
\hline
Measure: & $\mu=\sum_{k} a_k \delta_{\omega_k}$ & $\mu=\sum_{k}a_k\mu_k$\\
\hline
Domain: & $\mathbb{R}/(2\pi\ZZ)$ & unknown subset of a metric space\\
\hline
Data: & Fourier moments & samples from $\mu$\\
\hline
Key quantity: & $\min_{j\neq k}\abs{(\omega_j-\omega_k) \mod 2\pi}$ & $\min_{j\neq k}\ \operatorname{dist}(\operatorname{supp}(\mu_j),\operatorname{supp}(\mu_k))$\\
\hline
\end{tabular}
\caption{Comparison between traditional signal separation and our approach to machine learning classification.}
\label{tab:supervsclass}
\end{center}
\end{table}

\section{Relevant Definitions}
\label{sec:background}

In this section, we introduce many common notations and definitions used throughout the rest of this chapter.

Let $\mathbb{M}$ be a compact metric space with metric $\rho$, normalized so that $\mathsf{diam}(\mathbb{M})=\max_{x,y\in\MM}\rho(x,y) =\pi$.
\yadi{$\MM$}{Ambient space}
\yadi{$\rho$}{Metric on $\MM$}
\yadi{$\XX$}{Support of the data measure $\mu$}
This normalization facilitates our use of the $2\pi$-periodic kernel $\Phi_n$, while avoiding the possibility that points $x,y$ with $\rho(x,y)\approx 2m\pi$ for some integer $m$ would be considered close to each other. 
It is well known in approximation theory that positive kernels leads to a saturation and, hence, would not be appropriate for approximating probability measures as is commonly done.
However, in this chapter our main interest is to find supports of the measures rather than approximating the measures themselves.
So, in order to avoid cancellations, we prefer to deal with a positive kernel defined by
\be\label{eq:metrickerndef}
\Psi_n(x,y)=\Phi_n(\rho(x,y))^2,
\ee
where $\Phi_n$ is the kernel defined in \eqref{eq:trigkerndef}. \yadi{$\Psi_n$}{Localized kernel on $\MM$, \eqref{eq:metrickerndef}}
The localization property \eqref{eq:kernloc}, used with $\lceil S/2\rceil$ in place of $S$ implies that
\be\label{eq:metrickernloc}
\Psi_n(x,y) \le c\frac{n^2}{\max(1, (n\rho(x,y))^S)},
\ee
where $c>0$ is a constant depending only on $h$ and $S$.

For any point $x\in\MM$, and any sets $A,B\subseteq \MM$, we define the following notation for the balls and neighborhoods.
\be\label{eq:balldefclassification}
\ba
\dist(x,A)=&\inf_{y\in A} \rho(x,y),& \mathbb{B}(x,r)=&\{y\in\MM:\rho(x,y)\leq r\},\\
\dist(A,B)=&\inf_{y\in A}\dist(y,B),& \mathbb{B}(A,r)=&\{x\in\MM: \dist(x,A)\leq r\}.
\ea
\ee
\yadi{$\mathbb{B}(x,r)$}{Ball of radius $r$ centered at $x$}
\yadi{$\mathbb{B}(A,r)$}{$r$-neighborhood of $A$, \eqref{eq:balldefclassification}}
For any $A\subseteq\mathbb{M}$, we define $\operatorname{diam}(A)=\sup_{x,y\in A} \rho(x,y)$. \yadi{$\operatorname{diam}$}{Diameter of a subset of a metric space}

\subsection{Measures}
\label{subsec:measuredefs}

Let $\mu$ be a positive, Borel, probability measure on $\mathbb{M}$ (i.e. $\int_{\mathbb{M}}d\mu(y)=1$). 
We denote $\mathbb{X}\coloneqq \mathsf{supp}(\mu)$.
Much of this chapter focuses on $\XX$. 
However, we wish to treat $\XX$ as an \textbf{unknown} subset of a known ambient space $\MM$ rather than treating it as a metric space in its own right. 
In particular, this emphasizes the fact the data measure $\mu$  may not have a density, and may not be supported on the entire ambient space.

In the case of signal separation, we have seen that if the minimal amplitude for a certain point source is sufficiently small, we may not be able to detect that point source. 
Likewise, if the measure $\mu$ is too small on parts of $\XX$, we may not be able to detect those parts.
For this reason, we make some assumptions on the measure $\mu$ as in  \cite{cloningercluster}. 
The first property, \textit{detectability}, determines the rate of growth of the measure locally around each point in the support. 
The second property, \textit{fine-structure}, relates the measure to the classification problem by equipping the support with some well-separated (except maybe for some subset of relatively small measure) partition which may correspond to some different class labels in the data. 

\begin{definition}
\label{def:detectable}
We say a measure $\mu$ on $\mathbb{M}$ is \textbf{detectable} if there exist $\alpha\geq 0, \kappa_1, \kappa_2>0$ such that 
\be\label{eq:ballmeasurecon}
\mu(\mathbb{B}(x,r))\le \kappa_1 r^\alpha,\qquad x\in \mathbb{M},\ r>0,
\ee
and there exists $r_0>0$ such that
\be\label{eq:ballmeasurelow}
\mu(\mathbb{B}(x,r))\ge \kappa_2 r^\alpha,\qquad x\in\mathbb{X}, \ 0<r\le r_0.
\ee
\end{definition}

\begin{definition}
\label{def:finestructure}
We say a measure $\mu$ has a \textbf{fine structure} if there exists an $\eta_0$ such that for every $\eta\in (0,\eta_0]$ there is an integer $K_\eta$ and a partition $\mathbf{S}_\eta\coloneqq\{\mathbf{S}_{k,\eta}\}_{k=1}^{K_\eta+1}$ of $\mathbb{X}$ where both of the following are satisfied.
\ben
\item (\textbf{Cluster Minimal Separation}) For any $j,k=1,2,\dots,K_\eta$ with $j\neq k$ we have
\be\label{eq:minsep}
\operatorname{dist}(\mathbf{S}_{j,\eta},\mathbf{S}_{k,\eta})\geq 2\eta.
\ee
\yadi{$\eta$}{Minimal separation among classes at different levels}
\yadi{$K_\eta$}{Number of partition elements of $\XX$ separated by $2\eta$, Definition~\ref{def:finestructure}}
\item (\textbf{Exhaustion Condition}) We have
\be\label{eq:exhaustion}
\lim_{\eta\to 0^+} \mu(\mathbf{S}_{K_\eta+1,\eta})=0.
\ee
\een
We will say that $\mu$ has a \textbf{fine structure in the classical sense} if $\mu=\sum_{k=1}^K a_k\mu_k$ for some probability measures $\mu_k$, $a_k$'s are $>0$ and $\sum_k a_k=1$, and the compact subsets $\mathbf{S}_k\coloneqq \mathsf{supp}(\mu_k)$ are disjoint. 
In this case $\eta$ is the minimal separation among the supports and there is no overlap.
 
\end{definition}

\begin{remark}\label{rem:exception}
{\rm
It is possible to require the condition \eqref{eq:ballmeasurelow} on a subset of $\XX$ having measure converging to $0$ with $r$. 
This will add some difficulties in our proof of \eqref{eq:Jn} and Lemma~\ref{lemma:discprelim}. However, in the case when $\mu$ has a fine structure, this exceptional set can be absorbed in $\mathbf{S}_{K_\eta+1}$ with appropriate assumptions. 
We do not find it worthwhile to explore this further in this chapter.
\qed}
\end{remark}

\begin{example}
{\rm
Supposing that $\mu=\sum_{k=1}^K a_k \delta_{\omega_k}$ (as in the signal separation problem), then we see that $\mu$ is detectable with $\alpha=0$, $\kappa_1=\max_k |a_k|$, $\kappa_2=\min_k |a_k|$. 
It has fine structure in the classical sense  whenever $\eta<\min_{j\neq k} |\omega_j-\omega_k|$. In this sense, the theory presented in this chapter is a generalization of results for signal separation in this regime.
\qed} 
\end{example}

\begin{example}
{\rm
If $\mathbb{X}$ is a $\alpha$-dimensional, compact, connected, Riemannian manifold, then the normalized Riemannian volume measure is detectable with parameter $\alpha$. 
\qed}
\end{example}

\subsection{F-score}\label{sec:fscore}

We will give results on the theoretical performance of our measure estimating procedure by giving an asymptotic result involving the so-called F-score. 
The F-score for binary classification (true/false) problems is a measure of classification accuracy taking the form of the harmonic mean between precision and recall. 
In a predictive model, precision is defined as the fraction of true positive outputs over all the positive outputs of the model. 
Recall is the fraction of true positive outputs over all the actual positives. 
In a multi-class problem, we extend this definition as follows (cf. \cite{diagraphcluster_ohiostate_2011}). 
If $\{C_1,\dots,C_N\}$ is a partition of $\{x_j\}_{j=1}^M$ indicating the predicted output labels of a model and $\{L_1,\dots,L_K\}$ is the ground-truth partition of the data, then 
one can define the precision of $C_j$ against the true label $L_k$ by $|C_j\cap L_k|/|C_j|$ and the corresponding recall by $|C_j\cap L_k|/|L_k|$.
Taking the maximum of the harmonic means of the precisions and recalls with respect to all the ground truth labels leads to 
\be
F(C_j)=2\max_{k\in \{1,\dots,K\}} \frac{\abs{C_j\cap L_k}}{\abs{C_j}+\abs{L_k}}.
\ee
Then the F-score is given by
\be
F\left(\{C_j\}_{j=1}^N\right)=\frac{\sum_{j=1}^N \abs{C_j}F(C_j)}{\sum_{j=1}^N \abs{C_j}}.
\ee
Since we are treating the data as samples from a measure $\mu$, we replace cardinality in the above formulas with measure. Our fine structure condition gives us the true supports as $\{\mathbf{S}_{k,\eta}\}_{k=1}^{K_\eta}$ for any valid $\eta$, so we can define the F-score for the support estimation clusters $\{\mathcal{C}_{j,\eta}\}_{j=1}^N$ by
\be\label{eq:class_fscore}
\mathcal{F}_\eta(\mathcal{C}_{j,\eta})=2\max_{k\in \{1,\dots,K\}}\frac{\mu(\mathcal{C}_{j,\eta}\cap \mathbf{S}_{k,\eta})}{\mu(\mathcal{C}_{j,\eta})+\mu(\mathbf{S}_{k,\eta})},
\ee
and \yadi{$\mathcal{F}_\eta$}{$F$-score at separation $\eta$, \eqref{eq:fscoredef}}
\be\label{eq:fscoredef}
\mathcal{F}_\eta\left(\{\mathcal{C}_{j,\eta}\}_{j=1}^N\right)=\frac{\sum_{j=1}^N \mu(\mathcal{C}_{j,\eta})\mathcal{F}_\eta(\mathcal{C}_{j,\eta})}{\mu\left(\bigcup_{j=1}^N \mathcal{C}_{j,\eta}\right)}.
\ee
\begin{remark}\label{rem:fscore}
We observe that 
$$
1-2\frac{\mu(\mathcal{C}_{j,\eta}\cap \mathbf{S}_{k,\eta})}{\mu(\mathcal{C}_{j,\eta})+\mu(\mathbf{S}_{k,\eta})}=\frac{\mu(\mathcal{C}_{j,\eta}\Delta \mathbf{S}_{k,\eta})}{\mu(\mathcal{C}_{j,\eta})+\mu(\mathbf{S}_{k,\eta})},
$$
where in this remark only, $\Delta$ denotes the symmetric difference.
It follows that $0\le \mathcal{F}_\eta\le 1$.
If we estimate each support perfectly so $C_{j,\eta}=\mathbf{S}_{j,\eta}$ for all $j$ and each $C_{j,\eta}$ is $\eta$-separated from any other, then we see that $\mathcal{F}_\eta\left(\{\mathcal{C}_{j,\eta}\}_{j=1}^N\right)=1$. Otherwise, we will attain an F-score strictly lower than $1$. \qed
\end{remark}

\section{Main Results}
\label{sec:mainresults}

In this section we introduce the main theorems of this chapter, which involve the recovery of supports of a measure from finitely many samples. Theorem~\ref{thm:fullsuportdet}
 pertains to the case where we only assume the detectability of the measure. Theorem~\ref{thm:class_separation} pertains to the case where we additionally assume the fine structure condition. Before stating the results, we must introduce our discrete measure support estimator and support estimation sets. We define our  \textbf{data-based measure support estimator} by \yadi{$F_n$}{measure support estimator \eqref{eq:support_estimator_def}}
\be\label{eq:support_estimator_def}
F_n(x)\coloneqq \frac{1}{M}\sum_{j=1}^M \Psi_n(x,x_j).
\ee
This definition is then used directly in the construction of our \textbf{data-based support estimation sets}, given by \yadi{$\mathcal{G}_n$}{Support estimation set \eqref{eq:support_est_set_def}}
\be\label{eq:support_est_set_def}
\mathcal{G}_n(\Theta)\coloneqq \left\{x\in \mathbb{M}: F_n(x)\geq \Theta \max_{1\leq k\leq M} F_n(x_k)\right\}.
\ee

Our first theorem gives us bounds on how well $\mathcal{G}_n(\Theta)$ approximates the entire support $\mathbb{X}$ with the detectability assumption. 

\begin{theorem}\label{thm:fullsuportdet}
Let $\mu$ be detectable and suppose $M\gtrsim n^{\alpha}\log(n)$. Let $\{x_1,x_2,\dots,x_M\}$ be independent samples from $\mu$. There exists a constant $C>0$ such that if $\Theta<C<1$, then there exists $r(\Theta)\sim \Theta^{-1/(S-\alpha)}$ such that with probability at least $1-c_1/M^{c_2}$ we have
\be\label{eq:thm1}
\mathbb{X}\subseteq \mathcal{G}_n(\Theta)\subseteq \mathbb{B}\left(\mathbb{X},r(\Theta)/n\right).
\ee
\end{theorem}

Our second theorem additionally assumes the fine-structure condition on the measure, and gives conditions so that for any satisfactory $\eta$, the support estimation set $\mathcal{G}_n(\Theta)$ splits into $K_\eta$ subsets each with separation $\eta$, thus solving the machine learning classification problem in theory.

\begin{theorem}\label{thm:class_separation}
Suppose, in addition to the assumptions of Theorem~\ref{thm:fullsuportdet}, that $\mu$ has a fine structure, $n\gtrsim 1/(\eta\Theta^{1/(S-\alpha)})$, and $\mu(\mathbf{S}_{K_{\eta}+1,\eta})\lesssim \Theta n^{-\alpha}$. Define
\be\label{eq:thm2-1}
\mathcal{G}_{k,\eta,n}(\Theta)\coloneqq \mathcal{G}_n(\Theta)\cap \mathbb{B}(\mathbf{S}_{k,\eta},r(\Theta)/n).
\ee
Then, with probability at least $1-c_1/M^{c_2}$, $\{\mathcal{G}_{k,\eta,n}(\Theta)\}_{k=1}^{K_\eta}$ is a partition of $\mathcal{G}_{n}(\Theta)$ such that 
\be\label{eq:thm2-2}
\operatorname{dist}(\mathcal{G}_{j,\eta,n}(\Theta),\mathcal{G}_{k,\eta,n}(\Theta))\geq \eta\qquad j\neq k,
\ee 
and in this case, there exists $c<1$ such that
\be\label{eq:thm2-3}
\mathbb{X}\cap \mathbb{B}(\mathbf{S}_{k,\eta},cr(\Theta)/n)\subseteq \mathcal{G}_{k,\eta,n}(\Theta)\subseteq \mathbb{B}\left(\mathbf{S}_{k,\eta},r(\Theta)/n\right).
\ee
\end{theorem}

\begin{remark}\label{rem:meshnorm}
If $\mathcal{C}=\{z_1,\cdots, z_M\}$ is a random sample from $\mu$, $n\ge 1$ and $M\gs n^\alpha\log n$, then it can be shown (cf. \cite[Lemma~7.1]{mhaskardata}) that for any point $x\in\XX$, there exists some $z\in\mathcal{C}$ such that $\rho(x,z)\le 1/n$. 
Hence, the Hausdorff distance between $\XX$ and $\mathcal{C}$ is $\le 1/n$. 
If $\mu$ has a fine structure in the classical sense, and $n\gs \eta^{-1}$, then this implies that a correct clustering of $\mathcal{C}$ would give rise to a correct classification of every point in $\XX$.
This justifies our decision to construct the algorithm in Section~\ref{sec:alg} to classify only the points in $\mathcal{C}$. 
On the other hand, the use of the localized kernel as in the theorems above guide us about the choice of the points at which to query the label. \qed
\end{remark}

In Figure~\ref{fig:twomoons} we illustrate Theorem~\ref{thm:class_separation} applied to a simple two-moons data set. We see that the support estimation set, shown in yellow, covers the data points as well as their nearby area, predicting the support of the distribution from which the data came from. Furthermore, we show in the figure a motivating idea: by querying a single point in each component for its class label we can extend the label to the other points in order to classify the whole data set. This is how we utilize the active learning paradigm in our algorithm discussed in Section~\ref{sec:alg}.

\begin{figure}[!ht]
\begin{center}
\includegraphics[width=.4\textwidth]{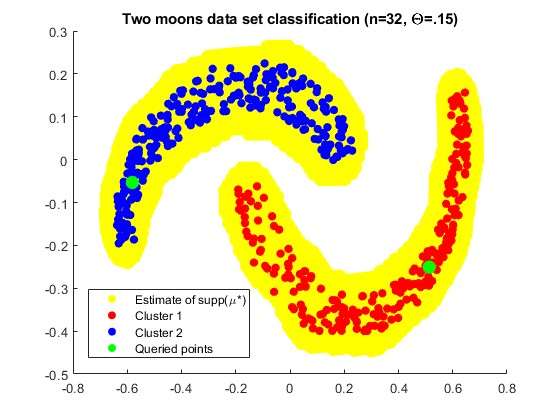}
\end{center}
\caption{Demonstration of the support estimation set $\mathcal{G}_{32}(0.15)$ (yellow) applied to a simple two-moons data set from \cite{twomoons} (blue and red). By querying one point from each component of the support estimation set and extending the label to the other points in the same component, we can classify the entire data set with 100\% accuracy.}
\label{fig:twomoons}
\end{figure}

Our final result examines the fidelity of our classification scheme in terms of the asymptotics of the F-score associated with our support estimation theorems as $\eta\to 0$.
 We show that our support estimation setup asymptotically approaches the ideal F-score of $1$.
   
\begin{theorem}\label{thm:fscore}
Suppose the assumptions of Theorem~\ref{thm:class_separation} are satisfied and that
\be
\lim_{\eta\to 0^+} \max_{0\leq k\leq K_\eta}\left(\frac{\mu(\mathbf{S}_{K_{\eta+1},\eta})}{\mu(\mathbf{S}_{k,\eta})}\right)=0.
\ee
Then, with probability at least $1-c_1/M^{c_2}$, we have
\be
\lim_{\eta\to 0^+} \mathcal{F}_\eta\left(\{\mathcal{G}_{k,\eta,n}(\Theta)\}_{k=1}^{K_\eta}\right)=1.
\ee
where $\mathcal{F}_\eta$ is the $F$-score with respect to $\mathbf{S}_\eta$.
\end{theorem}

\section{MASC Algorithm}
\label{sec:alg}

\subsection{Algorithm Description}

In the following paragraphs we describe the motivation and intuition of the algorithm MASC (Algorithm~\ref{alg:MASC}). Throughout this section we will refer to line numbers associated with Algorithm~\ref{alg:MASC}.

One obvious way to embed a data into a metric space with diameter $\le \pi$ is just to rescale it.
If the data is a compact subset of an ambient Euclidean space $\RR^q$, we may project the data on the unit sphere $\SS^q\subset \RR^{q+1}$ by a suitable inverse stereographic projection. 
The metric space $\SS^q$, equipped with the geodesic distance $\arccos(\circ,\circ)$ has diameter $\pi$ by construction.

One of the main obstacles we must overcome in an implementation of our theory is the following. 
In practice, we often do not know the minimal separation $\eta$ of the data classes beforehand, nor do we know optimal values for $\Theta,n$.
Taking a machine learning perspective, we develop a multiscale approach to remedy these technical challenges: treat $n,\Theta$ as hyperparameters of the model and increment $\eta$. 
Firstly, MASC will threshold out any data points not belonging to $\mathcal{G}_n(\Theta)$ (line 2).
For each value of $\eta$ (initialize while loop in line 4) we construct a (unweighted) graph where an edge goes between two points $x_i,x_j$ if and only if $\rho(x_i, x_j)<\eta$ (line 5). 
At this point, we have a method for unsupervised clustering by simply examining graph components (line 6, see below for discussion on $p$). 
The idea to implement active learning is to then query a modal point of each graph component (line 11), also referred to in this section as a cluster, with respect to $\Psi_n$ and extend that label to the rest of the cluster (line 13). 
A trade-off associated with this idea is the following: if we initialize $\eta$ too small (respectively, $n$ too large) then each point in the data set will be its own cluster and we will simply query the whole data set, whereas if we initialize $\eta$ too large (respectively, $n$ too small) then the whole data set will belong to a single cluster destroying any classification accuracy. 
Therefore, we initialize $\eta$ small and introduce a minimum cluster size threshold value $p$ to avoid this issue. Any cluster of size $<p$ will be removed from consideration (line 6), so we will not query any points until $\eta$ is large enough to produce a cluster of size $p$ or greater.

After the label extension is done in each cluster of size $\ge p$, we keep track of which points we queried (line 12), increment $\eta$ (line 16), and repeat (line 4). 
Sometime after the first incrementation of $\eta$, we will experience the combination of clusters which were previously disconnected. 
When this occurs we check whether each of the previously queried points in the new cluster have the same label (line 14). If so, then we extend it to the new cluster (line 15).
 Otherwise, we halt the extension of labels for all points in that cluster. 
 In this way, the method proceeds by a cautious clustering to avoid labeling points that are either 1) in a too-low density region, or 2) within a cluster where we have queried multiple points with contradicting labels.

Once $\eta$ is large enough that the data set all belongs to a single cluster, we will not gain any new information by incrementing $\eta$ further, and hence MASC will halt the iterations of $\eta$ (lines 7 and 8). The final process is to implement a method for estimating the labels of points that did not receive a predicted label in the first part, either because they belonged to a low-density region and were thresholded out or because they belonged to a cluster with conflicting queried points. The remaining task is equivalent to the semi-supervised regime of classification and we acknowledge that there is a vast variety of semi-supervised learning methods to choose from. In MASC, we have elected to use a traditional $\overline{k}$-nearest neighbors approach.

For a data point $x_j$, we denote the set of its nearest $\overline{k}$ neighbors which already have labels $\hat{y}(x_j)$ estimated from MASC by $\mathcal{A}_{j,\overline{k}}$. The $\overline{k}$-nearest neighbors formula to estimate the label of $x_j$ is then given by:
\be
\underset{k\in [K]}{\argmax} |\{x_i\in A_{j,\overline{k}}: \hat{y}(x_j)=k\}|,
\ee
with some way to decide on the choice of $k$ in the event of a tie. In binary classification tasks, the value of $\overline{k}$ can be chosen as an odd value to prevent ties. Otherwise, a tie can be broken by choosing the label of the nearest point with a tied label, a hierarchical ordering of the labels, at random, etc. In our Python implementation of the algorithm used to produce the figures in this chapter, we use the \verb|scipy.stats.mode| function, which returns the first label in the list of tied labels upon such a tie.

MASC will collect all points which do not yet have predicted labels (line 17), and apply the nearest-neighbors approach as describe above to each of these points (lines 19 and 20). At this point, every element in the data set will have a predicted label, so the algorithm will return the list of labels (line 21).

%Once the above has been repeated for all of the values of $\eta$ in the chosen range, the final step of interest is to implement the function approximation method discussed in Appendix~\ref{sec:manifoldapprox} on the remaining points which do not yet have a predicted label (also referred to as ``uncertain" points). 
%In view of Remark~\ref{rem:density_est}, computing
%\be\label{eq:labelextension}
%\underset{k\in [K]}{\argmax}\frac{1}{\abs{\mathcal{A}_k}}\sum_{x_i\in \mathcal{A}_k}\Phi_{n,q}(x_i\cdot x_j).
%\ee
%\yadi{$\Phi_{n,q}$}{Approximation kernel for unknown submanifolds of Euclidean sphere, \eqref{eq:manifoldkernel}}
%amounts to estimating which of the estimated supports of $\mu_k$'s has the highest density at the point $x$. As such, we refer to this step as the ``density estimation extension" of the class labels to the uncertain points.

In MASC, we require defining a starting $\eta$ and $\eta_{\operatorname{step}}$. Once the matrix with entries given by $\Psi_n(x_i,x_j)$ is calculated, one may search for the range of $\eta$ values which give non-trivial clusters of size $\geq p$ with relative ease. If $\eta$ is too small, no cluster will contain a sufficient number of points and if $\eta$ is too large, every point will belong to the same cluster, both of which we consider a ``trivial" case. Then $\eta_{\operatorname{step}}$ may be chosen to satisfy some total number of iterations across this domain. The values $n,\Theta,p,\overline{k}$ are considered hyperparameters.

\begin{algorithm}
\caption{Multiscale Active Super-resolution Classification (MASC)}
\label{alg:MASC}
\KwIn{Data set $X$, kernel degree $n$, threshold parameter $\Theta$, $\eta$ initialization, step size $\eta_{\text{step}}>0$, cluster size minimum $p$, oracle $f$, neighbor parameter $\overline{k}$.}
\KwOut{Predicted labels $\hat{y}$ for all points in $X$.}

$\mathcal{A} \leftarrow \emptyset$ \Comment*{Initialize queried point set}
    $V\gets \{x_i\in X: x_i\in \mathcal{G}_n(\Theta)\}$ \Comment*{Prune data to consider only those in threshold set \eqref{eq:support_est_set_def}}
$\operatorname{STOP}\gets \operatorname{FALSE}$ \;

\While{$\operatorname{STOP}=\operatorname{FALSE}$}{
	$E\gets \{(x_i,x_j)\in V\times V:\rho(x_i,x_j)<\eta, x_i\neq x_j\}$ \Comment*{Edge set consisting of points within $\eta$ distance from each other}
    $\{C_{\eta,\ell}\}_{\ell=1}^{K_\eta}\gets$ connected components of $G=(V,E)$ with size $\ge p$ \;
     % $\text{flag}(\ell) \leftarrow 0$ for all $\ell$ \;
     \If{$|C_{\eta,1}|=|V|$}{
     $\operatorname{STOP}\gets\operatorname{TRUE}$ \Comment*{End while loop once $G$ is connected}}

    \For{$\ell = 1$ \KwTo $K_n$}{
        \If{$C_{\eta,\ell} \cap \mathcal{A} = \emptyset$}{
            $x_i \leftarrow \underset{{x \in C_{\eta,\ell}}}{\argmax} \sum_{j=1}^{M} \Psi_n( x, x_j)$ \Comment*{Locate maximizer of $F_n$ (cf. \eqref{eq:support_estimator_def}) in $C_{\eta,\ell}$ without any queried points}
            $\mathcal{A} \leftarrow \mathcal{A}\bigcup \{x_i\}$ \Comment*{Append maximizer to queried point set}
            {\small $\hat{y}(x_j) \leftarrow f(x_i)$ for all $x_j \in C_{\eta,\ell}$} \Comment*{Query point and extend label to all of $C_{\eta,\ell}$}
            % $\text{flag}(\ell) \leftarrow 1$ \;
        }
        \ElseIf{$\forall x_i,x_j \in C_{\eta,\ell} \cap \mathcal{A}, f(x_i) = f(x_j)=:c_{\eta,\ell}$}{
            {\small $\hat{y}(x_j) \leftarrow c_{\eta,\ell}$ for all $x_j \in C_{\eta,\ell}$ \Comment*{If all queried points in component have same label, extend label to entire component}} 
            % $\text{flag}(\ell) \leftarrow 1$ \;
        }
    }

        $\eta \leftarrow \eta + \eta_{\text{step}}$ \;

}

%\If{$m$ is given and $|\alpha|<m$}{
%
%}

%\For{$k\in \operatorname{range}(\hat{y})$}
%{$\mathcal{A}_k\gets \{x\in X:\hat{y}(x)=k\}$ \Comment*{Set of all points with predicted label $k$}}
%
%$\mathcal{C}_{\text{uncertain}}\gets \{x\in X : x\notin \cup_{k}\mathcal{A}_k\}$ \Comment*{Set of points without a predicted label}

$\mathcal{C}_{\text{uncertain}}\gets \{x\in X : \hat{y}(x_j)=\text{DNE}\}$ \Comment*{Set of points which do not have a predicted label}

\For{$x_j \in \mathcal{C}_{\text{uncertain}}$}{
    %$\hat{y}(x_j) \leftarrow \underset{k}{\argmax}\frac{1}{\abs{\mathcal{A}_k}} \sum_{x_i \in \mathcal{A}_k} \Phi_{m,q}( x_i\cdot x_j )$ \Comment*{Density estimation extension to predict label for remaining points}
$\mathcal{A}_{j,\overline{k}}\gets \{x\in X\setminus \mathcal{C}_{\text{uncertain}}: x \text{ is the $\overline{k}$th closest element to $x_j$ or closer}\}$ \;
    
    $\hat{y}(x_j)\gets \underset{k\in [K]}{\argmax} |\{x_i\in A_{j,\overline{k}}: y_j=k\}|$ \Comment*{$\overline{k}$-nearest neighbors approach to estimate labels for uncertain points}
}

\Return $\hat{y}$.
\end{algorithm}

\subsection{Comparison With CAC and SCALe}
\label{subsec:CACcomparison}

In \cite{cloningercluster}, a similar theoretical approach to this chapter except on the Euclidean space was developed and an algorithm we will call ``Cautious Active Clustering" (CAC) was introduced. MASC and CAC are both multiscale algorithms using $\mathcal{G}_n(\Theta)$ to threshold the data set, then constructing graphs to query points and extend labels. The main difference between the algorithms is the following. In CAC, $\eta,\Theta$ are considered hyperparameters while $n$ is incremented, whereas in MASC, $n,\Theta$ are considered hyperparameters while $\eta$ is incremented. 
This adjustment serves three purposes:
\ben
\item It connects the algorithm closer to the theory, which states that a single $n,\Theta$ value will suffice for the right value of $\eta$. We do not know $\eta$ in advance, but by incrementing $\eta$ until all of the data belongs to a single cluster, we will attain a value close to the true value at some step. At this step, we will query points belonging roughly to the ``true" clusters and that information will be carried onward to the subsequent steps.

\item Consistency in query procedure: we use the same function to decide which points to query at each level, rather than it changing as the algorithm progresses.

\item It improves computation times since computing the $\Psi_{n}$ matrix for varying values of $n$ tends to take more time than incrementing $\eta$ and checking graph components.
\een
In MASC, we have the additional parameter $p$ specifying the minimum size of the graph component to allow a query. While this is new compared to CAC, the main purpose is to reduce the total number of queries to just those that contain more information. One could implement such a change to CAC as well for similar effect. A further difference is that CAC uses a localized summability kernel approach to classify uncertain samples, whereas MASC uses a nearest-neighbors approach.

SCALe, as introduced in \cite{mhaskar-odowd-tsoukanis} is an even more similar algorithm to MASC. The main difference between MASC and SCALe is the final step, where in the present method we use a nearest-neighbors approach to extend labels to uncertain points while in SCALe the choice was to use a  function approximation technique developed in \cite{mhaskarodowd}. Both methods have their pros and cons. Compared to SCALe, the nearest-neighbors approach of MASC:
\ben
\item  works in arbitrary metric spaces, without requiring a summability kernel as in SCALe.

\item extends labels to uncertain points (sometimes much) faster, reducing computation time while usually providing comparable or better results with sufficiently many queries, but

\item reduces accuracy in extremely sparse query setting, where the function estimation method with the manifold assumption empirically seems to extend labels more consistently.
\een

\section{Numerical Examples}
\label{sec:numericalclassification}

In this section, we look at the performance of the MASC algorithm applied to 1) a synthetic data set with overlapping class supports (Section~\ref{subsec:circleellipse}), 2) a document data set (Section~\ref{sec:document}), and 3) two different hyperspectral imaging data sets: Salinas (Section~\ref{sec:salinas}) and Indian Pines (Section~\ref{sec:indianpines}). On the Hyperspectral data sets, we compare our method with two other algorithms for active learning: LAND and LEND (Section~\ref{sec:comparison}). 
%Finally, we examine how much better our method does than querying just random samples.

For hyperparameter selection on our model as well as the comparisons, we have not done any validation but rather optimized the hyperparameters for each model on the data itself. So the results should be interpreted as being near-best-possible for the models applied to the data sets in question rather than a demonstration of generalization capabilities. While this approach is non-traditional for unsupervised/supervised learning, it has been done for other active learning research (\cite{murphy-lend}, for example) so we have elected to follow the same procedure in this chapter. Further, an exhaustive grid search was not conducted but rather local minima among grid values were selected for each hyperparameter. For MASC, we looked at $n$ in powers of $2$ and $\overline{k}$ values in multiples of $5$. For LAND we looked at $K,t$ at increments of $10$, and with LEND we used the same parameters from LAND and looked at integer $J$ values and $\alpha$ values in increments of $0.1$. For $\Theta$, we tried values less rigorously, meaning that better $\Theta$ values may exist than the ones chosen. Due to the nature of the algorithm, increasing $\Theta$ will increase the number of samples that the nearest-neighbors approach has to estimate, while reducing the number of labeled neighbors it has to do so. However, increasing $\Theta$ can also reduce the number of queries used, sometimes without deterioration in accuracy. So there may be some tradeoff, but we generally see the best results when $\Theta$ is chosen to threshold a small portion of the initial data (outlier removal). In Table~\ref{tab:parameters}, we summarize the choice of parameters for each of the data sets in the subsequent sections.

\begin{table}[htbp]
    \centering
    \textbf{MASC hyperparameter selection for each data set}\\
    \begin{tabular}{|l|c|c|c|c|}
        \hline
        Dataset & $\Theta$ & $\eta$ & $p$ & $\overline{k}$ \\
        \hline
        Circle+Ellipse & 0.12 & $[0.006, 0.036]$ & 15 & 5 \\
        (Section~\ref{subsec:circleellipse})               &      & (step size 0.005) &    &   \\
        \hline
        Document & 0.51 & $[0.08, 0.15]$ & 3 & 25 \\
        (Section~\ref{sec:document})         &      & (step size 0.002) &    &   \\
        \hline
        Salinas & 0.32 & $[0.21, 0.27]$ & 3 & 25 \\
        (Section~\ref{sec:salinas})       &      & (step size 0.005) &    &   \\
        \hline
        Indian Pines & 0.08 & $[0.03, 0.13]$ & 5 & 15 \\
        (Section~\ref{sec:indianpines})             &      & (step size 0.005) &    &   \\
        \hline
    \end{tabular}
        \caption{Selected hyperparameter values for our MASC algorithm applied to the data sets in the subsequent sections.}
    \label{tab:parameters}
\end{table}

\subsection{Circle on Ellipse Data}
\label{subsec:circleellipse}

Although the theory in this chapter focuses on the case where the supports of the classes are separated (or at least satisfy a fine-structure condition), our MASC algorithm still performs well at classification tasks of data with overlapping supports in the regions without overlap. 
To illustrate this, we generated a synthetic data set of 1000 points sampled along the arclength of a circle and another 1000 sampled along the arclength of an ellipse with eccentricity 0.79. 
For each data point, normal noise with standard deviation 0.05 was additively applied independently to both components. Figure~\ref{fig:circleellipse} shows the true class label for each of the points on the left and the estimated class labels on the right. 
We can see that the misclassifications are mostly localized to the area where the supports of the two measures overlap.
 Near the intersection points of the circle and ellipse the classification problem becomes extremely difficult due to a high probability that a data point could have been sampled from either the circle or ellipse.

\begin{figure}
\begin{center}
\begin{subfigure}[t]{.45\textwidth}
\begin{center}
\includegraphics[width=\textwidth]{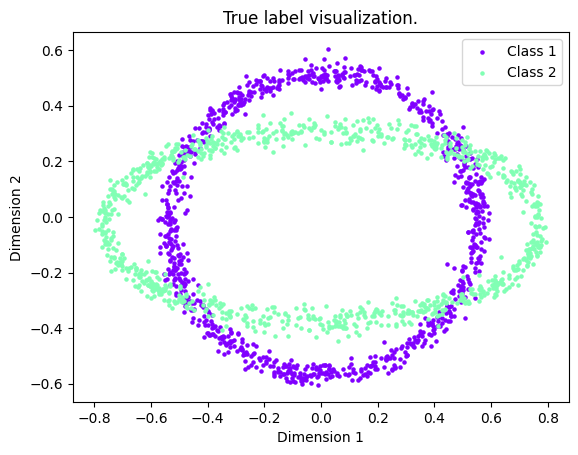}
\subcaption{True labels of the circle and ellipse data.}
\end{center}
\end{subfigure}
\begin{subfigure}[t]{.45\textwidth}
\begin{center}
\includegraphics[width=\textwidth]{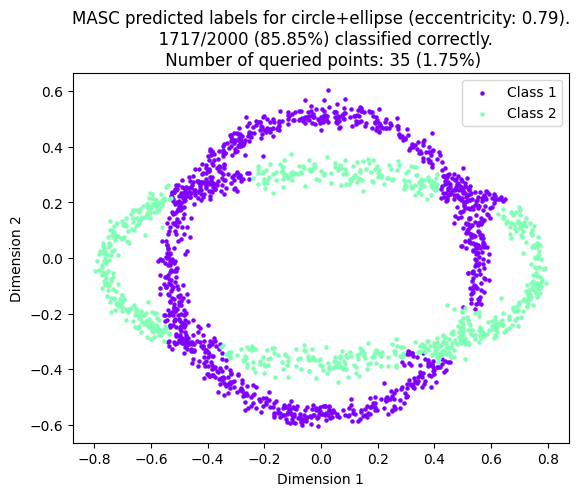}
\subcaption{Predicted labels using MASC with 35 queries, achieving 83\% accuracy.}
\end{center}
\end{subfigure}
\caption{This figure illustrates the result of applying MASC to a synthetic circle and ellipse data set. On the left are true labels of the given data, and on the right is the estimation attained by MASC.}
\label{fig:circleellipse}
\end{center}
\end{figure}

\subsection{Document Data}
\label{sec:document}

This numerical example uses the document data set provided by Jensen Baxter through Kaggle \cite{documentdata}. The data set contains 1000 documents total, 100 each belonging to a particular category from: business, entertainment, food, graphics, historical, medical, politics, space, sport, and technology. For prepossessing we run the data through the Python sklearn package's TfidfVectorizer function to convert the documents into vectors of length 1684. Then we implement MASC.

In Figure~\ref{fig:doc1} we see the results of applying MASC on the document data in two steps. On the left we see the classification task by MASC paused at line 17 of Algorithm~\ref{alg:MASC}, before labels have been extended via the nearest neighbor  portion at the end of the algorithm. On the right we see the result of the density estimation extension. In Figure~\ref{fig:doc2} we see on the left a confusion matrix for the result shown in Figure~\ref{fig:doc1}, allowing us to see which classes were classified the most accurately versus which ones had more trouble. We see the largest misclassifications had to do with documents that were truly ``entertainment" but got classified as either ``sport" or ``technology", and documents which were actually ``graphics" but got classified as ``medical". On the right of Figure~\ref{fig:doc2} we have a plot indicating the resulting accuracy vs. the number of queries which MASC was allowed to do. Naturally as the number of queries approaches 1000 this plot will gradually increase to 100\% accuracy. Lastly, in Figure~\ref{fig:doc3} we see a side-by-side comparison of the true labels for the document data set vs. the predicted labels.

\begin{figure}
\begin{center}
\begin{subfigure}[t]{.45\textwidth}
\begin{center}
\includegraphics[width=\textwidth]{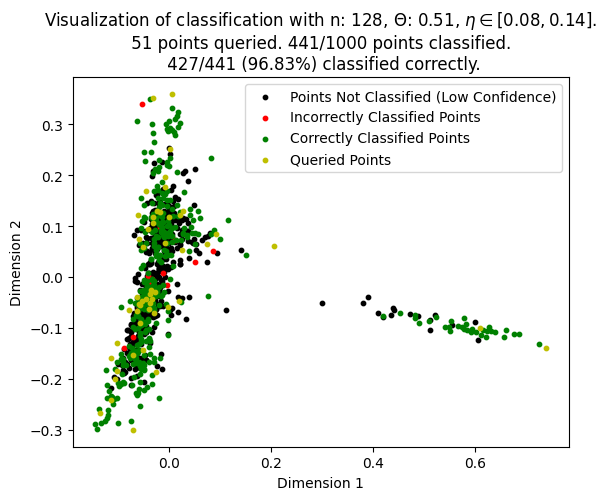}
\subcaption{Classification of certain points in MASC algorithm (before density estimation extension).}
\end{center}
\end{subfigure}
\begin{subfigure}[t]{.45\textwidth}
\begin{center}
\includegraphics[width=\textwidth]{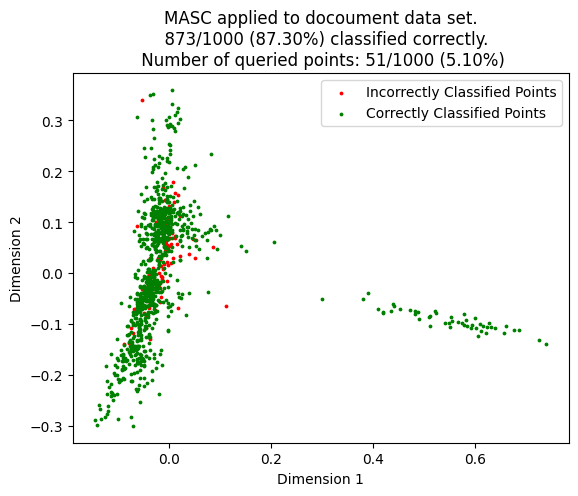}
\subcaption{Classification of remainder points using density estimation extension.}
\end{center}
\end{subfigure}
\caption{This figure illustrates the classification process undergone MASC on the document data set at two points. On the left, we see the classification of points before the $\overline{k}$-nearest neighbors extension. On the right, we see the result after $\overline{k}$-nearest neighbors extension.}
\label{fig:doc1}
\end{center}
\end{figure}

\begin{figure}[!ht]
\begin{center}
\begin{subfigure}[t]{.45\textwidth}
\begin{center}
\includegraphics[width=\textwidth]{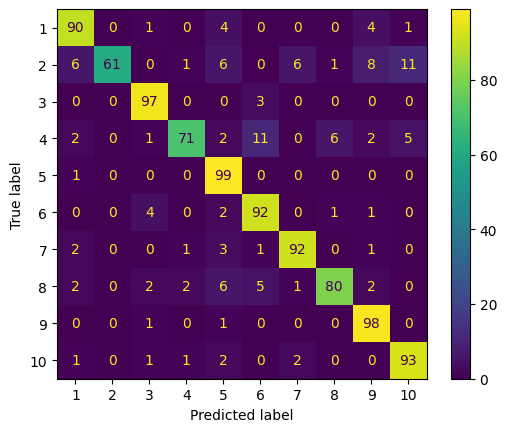}
\subcaption{Confusion matrix.}
\end{center}
\end{subfigure}
\begin{subfigure}[t]{.45\textwidth}
\begin{center}
\includegraphics[width=\textwidth]{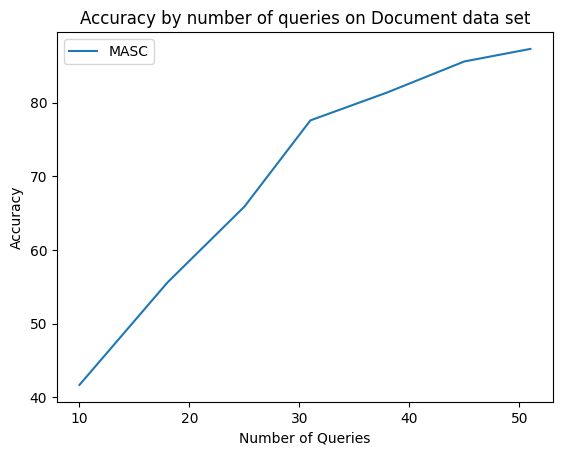}
\subcaption{Plot of MASC accuracy vs. number of allowed query points}
\end{center}
\end{subfigure}
\caption{Further details on the classification results for the document data set. (Left) Confusion matrix for single run of MASC algorithm. (Right) Accuracy of MASC algorithm vs. the number of queries used.}
\label{fig:doc2}
\end{center}
\end{figure}

\begin{figure}[!ht]
\begin{center}
\begin{subfigure}[t]{.45\textwidth}
\begin{center}
\includegraphics[width=\textwidth]{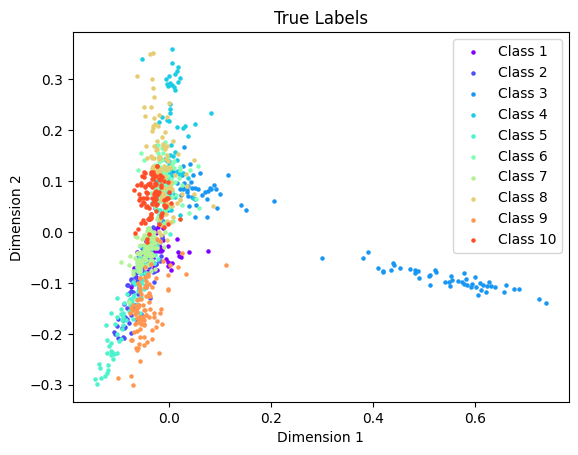}
\caption{True labels.}
\end{center}
\end{subfigure}
\begin{subfigure}[t]{.45\textwidth}
\begin{center}
\includegraphics[width=\textwidth]{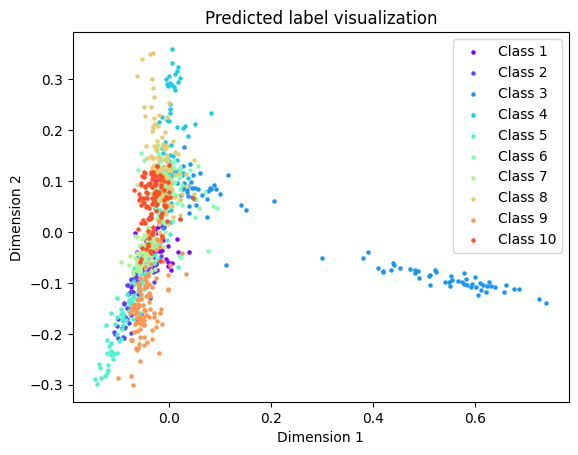}
\caption{Predicted labels.}
\end{center}
\end{subfigure}
\caption{Visual comparison of true labels (left) versus predicted labels output by the model (right) for the document data set.}
\label{fig:doc3}
\end{center}
\end{figure}

\subsection{Salinas Hyperspectral Data}
\label{sec:salinas}

This numerical example is done on a subset of the Salinas hyperspectral image data set from \cite{hsidata}. Our subset of the Salinas data set consists of 20034 data vectors of length 204 belonging to 10 classes of the 16 original classes. Specifically, we took half of the data points at random from each of the first 10 classes of the original data set. For preprocessing we ran PCA and kept the first 50 components. Then we implemented MASC.

In Figure~\ref{fig:sal1} we see the results of applying MASC on the Salinas data in two steps. 
On the left we see the classification task by MASC paused at line 17 of Algorithm~\ref{alg:MASC}, before labels have been extended via the nearest neighbor  portion at the end of the algorithm. 
At this stage, our algorithm has classified 1518 points with 99.60\% accuracy using 261 queries. 
On the right we see the result of the $\overline{k}$-nearest neighbors extension, where all 20034 points have been classified with 97.11\% accuracy. 
In Figure~\ref{fig:sal2} we see a confusion matrix for the result shown in Figure~\ref{fig:sal1}, allowing us to see which classes were classified the most accurately versus which ones had more trouble. We see the largest misclassification involved our predicted class 5, which included points from several other classes. Lastly, in Figure~\ref{fig:sal3} we see a side-by-side comparison of the true labels for the Salinas data set versus the predicted labels.

\begin{figure}[!ht]
\begin{center}
\begin{subfigure}[t]{.45\textwidth}
\begin{center}
\includegraphics[width=\textwidth]{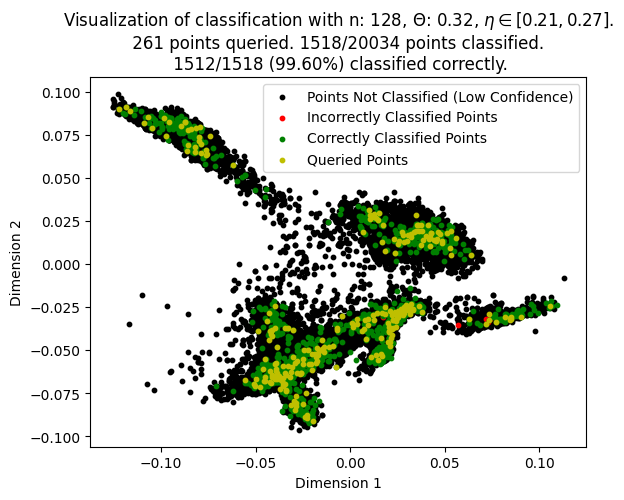}
\subcaption{Classification of certain points in MASC algorithm (before $\overline{k}$-nearest neighbors extension).}
\end{center}
\end{subfigure}
\begin{subfigure}[t]{.45\textwidth}
\begin{center}
\includegraphics[width=\textwidth]{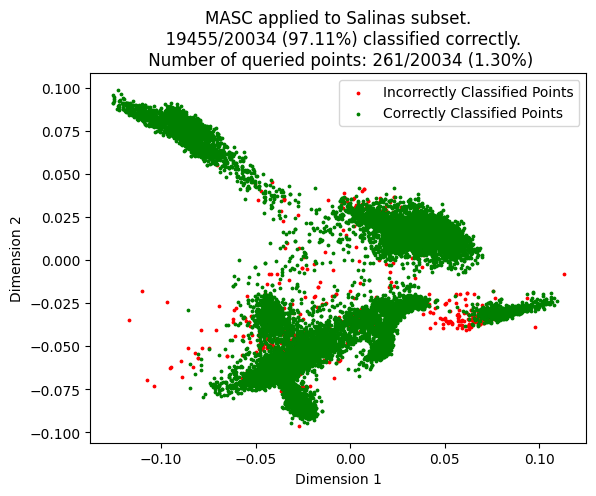}
\subcaption{Classification of remainder points using $\overline{k}$-nearest neighbors extension.}
\end{center}
\end{subfigure}
\caption{This figure illustrates the classification process undergone by MASC at two points on the Salinas hyperspectral data set. On the left, we see the classification of points before the $\overline{k}$-nearest neighbors extension. On the right, we see the result after $\overline{k}$-nearest neighbors extension.}
\label{fig:sal1}
\end{center}
\end{figure}

\begin{figure}[!ht]
\begin{center}
\includegraphics[width=.45\textwidth]{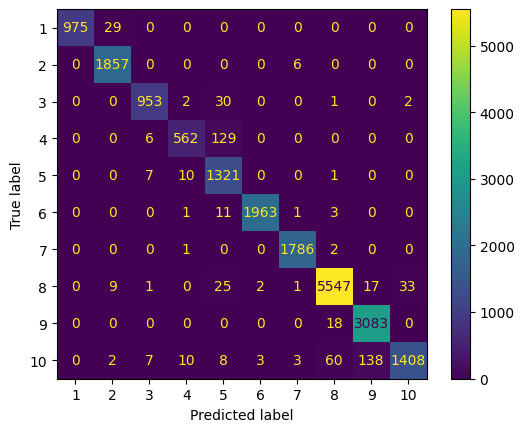}
\caption{Confusion matrix for single run of MASC algorithm on Salinas.}
\label{fig:sal2}
\end{center}
\end{figure}

\begin{figure}[!ht]
\begin{center}
\begin{subfigure}[t]{.45\textwidth}
\begin{center}
\includegraphics[width=\textwidth]{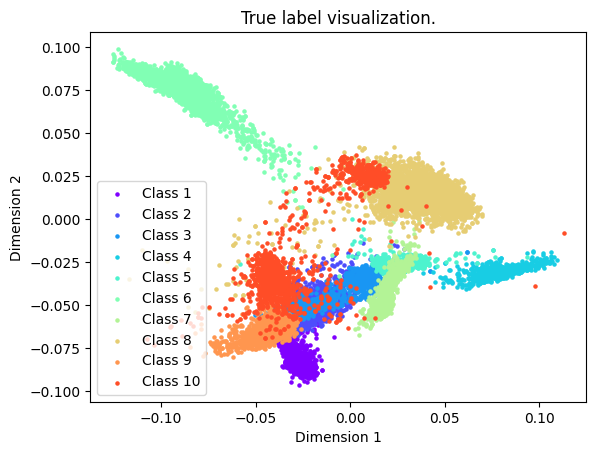}
\caption{True labels.}
\end{center}
\end{subfigure}
\begin{subfigure}[t]{.45\textwidth}
\begin{center}
\includegraphics[width=\textwidth]{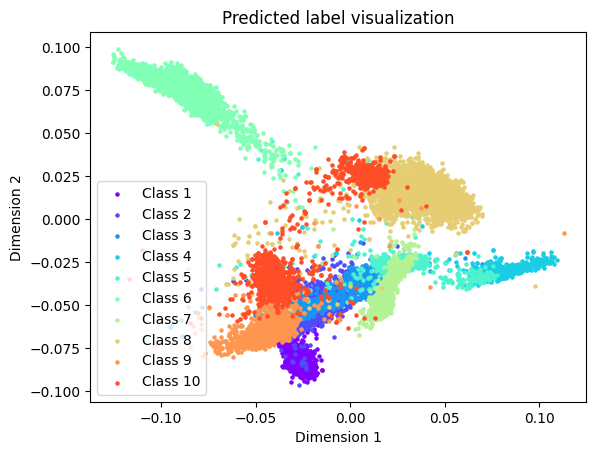}
\caption{Predicted labels.}
\end{center}
\end{subfigure}
\caption{Visual comparison of true labels (left) versus predicted labels output by the model (right) for the Salinas hyperspectral data set.}
\label{fig:sal3}
\end{center}
\end{figure}

\subsection{Indian Pines Hyperspectral Data}
\label{sec:indianpines}

This numerical example is done on a 5-class subset of the Indian Pines hyperspectral image data set from \cite{hsidata}. Our subset of the Indian Pines data set consists of 5971 data vectors of length 200 belonging to classes number 2,6,11,14,16 of the 16 original classes. For preprocessing we normalized each vector. Then we implement MASC.

In Figure~\ref{fig:ip1} we see the results of applying MASC on the Indian Pines data in two steps. On the left we see the classification task by MASC paused at line 17 of Algorithm~\ref{alg:MASC}, before labels have been extended via the nearest neighbor portion at the end of the algorithm. On the right we see the result of the $\overline{k}$-nearest neighbors extension. In Figure~\ref{fig:ip2} we see a confusion matrix for the result shown in Figure~\ref{fig:ip1}, allowing us to see which classes were classified the most accurately versus which ones had more trouble. As we can see from the confusion matrix, the largest error comes from distinguishing class 2 from 11 and vise versa. These classes correspond to portions of the images belonging to corn-notill and soybean-mintill. Lastly, in Figure~\ref{fig:ip3} we see a side-by-side comparison of the true labels for the Indian Pines data set versus the predicted labels.

\begin{figure}[!ht]
\begin{center}
\begin{subfigure}[t]{.45\textwidth}
\begin{center}
\includegraphics[width=\textwidth]{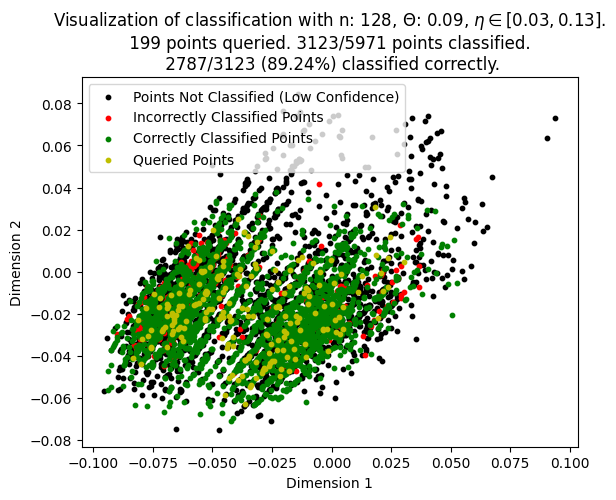}
\subcaption{Classification of certain points in MASC algorithm (before $\overline{k}$-nearest neighbors extension).}
\end{center}
\end{subfigure}
\begin{subfigure}[t]{.45\textwidth}
\begin{center}
\includegraphics[width=\textwidth]{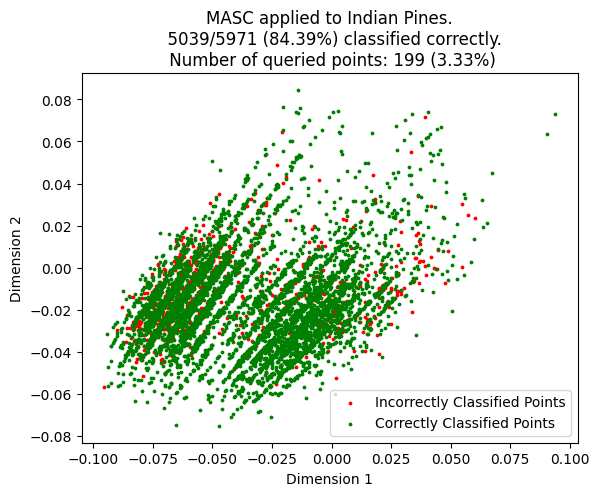}
\subcaption{Classification of remainder points using $\overline{k}$-nearest neighbors extension.}
\end{center}
\end{subfigure}
\caption{This figure illustrates the classification process undergone by MASC at two points on the Salinas hyperspectral data set. On the left, we see the classification of points before the $\overline{k}$-nearest neighbors extension. On the right, we see the result after $\overline{k}$-nearest neighbors extension.}
\label{fig:ip1}
\end{center}
\end{figure}

\begin{figure}[!ht]
\begin{center}
\includegraphics[width=.45\textwidth]{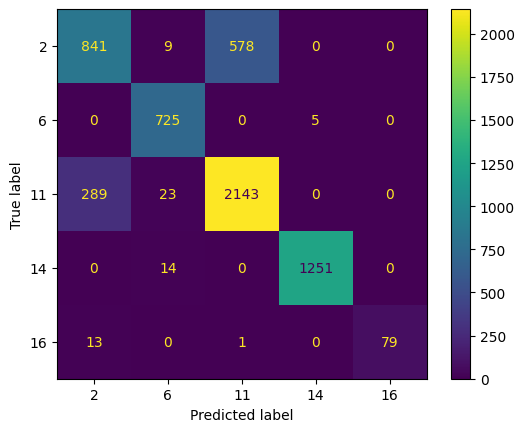}
\caption{Confusion matrix for result of MASC applied to Indian Pines.}
\label{fig:ip2}
\end{center}
\end{figure}

\begin{figure}[!ht]
\begin{center}
\begin{subfigure}[t]{.45\textwidth}
\begin{center}
\includegraphics[width=\textwidth]{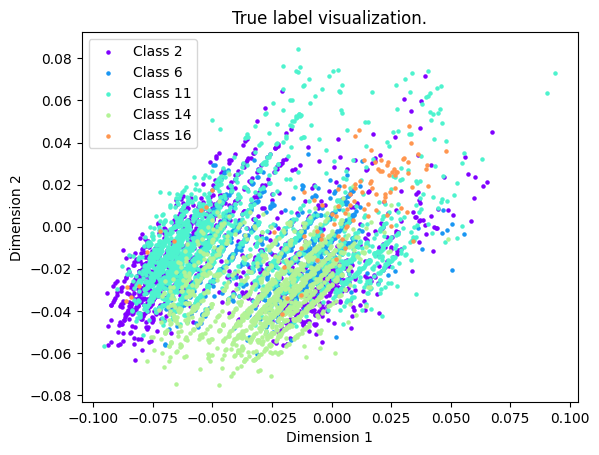}
\caption{True labels.}
\end{center}
\end{subfigure}
\begin{subfigure}[t]{.45\textwidth}
\begin{center}
\includegraphics[width=\textwidth]{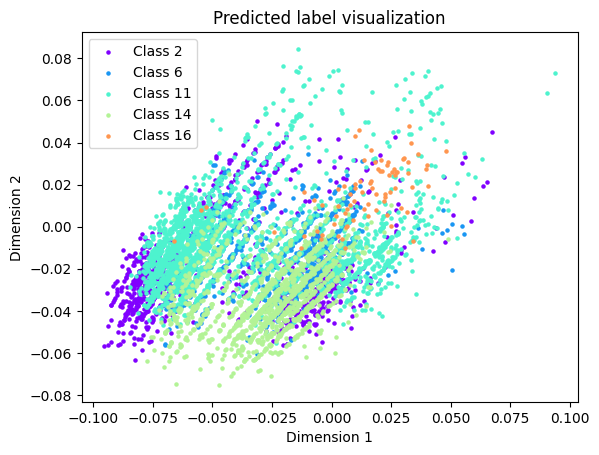}
\caption{Predicted labels.}
\end{center}
\end{subfigure}
\caption{Visual comparison of true labels (left) versus predicted labels output by the model (right) for the Indian Pines hyperspectral data set.}
\label{fig:ip3}
\end{center}
\end{figure}

\subsection{Comparison With LAND and LEND}
\label{sec:comparison}

We compare our method with the LAND \cite{murphy-land} algorithm and its boosted variant, LEND \cite{murphy-lend}. 
In Figure~\ref{fig:comparison}, we see the resulting accuracy that each algorithm achieves on both Salinas and Indian Pines for various query budgets. 
On the left, we observe that our method achieves a comparable accuracy to both LAND and LEND at around 50 queries, then gradually surpasses the accuracy of LAND as the number of queries surpasses around 200. 
On the right, our method achieves a lower accuracy for a small number of queries, but then outperforms both LAND and LEND after the budget exceeds about 60 queries.

The query budgets were decided by how many queries were used at various $\eta$ levels of while loop in the MASC Algorithm~\ref{alg:MASC}. We then forced the nearest-neighbors portion of the MASC algorithm to extend labels to the remainder of the data set at each such level, which is shown in the plot.

A separate aspect of comparison involves the run-time of both algorithms. In Table~\ref{tab:salcompare}, we see that while LEND has the highest accuracy on the Salinas data set with 261 queries, it takes significantly longer than the other two methods to attain this result. Of the three methods, MASC has the quickest run-time at 110.8s, achieving a better accuracy than LAND in less time. In Table~\ref{tab:ipcompare}, we see that MASC produces both the best result and has the fastest run-time for the case of 211 queries on the Indian Pines data set.

When deciding which algorithm to use for an active learning classification task, one has to consider the trade off between query budget/cost, computation time, and accuracy. Our initial results indicate that if the query cost is not so high compared to the run-time of the algorithm, then one may elect to use MASC with its lower run-time and simply query more points. However, if the query cost is high compared to the run-time, then one may instead elect to use an algorithm like LEND instead. The comparison results in this section are not meant to give an exhaustive depiction of which algorithm to use in any case, only illustrate that in two data sets of interest, MASC performs competitively with the existing methods in terms of either or both accuracy and run-time.

\begin{figure}[!ht]
\begin{center}
\begin{subfigure}[t]{.45\textwidth}
\begin{center}
\includegraphics[width=\textwidth]{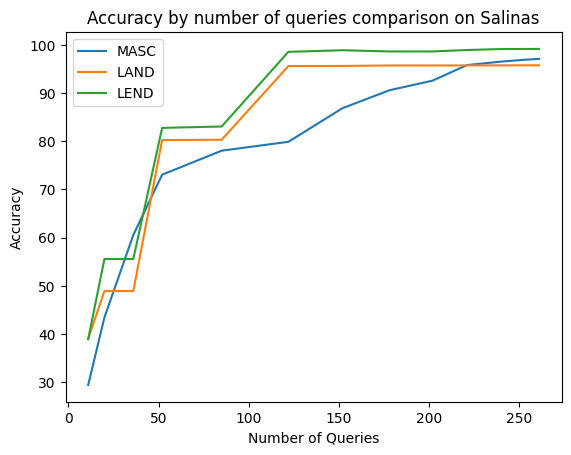}
\subcaption{Plot of accuracy vs. number of query points for Salinas.}
\end{center}
\end{subfigure}
\begin{subfigure}[t]{.45\textwidth}
\begin{center}
\includegraphics[width=\textwidth]{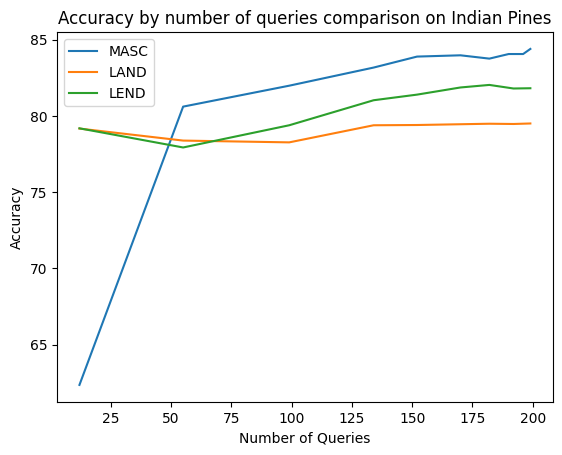}
\subcaption{Plot of accuracy vs. number of query points for Indian Pines.}
\end{center}
\end{subfigure}
\caption{Plots indicating the accuracy of MASC, LAND, and LEND for different query budgets, for both Salinas (left) and Indian Pines (right).}
\label{fig:comparison}
\end{center}
\end{figure}

\begin{table}
\begin{center}
\textbf{Comparison of MASC with LAND and LEND on Salinas subset}
\begin{tabular}{|c|c|c|c|}
\hline
Salinas & MASC & LAND & LEND\\
\hline
Accuracy & 97.1\% & 95.7\% & \textbf{99.2\%}\\
\hline
Run-time & \textbf{110.8s} & 190.0s & 669.1s\\
\hline
\end{tabular}
\end{center}
\caption{Comparison between MASC, LAND, and LEND on the Salinas data set using 261 queries.}
\label{tab:salcompare}
\end{table}

\begin{table}
\begin{center}
\textbf{Comparison of MASC with LAND and LEND on Indian Pines subset}
\begin{tabular}{|c|c|c|c|}
\hline
 & MASC & LAND & LEND\\
\hline
Accuracy & \textbf{84.4\%} & 79.5\% & 82.8\%\\
\hline
Run-time & \textbf{15.5s} & 19.6s & 97.6s\\
\hline
\end{tabular}
\end{center}
\caption{Comparison between MASC, LAND, and LEND on the Indian Pines data set using 211 queries.}
\label{tab:ipcompare}
\end{table}

\subsection{Active Versus Semi-Supervised Learning}

The bulk of the MASC algorithm serves to decide how to cluster the data and query points effectively. An important question to ask is: do we ultimately query useful points? To analyze this question, we compare the results using MASC to those using the same semi-supervised approach drawn from lines 17-21 of Algorithm~\ref{alg:MASC}. We will look at semi-supervised learning on the document data set with 51 random queries and on the Salinas data set with 261 random queries.

In Table~\ref{tab:MASCvsrand} we can see the results from running this experiment 10 times on each data set. On average, the results based on using random queries for lines 17-21 of Algorithm~\ref{alg:MASC} performed 4.60 and 9.03 standard deviations worse than MASC for the document and Salinas data sets respectively. Both cases show that proceeding by querying points randomly will with high probability achieve an accuracy much lower than MASC.

\begin{table}
\begin{center}
\textbf{Comparison of SCALe with function approximation on randomly queried points}\\
\begin{tabular}{|r|l|l|}
\hline
                                  & Document Data & Salinas Data \\
                                  \hline
Number of queries                 & 51            & 261          \\
\hline
MASC accuracy (\%)                   & 87.3         & 97.11           \\
\hline
Mean accuracy of the random trials (\%) & 70.2          & 87.5         \\
\hline
Standard Deviation of the random trials (\%)           & 3.72             & 1.06          \\
\hline
Standard deviations below MASC accuracy  & 4.60        & 9.03        \\
\hline
\end{tabular}
\caption{Table of comparison results between MASC and 10 KNN semi-supervised learning trials with randomly queried data on the Document and Salinas data sets. The mean and standard deviation of the 10 trials in each case are shown, as well has how many standard deviations above these means MASC achieves in the active learning setting.}
\label{tab:MASCvsrand}
\end{center}
\end{table}

\section{Proofs}
\label{sec:proofs}

In this section we give proofs for our main results in Section~\ref{sec:mainresults}. We assume that $\mathbb{X}\coloneqq \operatorname{supp}(\mu)\subseteq \mathbb{M}$ and $n\geq 1$ is given. Essential to our theory is the construction of an integral support estimator:
\be\label{eq:sigmametricdef}
\sigma_n(x)\coloneqq \int_{\mathbb{X}}\Psi_n(x,y)d\mu(y).
\ee
We also define the following two associated values which will be important:
\be\label{eq:InJndef}
I_n\coloneqq \max_{x\in\MM}|\sigma_n(x)|,\qquad J_n\coloneqq \min_{x\in\mathbb{X}}\abs{\sigma_n(x)}.
\ee
Informally, we expect the evaluation of $\sigma_n(x)/I_n$ to give us an estimation on whether or not the point $x$ belongs to $\mathbb{X}$. We encode this intuition by setting a thresholding (hyper)parameter $\theta>0$ in a support estimation set:
\be\label{eq:Sn}
\mathcal{S}_n(\theta)\coloneqq\left\{x\in\mathbb{M}:\sigma_n(x)\geq 4\theta I_n\right\}.
\ee
When the measure $\mu$ is detectable, we show that $\mathcal{S}_n(\theta)$ is an estimate to the support of $\mu$ (Theorem~\ref{thm:suppmu}). When the measure $\mu$ has a fine structure, we show that $\mathcal{S}_n(\theta)$ is partitioned exactly into $K_\eta$ separated components and each component estimates the support of the corresponding partition $\mathbf{S}_{k,\eta}$ (Theorem~\ref{thm:partmu}). These results then give us give us the ability to estimate the classification ability in the discrete setting via probabilistic results, as we investigate in Section~\ref{sec:discreteproofs}.

\subsection{Measure Support Estimation}
\label{sec:measureproofs}

In this section we develop key results to estimate the supports of measures defined on a continuum.
We first start with a useful lemma giving upper and lower bounds on $I_n,J_n$ respectively. Additionally for any given $x\in\mathbb{M}$, we determine a bound for the integral of $\Psi_n$ taken over points away from $x$.

\begin{lemma}\label{lem:integralest}
Let $n\geq 1$ and $S>\alpha$. Then there exist $C_1,C_2>0$ (depending on $\alpha,S,h$) such that
\be\label{eq:In}
I_n=\max_{x\in\MM}|\sigma_n(x)| \leq C_1n^{2-\alpha}
\ee
and
\be\label{eq:Jn}
J_n=\min_{x\in\mathbb{X}}\abs{\sigma_n(x)}\geq C_2n^{2-\alpha}.
\ee
\yadi{$C_1$}{Constant in the upper bound for $\sigma_n$, \eqref{eq:In}}
\yadi{$C_2$}{Constant in a lower bound for $\sigma_n$, \eqref{eq:Jn}}
In particular, $C_1\geq C_2$. For $d>0$ and any $x\in\mathbb{M}$,
\be\label{eq:Kn}
\int_{\mathbb{M}\setminus \mathbb{B}(x,d)}\Psi_n(x,y)d\mu(y)\leq C_1\frac{n^{2-\alpha}}{\max(1,(nd)^{S-\alpha})}.
\ee
\end{lemma}

In order to prove this lemma, we first recall a consequence of the Bernstein inequality for trigonometric polynomials (\cite{paibk}, Chapter~III, Section~3, Theorems~1 and Lemma~5).
\begin{lemma}\label{lemma:bernstein}
Let $T$ be a trigonometric polynomial of order $<2n$. 
Then
\be\label{eq:trigbern}
\|T\|=\max_{x\in\TT}|T'(x)|\le 2n\max_{x\in\TT}|T(x)|.
\ee
Moreover, if $|T(x_0)|=\|T\|$ then
\be\label{eq:triglowbd}
|T(x)|\ge \|T\|\cos(2nx), \qquad |(x-x_0) \mod 2\pi|\le \pi/(2n).
\ee
\end{lemma}
The following corollary gives a consequence of this lemma for the kernel $\Psi_n$, which will be used often in this chapter.
\begin{corollary}\label{cor:kernelbern}
Let $x, y, z, w\in \MM$, $n\ge 1$. Then
there are constants $c, C_0$ such that
\be\label{eq:kernmax}
cn^2\le \Psi_n(x,x)\le C_0 n^2.
\ee
\yadi{$C_0$}{upper bound for $\Psi_n(x,x)/n^2$, \eqref{eq:kernmax}.}
Moreover,
\be\label{eq:kernupbd}
\Psi_n(x,y)\le \Psi_n(x,x)\sim n^2,
\ee
\be\label{eq:kernbern}
|\Psi_n(x,y)-\Psi_n(z,w)|\ls n^3\left\{\rho(x,z)+\rho(y,w)\right\}.
\ee
and
\be\label{eq:kernlowbd}
|\Psi_n(x,y)| \gs n^2, \qquad \mbox{ for } \rho(x,y)\le \pi/(6n).
\ee
\end{corollary}
\begin{proof} \ % of Corollary~\ref{cor:kernelbern}
The estimate \eqref{eq:kernmax} follows from the fact that
$$
\Psi_n(x,x)=\Phi_n(0)^2 =\left(\sum_\ell h(\ell/n)\right)^2 \sim n^2,
$$ 
where the last estimate is easy to see using Riemann sums for $\int h(t)dt$.
We observe that $\Phi_n^2$ is a trigonometric polynomial, and it is clear that 
$$
|\Phi_n(t)|^2\le \Phi_n(0)^2.
$$
Consequently, \eqref{eq:kernupbd} follows from the definition of $\Psi_n$. 
The estimate \eqref{eq:kernbern} is easy to deduce from the fact that $\|(\Phi_n^2)'\|\ls n^3$, so that
\be
\ba
|\Psi_n(x,y)-\Psi_n(z,w)|\le& |\Phi_n^2(\rho(x,y))-\Phi_n^2(\rho(z,w))|\\
\le& |\Phi_n^2(\rho(x,y))-\Phi_n^2(\rho(z,y))|+|\Phi_n^2(\rho(z,y))-\Phi_n^2(\rho(z,w))|\\
\ls& n^3\left\{|\rho(x,y)-\rho(z,y)|+|\rho(z,y)-
\rho(z,w)|\right\}\\
\ls& n^3\left\{\rho(x,z)+\rho(y,w)\right\}.
\ea
\ee
The estimate \eqref{eq:kernlowbd} follows from \eqref{eq:triglowbd} and the definition of $\Psi_n$.
\end{proof}
 
\noindent\textit{Proof of Lemma~\ref{lem:integralest}.} We proceed by examining concentric annuli. Let $x\in \mathbb{M}$ be fixed, and set $A_0=\mathbb{B}(x,d)$ and $A_k=\mathbb{B}(x,2^kd)\setminus \mathbb{B}(x,2^{k-1}d)$ for every $k\geq 1$. First suppose that $nd\geq 1$. Then by \eqref{eq:metrickernloc}~and~\eqref{eq:ballmeasurecon}, we deduce
\be\label{eq:annuli}
\ba
\int_{\mathbb{M}\setminus \mathbb{B}(x,d)}\Psi_n(x,y)d\mu(y)=&\sum_{k=1}^\infty\int_{A_k}\Psi_n(x,y)d\mu(y)
\lesssim\sum_{k=1}^\infty \frac{\mu(A_k)n^2}{\max(1,2^{k-1}dn)^S}\\
&\lesssim\sum_{k=1}^\infty \frac{2^{k\alpha} d^\alpha n^2}{2^{S(k-1)}(dn)^S}
\lesssim n^{2-\alpha}(nd)^{\alpha-S}\sum_{k=1}^\infty 2^{k(\alpha-S)}
\lesssim  n^{2-\alpha}(nd)^{\alpha-S}.
\ea
\ee
If $nd=1$, we observe
\be\label{eq:annuli2}
\int_{A_0}\Phi_{n}(\rho(x,y))^2d\mu(y)\lesssim \mu(A_0)n^2\lesssim d^{\alpha}n^2=n^{2-\alpha}.
\ee
Combining \eqref{eq:annuli}~and~\eqref{eq:annuli2} when $nd=1$ yields \eqref{eq:In}. When $dn\leq 1$, we see 
\be
\int_{\mathbb{M}\setminus\mathbb{B}(x,d)}\Psi_n(x,y)d\mu(y)\leq I_n\lesssim n^{2-\alpha}.
\ee
Together with \eqref{eq:annuli}, this completes the proof of \eqref{eq:Kn}. 
There is no loss of generality in using the same constant $C_1$ in both of these estimates.
We see, in view of \eqref{eq:kernupbd}, \eqref{eq:kernlowbd}, and the detectability of $\mu$, that if $x\in\mathbb{X}$ it follows that
\be
\int_{\mathbb{X}} \Psi_n(x,y)d\mu(y)\gtrsim \int_{\mathbb{B}(x,\pi/(6n))} n^2 d\mu(y)\gtrsim n^{2-\alpha},
\ee
demonstrating \eqref{eq:Jn} and completing the proof.
\qed

\begin{theorem}\label{thm:suppmu}
Let $\mu$ be detectable and $S>\alpha$. 
If $\theta\leq C_2/(4C_1)$, then by setting
\be\label{eq:dtheta}
d(\theta)=\left(\frac{C_1}{C_2\theta}\right)^{1/(S-\alpha)},
\ee
it follows that (cf. \eqref{eq:Sn})
\be\label{eq:thm1inc}
\mathbb{X}\subseteq \mathcal{S}_{n}(\theta)\subseteq \mathbb{B}(\mathbb{X},d(\theta)/n).
\ee
\end{theorem}

\begin{proof}
From \eqref{eq:In}~and~\eqref{eq:Jn}, we see that for any $x\in\mathbb{X}$,
\be
\sigma_n(x)\geq J_n \geq \frac{C_2I_n}{C_1}.
\ee
With our assumption of $\theta\leq C_2/(4C_1)$, this proves the inclusion
\be
\mathbb{X}\subseteq \mathcal{S}_n(\theta).
\ee
Note that $C_1^2/C_2^2\ge 1 >1/4$, so that $\theta\leq C_2/(4C_1)< C_1/C_2$, and hence, $d(\theta)>1$. Then, for any $x\in \mathbb{M}$ such that $\dist(x,\mathbb{X})\geq d(\theta)/n$, we have by \eqref{eq:Kn} that
\be
\sigma_n(x)\leq \int_{\mathbb{M}\setminus \mathbb{B}(x,d(\theta)/n)}\Psi_n(x,y)d\mu(y)\leq C_1n^{2-\alpha}/d(\theta)^{S-\alpha}\leq \theta C_2n^{2-\alpha}\leq \theta I_n.
\ee
This demonstrates the inclusion
\be
\mathcal{S}_n(\theta)\subseteq \mathbb{B}(\mathbb{X},d(\theta)/n),
\ee
completing the proof.
\end{proof}

\begin{theorem}\label{thm:partmu}
Assume the setup of Theorem~\ref{thm:suppmu} and suppose $\mu$ has a fine structure. Define
\be\label{eq:Skn}
\mathcal{S}_{k,\eta,n}(\theta)\coloneqq \mathcal{S}_n(\theta)\cap \mathbb{B}(\mathbf{S}_{k,\eta},d(\theta)/n).
\ee
Let $n\geq 2d(\theta)/\eta$, $\mu(\mathbf{S}_{K_{\eta}+1,\eta})\leq \frac{C_2}{C_0}\theta n^{-\alpha}$, and $j,k=1,\dots,K_\eta$ with $j\neq k$.  Then 
\be\label{eq:Snpartition}
\mathcal{S}_n(\theta)=\bigcup_{k=1}^{K_\eta} \mathcal{S}_{k,\eta,n}(\theta)
\ee
and,
\be\label{eq:minsep2}
\operatorname{dist}(\mathcal{S}_{j,\eta,n}(\theta),\mathcal{S}_{k,\eta,n}(\theta))\geq \eta.
\ee
Furthermore,
\be\label{eq:thm2inc}
\mathbb{X}\cap \mathbb{B}(\mathbf{S}_{k,\eta},d(\theta)/n)\subseteq \mathcal{S}_{k,\eta,n}(\theta)\subseteq \mathbb{B}(\mathbf{S}_{k,\eta},d(\theta)/n).
\ee
\end{theorem}

\begin{proof}
The first inclusion in \eqref{eq:thm2inc} is satisfied from \eqref{eq:thm1inc} and the second is satisfied by the definition of $\mathcal{S}_{k,\eta,n}$. In view of the assumption that $\eta\geq 2d(\theta)/n$ and Definition~\ref{def:finestructure}, we see that 
\be
\dist(\mathbb{B}(\mathbf{S}_{j,\eta},d(\theta)/n),\mathbb{B}(\mathbf{S}_{k,\eta},d(\theta)/n))\geq \eta,
\ee
for any $j\neq k$. Since $\mathcal{S}_{k,\eta,n}(\theta)\subseteq \mathbb{B}(\mathbf{S}_{k,\eta},d(\theta)/n)$, it follows that the separation condition \eqref{eq:minsep2} must also be satisfied.
Now it remains to show \eqref{eq:Snpartition}.
Let us define, in this proof only,
\be
\mathbf{S}=\bigcup_{k=1}^{K_\eta}\mathbf{S}_{k,\eta}.
\ee
It is clear from \eqref{eq:Skn} that $\bigcup_{k=1}^{K_\eta} \mathcal{S}_{k,\eta,n}(\theta)\subseteq\mathcal{S}_n(\theta)$. 
We note that for any $x\in\mathbb{M}\setminus \mathbb{B}(\mathbf{S},d(\theta)/n)$, we have $\dist(x,\mathbf{S})\geq d(\theta)/n$ and as a result
\be
\ba
\sigma_n(x)=& \int_{\mathbf{S}_{K_\eta+1}}\Psi_n(x,y)d\mu(y)+\int_{\mathbf{S}}\Psi_n(x,y)d\mu(y)\\
\leq& C_0n^2\mu(\mathbf{S}_{K_{\eta}+1})+\int_{\mathbf{S}\setminus \mathbb{B}(x,d(\theta)/n)}\Psi_n(x,y)d\mu(y)&\text{(By \eqref{eq:kernupbd})}\\
\leq&C_2n^{2-\alpha}\theta+C_1n^{2-\alpha}d(\theta)^{\alpha-S}&\text{(By the assumption on $\mu(\mathbf{S}_{K_\eta+1})$ and \eqref{eq:Kn})}\\
\leq& 2C_2n^{2-\alpha}\theta \le 2J_n\theta &\text{(By \eqref{eq:dtheta})}\\
\leq&2\theta I_n.
\ea
\ee
Thus, $x\notin \mathcal{S}_n(\theta)$ and, equivalently, $\mathcal{S}_n(\theta)\subseteq \mathbb{B}(\mathbf{S},d(\theta)/n)$. Therefore we have shown $\mathcal{S}_n(\theta)=\mathbb{B}(\mathbf{S},d(\theta)/n)$, completing the proof.
\end{proof}

\subsection{Discretization}
\label{sec:discreteproofs}

In this section we relate the continuous support estimator and estimation sets to the discrete cases based on randomly sampled data. The conclusion of this section will be the proofs to the theorems from Section~\ref{sec:mainresults}. To aid us in this process we first state a consequence of the Bernstein Concentration inequality as a proposition.

\begin{proposition}\label{prop:concentration}
Let $X_1,\cdots, X_M$ be independent real valued random variables such that for each $j=1,\cdots,M$, $|X_j|\le R$, and $\mathbb{E}(X_j^2)\le V$. Then for any $t>0$,
\be\label{bernstein_concentration}
\mathsf{Prob}\left( \left|\frac{1}{M}\sum_{j=1}^M (X_j-\mathbb{E}(X_j))\right| \ge Vt/R\right) \le 2\exp\left(-\frac{MVt^2}{2R^2(1+t)}\right).
\ee
\end{proposition}

%\begin{proposition}[Bernstein concentration inequality] Let $Z_1,\dots,Z_M$ be independent random variables where for each $j=1,\dots,M$, $|Z_j|\leq R$ and $\mathbb{E}(Z_j^2)\leq V$. Then for any $t>0$ we have,
%\be
%\mathbf{P}\left(\abs{\frac{1}{M}\sum_{j=1}^M(Z_j-\mathbb{E}(Z_j))}\geq t\right)\leq 2\exp\left(-\frac{Mt^2}{2(V+Rt/3)}\right).
%\ee
%\label{prop:bernsteinconcentration}
%\end{proposition}

%The second result is the multiplicative Chernoff bounds.
%
%\begin{proposition}[Multiplicative Chernoff bounds]
%Let $M\geq 1$, $0\leq p\leq 1$, and $X_1,\dots,X_M$ be random variables taking values in $\{0,1\}$, with $\operatorname{Prob}(X_k=1)=p$. Then for any $\epsilon\in (0,1]$, we have
%\be
%\operatorname{Prob}\left(\sum_{k=1}^M X_k\leq (1-\epsilon)Mp\right)\leq \exp\left(-\epsilon^2Mp/2\right).
%\ee
%\label{prop:chernoffbound}
%\end{proposition}

Information and a derivation for the Bernstein concentration inequality are standard among many texts in probability; we list \cite[Section~2.1, 2.7]{concentration} as a reference. 
It is instinctive to use Proposition~\ref{prop:concentration} with $X_j=\Psi_n(x,x_j)$. 
This would yield the desired bound for any value of $x\in\MM$. 
However, to get an estimate on the supremum norm of the difference $F_n-\sigma_n$, we need to find an appropriate net for $\MM$ (and estimate its size) so that the point where this supremum is attained is within the right ball around one of the points of the net. 
Usually, this is done via a Bernstein inequality for the gradients of the objects involved.
In the absence of any differentiability structure on $\MM$, we need a more elaborate argument.

The following proposition is a consequence of \cite[Theorem~7.2]{mhaskardata}, and asserts the existence of a partition of $\mathbb{X}$ satisfying properties which will be helpful to proving our main results.

\begin{proposition}\label{prop:partition} 
Let $\delta>0$. There exists a partition $\{Y_k\}_{k=1}^N$ of $\XX$ such that for each $k$, $\mathsf{diam}(Y_k)\le 36\delta$, and $\mu(Y_k)\sim \delta^\alpha$. 
In particular, $N\ls \delta^{-\alpha}$.
\end{proposition}

Recall that we denote our data by $\mathcal{D}=\{x_j\}_{j=1}^M$, where each $x_j$ is sampled uniformly at random from $\mu$. In the sequel, we let  $\D_k=\D\cap Y_k$, $k=1,\cdots, N$.
The following lemma gives an estimate for $|\D_k|$.
\begin{lemma}\label{lemma:discpartition}
Let $0<\delta, \epsilon, t<1$. 
If 
\be\label{eq:Mcond}
M\gs t^{-2}\delta^{-\alpha}\log(c/(\epsilon\delta^\alpha)),
\ee  
then
\be\label{eq:discpartition}
\mathsf{Prob}\left(\max_{1\le k\le N}\left|\frac{\mu(Y_k)M}{|\D_k|}-1\right|\ge t\right)\ls \epsilon.
\ee
\end{lemma}

\begin{proof}\ %of Lemma~\ref{lemma:discpartition}.
Let $k$ be fixed, and in this proof only,  $\chi_k$ denotes the characteristic function of $Y_k$. 
Thought of as a random variable, it is clear that $|\chi_k|\le 1$, $\int \chi_k(z)d\mu(z)=\int \chi_k(z)^2d\mu(z)=\mu(Y_k)$. 
Moreover, $\sum_{j=1}^M \chi_k(z_j)=|\D_k|$.
So, we may apply Proposition~\ref{prop:concentration}, and recall that $\mu(Y_k)\sim\delta^\alpha$ to conclude that
\be\label{eq:pf1eqn1c}
\begin{aligned}
\mathsf{Prob}\left(\left|\frac{|\D_k|}{M}-\mu(Y_k)\right|\ge \frac{\mu(Y_k)t}{1+t}\right) &= \mathsf{Prob}\left(\left|\frac{|\D_k|}{M\mu(Y_k)}-1\right|\ge \frac{t}{1+t}\right)\\
&\le 2\exp\left(-\frac{M\mu(Y_k)t^2}{(1+t)(1+2t)}\right)\le 2\exp\left(-cM\delta^\alpha t^2\right).
\end{aligned}
\ee
(In the last estimate, we have used the fact that for $0<t<1$, $(1+t)(1+2t)\sim 1$.)
Next, we observe that
$$
\left|\frac{M\mu(Y_k)}{|\D_k|}-1\right|=\frac{M\mu(Y_k)}{|\D_k|}\left|\frac{|\D_k|}{M\mu(Y_k)}-1\right|.
$$
\vskip 2pt
So, if $\disp \left|\frac{|\D_k|}{M\mu(Y_k)}-1\right|< \frac{t}{1+t}$, then $\disp \frac{|\D_k|}{M\mu(Y_k)}\ge 1/(1+t)$, and hence,
$\disp \left|\frac{M\mu(Y_k)}{|\D_k|}-1\right|<t$.
Thus, for every $k$,
\be\label{eq:pf1eqn2c}
\mathsf{Prob}\left(\left|\frac{M\mu(Y_k)}{|\D_k|}-1\right|\ge t\right) \le \mathsf{Prob}\left(\left|\frac{|\D_k|}{M\mu(Y_k)}-1\right|\ge \frac{t}{1+t}\right)\le 2\exp\left(-cM\delta^\alpha t^2\right).
\ee
Since the number of elements $Y_k$ in the partition is $\ls \delta^{-\alpha}$,
we conclude that
\be\label{eq:pf1eqn3}
\mathsf{Prob}\left(\max_{k\in [N]}\left|\frac{M\mu(Y_k)}{|\D_k|}-1\right|\ge t\right)\ls \delta^{-\alpha}\exp\left(-cM\delta^\alpha t^2\right).
\ee
We set the right hand side of the above inequality to $\epsilon$ and solve for $M$ to complete the proof.
\end{proof}
%In the sequel, we write
%$$
%\Phi_n(\theta)=\sum_{\ell}h\left(\frac{|\ell|}{n}\right)\exp(i\ell\theta),
%$$
%and
%$$
%\Psi_n(x,y)=\Phi_n(\rho(x,y))^2.
%$$
%We note that the Bernstein inequality for derivatives of trigonometric polynomials and the fact that $|\Phi_n(\theta)|\le cn$ for all $\theta$ implies that for all $x,y, z,w\in\mathbb{M}$,
%\be\label{eq:kernbern}
%\begin{aligned}
%|\Psi_n(x,y)-\Psi_n(z,w)|&\le |\Psi_n(x,y)-\Psi_n(y,z)|+|\Psi_n(y,z)-\Psi_n(z,w)| \le cn^3\left\{|\rho(x,y)-\rho(y,z)|+\rho(y,z)-\rho(z,w)|\right\}\\
%&\le cn^3\left\{\rho(x,z)+\rho(y,w)\right\}.
%\end{aligned}
%\ee

In order to prove the bounds we want in Lemma~\ref{lem:Inbounds}, we rely on a function which estimates both our discrete and continuous measure support estimators $F_n$ and $\sigma_n$. We define this function as
\be
H_n(x)\coloneqq \sum_{k=1}^N \frac{\mu(Y_k)}{|\D_k|}\sum_{x_j\in\D_k}\Psi_n(x,x_j).
\ee
The following lemma relates this function to our continuous measure support estimator.

\begin{lemma}\label{lemma:discprelim}
Let $0<\gamma<2$, $n\ge 2$. 
There exists a constant $c(\gamma)$ with the following property. 
Suppose $0<\delta\le c(\gamma)/n$, $\{Y_k\}$ be a partition of $\XX$ as in Proposition~\ref{prop:partition}, and we continue the notation before. 
We have
\be\label{eq:discprelim}
\max_{x\in\mathbb{M}}\left|H_n(x) -\sigma_n(x)\right| \le (\gamma/2) I_n.
\ee
\end{lemma}
\begin{proof}\ % of Lemma~\ref{lemma:discprelim}
In this proof, all the constants denoted by $c_1,c_2,\cdots$ will retain their values.
Let $x\in\mathbb{M}$.
We will fix $\delta$ to be chosen later.
Also, let $r\ge \delta$ be a parameter to be chosen later, $\mathcal{N}=\{k : \mathsf{dist}(x, Y_k) <r\}$, $\mathcal{L}=\{k : \mathsf{dist}(x, Y_k) \ge r\}$ and for $j=0,1,\cdots$, $\mathcal{L}_j=\{k :  2^jr\le \mathsf{dist}(x, Y_k) <2^{j+1} r\}$. 

In view of \eqref{eq:kernbern}, we have for $k\in \mathcal{N}$ and $x_j\in \D_k$,
\be
\ba
\left|\mu(Y_k)\Psi_n(x,x_j)-\int_{Y_k}\Psi_n(x,y)d\mu(y)\right|\le& \int_{Y_k}\left|\Psi_n(x,x_j)-\Psi_n(x,y)\right|d\mu(y)\\
\ls & n^3\int_{Y_k}\rho(z,y)d\mu(y) \le c_1n^3\mathsf{diam}(Y_k)\mu(Y_k).
\ea
\ee
Consequenty, for $k\in \mathcal{N}$,
\be\label{eq:pf2eqn1c}
\left|\frac{\mu(Y_k)}{|\D_k|}\sum_{x_j\in\D_k}\Psi_n(x,x_j) -\int_{Y_k}\Psi_n(x,y)d\mu(y)\right| \le c_1n^3\mathsf{diam}(Y_k)\mu(Y_k).
\ee
Since $\cup_{k\in\mathcal{N}}Y_k\subseteq \BB(x,r)$, we have 
\be\label{eq:pf2eqn5}
\sum_{k\in\mathcal{N}}\mu(Y_k) =\mu\left(\cup_{k\in\mathcal{N}}Y_k\right)\le \mu(\BB(x,r))\ls r^\alpha.
\ee
We deduce from \eqref{eq:pf2eqn1c}, \eqref{eq:pf2eqn5} and the fact that $\mathsf{diam}(Y_k)\ls \delta$ that
\be\label{eq:pf2eqn2}
\left|\sum_{k\in\mathcal{N}}\left(\frac{\mu(Y_k)}{|\D_k|}\sum_{x_j\in\D_k}\Psi_n(x,x_j) -\int_{Y_k}\Psi_n(x,y)d\mu(y)\right)\right| \le 
c_3n^3\delta r^\alpha = c_3(n\delta) (nr)^\alpha n^{2-\alpha}. 
\ee
Next, let $k\in \mathcal{L}_j$ for some $j\ge 0$. 
Then the localization estimate \eqref{eq:metrickernloc} shows that for any $x_j\in \D_k$,
$$
\left|\mu(Y_k)\Psi_n(x,x_j)-\int_{Y_k}\Psi_n(x,y)d\mu(y)\right|\le c_4n^2 (2^jnr)^{-S}\mu(Y_k),
$$
so that
\be\label{eq:pf2eqn3}
\left|\frac{\mu(Y_k)}{|\D_k|}\sum_{x_j\in \D_k}\Psi_n(x,x_j) -\int_{Y_k}\Psi_n(x,y)d\mu(y)\right|\le c_4n^2\mu(Y_k) (2^jnr)^{-S}.
\ee
Arguing as in the derivation of \eqref{eq:pf2eqn5}, we deduce that 
$$
\mu\left(\cup_{k\in \mathcal{L}_j}Y_k\right)\ls (2^jr)^\alpha.
$$
 Since $S>\alpha$, we deduce that
if $r\ge 1/n$, then
\be\label{eq:pf2eqn4}
\begin{aligned}
\left|\sum_{k\in \mathcal{L}} \frac{\mu(Y_k)}{|\D_k|}\sum_{x_j\in \D_k}\Psi_n(x,x_j) -\int_{Y_k}\Psi_n(x,y)d\mu(y)\right|&\le \sum_{j=0}^\infty \sum_{k\in \mathcal{L}_j} \left|\frac{\mu(Y_k)}{|\D_k|}\sum_{x_j\in \D_k}\Psi_n(x,x_j) -\int_{Y_k}\Psi_n(x,y)d\mu(y)\right|\\
&\le c_4n^{2-\alpha}(nr)^{\alpha-S}\sum_{j=0}^\infty 2^{j(\alpha-S)} \le c_5n^{2-\alpha}(nr)^{\alpha-S}.
\end{aligned}
\ee
Since $S>\alpha$, we may choose $r\sim_\gamma n$ such that $c_5 (nr)^{\alpha-S}\le \gamma/4$, and then require $\delta\le \min(r, c_6(\gamma)/n)$ so that in \eqref{eq:pf2eqn2}, $c_3(nr)^\alpha n\delta\le \gamma/4$.
Then, recalling that (cf. \eqref{eq:In}) $I_n\sim n^{2-\alpha}$, \eqref{eq:pf2eqn4} and \eqref{eq:pf2eqn2} lead to \eqref{eq:discprelim}.
\end{proof}

In the following lemma, we establish a connection between the sum $F_n$ as defined in \eqref{eq:support_estimator_def} and the value $I_n$ from Section~\ref{sec:measureproofs}. Since we have already established the bound between $H_n$ and $\sigma_n$, we focus on the bound between $H_n$ and $F_n$ in this lemma.

\begin{lemma}\label{lem:Inbounds}
Let $n\ge 2$, $0<\beta<2$, $M\gs_\beta n^\alpha\log n$, and $\mathcal{D}=\{x_j\}_{j=1}^M$ be independent random samples from a detectable measure $\mu$. 
Then with probability $\ge 1-1/n$, we have for all $x\in \mathbb{M}$,
\be\label{eq:betacomp1}
 |F_n(x)-\sigma_n(x)|\leq \beta I_n.
\ee
Consequently,
\be\label{eq:betacomp2}
(1-\beta)I_n\leq \max_{x\in\mathbb{M}} F_n(x)\leq (1+\beta)I_n.
\ee
\end{lemma}

\begin{proof}\ % of Theorem~\ref{theo:discretization}.
Let $\gamma\in (0,1)$ to be chosen later, and $\{Y_k\}$ be a partition as in Lemma~\ref{lemma:discprelim}.
In view of Lemma~\ref{lemma:discprelim}, we see that
\be\label{eq:pf3eqn1}
(1-\gamma/2)I_n\le \max_{x\in\mathbb{M}} H_n(x)\le (1+\gamma/2) I_n.
\ee
In view of Lemma~\ref{lemma:discpartition}, we see that for $M\ge c(\gamma)n^\alpha\log n$, we have with probability $\ge 1-1/n$,
\be\label{eq:pf3eqn2}
\max_k \left|\frac{\mu(Y_k)M}{|\D_k|}-1\right|\le \gamma/2.
\ee
In particular, 
\be
1-\gamma/2\le \frac{\mu(Y_k)M}{|\D_k|}\le 1+\gamma/2.
\ee
Hence, \eqref{eq:pf3eqn1} leads to
\be\label{eq:pf3eqn3}
\max_{x\in\mathbb{M}}F_n(x)\le \frac{2+\gamma}{2-\gamma}I_n.
\ee
Using \eqref{eq:pf3eqn2} again, we see that
\be
\left|F_n(x)-H_n(x)\right| \le (\gamma/2)\max_{x\in\mathbb{M}}F_n(x)\le (\gamma/2)\frac{2+\gamma}{2-\gamma}I_n.
\ee
Together with \eqref{eq:discprelim}, this implies that
\be\label{eq:pf3eqn4}
\left|F_n(x)-\sigma_n(x)\right| \le (\gamma/2)\left(1+\frac{2+\gamma}{2-\gamma}\right)I_n=\frac{2\gamma}{4-\gamma}I_n
\ee
We now choose $\gamma=4\beta/(2+\beta)$, so that the right hand side of \eqref{eq:pf3eqn4} is $\beta I_n$. We can verify $0<\gamma<2$ whenever $\beta<2$.
\end{proof}

The following lemma gives bounds on $\max_{k\in [M]}F_n(x_k)$, which is a crucial component to the proof involving our finite data support estimation set $\mathcal{G}_n(\Theta)$. We note briefly the critical difference between Lemma~\ref{lem:Inbounds} and Lemma~\ref{lem:Gnboundsprep} is that the former considers the maximum of $F_n$ over the entire metric space $\MM$, while the latter considers the maximum over the finite set of data points which are sampled from the measure $\mu$.

\begin{lemma}
\label{lem:Gnboundsprep}
Let $\mathcal{D}=\{x_j\}_{j=1}^M$ be independent random samples from a detectable measure $\mu$. If $0<\beta<C_2/C_1$ and $M\gtrsim_\beta n^\alpha\log(n)$, with  probability $\geq 1-1/n$ we have
\be\label{eq:Fnxkbound}
\left(\frac{C_2}{C_1}-\beta\right)I_n\leq \max_{k\in [M]}F_n(x_k)\leq (1+\beta)I_n.
\ee
\end{lemma}

\begin{proof}
Necessarily, $\mathcal{D}\subseteq \XX$.
Using \eqref{eq:Jn}, we deduce that
\be
\max_{k\in [M]}F_n(x)\geq \max_{k\in [M]}\sigma_n(x_k)-\beta I_n\geq J_n -\beta I_n\ge C_2n^{2-\alpha}-\beta I_n \geq \left(\frac{C_2}{C_1}-\beta\right)I_n.
\ee
This proves \eqref{eq:Fnxkbound}.
The second inequality is satisfied by \eqref{eq:betacomp2} since $\max_{k\in[M]}F_n(x_k)\leq \max_{x\in\mathbb{M}}F_n(x)$. 
%By \eqref{eq:betacomp1}, we can see that
%\be\label{eq:maxFnxk}
%\max_{k\in[M]}F_n(x_k)\geq \max_{k\in[M]}\sigma_n(x_k)-\beta I_n.
%\ee
%Thus, it will suffice to show that $\max_{k\in[M]}\sigma_n(x_k)\gtrsim I_n$. Let $x^*$ be the maximizer over $\mathbb{X}$ of $\sigma_n(x)$. We have from Lemma~\ref{lem:integralest} that
%\be\label{eq:sigmaxs}
%\sigma_n(x^*)\gtrsim I_n.
%\ee
%Then by \eqref{eq:bernstein}, we see that there exists $c=c(\beta)$ such that for any $x\in \mathbb{B}(x^*,(2cn)^{-1})$,
%\be\label{eq:sigmax}
%\sigma_n(x)\geq (\beta+C^*) I_n.
%\ee
%If there exists some $x_k\in \mathcal{D}\cap \mathbb{B}(x^*,(2cn)^{-1})$, then combining  \eqref{eq:maxFnxk} and \eqref{eq:sigmax} 
%So, we would like to show there is an $x_k\in \mathbb{B}(x^*,(2cn)^{-1})$. Since $\mu(\mathbb{B}(x^*,(2cn)^{-1}))\gtrsim n^{-\alpha}$ and $M\gtrsim n^\alpha\log(n)$, we may use \ref{prop:chernoffbound} with $X_1,\dots,X_M$ defined by
%\be
%X_k=\begin{cases}1& x_k\in\mathbb{B}(x^*,(2cn)^{-1})\\0&\text{else.}\end{cases}
%\ee
%Thus, in view of Proposition~\ref{prop:chernoffbound} with $p\sim n^{-\alpha}$, there exists some $\overline{c}$ such that when $M\geq \overline{c}n^\alpha \log(n)$ we have
%\be
%\mathbf{P}\left(\sum_{k=1}^M X_k\leq 0\right)\leq \exp(-Mp/2)\leq \exp(-\overline{c}\log(n))\leq \frac{1}{2n}.
%\ee
%In other words, with probability at least $1-\frac{1}{2n}$ there exists some $k$ such that $x_k\in \mathbb{B}(x^*,(2cn)^{-1})$.
\end{proof}

Now with the prepared bounds on $\max_{k\in [M]}F_n(x_k)$, we give a theorem from which Theorems~\ref{thm:fullsuportdet} and~\ref{thm:class_separation} are direct consequences.
In the sequel, we will denote $C_2/C_1$ by $C^*$. \yadi{$C^*$}{$C_2/C_1$}

\begin{theorem}\label{thm:GnSnbound}
Let $n\geq 2$, and $\mathcal{D}=\{x_j\}_{j=1}^M$ be sampled from the detectable probability measure $\mu$.
Let $0<\Theta< 1$.
%\min\left\{\frac{C_2}{C_{1}(1+C^*)},\frac{4}{C^*},1\right\}$,
%\be\label{eq:Thetacond}
%0<\Theta< \min\left(1, \frac{C_2}{C_1(1+C^*)}, \frac{4}{C^*}\right).
%\ee
 If $M\gtrsim n^\alpha\log(n)$ then with probability at least $1-c_1M^{-c_2}$ we have
\be\label{eq:SnGncomp}
\mathcal{S}_n\left(\frac{(1+C^*)\Theta}{4}\right)\subseteq \mathcal{G}_n(\Theta, \mathcal{D})\subseteq\mathcal{S}_n(C^*\Theta/8).
\ee
\end{theorem}

\begin{proof}
In this proof, we will invoke Lemma~\ref{lem:Inbounds}~and~\ref{lem:Gnboundsprep} with $\beta=C^*\Theta/2$. For this proof only, define
\be
t=\frac{2\Theta+C^*\Theta(1+\Theta)}{8}=\frac{(1+\beta)\Theta+\beta}{4},
\ee
and suppose $x\in\mathcal{S}_n(t)$; i.e., $\sigma_n(x)\ge 4tI_n$.
 Then with probability $1-c_1M^{-c_2}$ we have
\be
\ba
\Theta\max_{k\in[M]}F_n(x_k)=&\frac{4t-\beta}{1+\beta}\max_{k\in[M]}F_n(x_k)\\
\leq&(4t-\beta)I_n &\text{(From \eqref{eq:Fnxkbound})}\\
\leq&\sigma_n(x)-\beta I_n &\text{(From \eqref{eq:Sn})}\\
\leq& F_n(x). &\text{(From \eqref{eq:betacomp1})}
\ea
\ee
This results in the first inclusion in \eqref{eq:SnGncomp} because $\Theta\le 1$ implies that
$$
\frac{(1+C^*)\Theta}{4}\geq \frac{2\Theta+C^*\Theta(1+\Theta)}{8},
$$
and hence,
\be
\mathcal{S}_n\left(\frac{(1+C^*)\Theta}{4}\right)\subseteq \mathcal{S}_n\left(\frac{2\Theta+C^*\Theta(1+\Theta)}{8}\right)\subseteq\mathcal{G}_n(\Theta,\mathcal{D}).
\ee
Now suppose $x\in\mathcal{G}(\Theta,\mathcal{D})$. Then
\be
\ba
4(C^*\Theta/8)I_n=& (C^*\Theta-\beta)I_n\\
\leq& \Theta \max_{k\in[M]}F_n(x_k)-\beta I_n&\text{(From \eqref{eq:Fnxkbound})}\\
\leq& F_n(x)-\beta I_n&\text{(From \eqref{eq:support_est_set_def})}\\
\leq& \sigma_n(x)&\text{(From \eqref{eq:betacomp1})}.
\ea
\ee
This gives us the second inclusion in \eqref{eq:SnGncomp}, completing the proof.
\end{proof}

We are now prepared to give the proofs of our main results from Section~\ref{sec:mainresults}. A key element of both proofs is to utilize Theorem~\ref{thm:GnSnbound} in conjunction with the corresponding results on the continuous support estimation set $\mathcal{S}_n(\theta)$ (Theorems~\ref{thm:suppmu}~and~\ref{thm:partmu}).

\noindent\textit{Proof of Theorem~\ref{thm:fullsuportdet}.} By \eqref{eq:thm1inc} and \eqref{eq:SnGncomp} with $\theta_{1}=\frac{(1+C^*)\Theta}{4}$, we have
\be\label{eq:lowerinclusion}
\mathbb{X}\subseteq \mathcal{S}_{n}(\theta_{1})\subseteq \mathcal{G}_{n}(\Theta).
\ee
Similarly, with $\theta_{2}=\frac{C^*\Theta}{8}$, and the definition of $d(\theta_{2})=\left( \frac{C_{1}}{C_{2}\theta_{2}} \right)^{1/(S-\alpha)}$ (from \eqref{eq:dtheta}), we see
\be\label{eq:gnrightinc}
\mathcal{G}_{n}(\Theta)\subseteq \mathcal{S}_{n}(\theta_{2})\subseteq \mathbb{B}\left( \mathbb{X},\left( \frac{C_{1}}{C_{2}\theta_2} \right)^{1/(S-\alpha)}\Big/n \right)=\mathbb{B}\left( \mathbb{X},\left( \frac{8C_{1}}{C^*C_{2}\Theta} \right)^{1/(S-\alpha)}\Big/ n\right).
\ee
The choice of $\theta$ in each case satisfies the conditions of Theorem~\ref{thm:suppmu} because
\be
\theta_{2}=\frac{C^*\Theta}{8}\leq \theta_{1}=\frac{(1+C^*)\Theta}{4}< \frac{C_{2}}{4C_{1}}.
\ee
Setting
\be\label{eq:rtheta}
r(\Theta):= \left( \frac{8C_{1}}{C^*C_{2}\Theta} \right)^{1/(S-\alpha)},
\ee
then \eqref{eq:gnrightinc} demonstrates \eqref{eq:thm1}.\qed

\noindent\textit{Proof of Theorem~\ref{thm:class_separation}.} We note that the inclusion $\mathcal{G}_{k,\eta,n}(\Theta)\subseteq \mathbb{B}(\mathbf{S}_{k,\eta},r(\Theta)/n)$ is already satisfied by \eqref{eq:thm2-1}. Let $\theta_1=\frac{(1+C^*)\Theta}{4}$, and observe that
\be
d(\theta_1)=\left(\frac{4C_1}{C_2(1+C^*)\Theta}\right)^{1/(S-\alpha)}=\left(\frac{C^*}{2(1+C^*)}\right)^{1/(S-\alpha)}r(\Theta),
\ee
as defined in \eqref{eq:rtheta}. We set $c=\left(\frac{C^*}{2(1+C^*)}\right)^{1/(S-\alpha)}$ and note $c<1$. Therefore,
\be
\ba
\mathbb{X}\cap \mathbb{B}(\mathbf{S}_{k,\eta},cr(\Theta)/n)=&\mathbb{X}\cap \mathbb{B}(\mathbf{S}_{k,\eta},d(\theta_1)/n)&\\
\subseteq& \mathcal{S}_n(\theta_1)\cap \mathbb{B}(\mathbf{S}_{k,\eta},d(\theta_1)/n)&\text{(From \eqref{eq:thm2inc})}\\
\subseteq& \mathcal{S}_n(\theta_1)\cap \mathbb{B}(\mathbf{S}_{k,\eta},r(\Theta)/n)&\text{(Since $d(\theta_1)<r(\Theta)$)}\\
\subseteq& \mathcal{G}_n(\Theta)\cap \mathbb{B}(\mathbf{S}_{k,\eta},r(\Theta)/n)&\text{(From \eqref{eq:SnGncomp})}\\
=&\mathcal{G}_{k,\eta,n}(\Theta),
\ea
\ee
completing the proof of \eqref{eq:thm2-3}.
Setting $\theta_2=\frac{C^*\Theta}{8}$, then $r(\Theta)=d(\theta_2)$ and \eqref{eq:SnGncomp} gives us the inclusion
\be
\mathcal{G}_{k,\eta,n}(\Theta)=\mathcal{G}_n(\Theta)\cap \mathbb{B}(\mathbf{S}_{k,\eta},r(\Theta)/n)\subseteq \mathcal{S}_n(\theta_2)\cap \mathbb{B}(\mathbf{S}_{k,\eta},d(\theta_2)/n)=\mathcal{S}_{k,\eta,n}(\theta_2).
\ee
Then, \eqref{eq:minsep2} implies \eqref{eq:thm2-2}.\qed

%Theorems~\ref{thm:fullsuportdet}~and~\ref{thm:class_separation} follow directly from Theorems~\ref{thm:suppmu}~and~\ref{thm:partmu} respectively, in view of Theorem~\ref{thm:GnSnbound}.

\subsection{F-score Result Proof}

In this section we give a proof for Theorem~\ref{thm:fscore}.

\begin{proof}
Observe that under the conditions of Theorem~\ref{thm:class_separation} we have
\be
\mu(\mathbf{S}_{k,\eta})\leq \mu(\mathcal{G}_{k,\eta,n})\leq \mu(\mathbf{S}_{k,\eta})+\mu(\mathbf{S}_{K_\eta+1,\eta}).
\ee
Therefore,
\be
\mathcal{F}_\eta(\mathcal{G}_{k,\eta,n})\geq 2\frac{\mu(\mathbf{S}_{k,\eta})}{2\mu(\mathbf{S}_{k,\eta})+\mu(\mathbf{S}_{K_{\eta}+1,\eta})}\geq 1-\frac{\mu(\mathbf{S}_{K_{\eta}+1,\eta})}{2\mu(\mathbf{S}_{k,\eta})+\mu(\mathbf{S}_{K_{\eta}+1,\eta})}.
\ee
Also, $\{\mathcal{G}_{k,\eta,n}\}_{k=1}^{K_\eta}$ is a partition of $\mathcal{G}_n(\Theta)\supseteq \mathbb{X}$. Then,
\be
\sum_{k=1}^{K_\eta}\mu(\mathcal{G}_{k,\eta,n})\mathcal{F}_\eta(\mathcal{G}_{k,\eta,n})\geq 1-\mu(\mathbf{S}_{K_{\eta}+1,\eta})\sum_{k=1}^{K_\eta}\frac{\mu(\mathcal{G}_{k,\eta,n})}{2\mu(\mathbf{S}_{k,\eta})+\mu(\mathbf{S}_{K_\eta+1,\eta})},
\ee
further implying
\be
1\geq \mathcal{F}_\eta\left(\{\mathcal{G}_{k,\eta,n}\}_{k=1}^{K_\eta}\right)\geq 1-\frac{\mu(\mathbf{S}_{K_{\eta}+1,\eta})}{2\min_{k\in \{1,\dots,K_\eta\}}\mu(\mathbf{S}_{k,\eta})}.
\ee
Thus, by our assumption
\be
\lim_{\eta\to 0^+}\mathcal{F}_\eta\left(\{\mathcal{G}_{k,\eta,n}\}_{k=1}^{K_\eta}\right)\leq \lim_{\eta\to 0^+} 1-\frac{\mu(\mathbf{S}_{K_{\eta}+1,\eta})}{2\min_{k\in \{1,\dots,K_\eta\}}\mu(\mathbf{S}_{k,\eta})}=1.
\ee
\end{proof}

\chapter{Conclusion}
\label{ch:conclusion}

This chapter contains excerpts from \cite{mhaskarodowd,mhaskarodowdclassification}.

Firstly, in Chapter~\ref{ch:manifoldapprox} we have discussed a central problem of machine learning, namely to approximate an unknown target function based only on the data drawn from an unknown probability distribution.
While the prevalent paradigm to solve this problem in general is to minimize a loss functional, we have initiated a new paradigm where we can do the approximation directly from the data, under the manifold assumption.
This is a substantial theoretical improvement over the classical manifold learning technology, which involves a two-step procedure: first to get some information about the manifold and then to do the approximation.
Instead, our method learns on the manifold without manifold learning.
Our construction in itself does not require any prior on the target function.
We derive uniform error bounds with high probability regardless of the nature of the distribution, provided we know the dimension of the unknown manifold.
The theorems are illustrated with some numerical examples. 
One of these examples is closely related to an important problem in magnetic resonance relaxometry, in which one seeks to find the proportion of water molecules in the myelin covering in the brain based on a model that involves the inversion of Laplace transform.

We view this work as the beginning of a new direction. 
As such, there are plenty of future research projects, some of which we plan to undertake ourselves.
\begin{itemize}
\item Find alternative methods that improve upon the error estimates on \textbf{unknown} manifolds, and more general compact sets.
The encoding described in Appendix~\ref{sec:encoding} gives a representation of a function on an unknown manifold.
Such an encoding is useful in the emerging area of approximation of operators. 
We plan to develop this theme further in the context of approximation of operators defined in different function spaces.
\item Explore real-life applications other than the examples which we have discussed in this chapter.
\item We feel that our method will work best if we are working in the right feature space.
One of the vexing problems in machine learning is to identify the right features in the data.
Deep networks are supposed to be doing this task automatically.
However, there is no clear explanation of whether they work in every problem or otherwise develop a theory of what ``features'' should mean and how deep networks can extract these automatically.
\end{itemize}

Secondly, in Chapter~\ref{ch:transferlearning} we have provided theory for the lifting of a function from one data space to another. 
Particularly, our theory looks at the case where there is data only in some region of the base space. 
We answered the question of where we can leverage this knowledge on the target space via the construction of a generalized distance between the spaces and the careful use of a kernel localized with respect to this distance. 
We also examined the resulting smoothness of the lifted function, and found that the smoothness of the original function as well as parameters associated with each data space play a role in this value.
Our work here has applications to local transfer learning, where one has some function approximation technique which works well on one domain (base space) and wishes to apply the learned knowledge to a new domain (target space). 
Under the framework we developed, this can be done without needing the global data from the original domain, and instead one can look at an area of interest in the target domain to decide which data they need from the original domain.
For large models defined on massive data this development could save considerable computation time and even open doors to new results which were computationally infeasible before due to requiring data on the entire space.
Future work in this direction also includes the field of image recovery, where one is given data of the transformation of an image and wishes to recover the original image. When a singular value decomposition of the transformation is available, our theory may be directly applicable.
In particular, we are working to understand the Radon transform in this manner.

Thirdly, in Chapter~\ref{ch:classification} we have introduced a new approach for active learning in the context of machine learning. 
We provide theory for measure support estimation based on finitely many samples from an unknown probability measure supported on a compact metric space. 
With an additional assumption on the measure, known as the fine structure, we then relate this theory to the classification problem, which can be viewed as estimating the supports of a measure of the form $\mu=\sum_{k=1}^K c_k\mu_k$. We have shown that this setup is a generalization of the signal separation problem. Therefore our theory unifies ideas from signal separation with machine learning classification. Since the measures we are considering may be supported on a continuum, our theory additionally relates to the super-resolution regime of the signal separation problem. 

We also give some empirical analysis for the performance of our new algorithm MASC, which was originally introduced in a varied form in \cite{mhaskar-odowd-tsoukanis}. The key focus of the algorithm is on querying high-information points whose labels can be extended to others belonging to the same class with high probability. This is done in a multiscale manner, with the intention to be applied to data sets where the minimal separation between supports of the measures for different classes may be unknown or even zero. We applied MASC to a document data set as well as two hyperspectral data sets, namely subsets of the Indian Pines and Salinas hyperspectral imaging data sets. In the process of these experiments, we demonstrate empirically that MASC is selecting high-information points to query and that it gives competitive performance compared to two other recent active learning methods: LAND and LEND. Specifically, MASC consistently outperforms these algorithms in terms of computation time and exhibits competitive accuracy on Indian Pines for a broad range of query budgets.

These three projects embody a new approach to machine learning in general, where the models are developed through ideas in approximation theory and harmonic analysis and justified by proven, constructive, convergence guarantees. This contrasts the current methodology in machine learning where the models are justified by existence results for the degree of approximation on their underlying hypothesis spaces, and then trained via optimization procedures often without knowledge of whether or not the constructive model will attain such approximation power. In this way, we view the notion of learning without training to be a powerful shift in methodology for machine learning problems which may yield further improvements to the field in the future.

\newpage

\chapter*{Appendices}
\addcontentsline{toc}{chapter}{Appendices}
\renewcommand\thesection{\Alph{section}}
\setcounter{section}{0}

%TAUBERIAN THEOREM
\section{Tauberian Theorem}\label{bhag:tauberian}

The following Tauberian theorem comes from \cite[Theorem~4.3]{tauberian}. This section is copied from the appendix of \cite{mhaskarodowdtransfer}.

We recall that if $\mu$ is an extended complex valued Borel measure on $\RR$, then its total variation measure is defined for a Borel set $B$ by
$$
|\mu|(B)=\sup\sum |\mu(B_k)|,
$$
where the sum is over a partition $\{B_k\}$ of $B$ comprising Borel sets, and the supremum is over all such partitions.

A measure $\mu$ on $\RR$ is called an even measure if $\mu((-u,u))=2\mu([0,u))$ for all $u>0$, and $\mu(\{0\})=0$. If $\mu$ is an extended complex valued measure on $[0,\infty)$, and $\mu(\{0\})=0$, we define a measure $\mu_e$ on $\RR$ by 
$$
\mu_e(B)=\mu\left(\{|x| : x\in B\}\right),
$$
and observe that $\mu_e$ is an even measure such that $\mu_e(B)=\mu(B)$ for $B\subset [0,\infty)$. In the sequel, we will assume that all measures on $[0,\infty)$ which do not associate a nonzero mass with the point $0$ are extended in this way, and will abuse the notation $\mu$ also to denote the measure $\mu_e$. In the sequel, the phrase ``measure on $\RR$'' will refer to an extended complex valued Borel measure having bounded total variation on compact intervals in $\RR$, and similarly for measures on $[0,\infty)$.

Our main Tauberian theorem is the following.

\begin{theorem}\label{theo:maintaubertheo}
Let $\mu$ be an extended complex valued measure on $[0,\infty)$, and $\mu(\{0\})=0$. We assume that there exist $Q, r>0$, such that each of the following conditions are satisfied.
\begin{enumerate}
\item 
\be\label{eq:muchristbd}
\tn\mu\tn_Q:=\sup_{u\in [0,\infty)}\frac{|\mu|([0,u))}{(u+2)^Q} <\infty,
\ee
\item There are constants $c, C >0$,  such that
\be\label{eq:muheatgaussbd}
\left|\int_\RR \exp(-u^2t)d\mu(u)\right|\le c_1t^{-C}\exp(-r^2/t)\tn\mu\tn_Q, \qquad 0<t\le 1.
\ee 
\end{enumerate}
Let $H:[0,\infty)\to\RR$, $S>Q+1$ be an integer, and suppose that there exists a measure $H^{[S]}$ such that
\be\label{eq:Hbvcondnew}
H(u)=\int_0^\infty (v^2-u^2)_+^{S}dH^{[S]}(v), \qquad u\in\RR,
\ee
and
\be\label{eq:Hbvintbdnew}
V_{Q,S}(H)=\max\left(\int_0^\infty (v+2)^Qv^{2S}d|H^{[S]}|(v), \int_0^\infty (v+2)^Qv^Sd|H^{[S]}|(v)\right)<\infty.
\ee
Then for $n\ge 1$,
\be\label{eq:genlockernest}
\left|\int_0^\infty H(u/n)d\mu(u)\right| \le c\frac{n^Q}{\max(1, (nr)^S)}V_{Q,S}(H)\tn\mu\tn_Q.
\ee
\end{theorem}

Proposition~\ref{prop:kernloc} is proved using this theorem with 
$$
\mu(u)=\mu_{x,y}(u)=\sum_{k : \lambda_k<u}\phi_k(x)\phi_k(y),
$$
and $r=d(x,y)$.
Proposition~\ref{prop:jointkernloc} is proved using this theorem with 
$$
\mu(u)=\mu_{x_1,x_2}(u)=\sum_{j,k : \ell_{j,k}<u}A_{j,k}\phi_{1,j}(x_1)\phi_{2,k}(x_2),
$$
and $r=d_{1,2}(x_1,x_2)$.

\section{Background on Manifolds}
\label{sec:manifoldintro}

This introduction to manifolds comes from \cite[A.2]{mhaskarodowd}. We mostly follow along with the notation and definitions in \cite{docarmo}. For details, we refer the reader to  texts such as \cite{boothby,docarmo,guillemin}.

\begin{definition}[Differentiable Manifold]
A (boundary-less) \textit{differentiable manifold} of dimension $q$ is a set $\mathbb{X}$ together with a family of open subsets $\{U_\alpha\}$ of $\mathbb{R}^q$ and functions $\{\vec{x}_\alpha\}$ such that
\be
\vec{x}_\alpha: U_\alpha \to \mathbb{X}
\ee
is injective, and the following 3 properties hold:
\begin{itemize}
\item $\bigcup_{\alpha}\vec{x}_\alpha(U_\alpha)=\mathbb{X}$,
\item $\vec{x}_\alpha(U_\alpha)\cap \vec{x}_\beta(U_\beta)=W\neq \emptyset$ implies that $\vec{x}^{-1}_\alpha(W),\vec{x}^{-1}_\beta(W)$ are open sets and $\vec{x}_\beta^{-1}\circ \vec{x}_\alpha$ is an infinitely differentiable function.
\item The family $\mathcal{A}_\mathbb{X}=\{(U_\alpha,\vec{x}_\alpha)\}$ is maximal regarding the above conditions.
\end{itemize}
\end{definition}

\begin{rem}{\rm
The pair $(U_\alpha,\vec{x}_\alpha)$ gives a \textit{local coordinate chart} of the manifold, and the collection of all such charts $\mathcal{A}_\mathbb{X}$ is known as the \textit{atlas}. \qed}
\end{rem}

\begin{definition}[Differentiable Map]
Let $\mathbb{X}_1,\mathbb{X}_2$ be differentiable manifolds. We say a function $\phi: \mathbb{X}_1\to\mathbb{X}_2$ is (infinitely) \textit{differentiable}, denoted by $\phi\in C^\infty(\mathbb{X})$, at a point $x\in\mathbb{X}_1$ if given a chart $(V,\vec{y})$ of $\mathbb{X}_2$, there exists a chart $(U,\vec{x})$ of $\mathbb{X}_1$ such that $x\in \vec{x}(U)$, $\phi(\vec{x}(U))\subseteq \vec{y}(V)$, and $\vec{y}^{-1}\circ \phi\circ \vec{x}$ is infinitely differentiable at $\vec{x}^{-1}(p)$ in the traditional sense.
\end{definition}

For any interval $I$ of $\mathbb{R}$, a differentiable function $\gamma: I\to \mathbb{X}$ is known as a $\textit{curve}$. If $x\in\mathbb{X}$, $\epsilon>0$, and $\gamma: (-\epsilon,\epsilon)\to \mathbb{X}$ is a curve with $x=\gamma(0)$, then we can define the \textit{tangent vector} to $\gamma$ at $\gamma(t_0)$ as a functional $\gamma'(t_0)$ acting on the class of differentiable functions $f:\mathbb{X}\to\mathbb{R}$ by
\be
\gamma'(t_0)f\coloneqq\frac{d(f\circ \gamma)}{dt}(t_0).
\ee
The \textit{tangent space} of $\mathbb{X}$ at a point $x\in\mathbb{X}$, denoted by $\mathbb{T}_x(\mathbb{X})$, is the set of all such functionals $\gamma'(0)$.

A \textit{Riemannian} manifold is a differentiable manifold with a family of inner products $\{\alg{\circ,\circ}_x\}_{x\in\mathbb{X}}$ such that for any $X,Y\in\mathbb{T}_x(\mathbb{X})$, the function $\varphi: \mathbb{X}\to\mathbb{C}$ given by $x\mapsto \alg{X(x),Y(x)}_x$ is differentiable. 
We can define an associated norm $\norm{X}=\alg{X(x),X(x)}_x$. 
The length $L(\gamma)$ of a curve $\gamma$ defined on $[a,b]$ is defined to be 
\be
L(\gamma)\coloneqq\int_a^b\norm{\gamma'(t)}dt.
\ee
 We will call a curve $\gamma:[a,b]\to \mathbb{X}$ a \textit{geodesic} if $L(\gamma)=\inf\{L(r): r:[a,b]\to\mathbb{X}, r\text{ is a curve}\}$. It is well-known that if $\gamma$ is a geodesic, then $\gamma'(t)\cdot \gamma''(t)=0$ for any $t\in [a,b]$.
 
 In the sequel, we assume that $\mathbb{X}$ is a compact, connected, Riemannian manifold.
Then for every $x,y\in\mathbb{X}$ there exists a geodesic $\gamma:[a,b]\to\mathbb{X}$  such that $\gamma(a)=x,\gamma(b)=y$. The quantity  $\rho(x,y)=L(\gamma)$ defines a metric on $\XX$ such that the corresponding metric topology is consistent with the topology defined by any atlas on $\XX$.

For any $x\in \XX$, there exists a neighborhood $V\subset \XX$ of $x$, a number $\delta=\delta(x)>0$ and a mapping $\mathcal{E} : (-2,2)\times U \to \XX$, where $U=\{(y,v) : y\in V,\  v\in T_y\XX,\ \|v\|_2<\delta\}$ such that $t\mapsto \mathcal{E}(t, y, v)$ is the unique geodesic of $\XX$ which, at $t=0$, passes through $y$ and has the property that $\partial{\mathcal{E}}/\partial t =v$ for each $(y,v)\in U$.
As a result, we can define the \textit{exponential map} at $x$ to be the function $\varepsilon_{x}: B_\mathbb{T}(x,\delta(x))\subset \mathbb{T}_x(\mathbb{X})\to \mathbb{X}$ by $\varepsilon_x(v)= \mathcal{E}(1, x, v)$. 
Intuitively, the line joining $x$ and $v$ in $\mathbb{T}_x(\mathbb{X})$ is mapped to the geodesic joining $x$ with $\varepsilon_x(v)$. 
We call the supremum of all $\delta(x)$ for which the exponential map is so defined  the \textit{injectivity radius} at $x$, denoted by $\iota(x)$.
We call $\iota^*=\inf_{x\in\mathbb{X}}\iota(x)$ the  \textit{global injectivity radius} of $\mathbb{X}$. 
Since $x\mapsto \iota(x)$ is a continuous function of $x$, and $\iota(x)>0$ for each $x$,  it follows that $\iota^*>0$ when $\mathbb{X}$ is compact. 
Correspondingly, on compact manifolds, one can conclude that for $y\in B_\mathbb{T}(x,\iota^*)$, $\rho(x,\varepsilon_x(y))=\norm{x-y}$.

Next, we discuss the metric tensor and volume element on $\XX$.
Let $(U,\vec{x})$ be a coordinate chart with $0\in U$, $\vec{x}(0)=x\in\mathbb{X}$, and $\partial_j(x)$ be the tangent vector at $x$ to the coordinate curve $t\mapsto \vec{x}((\underbrace{0,\dots,0}_{j-1},t,0,\dots,0))$. Then we can define the metric tensor $\mathbf{g}$ to be the matrix where $\mathbf{g}_{ij}=\alg{\partial_i(x),\partial_j(x)}_x$. When one expands the metric tensor $\mathbf{g}$ as a Taylor series in local coordinates on $\mathbb{B}(x,\iota^*)$, it can be shown \cite[Page 21]{roemanifold} that for any $\delta<\iota^*$, on the ball $\mathbb{B}(x,\delta)$ we have
\be
|\mathbf{g}|=1+O(\delta^2).
\ee
In turn, this implies
\be\label{eq:volume_element}
\sqrt{|\mathbf{g}|}-1\lesssim \delta^2.
\ee

The following proposition lists some important properties relating the geodesic distance $\rho$ on an unknown submanifold of $\SS^Q$ with the geodesic distance on $\SS^Q$ as well as the Euclidean distance on $\RR^{Q+1}$.
\begin{proposition}\label{prop:taylorprop} Let $\eta_x$ be defined as in Section~\ref{sec:manifoldapprox}.

{\rm (a)} For every $\eta_x(u)\in\mathbb{B}(x,\iota^*)$,
\be\label{eq:diffbound}
\abs{\arccos(x\cdot \eta_x(u))-\rho(x,\eta_x(u))}\lesssim\rho(x,\eta_x(u))^3.
\ee

{\rm (b)} For any $x,y\in\mathbb{X}$,
\be
\rho(x,y)\sim \arccos(x\cdot y).
\ee

\end{proposition}

\begin{proof}
First, we observe the fact that $\norm{x-y}_2\sim \arccos(x\cdot y)$ because $\norm{x-y}_2/2=\sin(\arccos(x\cdot y)/2)$ and $\theta/\pi \le \sin(\theta/2)\le \theta/2$ for all $\theta\in [0,\pi]$. Fix $x\in\mathbb{X}$ and let $\gamma$ be a geodesic on $\mathbb{X}$ parametrized by length $t$ from $x$. In particular we then have $\norm{\gamma'(0)}_2=1$ and $\gamma'(0)\cdot\gamma''(0)=0$.
 Taking a Taylor expansion for $\gamma(t)$ with $|t|<\iota^*$ (we recall that $\iota^*\leq 1)$, we can see
\be\label{eq:taylorgamma}
\ba
\gamma'(0)\cdot (\gamma(t)-\gamma(0))=&\gamma'(0)\cdot \left(\gamma'(0)t+\frac{1}{2}\gamma''(t)t^2+O(t^3)\right)\\
=&\norm{\gamma'(0)}^2_2t+\gamma'(0)\cdot \gamma''(0)t^2+O(t^3)\\
=&t+O(t^3).
\ea
\ee
For any $y\in\mathbb{B}(x,\iota^*)$, there exists a unique $u\in \mathbb{S}_x(\iota^*)$ such that $y=\eta_x(u)$. We can write $y=\gamma(t)$ for some geodesic $\gamma$. We know, $t=\rho(x,y)\geq \arccos(x\cdot y)\geq \norm{x-y}_2=\norm{\gamma(t)-\gamma(0)}_2$. Using the Cauchy-Schwarz inequality, we see
\be
0\leq t-\norm{x-y}_2\leq t-\gamma'(0)\cdot (\gamma(t)-\gamma(0))\lesssim t^3.
\ee
As a result we can conclude
\be
\rho(x,\eta_x(u))-\arccos(x\cdot \eta_x(u))\leq \rho(x,\eta_x(u))-\norm{\eta_x(u)-x}_2\lesssim \rho(x,\eta_x(u))^3,
\ee
showing \eqref{eq:diffbound}. Letting $c$ be the constant built into the notation of \eqref{eq:diffbound}, then if we fix $x\in\mathbb{X}$ and let $y\in\mathbb{B}\big(x,\sqrt{1/(2c)}\big)$, we have
\be
\frac{1}{2}\rho(x,y)\leq \rho(x,y)-c\rho(x,y)^3\leq \arccos(x\cdot y).
\ee
Furthermore, since $A=\overline{\mathbb{X}\setminus\mathbb{B}\big(x,\sqrt{1/(2c)}\big)}$ is a compact set and $g_x(y)=\arccos(x\cdot y)/\rho(x,y)$ is a continuous function of $y$ defined on $A$, we can conclude that $g_x$ attains a minimum on $A$. Therefore,
\be
\rho(x,y)\sim \arccos(x\cdot y)
\ee
for every $y\in \mathbb{X}$.
We note that the constants involved in this proof vary continuously with respect to the choice of $x$, so in the theorem we may simply use the supremum over all such constants which must be finite since $\mathbb{X}$ is compact.
\end{proof}

\section{Orthogonal Polynomials}
\label{sec:orthopolys}

Many of the results discussed in this section can be found with corresponding proofs in \cite[Chapter 7]{mhaskarpai}. As mentioned above, polynomials are a popular choice for approximants in many settings. On the real line, polynomials are ideal approximants for a wide variety of measures because the space of polynomials can be decomposed into a unique orthonormal basis. To illustrate, we give the following theorem.

\begin{theorem}
Let $\mu$ be a measure on $\mathbb{R}$ satisfying both
\ben
\item $\int_{\mathbb{R}}\abs{x}^n d\mu(x)<\infty$ for every $n\in\mathbb{N}$.

\item There exists $A\subseteq \mathbb{R}$ with $|A|=\infty$ such that for any $t\in A$ and $h>0$ we have
\be
\int_t^{t+h}d\mu(x)>0.
\ee
\een
Then there exists a unique sequence of polynomials $\{P_n\}_{n=1}^\infty$, where each $P_k$ is degree $k$ and has a positive leading coefficient $\gamma_k$, such that
\be
\int_{\mathbb{R}}P_n(x)P_m(x)d\mu(x)=\delta_{m,n}.
\ee
\end{theorem}

Using the same notation, each of these polynomials also satisfies the following three-term recurrence property:
\be
xP_n(x)=\frac{\gamma_n}{\gamma_{n+1}}P_{n+1}(x)+\int_{\mathbb{R}}tP_n^2(t)d\mu(t)P_n(x)+\frac{\gamma_{n-1}}{\gamma_n}P_{n-1}(x).
\ee
Given $n\in\mathbb{N}$, if one knows the moments 
\be
u_k=\int_{\mathbb{R}}x^kd\mu(x)
\ee
for $k=1,...,2n-1$, then the first $n$ orthogonal polynomials (not necessarily orthonormal) can be constructed by
\be
P_n(x)=\begin{vmatrix}
u_0 & u_1 & u_2 & \dots & u_{n-1} & 1\\
u_1 & u_2 & u_3 & \dots & u_n & x\\
u_2 & u_3 & u_4 & \dots & u_{n+1} & x^2\\
\vdots & \vdots & \vdots & \ddots &  \vdots & \vdots\\
u_{n-1} & u_{n} & u_{n+1} & \dots & u_{2n-2} & x^{n-1}\\
u_n & u_{n+1} & u_{n+2} & \dots &  u_{2n-1} & x^n
\end{vmatrix}.
\ee
We note that column $k$ enforces orthogonality against $x^k$. One of the primary examples of orthogonal polynomials on the real line are the \textit{orthonormalized Hermite polynomials}, defined by the Rodrigues-type formula:
\be\label{eq:hermitepoly}
H_n(x)=\frac{e^{x^2}}{\sqrt{2^n n!\sqrt{\pi}}}\left(\frac{d}{dx}\right)^n e^{-x^2}.
\ee
These polynomials satisfy the orthonormality property with respect to a Gaussian weight function, i.e.
\be
\int_{\mathbb{R}} H_n(x)H_m(x)e^{-x^2}dx=\delta_{m,n}.
\ee
These polynomials are also used to construct the Hermite functions, defined by
\be\label{eq:hermitefun}
h_n(x)=e^{-x^2/2}H_n(x),
\ee
which are orthonormal functions with respect to the standard Lebesgue measure on the real line.
The main orthogonal polynomials on the real line of interest in this dissertation are the Jacobi polynomials. We include the following excerpt about Jacobi polynomials from \cite{mhaskarodowdtransfer}. For $\alpha, \beta>-1$, $x\in (-1,1)$ and integer $\ell\ge 0$, the Jacobi polynomials $p_\ell^{(\alpha,\beta)}$ are defined by the Rodrigues' formula \cite[Formulas~(4.3.1), (4.3.4)]{szegopoly}
\be\label{eq:rodrigues}
(1-x)^\alpha(1+x)^\beta p_\ell^{(\alpha,\beta)}(x)
=\left\{\frac{2\ell+\alpha+\beta+1}{2^{\alpha+\beta+1}}\frac{\ell!(\ell+\alpha+\beta)!}{(\ell+\alpha)!(\ell+\beta)!}\right\}^{1/2}\frac{(-1)^\ell}{2^\ell \ell!}\frac{d^\ell}{dx^\ell}\left((1-x)^{\ell+\alpha}(1+x)^{\ell+\beta}\right),
\ee
where $z!$ denotes $\Gamma(z+1)$. The Jacobi polynomials satisfy the following well-known differential equation:
\begin{equation}\label{eq:jacobidifeq}
	{p''_n}^{(\alpha,\beta)}(x)(1-x^2)+(\beta-\alpha-(\alpha+\beta+2)x){p'_n}^{(\alpha,\beta)}(x)=-n(n+\alpha+\beta+1)p_n^{(\alpha,\beta)}(x).
\end{equation}
Each $p_\ell^{(\alpha,\beta)}$ is a polynomial of degree $\ell$ with positive leading coefficient, satisfying the orthogonality relation
\be\label{eq:jacobiortho}
\int_{-1}^1 p_\ell^{(\alpha,\beta)}(x)p_j^{(\alpha,\beta)}(x)(1-x)^\alpha(1+x)^\beta=\delta_{\ell,j},
\ee
and
\be\label{eq:pkat1}
p_\ell^{(\alpha,\beta)}(1)=\left\{\frac{2\ell+\alpha+\beta+1}{2^{\alpha+\beta+1}}\frac{\ell!(\ell+\alpha+\beta)!}{(\ell+\alpha)!(\ell+\beta)!}\right\}^{1/2}\frac{(\ell+\alpha)!}{\alpha!\ell!} \sim \ell^{\alpha+1/2}.
\ee
It follows that $p_\ell^{(\alpha,\beta)}(-x)=(-1)^\ell p_\ell^{(\beta,\alpha)}(x)$. 
In particular,  $p_{2\ell}^{(\alpha,\alpha)}$ is an even polynomial, and $p_{2\ell+1}^{(\alpha,\alpha)}$ is an odd polynomial.  
We note (cf. \cite[Theorem~4.1]{szegopoly}) that
\be\label{eq:evenjacobi}
\begin{aligned}
p_{2\ell}^{(\alpha,\alpha)}(x)=&2^{\alpha/2+1/4}p_\ell^{(\alpha,-1/2)}(2x^2-1)=2^{\alpha/2+1/4}(-1)^\ell p_\ell^{(-1/2,\alpha)}(1-2x^2)\\
p_{2\ell+1}^{(\alpha,\alpha)}(x)=&2^{\alpha/2+1/2}xp_\ell^{(\alpha,1/2)}(2x^2-1)=2^{\alpha/2+1/2}(-1)^\ell xp_\ell^{(1/2,\alpha)}(1-2x^2).
\end{aligned}
\ee
It is known \cite[Theorem A]{nowak2011sharp} that for $\alpha,\beta\ge -1/2$ and $\theta, \phi\in [0,\pi]$,
\be\begin{aligned}\label{eq:jacobigauss}
&\sum_{j=0}^\infty \exp(-j(j+\alpha+\beta+1)t)p_j^{(\alpha,\beta)}(\cos\theta)p_j^{(\alpha,\beta)}(\cos\phi)\\
\ls& (t+\theta\phi)^{-\alpha-1/2}(t+(\pi-\theta)(\pi-\phi))^{-\beta-1/2}t^{-1/2}\exp\left(-c\frac{(\theta-\phi)^2}{t}\right).
\end{aligned}\ee
We note that when $\beta=-1/2$, this yields
\be\label{eq:specialjacobigauss}
\sum_{j=0}^\infty \exp(-j(j+\alpha+1/2)t)p_j^{(\alpha,-1/2)}(\cos\theta)p_j^{(\alpha,-1/2)}(\cos\phi)
\ls t^{-\alpha-1}\exp\left(-c\frac{(\theta-\phi)^2}{t}\right).
\ee

\section{Network Representation}
\label{subsec:clenshaw}

\begin{algorithm}[!ht]
\caption{Clenshaw algorithm to compute $\sum_{k=0}^{n-1}C_kp_k$, where\\ $p_k(x)=a_kxp_{k-1}(x)+b_kp_{k-2}(x)$, $k=1,2,\cdots,n-1$, $b_{1}=0$.}
\label{alg:clenshaw}
\KwIn{$p_0$, $C_0,\cdots, C_{n-1}$, $x$, $a_{n+1},\cdots, a_1$, $b_{n+1},\cdots,b_1$.}
\KwOut{The value of $\sum_{k=0}^{n-1}C_kp_k$.}
$\mathsf{out1}\gets 0, \mathsf{out2}\gets 0, C_{-1} \gets 0,  C_n\gets 0$. \;
\For{$k=n+1,n,n-1,\dots, 1$}{
$\mathsf{temp}\gets a_k*\mathsf{out_1}*x+\mathsf{out2}$ \;
$\mathsf{out2}\gets b_k*\mathsf{out1}+C_{k-2}$ \;
$\mathsf{out1}\gets \mathsf{temp}$. \;
}
\Return $\mathsf{out1}*p_0$.
\end{algorithm}

This appendix comes from \cite{mhaskarodowd}. Let $\{p_k\}$ be a system of orthonormal polynomials satisfying a recurrence relation
\be\label{eq:clenshawrec}
p_k(x)=a_kxp_{k-1}(x)+b_kp_{k-2}(x), \qquad k=1,2,\cdots, \quad b_{1}=0.
\ee
The Clenshaw algorithm is a modification of the classical Horner method to compute polynomials expressed in the monomial basis that evaluates a polynomial expressed in terms of the orthonormalized polynomials $\{p_k\}$ \cite{clenshaw,gautschibk}.
To understand the method, let $P=\sum_{k=0}^{n-1}C_kp_k$.
It is convenient to write $C_k=0$ if $k\ge n$ or $k<0$.
The recurrence \eqref{eq:clenshawrec} shows that
\be\label{eq:clenshawstep}
C_kp_k(x)+C_{k-1}p_{k-1}(x)+C_{k-2}p_{k-2}(x)=\left(a_kC_kx+C_{k-1}\right)p_{k-1}(x)+\left(b_kC_k+C_{k-2}\right)p_{k-2}.
\ee
This leads to Algorithm~\ref{alg:clenshaw}.

\begin{figure*}[!ht]
\begin{center}
\includegraphics[scale=0.1]{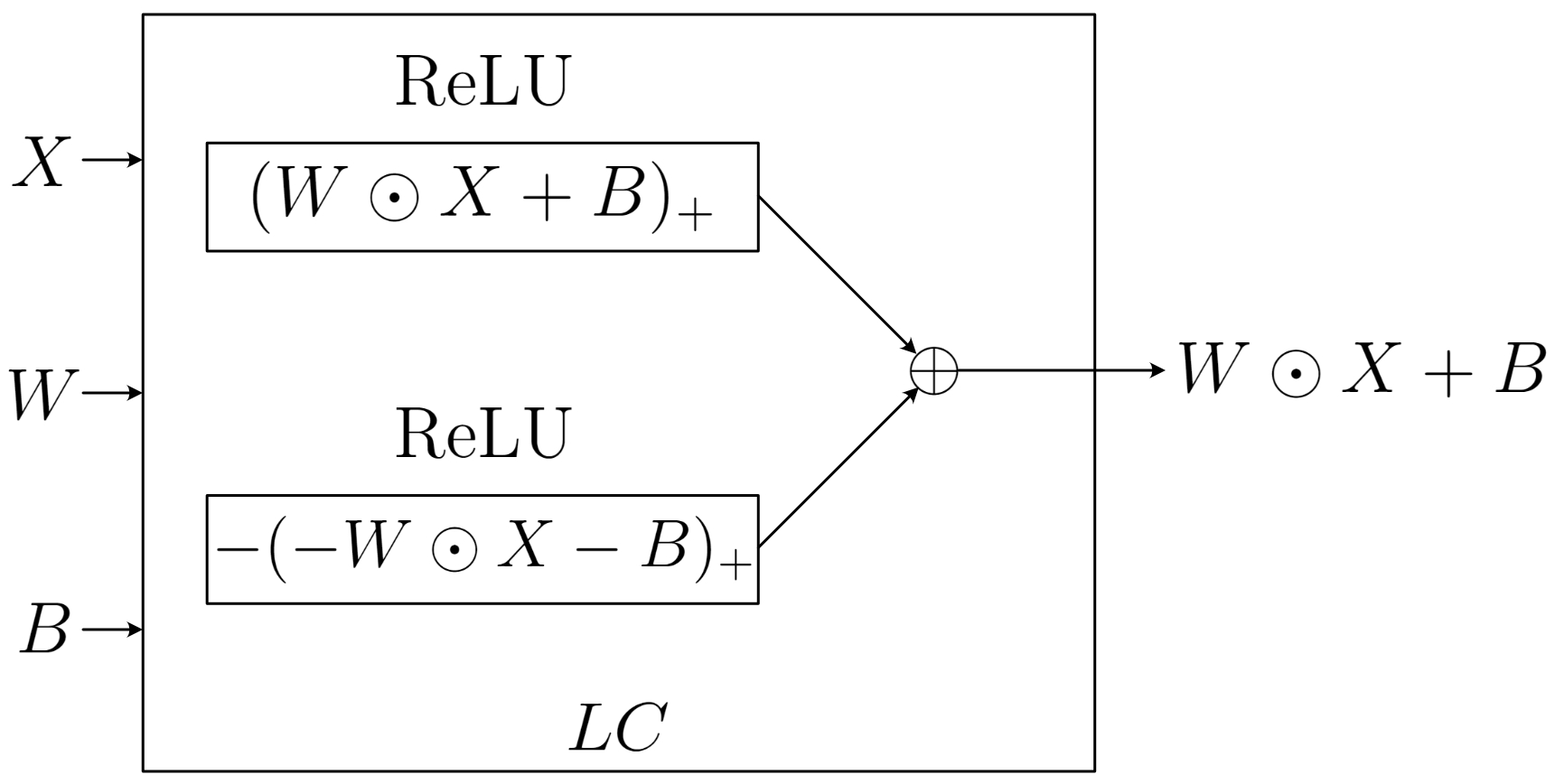} 
\end{center}
\caption{The implementation of a linear combination as a ReLU network. Here all operations are pointwise. The symbols $\odot$ represents Hadamard product of matrices, $\oplus$ is the sum of matrices.}
\label{fig:linearcomb}
\end{figure*} 

\begin{figure*}[!ht]
\begin{center}
\includegraphics[scale=0.15]{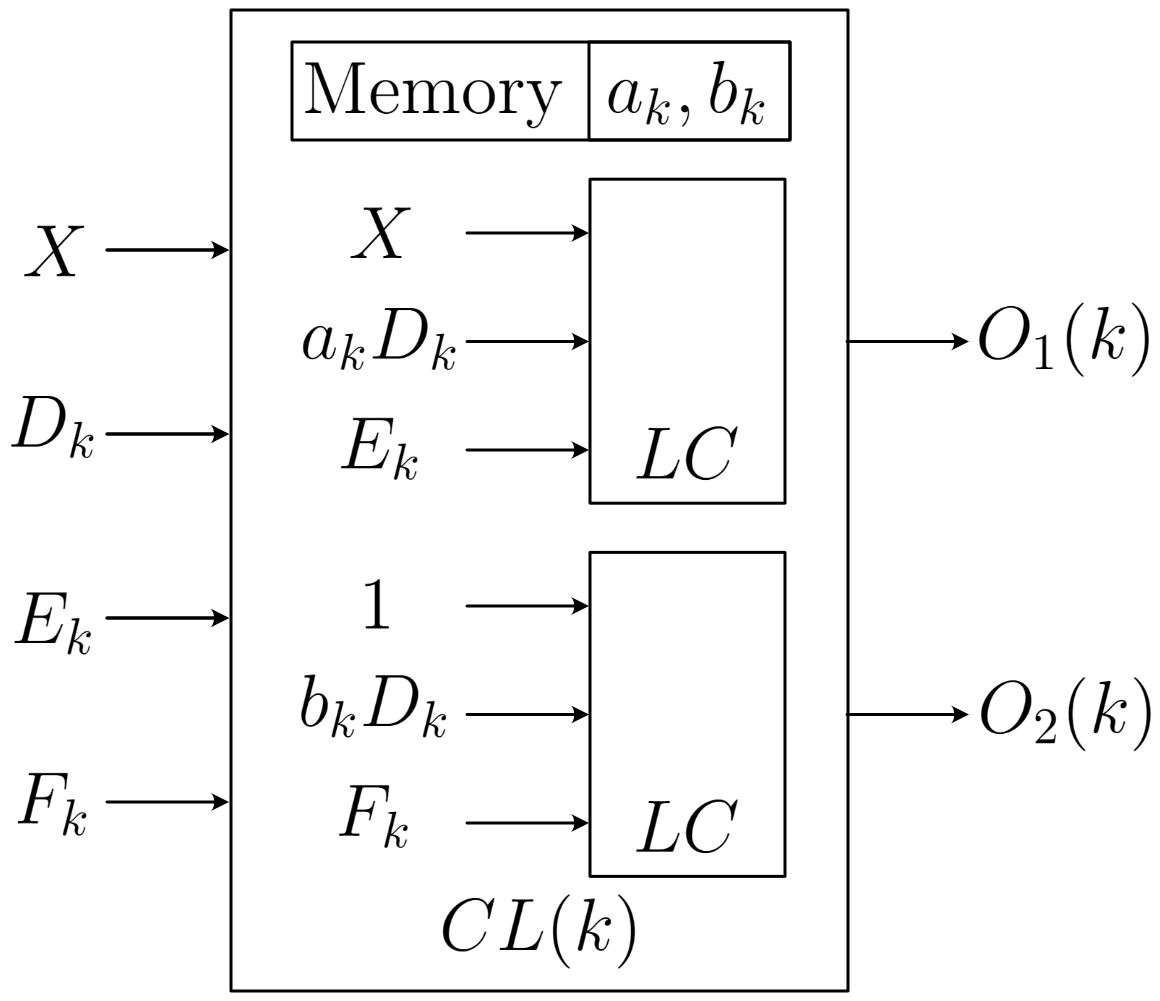} 
\end{center}
\caption{One step of the Clenshaw algorithm, using two circuits of the form LC (4 neurons) as in Figure~\ref{fig:linearcomb}. 
The circuit diagram is shown in general with four input pins and two output pins.}
\label{fig:clenshawstep}
\end{figure*}

\begin{figure*}[!ht]
\begin{center}
\includegraphics[scale=0.25]{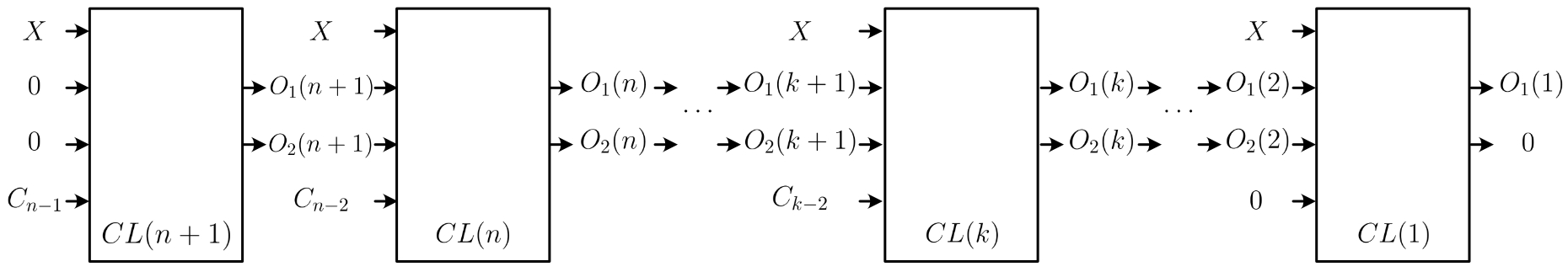} 
\end{center}
\caption{Unrolling the Clenshaw algorithm as a cascade of the circuits of the form CL($k$) as in Figure~\ref{fig:clenshawstep}.}
\label{fig:clenshawalg}
\end{figure*}
 
By algorithm unrolling, we may express this algorithm in terms of a deep neural network evaluating a ReLU activation function.
The network is a cascade of different circuits. 
The most fundamental is the implementation of a linear combination as a ReLU network (see  Figure~\ref{fig:linearcomb})
$$
ax+b=(ax+b)_+ -(-ax-b)_+.
$$
Using the circuits LC in Figure~\ref{fig:linearcomb}, we next construct a circuit to implement recursive reduction \eqref{eq:clenshawstep}.
This is illustrated in Figure~\ref{fig:clenshawstep}.
Finally, we unroll the Clenshaw algorithm by cascading the circuits CL($k$) from Figure~\ref{fig:clenshawstep} for $k=n+1$ down to $k=1$ with different inputs and outputs as shown in Figure~\ref{fig:clenshawalg}.
We use this in order to compute $\Phi_{n,q}(x\cdot y_j)$ by using the recursive formula for ultraspherical polynomials \eqref{eq:recurrence} in the following way.
We set
\be
\begin{aligned}
C_k&=\frac{\omega_q}{\omega_{q-1}}h(k/n)p_{q,k}(1),\\
a_k&= \begin{cases}\displaystyle{\frac{\sqrt{\Gamma(q)\Gamma(q+1)}}{\Gamma(q-1)}}& k=1 \vspace{5pt}\\ \displaystyle{\sqrt{\frac{(2k+q-3)(2k+q-1)}{k(n+q-2)}}}& k\geq 2\end{cases},\\
b_k&=\sqrt{\frac{(k-1)(k+q-3)(2k+q-1)}{k(k+q-2)(2k+q-5)}}.
\end{aligned}
\ee
For the matrix $X$ shown in Figure~\ref{fig:clenshawalg}, we consider the $(Q+1)\times N$ test data matrix $S$ where each column represents one test data $x$, and a $(Q+1)\times M$ train data matrix $R$ where column $j$ represents data point $y_j$. 
Then we set $X=S^TR$. In this way, we would return $\Phi_{n,q}(S^TR)$ from running Algorithm~\ref{alg:clenshaw}, with a time complexity of $O(NMn)$.

\section{Encoding}
\label{sec:encoding}

This appendix comes from \cite{mhaskarodowd}. Our construction in \eqref{eq:approximation} allows us to encode the target function in terms of finitely many real numbers.
For each integer $\ell\geq 0$, let $\{Y_{Q,\ell,k}\}_{k=1}^{\operatorname{dim}(\mathbb{H}_\ell^Q)}$ be an orthonormal basis for $\mathbb{H}_\ell^Q$ on $\mathbb{S}^Q$. We define  the encoding of $f$ by
\be\label{eq:encoding}
\hat{z}(\ell,k)\coloneqq \frac{1}{M}\sum_{j=1}^M z_jY_{Q,\ell,k}(y_j).
\ee
Given this encoding, the decoding algorithm is given in the following proposition.
\begin{proposition}\label{prop:encoding}
Assume $\Phi_{n,q}$ is given as in \eqref{eq:kernel}. Given the encoding of $f$ as given in \eqref{eq:encoding}, one can rewrite the approximation in \eqref{eq:approximation} as
\be\label{eq:decoder}
F_n(\mathcal{D};x)=\sum_{\ell=0}^n \Gamma_{\ell,n}\sum_{k=1}^{\operatorname{dim}(\mathbb{H}_\ell^Q)}\hat{z}(\ell,k)Y_{Q,\ell,k}(x)\qquad x\in \mathbb{S}^Q,
\ee 
where
\be\label{eq:connection}
\Gamma_{\ell,n}\coloneqq \frac{\omega_q\omega_{Q-1}}{\omega_Q\omega_{q-1}}\sum_{i=\ell}^n h\left(\frac{i}{n}\right)\frac{p_{q,i}(1)}{p_{Q,\ell}(1)}C_{Q,q}(\ell,i),
\ee
and $C_{Q,q}(\ell,i)$ is defined in \eqref{eq:dimchange}.
\end{proposition}

\begin{proof}
The proof follows from writing out
\be
F_n(\mathcal{D};x)=\frac{\omega_q}{M\omega_{q-1}}\sum_{j=1}^M z_j \sum_{i=1}^n h\left(\frac{i}{n}\right)p_{q,i}(1)p_{q,i}(x\cdot y_j),
\ee
making substitutions using \eqref{eq:dimchange}, \eqref{eq:reproduce}, and  \eqref{eq:summation}, then collecting terms.
\end{proof}

\begin{rem}{\rm
The encoding \eqref{eq:encoding} is not parsimonious. Since the basis functions $\{Y_{Q,\ell,k}\}_{\ell=0,k=1}^{n,\operatorname{dim}(\mathbb{H}_\ell^Q)}$ is not necessarily independent on $\XX$, the encoding can be made parsimonious by exploiting linear relationships in this system. Given a reparametrization the functions as $\{Y_j\}_{j=1}^{\sum_{\ell=0}^n \operatorname{dim}(\mathbb{H}_\ell^Q)}$, we  form the discrete Gram matrix $G$ by the entries
\be\label{eq:discrete_gram}
G_{i,j}\coloneqq \frac{1}{M}\sum_{k=1}^M Y_i(y_k)Y_j(y_k) \approx\int_\mathbb{X} Y_i(y) Y_j(y) f_0d\mu^*(y).
\ee
In practice, one may formulate a QR decomposition by fixing some first basis vector and proceeding by the Gram-Schmidt process until a basis is formed, then setting some threshold on the eigenvalues to get the desired dependencies among the $Y_j$'s. \qed}
\end{rem}

\newpage

\addcontentsline{toc}{chapter}{Bibliography}

%\nocite{*}
\printbibliography

\end{document}